\documentclass[sotoncolour]{uosthesis}      
\graphicspath{{Figures/}}   
\usepackage[round]{natbib}            
\usepackage{bibentry}          
\nobibliography*               
\usepackage{attrib}            
\hypersetup{colorlinks=true}   
\newcommand{\BibTeX}{{\rm B\kern-.05em{\sc i\kern-.025em b}\kern-.08em T\kern-.1667em\lower.7ex\hbox{E}\kern-.125emX}}

\newcounter{address}
\newcommand{\V}[1]{\ensuremath{\boldsymbol{#1}}\xspace}

\newcommand{\0}{\V{0}}

\newcommand{\N}[1]{\ensuremath{\mathit{#1}}\xspace}

\newcommand{\mc}[1]{\ensuremath{\mathcal{#1}}\xspace}

\newcommand{\D}{\mc{D}}

\newcounter{alg}

\usepackage{csquotes}
\usepackage{nicefrac}
\newcommand{\Dd}[1]{\mathcal{D}^{#1}}

\newcommand{\Ss}{\mathcal{S}}
\newcommand{\PP}{\mathbf{P}}
\newcommand{\val}{\mathrm{val}}

\usepackage{array, makecell}

\DeclareMathOperator*{\argmin}{arg\,min}
\newtheorem{prop}{Proposition}

\usepackage{amssymb}
\usepackage{pifont}
\usepackage{substitutefont}
\substitutefont{TS1}{fds}{cmr}
\substitutefont{OT1}{fds}{cmr}

\department  {School of Electronics and Computer Science}
\DEPARTMENT  {\MakeUppercase{\deptname}}
\group       {Web and Internet Science Group}
\GROUP       {\MakeUppercase{\groupname}}
\title{Model Monitoring in the Absence of Labeled Data via Feature Attributions Distributions}
\authors    {Carlos Mougan Navarro} 
\addresses  {\groupname\\\deptname\\\univname}
\date       {\today}

\orcidid{0000-0002-1825-0097}
\doi{10.1002/0470841559.ch1}
\volume{n}{m} 
\subject    {}
\keywords   {}

\usepackage{csquotes}
\usepackage{multirow}
\usepackage{verbatim}
\usepackage{color, colortbl}
\usepackage{dirtytalk}
\usepackage{pdflscape}
\usepackage{tikz}
\newcommand{\tikzxmark}{%
\tikz[scale=0.23] {
    \draw[line width=0.7,line cap=round] (0,0) to [bend left=6] (1,1);
    \draw[line width=0.7,line cap=round] (0.2,0.95) to [bend right=3] (0.8,0.05);
}}
\newcommand{\tikzcmark}{%
\tikz[scale=0.23] {
    \draw[line width=0.7,line cap=round] (0.25,0) to [bend left=10] (1,1);
    \draw[line width=0.8,line cap=round] (0,0.35) to [bend right=1] (0.23,0);
}}

\begin{document}
\pagenumbering{gobble} 

\frontmatter
\maketitle

\begin{abstract}

Model monitoring involves analyzing AI algorithms once they have been deployed and detecting changes in their behaviour.
This thesis explores machine learning model monitoring ML before the predictions impact real-world decisions or users. This step is characterized by one particular condition: the absence of labelled data at test time, which makes it challenging, even often impossible, to calculate performance metrics.

The thesis is structured around two main themes: \emph{(i) AI alignment}, measuring if AI models behave in a manner consistent with human values and \emph{(ii) performance monitoring}, measuring if the models achieve specific accuracy goals or desires. 

The thesis uses a common methodology that unifies all its sections. It explores feature attribution distributions for both monitoring dimensions. Using these feature attribution explanations, we can exploit their theoretical properties to derive and establish certain guarantees and insights into model monitoring.

For AI Alignment, we explore whether the distributions of feature attributions are distinct for different social groups and propose a new formalization of equal treatment. This novel metric assesses how well AI decisions adhere to ethical standards and political-philosophical values. Our notion of Equal Treatment tests for statistical independence of the explanation distributions over populations with different protected characteristics. We show the theoretical properties of our formalization of equal treatment and devise an equal treatment inspector based on the AUC of a classifier two-sample test. 

For performance monitoring, we define \emph{explanation shift} as the statistical comparison between how predictions from training data are explained and how predictions on new data are explained. We propose explanation shift as a key indicator to investigate the interaction between distribution shifts and learned models.  We introduce an Explanation Shift Detector that operates on the explanation distributions, providing more sensitive and explainable changes in interactions between distribution shifts and learned models. We compare explanation shifts with other methods that are based on distribution shifts, showing that monitoring for explanation shifts results in more sensitive indicators for varying model behavior. We provide theoretical and experimental evidence and demonstrate the effectiveness of our approach on synthetic and real data.

Finally, to explain model degradation we use a second model, that predicts the uncertainty estimates of the first

Additionally, we release two open-source Python packages, \texttt{skshift} and \texttt{explanationspace}, which implement our methods and provide usage tutorials for further reproducibility.

\end{abstract}
\tableofcontents

\chapter*{\centering \large\textbf{LIST OF PUBLICATIONS}}
\addcontentsline{toc}{section}{\bf ~~~~~~~~~~~LIST OF PUBLICATIONS}
\noindent

The research presented in Chapter~ \ref{ch:et} has been disseminated through the following publication:
\begin{itemize}
\item  \textbf{1.}\textbf{Mougan, C}, Ferrara, A., State, L., Ruggieri, S., $\&$ Staab, S. Beyond Demographic Parity: Redefining Equal Treatment\\
    \textit{NeurIPS 2023 workshop on AI meets Moral Philosophy}\\
    \cite{equalTreatment}\\
\end{itemize}

The research presented in Chapter \ref{ch:exp.shift} has been disseminated through the following publication:
\begin{itemize}
    \item \textbf{2.}\textbf{Mougan, C.}, Broelemann, K., Kasneci, G., Tiropanis, T., and Staab, S. (2022)\\
    Explanation shift: Detecting distribution shifts on tabular data via the explanation space.
    \textit{NeurIPS 2022 Workshop on Distribution Shifts: Connecting Methods and Applications}~\cite{explanationShift}
	
\end{itemize}

The research presented in Chapter~\ref{ch:monitoring} has been disseminated through the following publication:
\begin{itemize}

    \item \textbf{3.}\textbf{Mougan, C.}, $\&$ Nielsen, D. S. (2023).\\ 
    Monitoring Model Deterioration with Explainable Uncertainty Estimation via Non-parametric Bootstrap.\\
    \textit{ In AAAI Conference on Artificial Intelligence,2023}~\cite{mougan2022monitoring}
	
\end{itemize}

Additional publications, while not directly incorporated into this thesis, have enriched the understanding of the field and influenced some of the proposed solutions:

\begin{itemize}

    \item \textbf{4.}\textbf{Mougan, C.}, Plant, R., Teng, C., Bazzi, M., Ejea, A. C., Chan, R. S.-Y., Jasin, D. S., Stoffel, M., Whitaker, K. J., and Manser, J. (2023).\\
    How to data in datathons.\\
    \textit{Advances in Neural Information Processing Systems}~\cite{dsg}
    
    \item \textbf{5.}\textbf{Mougan, C.}, Alvarez, J., Ruggieri, S., and Staab, S. (2023).\\
    Fairness implications of encoding protected categorical attributes.\\
    \textit{In Proceedings of the 2023 AAAI/ACM Conference on AI, Ethics, and Society, AIES Association for Computing Machinery}~\cite{FairEncoder}

    \item \textbf{6.}\textbf{Mougan, C}, Masip, D., Nin, J., and Pujol, O. (2021).\\
    Quantile encoder: Tackling high cardinality categorical features in regression problems.
    \textit{Modeling Decisions for Artificial Intelligence - 18th International Conference, MDAI 2021} ~\cite{DBLP:conf/mdai/QuantileEncoder}

    \item \textbf{7.}\textbf{Mougan, C},  Kanellos, G, and Gottron, T. \\
    Desiderata for Explainable AI in Statistical Production Systems of the European Central Bank.
    \textit{ECMLPKDD Workshop on Bias (2021)}~\cite{desiderataECB}

    \item \textbf{8.}\textbf{Mougan, C}, Kanellos, G., Micheler, J., Martinez, J., $\&$ Gottron, T. (2022).\\ 
    Introducing explainable supervised machine learning into interactive feedback loops for statistical production system. 
    \textit{Irving Fisher Committee (IFC) - Bank of Italy workshop on Data science in central banking: Applications and tools.}~\cite{DBLP:journals/corr/csdb_ml}

    \item \textbf{9.}Yasar, A. G., Chong, A., Dong, E., Gilbert, T. K., Hladikova, S., Maio, R., \textbf{Mougan, C}, Shen, X., Singh, S., Stoica, A.-A., Thais, S., and Zilka, M. (2023). \\
    AI and the EU digital markets act: Addressing the risks of bigness in generative AI.\\
    \textit{In Proceedings of the 1st Workshop on Generative AI and Law. International Conference on Machine Learning (ICML)}~\cite{workshop2023generative}

    \item \textbf{10.}Papageorgiou, I. and \textbf{Mougan, C} (2023). \\
    Processing sensitive personal data for bias monitoring under the AI act: Necessity principle. \\
    \textit{In Beyond Data Protection Conference 2023}\\
    \textit{NeurIPS 2023 workshop on Regulatable ML}~\cite{papageorgiou2023processing}

    \item \textbf{11.}Alvarez, J. M., Colmenarejo, A. B., Elobaid, A., Fabbrizzi, S., Fahimi, M., Ferrara, A., Ghodsi, S., \textbf{Mougan, C.}, Papageorgiou, I., Reyero, P., Russo, M., Scott, K. M., State, L., Zhao, X., and Ruggieri, S.\\
    Policy advice and best practices on bias and fairness in AI. \\
    \textit{Ethics and Information Technology}\cite{alvarez2024policy}

    \item \textbf{12.}Leslie, D., Ashurst, C., Gonzalez, N. M., Griffiths, F., Jayadeva, S., Jorgensen, M., Katell, M., Krishna, S., Kwiatkowski, D., Martins, C. I., Mahomed, S., \textbf{Mougan, C.}, Pandit, S., Richey, M., Sakshaug, J. W., Vallor, S., and Vilain, L. \\
    Frontier AI, power, and the public interest: Who benefits, who decides? \textit{Harvard Data Science Review} \cite{leslie2024frontier}

\end{itemize}
\section*{\centering \large\textbf{UNDER REVIEW}}
\begin{itemize}

\item \textbf{13.}\textbf{Mougan, C}, Broelemann,D., Masip, K., Kasneci, G., Tiropanis, T., and Staab, S. (2022).Explanation shift: How did the distribution shift impact the model?

\item \textbf{14.}Pugnana, A., \textbf{Mougan, C}, Saatrup , D., Model Agnostic Explainable Selective Regression via Uncertainty Estimation

\item \textbf{15.}\textbf{Mougan C.}, Brand, J., Kantian Deontology Meets AI Alignment: Towards Morally Robust Fairness Metrics ~\cite{mougan2023kantian}

\item \textbf{16.}Castillo-Eslava, F., \textbf{Mougan, C}, Romero-Reche, A., and Staab, S. (2023).
    The role of large language models in the recognition of territorial sovereignty: An analysis of the construction of legitimacy~\cite{LLMTerritorial}

    \item \textbf{17.}Yasar, A. G., Chong, A., Dong, E., Gilbert, T. K., Hladikova, S., \textbf{Mougan, C.}, Shen, X., Singh, S., Stoica, A.-A., and Thais, S. (2024). Integration of generative AI in the digital markets act: Contestability and fairness from a cross-disciplinary perspective. In Working Paper Series. London School of Economics.\cite{yasar2024}

    \item  \textbf{18.} \textbf{Mougan, C.} AI Alignment via Virtue-Based Selected Feedback

\end{itemize}
\authorshipdeclaration{}
\mainmatter
\chapter{Introduction to Model Monitoring} \label{chapter:intro}

Model monitoring refers to the practice of observing and understanding the behaviour of Artificial Intelligence (AI) systems in production~\citep{DesignMLSystems,mle}. It involves collecting and analyzing data from deployed AI models and systems to ensure they are functioning as intended and to identify and diagnose any issues or anomalies that may arise during their operation.

\begin{figure}[ht]
\centering
\includegraphics[width=\textwidth]{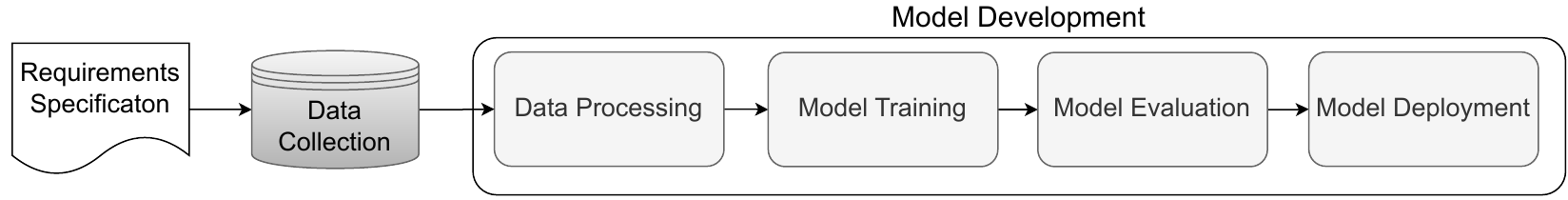}
\caption{Model monitoring occurs within the AI system life cycle. During model development, task requirements and data are collected, models are trained and evaluated, and parameters are refined. Finally, the model is assessed against benchmarks and metrics using a hold-out dataset.}\label{fig:AI.train}
\end{figure}

The process of model monitoring unfolds within the broader context of the AI system life cycle, as illustrated in Figure~\ref{fig:AI.train}. After the application task requirements specifications are collected and needed data is selected, AI models undergo training and evaluation using available data during the model developmental phase. In this phase, model parameters are learned, tuned and refined. Then, the model is evaluated using the available data from a hold-out set against established benchmarks and metrics. 

Subsequently, based on the results of this evaluation, models are deployed into production environments. This shift marks the transition from the controlled development environment, with available ground truth data, to the dynamic and frequently unpredictable landscape of real-world deployment~\citep{DBLP:journals/csur/PaleyesUL23}, as illustrated in Figure~\ref{fig:AI.train}. This stage not only marks a significant point in the AI life cycle but also a moment when understanding how the model is going to behave, which is crucial for the successful and responsible application of AI models.

\begin{figure}[ht]
\centering
\includegraphics[width=\textwidth]{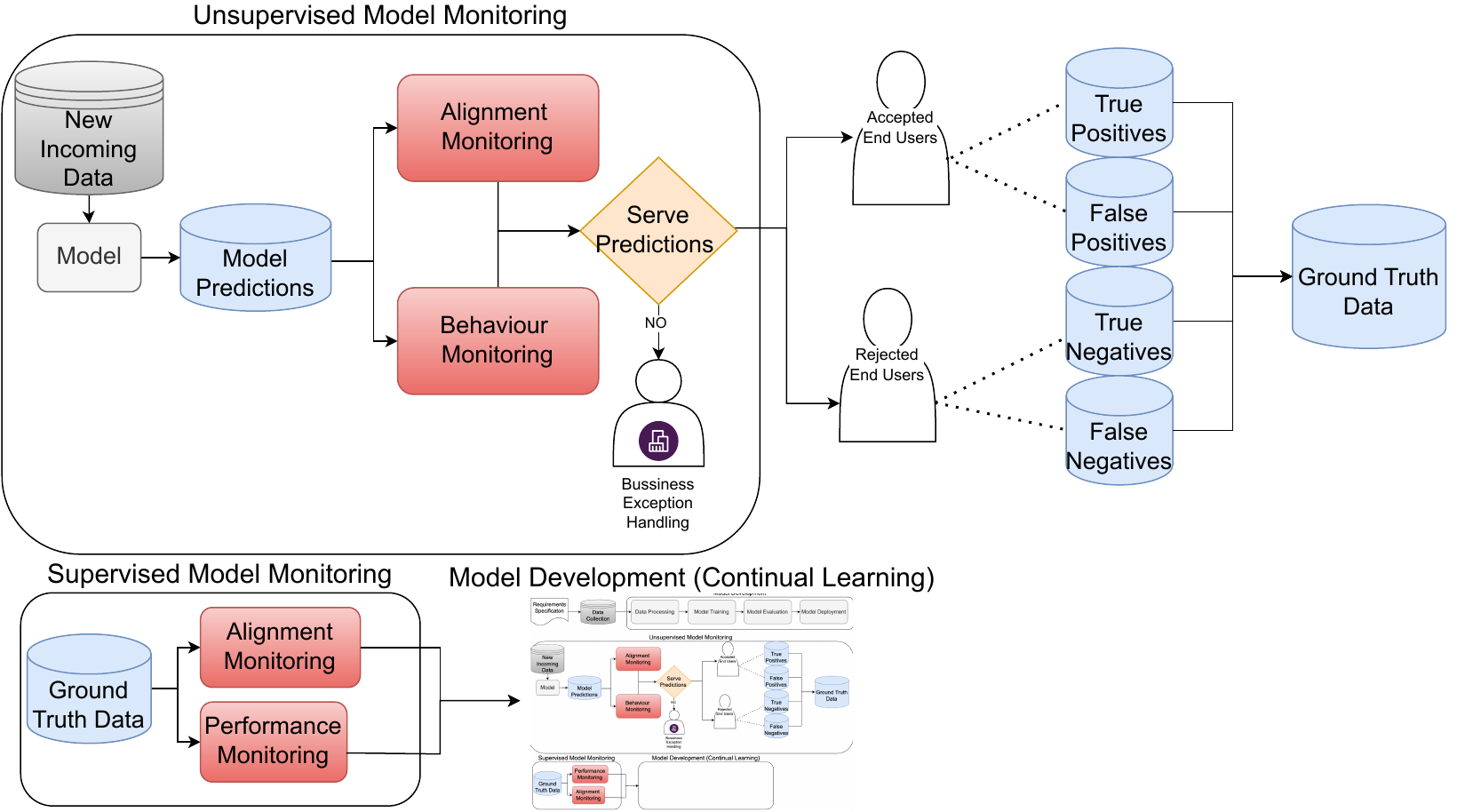}
\caption{Life cycle of monitoring a supervised learning binary classifier. Once a model is deployed, it encounters new, unlabelled data, prompting the need to observe and analyse its behaviour to detect potential issues before serving the predictions in the real world. In some cases, after serving the model to users, the acquisition of ground truth data allows for further evaluation. Leveraging these results and the available ground truth data, the original model can undergo retraining.}\label{fig:AI.lifecycle}
\end{figure}

We can distinguish two types of model monitoring: before serving the model predictions (unsupervised model monitoring) or after, when it is often possible to obtain ground truth data collection (supervised model monitoring).  However, it is at the first stage of model monitoring, depicted in Figure~\ref{fig:AI.lifecycle}, where the primary focus of this thesis lies; when newly arrived, unlabeled data triggers the model to make predictions, right before serving those predictions to the world (e.g. users). It is precisely during this phase that we observe and analyse the model's predictions to detect any undesirable behaviour before it reaches real-world applications. This specific phase,  prior to serving the model predictions, presents a unique challenge: the \emph{absence of ground truth data} that makes it extremely hard and often impossible to estimate conventional model performance metrics such as accuracy, precision, AUC, or F1-score and social inequality metrics such as equal opportunity, average absolute odds or equal odds. Within this step of an AI model's life cycle, we pinpoint two paramount challenges of model monitoring: 

\begin{itemize}
    \item \textbf{AI Alignment}: relies on ensuring that the AI systems deployed align with a diverse set of human values.
    
    \item \textbf{Model Monitoring}: Accounting for changes in the model's behaviour and performance on the new incoming data.
\end{itemize}

After performing an unsupervised model monitoring assessment, we reach the stage in which we have to decide whether to serve those predictions. This assessment is particular to each task, and the final decision lies within the domain of engineering and business teams. Opting not to serve the predictions to the users leads us back to revisiting the problem definition, as the absence of labeled data prevents retraining. On the other hand, if we accept those predictions and decide to serve them to the users, then in some situations, there is the possibility of collecting ground truth data.  Acquiring labeled data after deployment poses a considerable challenge and is often impractical,  exhibiting well-known biases, such as confounding, selection, and missingness~\citep{feng2023towards,DBLP:conf/aaai/Ruggieri0PST23}.
For example, in a loan lending scenario, after the period is due, we will be able to collect the data for the users who we gave the loan, but we won't be able to collect it for the users who we did not offer the loan, creating then a biased ground truth data.

Once we have obtained ground truth data, the second step of model monitoring can be performed. In this step, other metrics based on the model errors can be performed. If those metrics are satisfactory enough, we can proceed to make predictions in the new incoming data; if they are not and we want to optimize further, in some situation, we can retrain the model on this new data, a process called \enquote{continual learning}~\citep{continual_learning,DBLP:conf/nips/PhamLH21,DBLP:conf/icml/KnoblauchHD20}.

This thesis centres on the previous step, the unsupervised monitoring stage, where labeled post-deployment data is unavailable. The techniques applicable in this context are relatively limited. Specifically, this thesis seeks to understand how features attribution distribution can be used as this post-deployment monitoring stage.\footnote{The foundations of feature attribution methods will be discussed in Chapter~\ref{ch:foundations}}. This focus leads to the first research question of the thesis:

\say{(Research Question 1) \emph{How can we use feature attributions distributions for post-deployment monitor in the absence of labeled truth data?}}

We will divide this thesis's overarching research question into two main parts: monitoring for performance and AI Alignment. We now introduce more fine-grained research questions aligned with the current challenges of both dimensions of unsupervised model monitoring. For AI alignment, we start by providing two social science perspectives and then the technical contribution of the thesis that aims to align with them. For model monitoring, we divide the problem into two tasks: understanding changes in the model under distribution shift and building indicators of model performance deterioration. An overarching principle of this thesis is its commitment to open science. To uphold the focus on openness and reproducibility, we will compile a set of requirements designed to ensure replicable research outcomes and the development of practical software tools.

\section{AI Alignment and Feature Attributions Challenges}

The AI alignment field relies on ensuring that AI systems are oriented with a diverse set of human values, encompassing moral, societal, justice-oriented, and other relevant principles~\citep{DBLP:journals/mima/Gabriel20,gabriel2022toward}. 

In the field of computer science, researchers have traditionally quantified various statistical disparities that exist among protected attributes. These measures are often labelled using terms from the social sciences~\citep{DBLP:conf/fat/Binns18,mitchell2021algorithmic}, such as equal opportunity~\citep{DBLP:conf/nips/HardtPNS16} and equal treatment~\citep{DBLP:conf/fat/CandelaWHJRAH023}. These notions have been incorporated into the study of algorithmic fairness to address historical and systemic biases that can arise in decision-making systems. Researchers have subsequently proposed methods and algorithmic approaches to improve these metrics. Deciding which metrics are best suited in each situation is an important yet complex task as researchers have demonstrated incompatibility between certain metrics, thus arguing that several AI fairness concepts cannot be achieved simultaneously~\citep{DBLP:journals/bigdata/Chouldechova17,DBLP:conf/innovations/KleinbergMR17,DBLP:conf/nips/HardtPNS16,DBLP:conf/nips/HsuMN022}.  It should be noted that concepts in social sciences are not designed with mathematical precision, which makes it challenging to quantify statistical differences in a straightforward mathematical way directly. Instead, they emerge from the understanding of historical events, social inequalities, and discriminatory practices.

At the intersection of social sciences fields and computer science, there exists an already recognized misalignment~\citep{DBLP:conf/aies/FazelpourL20}. In light of this, there are several research lines that propose frameworks to help understand existing notions of algorithmic fairness in social science terms~\citep{DBLP:conf/fat/HeidariLGK19,DBLP:conf/aies/SimonsBW21}. 
Within the context of this thesis, which focuses on post-deployment monitoring without labeled data and explores the use of feature attribution distributions, we pose the second research question:

\say{(Research Question 2) \emph{Can feature attribution distributions be used to measure alignment of post-deployed models in a unsupervised monitoring set up??}}

This thesis will introduce novel metrics based on feature attribution distributions, aligned with two core principles of AI alignment, demonstrating how these distributions can be utilized to effectively measure them:
\begin{itemize}

    \item Moral values grounded in Kantian deontological philosophy.

    \item Equality notions based on distributive justice principles that are derived from political philosophy.

\end{itemize}

\section{Performance Monitoring with Feature Attributions Distributions}

The second and subsequent part of this PhD thesis centres on performance aspects of model monitoring, covering chapters~\ref{ch:exp.shift} and \ref{ch:monitoring}. In supervised machine learning, model monitoring of predictive performance is an important step of the ML model pipeline as it directly relates to the quality of the application~\citep{DesignMLSystems,mle}. Monitoring machine learning models once they have been deployed is not an easy task. Model monitoring ensures that a machine learning application in a production environment displays consistent behaviour over time. Understanding how distribution shift impacts the model behaviour is paramount in deployed ML models.  It enables us to gauge how well models adapt to changing data distributions, ensuring their reliability in dynamic real-world scenarios. Moreover, such investigations aid in addressing critical potential societal issues and ethical concerns by identifying and mitigating post-deployment biases promoting responsible AI deployment. By understanding how models respond to distribution shifts, we can enhance model accountability, allocate resources efficiently, and meet regulatory requirements.

With this challenge in mind and the usage of feature attribution distributions this thesis covers the following research question:

\emph{(Research Question (3)\emph{By looking at changes in the distributions of explanations, can we measure how did distribution shift impact the model?}}\footnote{To be answered in Chapter~\ref{ch:exp.shift}}

Current methods to detect shifts in the input or output data distributions have limitations in identifying model behaviour changes, where we recollect our research requirement \emph{(REQ-05)}, measuring the impact of distribution shift in input data in the model. Extensive related work has aimed at detecting that data is from out-of-distribution. To this end, they have created several benchmarks that measure whether data comes from in-distribution or not~\citep{wilds,wilds_ext,malinin2021shifts,malinin2022shifts,DBLP:conf/nips/MalininBGGGCNPP21}. However, many methods developed for detecting out-of-distribution data are specific to neural networks processing image and text data and can not be applied to traditional machine learning techniques. These methods often assume that the relationships between predictor and response variables remain unchanged, i.e., no concept shift occurs. In this thesis, we focus on unlabeled monitoring, where ground truth from deployment data is not available, forming research requirement \emph{(REQ-06)}.  In this thesis, a research requirement is the statistical comparison between how predictions from training data are explained and how predictions on new data are explained, \emph{(REQ-07)}. Also, identifying and having explanation capabilities of the changes in the model behaviour is research requirement \emph{(REQ-08)}

Once we have proposed and studied methods, based on feature attributions distributions, to understand the impact of distribution shifts on the model, the next step is to assess whether these shifts have affected the model's performance. This task is a well known impossible challenge without further assumption, therefore developing an approach to measure performance decay is notoriously difficult. In the final research question of this thesis, we will study this open ended challenge , exploring how performance degradation can be identified and explained under a specific set of conditions using feature attributions distributions:

\emph{(Research Question (4)\emph{How can we identify and explain changes in model performance using feature attribution distribution?}}\footnote{To be answered in Chapter~\ref{ch:monitoring}}

\subsection{Reproducibility and Software Contribution Objectives}\label{sec:reproducibility}

Having well-defined software requirements in a computer science thesis is crucial to ensure clarity, reproducibility, and applicability of the research. Open-source software enables other researchers to replicate the study findings, validating the results and ensuring their reliability across different settings, which is a fundamental aspect of scientific research, as captured in requirement \emph{REQ-09}. In addition to releasing open-source packages, the software used for experiments should be available in the form of a software repository. This facilitates systematic evaluation, scrutiny, and comparison of the developed experiments against other benchmarks or existing solutions, as detailed in requirement \emph{REQ-10}.

\clearpage
\subsection{Summary Requirements for Model Monitoring in the Absence of Labeled Data via Feature Attributions Distributions}

\begin{table}[ht]

\small
  \caption{Requirements Elicitation Table for Equal Treatment}
  \label{tab:requirements}
  \begin{tabular}{|p{0.2\linewidth}|p{0.5\linewidth}|p{0.2\linewidth}|}
    \hline
    \rowcolor{lightgray}
    \textbf{Requirement} & \textbf{Description} & \textbf{Source} \\
    \hline

    \multicolumn{3}{|c|}{\textbf{Alignment Requirements}} \\
    \hline
    
    \makecell{REQ-01 \\ Unlabeled Data }&  Metrics should prioritize assessing disparities in model decisions rather than fixating on model errors. 
    & Society: Deontological Ethics\\
    \hline

    \makecell{REQ-02 \\ Group \\Discrimination }&Decisions based on protected characteristics that individuals did not choose at birth. & Society: Politics \& Ethics \\
    \hline
    \makecell{REQ-03 \\ Measuring Proxy \\Discrimination }&   Detecting that certain features may indirectly contribute to biased decisions of non-chosen factors at birth (protected attributes) & Society(Liberal Political Philosophy)\\
    \hline
    \makecell{REQ-04 \\ Explanation\\ Capability}&  Offer both theoretical underpinnings and empirical validation of the sources driving discrimination.& End User \\
    \hline
    \multicolumn{3}{|c|}{\textbf{Performance Monitoring}} \\
    \hline
    \makecell{REQ-05 \\ Model-Data \\Interactions}&The system shall measure and quantify the impact of distribution shifts in input data in the model & Monitoring Objectives \\
    \hline
    \makecell{REQ-06 \\ Unlabeled Data }& The system should not require labelled test data & Monitoring Objectives \\
    \hline
    \makecell{REQ-07 \\ Explanation\\ Distributions}& The system shall identify model explanation changes resulting from distribution shifts. & Monitoring Objectives \\
    \hline
    \makecell{REQ-08 \\Explanation \\ Capabilities}& The system shall provide insights into how models respond to different types of distribution shifts. & System Users \\
    \hline
    \multicolumn{3}{|c|}{\textbf{Software Requirements}} \\
    \hline
    \makecell{REQ-09 \\ Open Source }& An open-source software package with clear and comprehensive documentation, including installation guides, usage instructions, tutorials covering both basic and advanced usage. & Developers \& Scientific Community \\
    \hline
    \makecell{REQ-10 \\ Reproducibility} & The experiments should be available with sample datasets and examples showcasing the experiments used in this thesis. Additionally, there should be a feedback mechanism for users to report issues and ask questions. & Developers \& Scientific Community \\
    \hline
  \end{tabular}
\end{table}

\clearpage
\section{Overview of the Thesis}

The opening chapter of this thesis sets the stage for what is unsupervised model monitoring within machine learning systems. This chapter is divided into two main sections, each addressing different facets of AI monitoring: AI alignment challenges and performance monitoring.

Chapter \ref{ch:foundations}, titled "Foundations," provides the theoretical framework essential for understanding and addressing model monitoring and AI alignment. The chapter begins by establishing the necessary mathematical notation. It then provides the foundations for feature attribution explanations (Shapley values and LIME). Then, AI alignment explores how AI can adhere to distributive justice and moral philosophy principles such as egalitarianism, socialism, liberalism, utilitarianism, and deontology. The chapter also introduces the different types of distribution shifts and model monitoring challenges.

Chapter~\ref{ch:et}, titled "Measuring AI Alignment with Equal Treatment," is dedicated to exploring and establishing methodologies for ensuring AI systems align with philosophical notions of equal treatment. The chapter uses illustrative examples, such as college admissions and paper-blind reviews, to demonstrate the practical implications and challenges of AI in decision-making processes that require impartiality. It present a methodology for monitoring equal treatment in AI systems. The chapter engages with theoretical analyses, comparing various fairness metrics like equal treatment, demographic parity, and individual fairness, and includes experiments with both synthetic and real data to test these concepts. Through these discussions and experiments, the chapter aims to refine the tools and methods needed to evaluate and ensure that AI systems administer equal treatment across different scenarios, aligning closely with both distributive justice concepts and moral values.

Chapter \ref{ch:exp.shift}, titled "How Did Distribution Shift Impact the Model?", examines the effects of various types of distribution shifts on machine learning models, particularly focusing on how these shifts influence explanations provided by models. The chapter introduces a model designed to detect explanation shifts and explores relationships between common distribution shifts—like covariate shift, prediction shift, and concept shift—and their impacts on model explanations. It studies analytical and empirical evaluations using both synthetic data and real-world datasets, such as StackOverflow survey data and ACS derived datasets. It concludes with a discussion on the challenges of monitoring models under distribution shifts and the implications for model explainability.

Chapter ~\ref{ch:monitoring}, titled "Building Indicators of Model Performance Deterioration,"  develops an indicator to detect and evaluate the deterioration in machine learning model performance over time. It reviews related works on model monitoring and the role of uncertainty in assessing model performance. The chapter describes various methods to evaluate model deterioration detection systems and estimate the uncertainty that can lead to model decay. Experiments are conducted using synthetic and classical ML data benchmarks to assess these methods. It concludes with a discussion on the limitations of current approaches and the potential for further research in the area of dynamic model monitoring and maintenance.

Chapter \ref{ch:conclusion}, the conclusion of the thesis, synthesizes the key findings, delineates the contributions made to the field, and outlines future research directions. It summarizes the results of investigations into AI model monitoring, fairness, and performance.

\chapter{Foundations}\label{ch:foundations}
In this chapter, we present the foundations on which this thesis is build. We first introduce the formalization and then provide \emph{(i)} explainable AI foundations \emph{(ii)}  model monitoring and finally \emph{(iii)} AI alignment foundations.

\section{Mathematical Notation}

\begin{table}[ht]
\centering
\caption{Notation used throughout the paper.}\label{tab:notation}
\begin{tabular}{ll}
\hline
Notation                    & Meaning                                                           \\\hline
$f_\theta$                  & Supervised learning model.                                        \\
$X$                         & Predictive features.                                              \\
$Y$                         & Target feature.                                                   \\
$(x_i, y_i)$                & Observation in the training set.                                  \\
$\mathrm{dom}(Y)$           & Domain of the target feature.                                     \\
$f_\theta(x)$               & Model prediction for $x$                                          \\
$\Dd{tr}$                   & Training set.                                                     \\
$\Dd{new}$                  & Unseen dataset at training.                                       \\
$\D_X{}$                    & Projection of $\Dd{}$ on $X$.                                     \\
$\Dd{tr}_X$                 & Covariate training dataset input.                                 \\
$\Dd{new}_X$                & Covariate unseen dataset input.                                   \\
$f_\theta({\D_X})$          & Empirical prediction distribution.                                \\
$N$                         & Size of $X$.                                                      \\
$B$                         & Number of bootstrap samples.                                      \\
$b$                         & The index of a particular bootstrap sample.                       \\
$f_\theta(x^{\star})$       & Instance to be predicted.                                         \\
$\mathcal{S}(f, x^{\star})$ & Shapley values for $x^{\star}$ and model $f$.                     \\
$\mathcal{S}_i(f, x^{\star})$& Shapley value of feature $X_i$.                                  \\
$\Ss_j(f_\theta,\D_X)$      & Shapley value of feature $j$ calculated from dataset $\D_X$.      \\
$\PP(\D_X)$                 & Distribution of dataset $X$.                                      \\
$\PP(\Ss(f_\theta,\D_X))$   & Distribution of Shapley values calculated from dataset $\D_X$.    \\
\hline
\hline
\end{tabular}
\end{table}
In supervised learning, a function $f_\theta: X  \to Y$, also called a model, is induced from a set of observations, called the training set, $\Dd{} =\{(x_1,y_1),\ldots,(x_n,y_n)\} \sim X \times Y$, 
where $X = \{X_1, \ldots, X_p\}$ represents predictive features and $Y$ is the target feature. 
The domain of the target feature is $\mathrm{dom}(Y)=\{0, 1\}$ (binary classification) or $\mathrm{dom}(Y) = \mathbb{R}$ (regression). For binary classification, we assume a probabilistic classifier, and we  denote by $f_\theta(x)$ the estimates of the probability $P(Y=1|X=x)$ over the (unknown)  distribution of $X \times Y$. For regression, $f_\theta(x)$ estimates $E[Y|X=x]$. We call the projection of $\D$ on $X$, written $\D_X{} = \{x_1, \ldots, x_n\} \sim X$,  the empirical \textit{input distribution}. 
The dataset $f_\theta({\D_X}) = \{f_\theta(x) \ |\ x \in \D_X{} \}$ is called the empirical \textit{prediction distribution}.

\section{Feature Attribution Explainations}\label{sec:xai.foundations}

Explainability has become an important concept in legal and ethical data processing guidelines and machine learning applications ~\cite{intuitive_appeal}. A wide variety of methods have been developed to account for the decision of algorithmic systems ~\cite{guidotti_survey,DBLP:conf/fat/MittelstadtRW19,DBLP:journals/inffus/ArrietaRSBTBGGM20}.

\subsection{Shapley Values}
One of the most popular approaches to explaining machine learning models has been the use of Shapley values to attribute relevance to features used by the model~\cite{lundberg2020local2global,lundberg2017unified}. The Shapley value is a concept from coalition game theory that aims to allocate the surplus generated by the grand coalition in a game to each of its players~\cite{shapley}. The Shapley value $\mathcal{S}_j$ for the $j$'th player can be defined via a value function $\mathrm{val}:2^N \to \mathbb{R}$ of players in $T$:

\begin{small}
\begin{gather}
\mathcal{S}_j(\mathrm{val}) = \sum_{T\subseteq N\setminus \{j\}} \frac{|T|!(p-|T|-1)!}{p!}(\mathrm{val}(T\cup \{j\}) - \mathrm{val}(T))
\end{gather}
\end{small}

In machine learning, $N=\{1,\ldots,p\}$ is the set of features occurring in the training data, and $T$ is used to denote a subset of $N$. Given that $x$ is the feature vector of the instance to be explained, and the term $\mathrm{val}_x(T)$ represents the prediction for the feature values in $T$ that are marginalized over features that are not included in $T$:

\begin{gather}
\mathrm{val}_{f,x}(T) = E_{X|X_T=x_T}[f(X)]-E_X[f(X)]
\end{gather}

The Shapley value framework satisfies several theoretical properties, which we will describe in the following section. Our approach works with feature attribution techniques that fulfill efficiency and non-informative properties, and we use Shapley values as an example.

\subsubsection{Efficiency.} 

\textit{The feature contributions must add up to the difference of prediction $x^{\star}$ and the expected value}:
\begin{gather}
    \sum_{j \in N} \Ss_j(f, x^{\star}) = f(x^{\star}) - E[f(X)])
    \label{eq:SHAP-efficiency}
\end{gather}

It is important to differentiate between the theoretical Shapley values and the different implementations that approximate them. We use TreeSHAP as an efficient implementation of an approach for tree-based models of Shapley values~\cite{lundberg2020local2global,molnar2019,Zern2023Interventional}, particularly we use the observational (or path-dependent) estimation  ~\cite{DBLP:journals/corr/abs-2207-07605,DBLP:conf/nips/FryeRF20,DBLP:journals/corr/ShapTrueModelTrueData} and for linear models we use the correlation dependent implementation that takes into account feature dependencies \cite{DBLP:journals/ai/AasJL21}.

The following property only holds for the interventional variant (SHAP values), but not for the observational variant.

\subsubsection{Uninformativeness.}\label{eq:SHAP-uninformativeness}

\textit{A feature $X_j$ that does not change the predicted value (i.e., for all $x, x'_j$: $f(x_{N\setminus \{j\}}, x_j) = f(x_{N\setminus \{j\}}, x'_j)$) have a Shapley value of zero, i.e., $\Ss_j(f, x^{\star}) = 0$.}

There are two variants for the term $\val(T)$~\cite{DBLP:journals/ai/AasJL21,DBLP:journals/corr/ShapTrueModelTrueData,Zern2023Interventional}: the \textit{observational} and the \textit{interventional}. When using the observational conditional expectation, we consider the expected value of $f$ over the joint distribution of all features conditioned to fix features in $T$ to the values they have in $x^{\star}$:
\begin{equation}\label{def:val:obs}
\val(T) = E[f(x^{\star}_T, X_{N\setminus T})|X_T=x^{\star}_T]
\end{equation}
where $f(x^{\star}_T, X_{N\setminus T})$ denotes that features in $T$ are fixed to their values in $x^{\star}$, and features not in $T$ are random variables over the joint distribution of features.
Opposed, the interventional conditional expectation considers the expected value of $f$ over the marginal distribution of features not in $T$: 
\begin{equation}\label{def:val:int}
\val(T) = E[f(x^{\star}_T, X_{N\setminus T})]
\end{equation}

In the interventional variant, the marginal distribution is unaffected by the knowledge that $X_T=x^{\star}_T$~\citep{DBLP:journals/ai/AasJL21}. In general, the estimation of (\ref{def:val:obs}) is difficult, and some implementations (e.g., SHAP) consider (\ref{def:val:int}) as the default one~\cite{DBLP:journals/corr/ShapTrueModelTrueData}. In the case of decision tree models, TreeSHAP offers both possibilities\citep{lundberg2020local2global}.

The interventional SHAP values are said to stay \enquote{true to the model}, which means they will only give credit to the model's features and do not consider any correlation between inputs. In contrast, the full conditional SHAP values stay \enquote{true to the data}, considering how the model would behave when respecting the correlations in the input data~\citep{DBLP:journals/corr/ShapTrueModelTrueData}. Therefore, the full conditional SHAP values only credit the features relevant to the model behavior, while the interventional SHAP values allow us to explore the effect of changing inputs.

The only option supported for sparse cases is the interventional option. Thus, interventional SHAP values can be useful for sparse models, while full conditional SHAP values benefit models with highly correlated inputs. Overall, both interventional and full conditional SHAP values are useful for explaining the behaviour of machine learning models and providing insights into the most important features of model performance.

\paragraph{Linear Models and IID Data} In the case of a linear model $f_\beta(x) = \beta_0 + \sum_j \beta_j \cdot x_j$, the SHAP values turns out to be $\Ss(f, x^{\star}) = \beta_i(x^{\star}_i-\mu_i)$ where $\mu_i = E[X_i]$. For the observational case, this holds only if the features are independent \citep{DBLP:journals/ai/AasJL21}.

\subsection{LIME: Local Interpretable Model-Agnostic Explanations}

Another widely used feature attribution technique is LIME (Local Interpretable Model-Agnostic Explanations). The intuition behind LIME is to create a local linear model that approximates the behavior of the original model in a small neighbourhood of the instance to explain~\citep{ribeiro2016why,ribeiro2016modelagnostic}, whose mathematical intuition is very similar to the Taylor/Maclaurin series. Instead of attempting to interpret the entire model globally, LIME focuses on understanding predictions at a local level. It achieves this by perturbing or sampling data points in the vicinity of the prediction of interest and fitting a simple, interpretable model to approximate the complex model's behavior in that local region. The interpretable model, often a linear model, decision tree or a decision rule, can be easily understood and analyzed~\citep{DBLP:journals/corr/abs-1805-10820}.

This approach effectively "explains" a model's prediction by providing insight into which features were most influential in the local context. LIME quantifies the feature importance, enabling users to grasp why the model made a particular prediction for a specific data point. This local approximation is crucial because it reflects the model's behavior regarding the instance in question, which may differ from its global behavior.  In this thesis  we use the implementation based on the Euclidean version of LIME, called \enquote{tabular LIME} and it involves the following steps:
\begin{itemize}
    \item \textbf{Selection of Data Point}: Choose the data point for which you want an explanation.
    
    \item \textbf{Data Perturbation:} Generate a dataset of similar data points by perturbing the chosen instance. This step is crucial for approximating local behavior.

    \item \textbf{Prediction Collection:} Obtain model predictions for the perturbed data points.

    \item \textbf{Surrogate Model Creation:} Fit an interpretable model, such as a linear regression model, to the perturbed data and their corresponding model predictions. This surrogate model approximates the complex model's behavior locally.

    \item \textbf{Explanation Generation:} The surrogate model's coefficients are used to estimate the importance of each feature for the chosen prediction, resulting in an explanation of the model's decision for that instance.
\end{itemize}

LIME is a versatile and model-agnostic technique, which means it can be applied to virtually any machine learning model without requiring knowledge of the model's internal architecture. This makes it a valuable tool for interpretable AI in diverse domains.

\section{AI Alignment}\label{sec:foundations.alginemnt}
AI value alignment aims to ensure that AI is properly aligned with human values~\citep{DBLP:journals/mima/Gabriel20}. As computer systems continue to gain greater autonomy and impact,  AI value alignment takes on greater importance. The rapid decision-making capabilities of AI systems, combined with their autonomy, raise the stakes when it comes to ensuring that these systems align with human values. In real-time scenarios, AI may encounter complex ethical dilemmas, making it increasingly challenging for human oversight and evaluation of each action. This lack of oversight can lead to issues of accountability and ethical risks, underscoring the need to establish mechanisms that ensure AI acts according to a set of human values.

We can divide the alignment challenge into two parts~~\citep{DBLP:journals/mima/Gabriel20}. A first \emph{normative} question asks which values AI should align. A second challenge is  \emph{technical} and focuses on mathematically formalising those values or principles in AI systems to capture the previously discussed notions. 

The first part of the challenge raises questions about the very nature of "value" and its subjectivity, as it serves as a placeholder for a wide array of concepts. Values are fundamental beliefs or principles that guide behavior, decision-making, and judgments about what is right or wrong, good or bad, desirable or undesirable. This leads us into a meta-normativity tecno-legal debate about the objectivity of values.  In this thesis, we acknowledge  this meta-normativity debate, but for the purpose of practicality, we opt for a classic and shallow definition of values:

\blockquote{\textit{In its broadest sense, \enquote{value theory} is a catch-all label used to encompass all branches of moral philosophy, social and political philosophy, aesthetics, and sometimes feminist philosophy and the philosophy of religion}}~\cite{sep-value-theory}

In this thesis, we will limit the discussion of AI alignment to two main branches of philosophy and two specific values of those branches, which will be discussed in the next section: \textit{(i)} the branch of political philosophy with the value of distributive justice and \textbf{ \textit{(ii}}  the branch of ethics, centred on the value of moral rightness.

The second part of the challenge relies on the mathematical formalization of those values. The challenge becomes especially intricate when considering the diverse and dynamic nature of human values, ranging from ethical principles and cultural norms to personal preferences and moral priorities. Developing a mathematical framework that captures this intricate web of values demands a balance between flexibility and precision. It often requires acknowledging that human values are not fixed or universally agreed upon. A fundamental question is, \textit{if values lack an objectively quantifiable foundation, how can AI be developed to align with them?} \cite{DBLP:journals/mima/Gabriel20} proposes a distinction in order to ease this translation of human values into computational metrics:

\blockquote{\textit{Here it is useful to draw a distinction between minimalist and maximalist conceptions of value alignment. The former involves tethering artificial intelligence to some plausible schema of human value and avoiding unsafe outcomes. The latter involves aligning artificial intelligence with the correct or best scheme of human values on a society-wide or global basis. While the minimalist view starts with the sound observation that optimizing exclusively for almost any metric could create bad outcomes for human beings, we may ultimately need to move beyond minimalist conceptions if we are going to produce fully aligned AI. This is because AI systems could be safe and reliable but still a long way from what is best—or from what we truly desire.}\citep{DBLP:journals/mima/Gabriel20} }

In this thesis, we deviate from minimalist alignments and explore the foundational aspects of two distinct branches of philosophy: political philosophy and moral philosophy. These branches can serve to establish the fundamental humanities' fundamental pillars, guiding systems in accordance with the values, principles and standards found within these fields. We will now present the fundamental concepts underpinning each of these values and relate them to their significance in the landscape of AI alignment.

\subsection{Distributive Justice Political Philosophy Notions}

In philosophy, longstanding discussions about what constitutes a fair or unfair political system have led to established frameworks of distributive justice. Justice is a broad and foundational concept that pertains to the fair and equitable treatment of individuals and groups in society. It encompasses the idea of treating people with fairness, impartiality, and according to moral or legal principles. Even though there are several types of justice (retributive, procedural, restorative, or corrective justice), this thesis focuses on the intersection of distributive justice and supervised learning algorithms. Distributive justice is a specific subset of justice that deals with the fair allocation and distribution of resources, benefits, and burdens in a society. It is concerned with questions of how economic and social goods (such as income, wealth, education, healthcare, and opportunities) should be distributed among individuals and groups, emphasizing that circumstances not chosen at birth should not be the basis for disparate distribution~\citep{rawls1971theory,distributiveJustice}.

Different theories of distributive justice propose varying principles for achieving a fair distribution of goods in society. We will now introduce a few of them discussed in machine learning research. 

\subsubsection{Egalitarianism and Equal Opportunity}

Egalitarianism is a moral and political philosophy that advocates for the ideal of equality among individuals in society. At its core, egalitarianism stands for the belief that all people should be treated with equal dignity, respect, and worth, regardless of their inherent characteristics, social status, or circumstances of birth. This philosophy places a strong emphasis on the principle of opportunity, asserting that individuals should have equal access to resources, opportunities, and benefits. Egalitarianism strives to create a society characterized by fairness and justice, aiming to minimize, if not eliminate, disparities in access to income, wealth, and essential services. It asserts that factors beyond an individual's control should not be translated into social inequalities.

In this context, \textbf{equal opportunity} is a fundamental mechanism for achieving egalitarian goals. Egalitarians argue that by providing individuals with an equal chance to access opportunities like education, employment, and social services, society can move closer to the ideal of fairness, as individuals are not systematically disadvantaged due to factors beyond their control.

Equal opportunity is a practical means for addressing structural barriers that hinder individuals from realizing their potential and is instrumental in levelling the playing field. Egalitarians often advocate for policies and practices that promote both equal opportunity and their broader egalitarian objectives, such as affirmative action to rectify historical injustices, progressive taxation to reduce economic disparities, and social safety nets to ensure that all individuals have access to essential services. In essence, equal opportunity is a fundamental component of the egalitarian vision, working hand in hand with egalitarianism to reduce inequalities and foster a more just and equitable society.

The concept of equal opportunity has been translated into computable metrics with the same name~\citep{DBLP:conf/nips/HardtPNS16}, as we will discuss later when we review fairness metrics. From a machine learning perspective, the technical drawback is that metrics for equal opportunity require label annotations for true positive outcomes, which are not always available after the deployment of a model~\citep{feng2023towards}.

\subsubsection{Socialism and Equal Outcome}

Socialism is a political and economic ideology that advocates for collective or government ownership of the means of production, distribution, and exchange~\footnote{Socialism, throughout its history, has featured a multitude of perspectives and theories, frequently characterized by disparities in their conceptual, empirical, and normative principle. For example, \cite{rappoport1924dictionary} in his book of dictionary of socialism distinguished $\sim 40$ definitions of socialism.}. Socialism has been shaped by numerous influential thinkers throughout its history, such as~\cite{marx2019communist,engels2023condition}.
 
The basic intuition is that it aims to reduce economic inequalities and promote social equality by redistributing wealth and resources from the affluent to the less privileged.  While socialism does not necessarily mandate equal outcomes for all individuals, it does aim to reduce extreme disparities in wealth and income, fostering a more equitable distribution of resources. This makes socialism inherently concerned with reducing the gap in outcomes between different socio-economic groups.

In this context, \textbf{equal outcome}, is the concept of ensuring that every individual in a society attains the same or very similar results in terms of wealth, income, and social well-being. While some forms of socialism aspire to achieve more equal outcomes, it's not an inherent feature of all socialist systems. The pursuit of equal outcomes can be seen as a more extreme or utopian vision, as it often requires substantial intervention and redistribution of resources. Many socialist systems seek to balance the principles of equality and individual merit, providing a safety net and essential services while still allowing for varying degrees of success and personal achievement.

From a technical standpoint, within the framework of this PhD, we translate socialism to be construed as the aspiration for all individuals to attain the same outcome, irrespective of their individual efforts and that factors beyond an individual's control should not be translated into social inequalities. This conceptualization is operationalized by employing a random estimator, a tool that assigns outcomes without regard to any of the input. This strategic use of a random estimator not only aligns with socialist principles but also facilitates a meaningful comparison with other fairness metrics.

This concept of a random estimator, as a means to uphold the principles of socialism, has also been articulated as a form of bias preservation in a study by~\cite{wachter2020bias}.

\subsubsection{Liberalism and Equal Treatment}

Liberalism is a political and philosophical tradition that places a strong emphasis on individual rights, personal freedom, and limited government intervention in the lives of citizens~\citep{friedman1990free,nozick1974anarchy}. Liberals argue that individuals should be treated based on their merits and character rather than arbitrary factors such as race, gender, or socioeconomic status\footnote{We use the term \emph{liberalism} to refer to the perspective exemplified by~\cite{friedman1990free}. This perspective can also be referred to as \emph{neoliberalism} or \emph{libertarianism}~\citep{distributiveJustice,wiki_friedman}. }. 

Liberalism pursues the idea that societies should strive to eliminate discriminatory practices and institutional barriers that prevent certain individuals or groups from enjoying the same rights and opportunities as others. This commitment to equal treatment is often reflected in liberal democracies through anti-discrimination laws, anti-affirmative action policies, and efforts to ensure equal access to education and employment.

Equal treatment has also been defined as ``equal treatment-as-blindness" or neutrality~\citep{sunstein1992neutrality,miller1959myth}. This expression is often used in the context of policies or practices that aim to ensure fairness and impartiality by treating individuals without regard to certain characteristics, such as race, gender, or other potential sources of bias. This concept is closely related to the idea of treating everyone with equality, regardless of their inherent attributes or background.

When it is said that equal treatment is "blind" to certain characteristics, it means that those characteristics should not influence how individuals are treated in a particular context~\footnote{\cite{rawls1971theory}  introduces the concept of the "veil of ignorance." The veil of ignorance is a thought experiment suggesting that to achieve fairness and justice, individuals should design societal structures without knowing their own characteristics, such as their race, gender, or socioeconomic status. In this thesis, we understand that it also implies proxy discrimination, discrimination based on features that correlate with the protected attribute. This concept aligns with the idea of making decisions "blind" to certain characteristics to ensure impartiality and fairness.}. The goal is to create a system or environment where decisions and actions are based on merit, qualifications, or relevant criteria rather than on factors that could lead to unfair discrimination. Achieving completely equal treatment has been argued to be a social desiderata that is impossible to achieve.

\subsection{Moral Philosophy Notions}
Moral philosophy explores what constitutes ethical behaviour and has given rise to established frameworks, two prominent among them being deontology and utilitarianism. Ethical considerations are fundamental to how individuals and societies make decisions, in this thesis we limit our review to these two perspectives.

\subsubsection{Utilitarianism}

Utilitarianism is one of the most powerful and persuasive approaches to normative ethics in the history of philosophy~\footnote{Consequentialism is a broader ethical theory that asserts that the morality of an action is determined solely by the outcomes they produce, considering the overall good or bad resulting from those consequences. Utilitarianism, in the other hand, is a form of consequentialism where the right action is understood entirely in terms of consequences produced.~\citep{sep-consequentialism}}. Utilitarianism, traditionally understood through the likes of Jeremy Bentham or John Stuart Mill, focuses on the consequences of an action to determine its ethical value, considering the utility or benefit brought about by the action~\citep{sep-mill,mill1966utilitarianism,bentham1996collected}. In traditional utilitarian accounts, utility is identified as pleasure, well-being, or happiness~\citep{sep-utilitarianism-history}. The goal of moral action is therefore generally construed as increasing the amount of pleasure or happiness in our communities, with various proposed methods to measure these outcomes.

\paragraph{The Principle of Utility.} At the core of utilitarianism lies the principle of utility, which serves as the guiding criterion for evaluating actions. This principle states:

\blockquote{\emph{the greatest amount of good for the greatest number}}\footnote{This sentence is a concise summary of the principle of utility. While it's not attributed to a specific individual, the concept is closely associated with utilitarian philosophers such as Jeremy Bentham and John Stuart Mill. \cite{bentham1996collected},  in Chapter 1, titled "Of the Principle of Utility," Bentham discusses the foundation of his utilitarian philosophy. While the exact wording varies, the essence of the principle aligns with the idea of maximizing happiness or pleasure for the greatest number of people.}.

Even if the sentence can provide a synthetic summary of utilitarianism, it is not perfectly accurate. While an action may enhance happiness for the majority (the greatest number of people), it can fall short of maximizing the overall good if the smaller group, whose happiness remains unchanged or decreases, experiences significant losses that outweigh the gains of the larger group. The principle of utility prohibits sacrificing the well-being of the smaller group for the benefit of the greater number unless the net increase in overall good surpasses any available alternative.
The ethical worth of an action is contingent upon its contribution to the general well-being and happiness of the affected individuals. One key aspect of the principle of utility is its emphasis on impartiality. It considers the well-being of all individuals equally, without giving special preference to oneself or a particular group. This universality contributes to the idea that the morally right action is the one that maximizes collective happiness.

\paragraph{Consequentialist Evaluation.} Utilitarianism prompts us to assess actions based on their consequences.  The moral calculus involves weighing the positive and negative outcomes, aiming for the greatest net happiness. This principle requires assessing not just the increase in happiness for the majority but also considering potential sacrifices made by a smaller group~\citep{sep-consequentialism}. Classic utilitarianism is consequentialist, denying that moral rightness depends on anything other than consequences. It contends that breaking promises or considering the circumstances of an act is only relevant insofar as they affect future consequences. Despite its seeming simplicity, classic utilitarianism comprises various interconnected claims, ranging from its consequentialist nature to specific views on evaluating consequences and considering the well-being of all individuals equally\footnote{Classic utilitarianism, while seemingly straightforward in its emphasis on consequences, is, in fact, a complex amalgamation of distinct claims about the moral rightness of actions. These encompass various dimensions, from the overarching consequentialist stance to specific perspectives on evaluating consequences. Classic utilitarianism includes actual consequentialism, direct consequentialism, evaluative consequentialism, hedonism, maximizing consequentialism, aggregative consequentialism, total consequentialism, universal consequentialism, equal consideration, and agent-neutrality. Each of these facets adds layers of intricacy to the theory, demonstrating that classic utilitarianism entails a multifaceted framework that goes beyond a simplistic focus on consequences\citep{sep-consequentialism}.}.

Guided by the principle of utility, this ethical framework seeks to maximize overall happiness, embracing a consequentialist evaluation that adapts to the nuances of individual situations.

\subsubsection{Deontology}

Deontological ethics, with Immanuel Kant as its central figure and in particular through his seminal work \emph{Groundwork of the Metaphysics of Morals} \cite{kant1785groundwork}, emphasizes that it is through duties, or universal set of rules, that the ethical value of actions is determined. Generally, deontology maintains that these duties flow from the recognition that all humans are intrinsically valuable---it is our rational nature that gives us reason to act in one way or another. This duty, or rule-based, moral theory thus aligns well with various declarations, such as the \emph{Universal Declaration of Human Rights}, that recognize a set of inalienable rights based on the individual, such as the right to life and the right to equal treatment under the law; with deontology, moral action involves respecting inescapable duties that lay the fundamental foundation of the rights of others. While not all duties entail rights, all rights suppose a corresponding duty~\citep{watson2021right}.

\subsubsection{Kantian Critique of Utilitarianism}

After providing a concise outline of Kantian deontology and utilitarianism, we now consider the Kantian critique of utilitarianism\footnote{ Kant, however, did not directly address the utilitarian of his time, Jeremy Bentham, so it is helpful to look to contemporary Kantians such as Onora O'Neill and Christine Korsgaard, who provide a direct critique of utilitarian ethics. The following provides a foundational, but certainly not exhaustive, list of critiques.}.

\paragraph{Intrinsic value and the pursuit of the 'good'.} Utilitarianism believes that all moral actions ought to be aimed at obtaining the ‘good’ as an impersonal absolute value, such as happiness or pleasure, that can be maximized. Peter Singer writes that the classical utilitarian approach,

\blockquote{\textit{regards sentient beings as valuable only insofar as they make possible the existence of intrinsically valuable experiences like pleasure.}}~\citep{singer2011}.

Singer, himself a utilitarian, continues this notion that it is not the individuals that matter but the defined utility through the replaceability argument. While he appreciates that all individuals have inherent worth, this value does not extend far enough to guarantee its protection. Singer has thus argued that there are circumstances (albeit rare) where killing humans is the right thing to do if it brings about more happiness in the community~\cite{singer2011}.

This, of course, contrasts the second formula of the categorical imperative that stresses that each individual is intrinsically valuable and is not instrumentally subordinated to the well-being, happiness, or pleasure of others. An individual cannot be replaced by another individual even if their existence brings about more happiness. This is because Kantians disagree that the 'good' can be understood as something abstract and impersonal, something that is simply just 'good'. Rather, the good is a relational concept because, as Korsgaard argues, all value is tethered. Happiness or pleasure matters because they are \emph{good-for} individuals~\citep{korsgaard2018fellow}. Therefore, crucially, something can only be considered as universally good when it is \emph{good-for} everyone, from every individual's perspective.

\paragraph{The issue of aggregation.} Because utilitarians do not consider individuals intrinsically valuable, the moral decision procedure thus allows aggregating the good. There is no prohibition against using individuals, even if they are innocent, as mere means, as long as the loss of their dignity is for maximizing a sufficient amount of \enquote{good} in the community. Yet, this aggregation not only may violate the second formula of the categorical imperative but also consider the tethered notion of the \enquote{good}. As Korsgaard contends:

\blockquote{\emph{what is good-for me plus what is good-for you is not necessarily good-for \emph{anyone}.} (page 157~\cite{korsgaard2018fellow})}

There is a problem when maximizing the good involves taking something away from an individual: if giving something to person A brings him more pleasure than it would to person B, the utilitarian would say we ought to give it to person A. But this action would just be good for person A, and it would be worse for person B. Because it is tethered, the good is not something we can aggregate~\citep{korsgaard2018fellow}. This particularly is more concerning when those on the receiving end of aggregation are already marginalized individuals.

\paragraph{Optimizing outcomes. } 
The utilitarian approach involves calculating outcomes, whether it be maximizing benefit or reducing harm, and is not a robust approach to identifying which actions should be prohibited. As Onora O’Neill points out, circumstances can change the outcome. Take the act of lying as an example: typically, lying is considered to be immoral and, thus, generally prohibited. Yet not all lies, such as ‘white lies’, cause harm. The same goes for telling the truth, which often provides a benefit, but it can also inflict emotional harm. Focusing on consequences is unpredictable and, therefore, provides a less robust standard to measure the rightness and wrongness of actions. The link between harmful consequences and actions is unstable. The connection between norms, or duties, and actions is a more robust approach to morally evaluating actions. Duties, such as the duties of civility or to respect the freedom of others, provide a more useful and robust way of evaluating actions\citep{o2022philosopher}. 

Utilitarianism thus has a limitless scope of acceptable actions as it does not depend on the type of action but rather the consequence, or utility, it brings about. Being dishonest for the utilitarian, for example, is not \emph{categorically} condemned; for this reason, it lacks precision as the rightness or wrongness of an action can only be determined until after it is carried out and the consequences are known. The Kantian approach, while it has a more restricted scope given its categorical stance on duties, provides a precise and fundamental standard of action, where moral authority persists irrespective of any potential harmful consequence. Being dishonest (to some Kantians) is always wrong irrespective of unforeseen consequences.

Later, in the discussion section of AI alignment, we will map the Kantian criticism of current utilitarianism to our proposed Equal Treatment.

\section{Technical Types of Fairness}\label{sec:foundations.fairness}
In this section, we compare our proposed notion and model for equal treatment with respect to other metrics found in the literature. In alignment with regulatory frameworks like U.S. Title VII of the Civil Rights Act of 1964, prohibiting discrimination in employment based on race, sex, and other protected characteristics~\citep{titleVII1964} and prior work by~\citet{DBLP:journals/corr/abs-2105-01441}; we distinguish between two discrimination types: \emph{disparate treatment}, aiming to prevent intentional model discrimination, and \emph{disparate impact}, addressing unintentional effects across subpopulations. The former examines the system's intentions, as realized in model predictions $f_\theta(X)$, while the latter considers effects, interpreted as the error $f_\theta(X)-Y$.

\subsection{Disparate Impact Fairness Metrics: Approaches that Rely on Labeled Data}

\subsubsection{Equal Opportunity} 

Two formulations of the Equal Opportunity fairness metric are defined in the literature. First, it is the difference in the True Positive Rate (TPR) between the protected group and the reference group~\cite{DBLP:conf/nips/HardtPNS16,bigdata_potus,bigdata_potus1}:
\begin{gather*}
\text{TPR} = \frac{TP}{TP + FN}\\
\text{EOF}_{z}= \text{TPR}_z - \text{TPR}
\end{gather*}
Second -- a formulation proposed by~\citet{DBLP:conf/fat/HeidariLGK19} and that does not rely on the false negatives -- a model $f_\theta$ achieves equal opportunity if $P(f_\theta(X)=1|Y=1,Z=z) = P(f_\theta(X)=1|Y=1)$.

Equal Opportunity comes with a technical limitation, as it necessitates the availability of labeled data for true positives. For instance, in a loan lending scenario, true positives represent users who are granted a loan and successfully repay it. Obtaining such data can be time-consuming, requiring the completion of the entire loan cycle. This reliance on true positive rates in formulating Equal Opportunity presents an additional challenge involving the calculation of False Negatives. In the context of the loan scenario, a False Negative occurs when a loan is denied to someone capable of repayment, but acquiring this data may not be feasible. An alternative approach involves maintaining a holdout set of randomly selected users to whom loans are uniformly granted, allowing for monitoring in that controlled environment. Nevertheless, the implementation of a holdout set comes with potential economic and societal implications. Careful consideration is needed, as this approach may impact both financial outcomes and broader societal dynamics.

\subsubsection{Treatment Equality}

Despite its similar name, Treatment Equality does not necessarily align with the same distributive justice values as equal treatment, and the mathematical metric it employs is notably distinct~\citep{berk2021fairness}. 

In the original formulation by~\cite{berk2021fairness}, the concept is defined as follows:

\blockquote{\textit{Treatment equality is achieved when the ratio of false negatives and false positives is the same for both protected group categories. The term “treatment” is used to convey that such ratios can be a policy lever with which to achieve other kinds of fairness. For example, if false negatives are taken to be more costly for men than women so that conditional procedure accuracy equality can be achieved, men and women are being treated differently by the algorithm. Incorrectly classifying a failure on parole as a success, say, is a bigger mistake for men. The relative numbers of false negatives and false positives across protected group categories also can by itself be viewed as a problem in criminal justice risk assessments addresses similar issues, but through the false negative rate and the false positive rate: our b/(a+b) and c/(c + d), respectively.}}

\[
    \frac{FP_z}{FN_z}=\frac{FP}{FN}
\]

The authors do not specify the principle or value with which the notion should align, and this clarity is absent from their definition. An illustrative sentence is \enquote{\textit{men and women are being treated differently by the algorithm}}. However, this statement may not be entirely accurate, as the algorithm may treat men and women equally while exhibiting different error rates. For instance, a constant classifier that grants loans to everyone treats everyone equally. Still, if males and females have distinct patterns of loan repayment, the classifier will yield different errors, resulting in distinct $\frac{FP}{FN}$ values.

\subsection{Disparate Treatment Fairness Metrics: Approaches that Determine Unfairness without Labelled Data}
In this section we compare against metrics that rely only on $X$ and $f_\theta$ rather than on disparities of model errors.

\subsubsection{Fairness Through Unawareness.}
This bias mitigation method was initially proposed by~\citet{DBLP:conf/aaai/Grgic-HlacaZGW18} and is often used as a baseline. 

\begin{definition}\textit{(Fairness Through Unawareness}. An algorithm is fair if no protected
attributes $Z$ is explicitly used as an input of the algorithm.
\end{definition}

\subsubsection{Demographic Parity.}

\begin{definition}\textit{(Demographic Parity (DP))}. A model $f_\theta$ achieves demographic parity if $f_\theta(X) \perp Z$\label{def:dp}.
\end{definition}

Thus, demographic parity holds if  $\forall z.\,P(f_\theta(X)|Z=z)=P(f_\theta(X))$. For binary $Z$'s, we can derive an unfairness metric as $d(P(f_\theta(X)|Z=1),P(f_\theta(X))$, where $d(\cdot)$ is a distance between probability distributions.

A refined metric introduced by legal scholars aiming to align with European Court of Justice \enquote{gold standard} ~\citep{wachter2021fairness} is Conditional Demographic Parity, that extends Demographic Parity by fixing one or more attributes.

\begin{definition}\textit{(Conditional Demographic Parity (CDP))}. A model $f_\theta$ achieves conditional demographic parity if $P(f_\theta(X)|X_i=\tau,Z=z)=P(f_\theta(X)|X_i=\tau)$, where $\tau$ is any fixed value.\label{def:cdp}
\end{definition}

CDP differs from DP only in the sense that one or more additional covariate conditions are added.

\subsubsection{Individual Fairness.}

Individual fairness principle is that any two individuals who are similar concerning a particular task should be classified similarly, a topic with a broad literature on fairness, notably in social choice theory, game theory, economics, and law as well, for example, John Rawls' in \textit{Justice as Fairness} states\enquote{citizens who are similarly endowed and motivated should have similar opportunities to hold office}~\cite{rawls1958justice}. Within the computer science field, to accomplish
this individual-based fairness, a distance metric defines the similarity between the individuals. Several individual fairness metrics have been proposed, such as the Lipschitz condition or the earth-mover distance between two distributions~\cite{DworkHPRZ2012_IndFair}.

We say that a model $f$ achieves individual fairness if similar individuals received similar predictions:
\[ \forall X, X'\ d (f(X),f(X')) \leq \mathcal{L} \cdot d(X,X') 
\]
Where $\mathcal{L}>0$ is a constant and the two $d(\cdot,\cdot)$'s are distance functiosn between models' outputs and between instances respectively~\citep{DworkHPRZ2012_IndFair}.



\subsubsection{Counterfactual Fairness.}

Counterfactual fairness, as defined by~\citet{DBLP:conf/nips/KusnerLRS17}
captures the intuition that a decision is fair towards an individual if it is the same in
\emph{(i)} the actual case and \emph{(ii)} in a counterfactual case where the individual belonged
to a different protected attribute group.

\begin{definition}\textit{Counterfactual Fairness}
We say that the model $f_\theta$  is counterfactually fair if:
\[
    \forall X \in X\ \forall z, z' \in Z\    
    P(f_\theta(X_{Z=z}) | X = X, Z = z) = P(f_\theta(X_{Z=z'}) | X = X, Z = z)
\]
where $X_{Z=z}$ and $X_{Z=z'}$ are the instances $X$ in the (counterfactual) worlds where $Z=z$ and $Z=z'$ respectively.
\end{definition}

\subsubsection{Causal Feature Selection}

Another avenue of research in defining fairness operates within the causal fairness paradigm~\citep{DBLP:conf/sigmod/SalimiRHS19}. The following definition proposes a method for identifying a subset of features that ensures fairness in the data by conducting conditional independence tests between various feature subsets.

\begin{definition}\textit{(Causal Fairness)}. For a given set of admissible variables, $A$, a model $f_\theta$ is considered fair if  $P(f_\theta(X)=y|\texttt{do}(S)=s',\texttt{do}(A=a)=P(f_\theta(X)=y|\texttt{do}(S)=s,\texttt{do}(A=a)$ for all values of $A$,$S$.
\end{definition}

Here, the \texttt{do-operator} is defined as an intervention or modification of initial attributes to a specific value, observing its effect~\citep{pearl2009causality}. 

\begin{table}[htbp]
\centering
\caption{Summary of Technical Definitions of Fairness Metrics}
\label{tab:fairness_metrics}
\begin{tabular}{@{}ll@{}}
\toprule
\textbf{Fairness Metric} & \textbf{Definition} \\ \midrule

Equal Opportunity (EOF)              & \begin{tabular}[c]{@{}l@{}}$\text{EOF}_{z}= \text{TPR}_z - \text{TPR}$ \\ $\quad P(f_\theta(X)=1|Y=1,Z=z) = P(f_\theta(X)=1|Y=1)$\end{tabular} \bigskip \\
Treatment Equality                   & $\nicefrac{FP_z}{FN_z}=\nicefrac{FP}{FN}$ \\ \bigskip \\
Fairness Through Unawareness         & \begin{tabular}[c]{@{}l@{}}An algorithm is fair if no protected \\ attributes $Z$ is explicitly used as an input \\ of the algorithm\\ $f: X \rightarrow Y$ for which $Z \not\in X$ \end{tabular} \\ \bigskip \\
Demographic Parity (DP)              & $f_\theta(X) \perp Z$ \\ \bigskip \\
Conditional Demographic Parity (CDP) & $P(f_\theta(X)|X_i=\tau,Z=z)= P(f_\theta(X)|X_i=\tau)$\\ \bigskip \\
Individual Fairness                  & \begin{tabular}[c]{@{}l@{}}$\forall X, X'\ d (f(X),f(X')) \leq \mathcal{L} \cdot d(X,X')$\end{tabular} \\ \bigskip \\
Counterfactual Fairness              & \begin{tabular}[c]{@{}l@{}}$\forall X \in X\ \forall z, z' \in Z$ \\ $P(f_\theta(X_{Z=z}) | X = X, Z = z)$ \\ $= P(f_\theta(X_{Z=z'}) | X = X, Z = z)$\end{tabular} \\ \bigskip \\
Causal Fairness                      & \begin{tabular}[c]{@{}l@{}}For a given set of admissible variables, $A$, \\ model $f_\theta$ is fair if\end{tabular} \\
& $P(f_\theta(X)=y|\texttt{do}(S)=s',\texttt{do}(A=a)$ \\
& $=P(f_\theta(X)=y|\texttt{do}(S)=s,\texttt{do}(A=a)$ \\
& \quad \quad \quad $\forall$ values of $A$, $S$ \\ \bottomrule
\end{tabular}
\end{table}
\clearpage

\section{Model Monitoring}\label{sec:foundations.monitoring}

Model monitoring techniques help to detect unwanted changes in the performance of machine learning applications in a production environment. One of the biggest challenges in model monitoring is distribution shift, which is also one of the primary sources of model degradation~\citep{datasetShift,continual_learning}. 

Diverse types of model monitoring scenarios require different supervision techniques. We can distinguish two main groups: Supervised learning and unsupervised learning. Supervised monitoring, is the appealing one since the label is available and performance metrics can easily be tracked.  Acquiring labelled data post-deployment poses a considerable challenge and is often impractical,  exhibiting well-known biases, such as confounding, selection, and missingness~\citep{feng2023towards}.\footnote{\emph{Confounding Bias:} Misleading results due to an unconsidered external variable.
\emph{Selection Bias:} Skewed results from non-representative data samples.
\emph{Missingness Bias:} Distortion from systematically absent data~\citep{good2012common}.} While attractive, these techniques are often unfeasible as they rely either on having ground truth labelled data available or maintaining a hold outset, which leaves the challenge of how to monitor ML models to the realm of unsupervised learning~\citep{continual_learning}. Popular unsupervised methods that are used in this respect are the Population Stability Index (PSI) and the Kolmogorov-Smirnov test (K-S), all of which measure how much the distribution of the covariates in the new samples differs from the covariate distribution within the training samples~\citep{DBLP:conf/nips/RabanserGL19}. These methods are often limited to real-valued data, low dimensions, and require certain probabilistic assumptions~\citep{continual_learning,ShiftsData}.

Monitoring machine learning models in production is not an easy task. There are situations when the true label of the deployment data is available, and performance metrics can be monitored. However, there are cases where it is not, and performance metrics are not so trivial to calculate once the model has been deployed. Model monitoring ensures that a machine learning application in a production environment displays consistent behavior over time. 

Being able to explain or remain accountable for the performance or the deterioration of a deployed model is crucial, as a drop-in model performance can affect the whole business process ~\citep{desiderataECB}, potentially having catastrophic consequences\footnote{The Zillow case is an example of consequences of model performance degradation in an unsupervised monitoring scenario, see
\url{https://edition.cnn.com/2021/11/09/tech/zillow-ibuying-home-zestimate/index.html} (Online accessed January 26, 2022).}. Once a deployed model has deteriorated, models are retrained using previous and new input data in order to maintain high performance. This process is called continual learning ~\citep{continual_learning}, and it can be computationally expensive and put high demands on the software engineering system. Deciding when to retrain machine learning models is paramount in many situations~\citep{desiderataECB}.

\subsection{Distribution Shift Formalization}
The core machine learning assumption is that training data $\Dd{tr}$ and novel data  $\Dd{new}$ are sampled from the same underlying distribution $\PP(X, Y)$. The twin problems of \emph{model monitoring} and recognizing that new data is \emph{out-of-distribution}  can now be described as predicting an absolute or relative performance drop between $\texttt{perf}(\D^{tr})$ and $\texttt{perf}(\D^{new})$, where 
$\texttt{perf}(\D)=\sum_{(x,y)\in \D} \ell_\text{eval}(f_\theta(x),y)$, $\ell_\text{eval}$ is a metric like 0-1-loss (accuracy), but $\Dd{new}_Y$ is  unknown and cannot be used for such judgment.

Therefore related work analyses distribution shifts between training and newly occurring data.
Let two datasets $\D, \D'$ define two empirical distributions $\PP(\D), \PP(\D')$, then we write $\PP(\D) \nsim \PP(\D')$ to express that 
$\PP(\D)$ is sampled from a different underlying distribution than $\PP(\D')$ with high probability $p>1-\epsilon$ allowing us to formalize various types of distribution shifts.
\\

\subsection{Types of Distribution Shift}\label{subsec:typesShift}

This section aims to provide an illustrative introduction to distribution shift; for more in-depth surveys see~\cite{datasetShift}.

\subsubsection{Covariate Shift}
The most classical and straightforward type of distribution shift is covariate shift, where the distribution from the training/source set $P(\D_X^{tr})$ differs from the test $P(\D_X^{te})$, meaning that the input distribution changes, but the conditional probability of a label given an input remains the same.~\cite{SHIMODAIRA2000227,sugiyama,sugiyama2}

\begin{definition}[Univariate data shift] There is a univariate data shift between $\PP(\D^{tr}_{X})=\PP(\D^{tr}_{X_1}, \ldots,\D^{tr}_{X_p})$ and $\PP(\D^{new}_{X})=\PP(\D^{new}_{X_1},\ldots,\D^{new}_{X_p})$, if\,$\exists i\in\{1\ldots p\}: \PP(\D^{tr}_{X_i})\nsim \PP(\D^{new}_{X_i})$.
\end{definition}

\begin{definition}[Covariate data shift] There is a covariate data shift between $P(\D^{tr}_X)=\PP(\D^{tr}_{X_1},\ldots,\D^{tr}_{X_p})$ and $\PP(\D^{new}_X)=\PP(\D^{new}_{X_1},\ldots,\D^{new}_{X_p})$  if  $\PP(\D_X^{tr}) \nsim \PP(\D_X^{new})$, which cannot only be caused by univariate shift. 
\end{definition}

An example of this could be training a machine learning model to predict the gender of a person based on physiological characteristics. During the last century, population height and weight distribution have changed~\cite{humanheight}, but the proportion of males/females is not affected by this temporal evolution. It is important to note that there is no causal relationship between our source data $P(\D_X^{tr})$ and our test data $P(\D_X^{te})$ at this stage.\footnote{There can be confounding factors in the modelling or even statistical assumptions that do not hold in a real case scenario. But, for the sake of explanation, examples are overly simplistic}

Another example of this distribution shift could be when one wants to build a machine-learning model to detect whether someone has COVID-19 by their coughing sound. In the data collection phase, recordings are collected from the hospitals. These recordings are recorded in controlled conditions: silent rooms with clear starting and ending times. However, when you deploy the model and provide the service to the population (for example, as a phone application), users can test coughing directly with the phone app microphones, inducing a covariate bias in the test data. The starting point could be before or in the middle of the cough. Also, the coughing can have some background noise due to not being performed in laboratory conditions.\\

Both of the above examples, might create a distribution shift that will jeopardize the model's performance.\\

\subsubsection{Predictions or Label Shift}
Only the prior class probabilities change between the training 
and test sets. In this case, this second type of distribution shift occurs when only the class probabilities $P(\D_Y^{te})\neq P(\D_Y^{te})$ 
changes and everything else prevails $P(\D_X^{tr}) = P(\D_X^{te})$~\cite{priorShift}.This poses a risk towards monitoring the performance of a model since unwanted changes in the real test target distribution will make unfeasible performance metrics.

\begin{definition}[Predictions Shift]
There is a predictions shift between distributions $\mathbf{P}(\Dd{tr}_X)$ and $\mathbf{P}(\Dd{new}_X)$ related to model $f_\theta$ if $\mathbf{P}(f_\theta(\Dd{tr}_{X})) \nsim \mathbf{P}(f_\theta(\Dd{new}_X))$.
\end{definition}

An example of this shift could be economic inflation, where buying the same set of products (e.g., gas, tomatoes, and potatoes) in 2001 will cost a very different price than in 2022.\\

\subsubsection{Concept or Posterior Shift}
This type of shift is characterized by input  data 
 and predictions distribution that remains the same, and what changes is the conditional distribution of the output given an input changes:

\begin{definition}[Concept Shift] There is a concept shift between $\PP(\D^{tr})=P(\D_X^{tr}, \D_Y^{tr})$  and $\PP(\D^{new})=P(\D_X^{new}, \D_Y^{new})$ if  conditional distributions change, i.e.\
$\frac{\PP(\D^{tr}_Y)}{\PP(\D^{tr}_X)} \nsim \frac{\PP(\D^{new})_Y}{\PP(\D^{new}_X)}$. 
\end{definition}

An example of this is the Zillow case, an example of the consequences of model performance degradation due to concept drift. The price of an apartment in the city centre before COVID-19 could have drastically dropped during/after COVID-19 due to people living in the city centre, even if the apartment stayed the same.

\subsection{Model Monitoring Limitation}

This thesis addresses a common scenario in machine learning applications, characterized by a specific setup: during training, both covariates, denoted as $\mathcal{D}_X^{tr}$, and labels, $\mathcal{D}_Y^{tr}$, are available. However, only covariates, $\mathcal{D}_X^{te}$, are known at the testing stage. This situation, where practitioners possess a labelled training set but must deal with unlabeled data at deployment, complicates model monitoring and evaluation.

In such a setup, detecting covariate shifts—changes in input data distribution—is theoretically and practically achievable, as evidenced by existing literature~\citep{DBLP:conf/aaai/RamdasRPSW15,DBLP:conf/nips/RabanserGL19}. Conversely, identifying concept shifts, which refer to changes in the relationship between the input data and the target labels, inherently depends on access to labelled data.

Consequently, accurately estimating the degradation of a model's performance when transitioning from the source data, $\texttt{perf}(\mathcal{D}^{tr})$, to new, unseen data, $\texttt{perf}(\mathcal{D}^{new})$, is recognized as an infeasible task without introducing additional assumptions or characterizing concept shift~\citep{garg2022leveraging}.

\chapter{Measuring Alignment with Equal Treament}\label{ch:et}
In philosophy, debates surrounding the principles of a fair political system and the essence of ethical decision-making have given rise to well-defined frameworks addressing distributive justice~\citep{distributiveJustice,kymlicka2002contemporary} and moral decision-making processes~\citep{sep-consequentialism,kant1785groundwork}

The \emph{liberalism} school of thought\footnote{We use the term \emph{liberalism} to refer to the perspective exemplified by~\cite{friedman1990free}. This perspective can also be referred to as \emph{neoliberalism} or \emph{libertarianism}~\citep{distributiveJustice,wiki_friedman}.}, put forward by scholars such as~\cite{friedman1990free} and~\cite{nozick1974anarchy}, argues for \emph{equal treatment} of individuals regardless of their protected characteristics.  On the other hand, deontology, grounded in the philosophy of Immanuel Kant, emphasizes that actions must be judged based on their adherence to a set of moral principles or duties, regardless of the consequences.

Verdicts like the recent Supreme Court decision about college admissions,\footnote{\url{https://www.jdsupra.com/legalnews/the-impact-of-the-supreme-court-s-7330075/}} may suggest that a proper formalization of \emph{equal treatment} is required in software supporting these processes. 
 
By analyzing the college admission use case\footnote{\url{https://en.wikipedia.org/wiki/Students_for_Fair_Admissions_v._Harvard}} in Section~\ref{examples:use.case2} and the paper blind reviews case in Section~\ref{examples:blind.reviews}, we observe that existing approaches fail to meet expected requirements. For example, in college admissions, \emph{fairness-by-unawareness} \citep{DBLP:conf/aaai/Grgic-HlacaZGW18} fails to recognize discrimination through proxy variables, and \emph{individual fairness} \citep{DworkHPRZ2012_IndFair} cannot identify group discrimination.
\emph{Equal opportunity} \cite{DBLP:conf/nips/HardtPNS16} and \emph{causal fairness} \cite{pearl2009causality} are good candidates for \emph{equal treatment}, but require labelled data and background knowledge, respectively,  neither of which is available in most practical scenarios. \emph{Equal treatment} has been translated by researchers  \citep{DBLP:conf/aies/SimonsBW21,DBLP:conf/fat/HeidariLGK19,wachter2020bias} into  Demographic Parity (DP), also called Statistical Parity, which compares the distributions of predicted outcomes of a model $f$ for different social groups. However, we show that Demographic Parity may indicate fairness, even when groups are discriminated against according to the \emph{equal treatment} principle.

We propose the novel fairness notion of Equal Treatment\footnote{We distinguish the philosophical principle of ``\emph{equal treatment}'' from our proposed fairness notion,
`Equal Treatment'' (ET), by printing the first in italics and capitalizing the initials of the latter.} for supervised machine learning models, which satisfies the requirements for a formalization of 
\emph{equal treatment}. 

The notion of Equal Treatment considers the contribution of
non-protected features to the output of the machine learning model as explained by Shapley values. If two social groups are treated the same, the distributions of feature contributions, which we call \emph{explanation distributions}, will not be distinguishable. Thus, Equal Treatment tests independence of Shapley values with the protected feature.
We introduce a decision tool, the \enquote{Equal Treatment Inspector}, that implements this idea. When detecting unequal treatment, it explains the features involved in such inequality, supporting the understanding of the roots of un-equal treatment in the machine learning model. 

In summary, this chapter contributions are:
\begin{enumerate}\itemsep0em
    \item The definition of Equal Treatment, based on {explanation distributions}, as a formalization of \emph{equal treatment}.

    \item The definition of an \enquote{Equal Treatment Inspector} workflow, based on a classifier two sample test, for {recognizing and explaining un-equal treatment}.
    
    \item The study of the formal relationships between Demographic Parity and Equal Treatment.

    \item Extensive experiments, both on synthetic and natural datasets, to demonstrate our method and compare it with related work.

    \item An open-source Python package \texttt{explanation\-space} implementing the \enquote{Equal Treatment Inspector}, which is \texttt{scikit-learn} compatible, and includes documentation and tutorials.
     
\end{enumerate}

\clearpage
\section{Illustrative Examples}
\subsection{College Admissions}\label{examples:use.case2}

To illustrate the difference between equal opportunity, equal outcomes, and equal treatment, we consider the hypothetical use case, the admission process of a university:

In June 2023, the U.S. Supreme Court struck down race-conscious admission programs at some universities \citep{killenbeck2022brief}\footnote{\tiny\url{https://www.scotusblog.com/2023/06/supreme-court-strikes-down-affirmative-action-programs-in-college-admissions/}}, ruling these as discriminatory against certain racial groups. This decision marked a significant shift from previous policies that used race as one of several factors in a holistic admissions process to promote a diverse student body\footnote{\tiny\url{https://www.politico.com/news/2023/06/29/supreme-court-ends-affirmative-action-in-college-admissions-00104179}}. While the ruling does not completely forbid universities from considering applicants' racial features, it emphasizes a more colorblind approach to admissions. This move towards equal treatment focuses less on achieving equal outcomes (diverse representation) and more on not considering race as an admissions factor. Now we discuss these notions within the context of fair machine learning.

\textit{Equal Opportunity}. In this approach, the university focuses on ensuring that students from different backgrounds have an equal chance of being admitted if they have high potential or qualifications. For instance, students from under-resourced high schools are given the same opportunity as those from well-funded schools if they demonstrate the same exceptional talent or potential. From a technical perspective, equal opportunity has been operationalized by the true positive rate. The university may implement this strategy by adjusting admission criteria based on educational opportunities tied to race. However, this can lead to challenges like overcompensating for disadvantages or unintentionally lowering standards for certain groups.

\textit{Equal Outcomes}. This measure requires that the distribution of acceptance rates is similar, independently of the student (cf. Definition~\ref{def:dp}). This could mean setting targets or quotas to ensure representation from various groups, regardless of their individual academic credentials. For instance, if 20\% of applicants are from a certain ethnicity, the university might aim to have 20\% of their admits from that ethnicity as well. While this promotes a diverse student body, it can overlook individual merit and potentially lead to tension between equity and academic standards. Particularly given the recent ruling of the US Supreme Court, considering ethnicity to match the proportion of applicants from that ethnicity would likely be considered illegal\footnote{\tiny\url{https://www.jdsupra.com/legalnews/the-impact-of-the-supreme-court-s-7330075/}}. It's also important to note that outcomes can have similar rates due to random chance.

\textit{Equal Treatment}. In this scenario, the university ensures that every application is evaluated based on the same criteria, such as academic achievement, extracurricular involvement, and personal essays, without bias towards the applicant's background. This means that a student’s socioeconomic status, high school's reputation, or geographic location (protected attributes) would not influence their likelihood of admission. The challenge here is to ensure the following in the admissions process: \emph{(R1) \textit{Group Discrimination}} different decisions are made based on inherent or protected characteristics of individuals, such as race, gender, ethnicity, or other attributes that are present at birth and beyond an individual's choice or control; \emph{(R2) \textit{Unlabeled Data}}, focus on the disparate decisions, metrics should evaluate the decisions of the model instead of statistical differences on the errors; \emph{(R3) \textit{No Background Knowledge}}, metrics should not necessitate an understanding of causal or structural aspects, ensuring practical applicability in real-world scenarios;  \emph{(R4) \textit{ Proxy discrimination}}:
metrics should be capable of detecting whether the model behaviour is genuinely free from discriminatory features, arising from proxy discrimination. This involves scrutinizing the differential impact that various features may have on the predictions, ensuring that no indirect bias is influencing the outcome; \emph{(R5) \textit{Explanation Capabilities}} focuses on the necessity for fairness metrics to not only detect biases or discriminations but also to explain them.

In the university admission case study, each approach to fairness in admissions – equal opportunity, equal outcomes, and equal treatment – has its merits and challenges. Equal opportunity aims to level the playing field based on potential, equal outcomes, strive for a representative student body, and equal treatment focuses on a uniform and unbiased evaluation process. We leave the normative discussion of which fairness paradigm should be pursued by policy to the discourse in the social sciences and the broad public. This case illustrates the complexity of implementing fair practices in college admissions, highlighting that there is no one-size-fits-all solution.

\subsection{Paper Blind Reviews}\label{examples:blind.reviews}

To illustrate the difference between equal opportunity, equal outcomes, and equal treatment, based on the previously discussed framework and applied to AI, we consider the example of conference papers' blind reviews and focus on the protected attribute of the country of origin of the paper's author, comparing Germany and the United Kingdom.

For \emph{equal opportunity}, we quantify fairness by the true positive rate. In other words, it is the acceptance ratio given that the quality of the paper is high. Achieving equal opportunity will imply that these ratios are similar between the two countries. In blind reviews, the purpose is to evaluate the paper's quality and the research's merit without being influenced by factors such as the author's identity, affiliations, background or country. If we were to enforce equal opportunity in this use case, we would aim for similar true positive rates for submissions from different countries. However, this approach could lead to unintended consequences, such as unintentionally favouring certain countries to meet some quota, overcorrection or quotas of affirmative action towards certain countries.

For \emph{equal outcomes}, we require that the distribution of acceptance rates is similar, independently of the quality of the paper (cf. Definition~\ref{def:dp}). So that the ratio of accepted papers is similar for each country. Note that the outcomes can have similar rates due to random chance, even if there is a country bias in the acceptance procedure.

For \emph{equal treatment}, we require that the contributions of the features used to make a decision on paper's acceptance has similar distributions (cf. Definition~\ref{def:et}). Equality of treatment through blindness is more desirable than equal opportunity or equal outcomes because it ensures that all submissions are evaluated solely on the basis of their quality, without any bias or discrimination towards any particular country. By achieving equality of treatment through blindness, we can promote fairness and objectivity in the review process and ensure that all papers have an equal chance to be evaluated on their merits. We complement the previous use case recollecting the following requirements: \emph{(R1) \textit{Group Discrimination}} different decisions are made based on belonging to an institution,; \emph{(R2) \textit{Unlabeled Data}}, focus on the disparate decisions, metrics should evaluate the decisions of the model instead of statistical differences on the errors; \emph{(R3) \textit{No Background Knowledge}}, metrics should not necessitate an understanding of causal or structural aspects, ensuring practical applicability in real-world scenarios;  \emph{(R4) \textit{ Proxy discrimination}}:
metrics should be capable of detecting whether the model behaviour is genuinely free from institutional discriminatory features, arising from proxy discrimination.\emph{(R5) \textit{Explanation Capabilities}} focuses on the necessity for fairness metrics to not only detect biases or discriminations but also to explain them. 

\clearpage

\section{Fairness Formalization}

We assume a feature modelling protected social groups is denoted by $Z$, called \textit{protected feature}, and assume it to be binary valued in the theoretical analysis. $Z$ can either be included in the predictive features $X$ used by a model or not. If not, we assume that it is still available for a test dataset. Even without the protected feature in training data, a model can discriminate against the protected groups by using correlated features as a proxy of the protected one~\citep{DBLP:conf/kdd/PedreschiRT08}.

We write $A \bot B$
to denote statistical independence between the two sets of random variables $A$ and $B$, or equivalently, between two multivariate proability distributions. We define two common fairness notions and corresponding fairness metrics that quantify a model's degree of discrimination or unfairness \cite{mehrabi2021survey}. 

\begin{definition}\textit{(Demographic Parity (DP))}. A model $f_\theta$ achieves demographic parity if $f_\theta(X) \perp Z$
\end{definition}

Thus, demographic parity holds if  $\forall z.\,P(f_\theta(X)|Z=z)=P(f_\theta(X))$. For binary $Z$'s, we can derive an unfairness metric as $d(P(f_\theta(X)|Z=1),P(f_\theta(X))$, where $d(\cdot)$ is a distance between probability distributions.

\begin{definition}\textit{(Equal Opportunity (EO))}\label{eq:TPR} A model $f_\theta$ achieves equal opportunity if $\forall z.\,P(f_\theta(X)|Y=1,Z=z) = P(f_\theta(X)=1|Y=1)$.
\end{definition}
As before, we can measure unfairness for binary $Z$'s as  $d(P(f_\theta(X)|Y=1,Z=1),P(f_\theta(X)=1|Y=1))$. Equal opportunity comes with the problem that labels for correct outcomes are required.

\section{A Model for Monitoring Equal Treatment}\label{sec:equal:method}

\subsection{Formalizing Equal Treatment}

To establish a criterion for equal treatment, we rely on the notion of explanation distributions.

\begin{definition}\textit{(Explanation Distribution)} An explanation function $\Ss:{\cal F}\times X\to \mathbb{R}^p$ 
maps a model $f_\theta \in {\cal F}$ and an input instance $x \in X$ into a vector of reals $\Ss(f_\theta, x) \in \mathbb{R}^p$. We extend it by mapping an
input distribution $\D_X$ into an (empirical) \textit{explanation distribution} as follows:
$\Ss(f_\theta, \D_X) = \{ \Ss(f_\theta, x)\ |\ x \in \D_X\} \subseteq \mathbb{R}^p$.
\end{definition}

We will use Shapley values as an explanation function (cf.\ Appendix~\ref{sec:xai.foundations}).
 We introduce next the new fairness notion of Equal Treatment, which considers independence over the explanations of the model's outputs.

\begin{definition}\textit{(Equal Treatment (ET))}. A model $f_\theta$ achieves equal treatment if  $\Ss(f_\theta,X) \perp Z$.\label{def:et}
\end{definition}
As we will see later in Section~\ref{sec:theory}, Equal Treatment is a stronger definition than Demographic Parity since it not only requires that the distributions of the predictions are similar but that the process of how predictions are made is also similar.

\subsection{Equal Treatment Inspector}\label{sec:demographicParityInspector}
\begin{figure}[ht]
    \centering
    \includegraphics[width=\columnwidth]{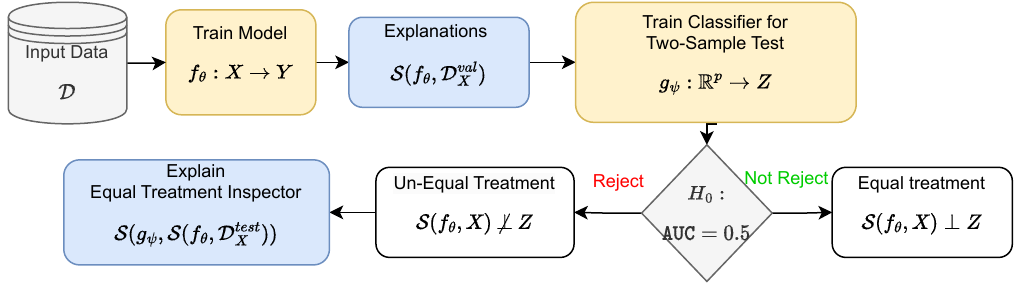}
    \caption{Equal Treatment Inspector workflow. The  model $f_\theta$ is learned based on training data, $\D = \{(x_i,y_i)\}$,  and outputs the explanations $\Ss(f_\theta,\D_X)$. The Classifier Two-Sample Test receives the explanations to predict the protected attribute, $Z$. The AUC of the two-sample test classifier $g_\psi$ decides for or against \emph{equal treatment}.  We can interpret the driver for unequal treatment on $g_\psi$ with explainable AI techniques}
    \label{fig:equal:workflow} 
\end{figure}

Our approach is based on the  properties of the Shapley values (cf. Appendix \ref{sec:xai.foundations})  and the prior work of ~\cite{DBLP:conf/iclr/Lopez-PazO17} of classifier two-sample tests. We split the available data into three parts $\Dd{tr},\Dd{val},\Dd{te} \subseteq X \times Y$. Here $\Dd{tr}$ is the training set of $f_\theta \in \cal F$ (not required if $f_\theta$ is already trained). Following the intuition above, $\Dd{val}$  is used to train another model $g_\psi$ on the distribution $\Ss(f_\theta, \Dd{val}_{X\setminus{Z}}) \times \Dd{val}_Z$ where the predictive features are in the explanation distribution $\Ss(f_\theta, \Dd{val}_{X\setminus{Z}})$ (excluding $Z$) and the target feature is the  protected feature $Z$.  The model belongs to a family $\cal G$, possibly different from $\cal F$. The parameter $\psi$ optimizes a loss function $\ell$:
\begin{gather}\label{eq:fairDetector}
\psi = \argmin_{\tilde{\psi}} \sum_{(x, z) \in \Dd{val} } \ell( g_{\tilde{\psi}}(\Ss(f_\theta,x)) , z )
\end{gather}
Finally, we use $\Dd{te}$ for testing the approach and for comparison with baselines.  To evaluate whether there is a fairness violation or not, we perform a statistical test of independence based on the AUC. See the workflow of Figure \ref{fig:workflow} for a visualization of the process.

Besides detecting and measuring fairness violations in machine learning models, a common desideratum is to understand what are the specific features driving the discrimination. Using the \enquote{Equal Treatment Inspector} as an auditor method that aims to depict and quantify possible fairness violations does not 
only report a metric but provides information on \textit{which features are the cause of the un-equal treatment}. We propose to solve this issue by applying explainable AI techniques to the Inspector. The \enquote{Equal Treatment Inspector} can provide different types of explanations, re-purposing the goal of the traditional explainable AI field, from understanding the model predictions to accounting for the reasons of un-equal treatment.

\clearpage
\section{Theoretical Analysis}\label{sec:theory}

Throughout this section, we assume an exact calculation of the Shapley values $\Ss(f_\theta, x)$ for instance $x$, possibly for the observational and interventional variants (see (\ref{def:val:obs},\ref{def:val:int}) in the foundations Chapter~\ref{sec:xai.foundations}). In the experimental section, we will use non-IID data and non-linear models.


\subsection{Equal Treatment Given Shapley Values of Protected Attribute}\label{sec:expSpaceIndependence}

Can we measure Equal Treatment by looking only at the Shapley value of the protected feature? The following result considers a linear model (with unknown coefficients) over \textit{independent} features. In such a very simple case, resorting to Shapley values leads to an exact test of both Demographic Parity and Equal Treatment, which turn out to coincide. In the following, we write $\mathit{distinct}(\D_X{}, i)$ for the number of distinct values in the $i$-th feature of dataset $\D_X{}$, and $\Ss(f_\beta, \D_X{})_i \equiv 0$ if the Shapley values of the $i$-th feature are all $0$'s.

\begin{lemma}\label{lemma:1} Consider a linear model $f_\beta(x) = \beta_0 + \sum_j \beta_j \cdot x_j$. Let $Z$ be the $i$-th feature, i.e. $Z = X_i$, and let $\D_X{}$ be such that $\mathit{distinct}(\D_X{}, i)>1$.
If the features in $X$ are independent, then $\Ss(f_\beta, \D_X{})_i \equiv 0 \Leftrightarrow f_\beta(X) \perp Z \Leftrightarrow \Ss(f_\beta,X) \perp Z$.
\end{lemma}
\begin{proof}
It turns out $\Ss(f_\beta, x)_i = \beta_i \cdot (x_i - E[X_i])$. This holds in general for the interventional variant (\ref{def:val:int}), and assuming independent features, also for the observational
variant (\ref{def:val:obs})
\citep{DBLP:journals/ai/AasJL21}. Since $\mathit{distinct}(\D_X{}, i)>1$, we have that $\Ss(f_\beta, \D_X{})_i \equiv 0$ iff  $\beta_i = 0$. By independence of $X$, this is equivalent to $f_\beta(X) \perp X_i$, i.e., $f_\beta(X) \perp Z$. Moreover, by the propagation of independence, this is also equivalent to $\Ss(f_\beta,X) \perp Z$.
\end{proof}

However, the result does not extend to the case of dependent features. 

\begin{example}
Consider $Z = X_2 = X_1^2$, and the linear model $f_\beta(x_1, x_2) = \beta_0 + \beta_1 \cdot x_1$ with $\beta_1 \neq 0$ and $\beta_2 = 0$, i.e., the protected feature is not used by the model. In the interventional variant, theuninformativeness property implies that $\Ss(f_\beta, x)_2 = 0$. However, this does not mean that $Z = X_2$ is independent of the output because $f_\beta(X_1, X_2) = \beta_0 + \beta_1 \cdot X_1 \not \perp X_2$. In the observational variant, \cite{DBLP:journals/ai/AasJL21} show that:
\[
val(T) = \sum_{i \in N\setminus T} \beta_i \cdot E[X_i| X_T = x^{\star}_T] + \sum_{i \in T} \beta_i \cdot x^{\star}_i
\]
\noindent from which, we calculate:
$\Ss(f_\beta, x^{\star})_2 = \frac{\beta_1}{2} E[X_1|X_2 =x^{\star}_2]$.
We have $\Ss(f_\beta, \D_X)_2 \equiv 0$ iff $E[X_1|x_2 =x^{\star}_2] = 0$ for all $x^{\star}$ in $\D_X{}$. For the  marginal distribution $P(X_1=v) = 1/4$ for $v=1, -1, 2, -2$, and considering that $X_2=X_1^2$, it holds that $E[X_1|x_2 =v] = 0$ for all $v$. Thus $\Ss(f, \D_X)_2 \equiv 0$. However, again $f_\beta(X_1, X_2) = \beta_0 + \beta_1 \cdot X_1 \not \perp X_2$.
\end{example}
The counterexample shows that focusing only on the Shapley values of the protected feature is not a viable way to prove Demographic Parity of a model -- and, a fortiori, neither to prove Equal Treatmentof the model, as will show in Lemma \ref{lemma:inc}.


\subsection{Equal Treatment vs Demographic Parity vs Fairness of the Input}\label{subsec:ETvsEOvsInd}



\label{sec:outcome-fairness}
We start by observing that \emph{equal treatment} (independence of the explanation distribution from the protected attribute) is a sufficient condition fordemographic parity (independence of the prediction distribution from the protected attribute).

\begin{lemma}\label{lemma:inc} If $\Ss(f_\theta, X) \perp Z$ then $f_\theta(X) \perp Z$.
\end{lemma}
\begin{proof}
By the propagation of independence in probability distributions, the premise implies $(\sum_i \Ss_i(f_\theta, X) + c) \perp Z$ where $c$ is any constant. By setting $c=E[f(X)]$ and by the efficiency property, we have the conclusion.
\end{proof}

Therefore, a Demographic Parity violation (on the prediction distribution) is also an Equal Treatment violation (in the explanation distribution). Equal Treatment accounts for a stricter notion of fairness. The other direction does not hold. We can have dependence of $Z$ from the explanation features, but the sum of such features cancels out, resulting in perfect Demographic Parity on the prediction distribution. This issue is also known as Yule's effect~\citep{DBLP:conf/aaai/Ruggieri0PST23}.

\begin{example}\label{ex42} Consider the model $f(x_1, x_2) = x_1 + x_2$. Let $Z \sim Ber(0.5)$, $A \sim U(-3, -1)$, and $B \sim N(2, 1)$ be independent, and let us define:
\[ X_1 = A \cdot Z + B \cdot (1-Z) \quad X_2 = B \cdot Z + A \cdot (1-Z)\]
We have $f(X_1, X_2) = A + B \perp Z$ since $A, B, Z$ are independent.
Let us calculate $\Ss(f, X)$ in the two cases $Z=0$ and $Z=1$. If $Z=0$, we have $f(X_1, X_2) = B + A$, and then $\Ss(f, X)_1 = B-E[B] = B-2 \sim N(0, 1)$ and $\Ss(f, X)_2 = A-E[A] = A+2 \sim U(-1, 1)$. Similarly, for $Z=1$, we have $f(X_1, X_2) = A + B$, and then $\Ss(f, X)_1 = A-E[A]=A+2 \sim U(-1, 1)$ and $\Ss(f, X)_2 = B-E[B] = B-2 \sim N(0, 1)$. This shows:
\[ P(\Ss(f, X) | Z=0) \neq P(\Ss(f, X) | Z=1)\]
and then $\Ss(f, X) \not \perp Z$. Notice this example holds both for the interventional and the observational cases, as we exploited Shapley values of a linear model over independent features, namely $A, B, Z$.
\end{example}


Statistical independence between the input $X$ and the protected attribute $Z$, i.e.,  $X \perp Z$, is another fairness notion. It targets fairness of the (input) datasets, disregarding the model $f_\theta$. For fairness-aware training algorithms, which are able not to (directly or indirectly) rely on $Z$, violation of such a notion of fairness does not imply Equal Treatmentviolation nor Demographic Parity violation.

\begin{example}
Let $X = X_1,X_2,X_3$ be independent features 
such that $E[X_1] = E[X_2] = E[X_3] = 0$, and $X_1, X_2 \perp Z$, and $X_3 \not \perp Z$. The target feature is defined as $Y = X_1 + X_2$, hence it is also independent from $Z$. Assume a linear regression model $f_\beta(x_1, x_2,x_3) = \beta_1 \cdot x_1 + \beta_2 \cdot x_2 + \beta_3 \cdot x_3$ trained from a sample data from $(X, Y)$ with $\beta_1, \beta_2 \approx 1$ and $\beta_3 \approx 0$. Intuitively, this occurs when a number of features are collected to train a classifier without a clear understanding of which of them contributes to the prediction. It turns out that $X \not \perp Z$ but, for $\beta_3 = 0$ (which can be obtained by some fairness regularization method \citep{DBLP:conf/icdm/KamishimaAS11}), we have $f_\beta(X_1, X_2, X_3) = \beta_1 \cdot X_1 + \beta_2 \cdot X_2 \perp Z$. By reasoning as in the proof of Lemma \ref{lemma:1}, we have $\Ss(f_\beta, X) = (\beta_1 \cdot X_1, \beta_2 \cdot X_2, 0)$ and then $\Ss(f_\beta, X) \perp Z$. This holds both in the interventional and in the observational variants.
\end{example}

The above represents an example where the input data depends on the protected feature, but the model and the explanations are independent. 

\subsection{Statistical Independence Test via Classifier AUC Test}\label{sec:stat.independence}

In this subsection, we introduce a statistical test of independence based on the AUC of a binary classifier. The test of $W \perp Z$ is stated in general form for multivariate random variables $W$ and a binary random variable $Z$ with $dom(Z) = \{0, 1\}$. In the next subsection, we will instantiate it to the case $W = \Ss(f_\theta, X)$.

Let $\D = \{ (w_i, z_i) \}_{i=1}^n$ be a dataset of realizations of the random sample $(W, Z)^n \sim \mathcal{F}^n$ where $\mathcal{F}$ is unknown.
The independence $W \perp Z$ can be tested via a two-sample test. In fact,
we have $W \perp Z$ iff $P(W|Z)=P(W)$ iff $P(W|Z=1) = P(W|Z=0)$. We test whether the positives and negatives instances in $\D$ are drawn from the same distribution by a novel two-sample test, which does not require permutation of data nor equal proportion of positive and negatives as in \cite[Sections 2 and 3]{DBLP:conf/iclr/Lopez-PazO17}. We rely on a probabilistic classifier $f: W \rightarrow [0, 1]$, for which $f(w)$ estimates $P(Z=1|W=w)$, and on its AUC:
\begin{gather}\label{eq:auc}
AUC(f) = E_{(W,Z), (W',Z') \sim \mathcal{F}}[I( (Z-Z')(f(W)-f(W')) > 0) + \nicefrac{1}{2} \cdot I(f(W)=f(W')) | Z \neq Z'] 
\end{gather}

Under the null hypothesis $H_0: W \perp Z$, we have $AUC(f)=\nicefrac{1}{2}$. 
\begin{lemma}\label{lemma:test}
If $W \perp Z$ then  $AUC(f)=\nicefrac{1}{2}$ for any classifier $f$.
\end{lemma}
\begin{proof}
Let us recall the definition of the Bayes Optimal classifier $f_{opt}(w) = P(Z=1|W=w)$.
For any classifier $f$, we have:
\begin{gather}\label{eq:aucineq} 
AUC(f_{opt}) \geq AUC(f) \geq 1 - AUC(f_{opt})
\end{gather}
The first bound $AUC(f_{opt}) \geq AUC(f)$ follows because the Bayes Optimal classifier minimizes the Bayes risk~\citep{DBLP:conf/ijcai/GaoZ15}.
Assume the second bound does not hold, i.e., for some $f$ we have $AUC(f_{opt}) < 1 - AUC(f)$. Consider the classifier $\bar{f}(w) = 1-f(w)$. We have $AUC(\bar{f}) \geq 1-AUC(f)$, and then $\bar{f}$ would contradict the first bound because $AUC(f_{opt}) < AUC(\bar{f})$.

If $W \perp Z$, then $P(Z=1|W=w) = P(Z=1)$, and then $f_{opt}(w)$ is constant. By (\ref{eq:auc}), this implies $AUC(f_{opt})=\nicefrac{1}{2}$.
By (\ref{eq:aucineq}), this implies
$AUC(f)=\nicefrac{1}{2}$ for any classifier.
\end{proof}

As a consequence, any statistics to test $AUC(f)=\nicefrac{1}{2}$ can be used for testing $W \perp Z$. 
%
A classical choice is to resort to the Wilcoxon–Mann–Whitney test, which, however, assumes that the distributions of scores for positives and negatives have the same shape. Better alternatives include the Brunner–Munzel test \citep{DBLP:journals/csda/NeubertB07} and the Fligner–Policello test \citep{fligner1981robust}. The former is preferable, as the latter assumes that the distributions are symmetric.

\subsection{Testing for Equal Treatment via an Inspector}

We instantiate the previous AUC-based method for testing independence to the case of testing for Equal Treatment via an Equal TreatmentInspector.

\begin{theorem}\label{thm:main}
Let $g_\psi: \Ss(f_\theta, X) \rightarrow [0,1]$ be an \enquote{Equal Treatment Inspector} for the model $f_{\theta}$, and $\alpha$ a significance level. We can test the null hypothesis $H_0: \Ss(f_\theta, X) \perp Z$ at $100 \cdot (1-\alpha)\%$ confidence level using a test statistics of $AUC(g_\psi) = \nicefrac{1}{2}$. 
\end{theorem}

\begin{proof}
Under $H_0$, by Lemma \ref{eq:auc} with $W = \Ss(f_\theta, X)$ and $f = g_\psi$, we have $AUC(g_\psi) = \nicefrac{1}{2}$. 
\end{proof}

Results of such a test can include $p$-values of the adopted test for $AUC(g_\psi) = \nicefrac{1}{2}$. Alternatively, confidence intervals for $AUC(g_\psi)$ can be reported, as returned by the Brunner–Munzel test or by the methods  \citep{delong1988comparing,DBLP:conf/nips/CortesM04,gonccalves2014roc}.

\subsection{Statistical Independence Test via Classifier AUC Test}\label{app:stat.independence.exp}

We complement the experiments of Section \ref{sec:experiments} by reporting in Table \ref{tab:auc.stats} the results of the C2ST for group pair-wise comparisons. As discussed in Section \ref{sec:stat.independence}, we perform the statistical test $H_0: AUC=\nicefrac{1}{2}$ of the \enquote{Equal Treatment Inspector} using a Brunner-Munzel one tailed test against $H_1: AUC>\nicefrac{1}{2}$ as implemented by~\citet{2020SciPy-NMeth}. 
Table \ref{tab:auc.stats} reports the empirical AUC on the test set, the confidence intervals at 95\% confidence level (columns \enquote{Low} and \enquote{High}), and the p-value of the test.  
The \enquote{Random} row regards a randomly assigned group and represents a baseline for comparison. The statistical tests clearly show that the AUC is significantly different from $\nicefrac{1}{2}$, also when correcting for multiple comparison tests.

\begin{table}[ht]
\centering
\caption{Results of the C2ST on the \enquote{Equal Treatment Inspector}. 
}\label{tab:auc.stats}
\begin{tabular}{c|ccccc}
\textit{\textbf{Pair}} & \textbf{AUC} & \textbf{Low} & \textbf{High} & \textbf{pvalue} & \textbf{Test Statistic} \\ \hline
\textit{Random}        & 0.501        & 0.494        & 0.507         & 0.813           & 0.236              \\
\textit{White-Other}   & 0.735        & 0.731        & 0.739         & $< 2.2e\text{-}16$             & 97.342             \\
\textit{White-Black}   & 0.62         & 0.612        & 0.627         & $< 2.2e\text{-}16$            & 27.581             \\
\textit{White-Mixed}   & 0.615        & 0.607        & 0.624         & $< 2.2e\text{-}16$             & 23.978             \\
\textit{Asian-Other}   & 0.795        & 0.79         & 0.8           & $< 2.2e\text{-}16$             & 107.784            \\
\textit{Asian-Black}   & 0.667        & 0.659        & 0.676         & $< 2.2e\text{-}16$             & 38.848             \\
\textit{Asian-Mixed}   & 0.644        & 0.634        & 0.653         & $< 2.2e\text{-}16$             & 28.235             \\
\textit{Other-Black}   & 0.717        & 0.708        & 0.725         & $< 2.2e\text{-}16$             & 48.967             \\
\textit{Other-Mixed}   & 0.697        & 0.688        & 0.707         & $< 2.2e\text{-}16$             & 39.925             \\
\textit{Black-Mixed}   & 0.598        & 0.586        & 0.61          & $< 2.2e\text{-}16$             & 15.451            
\end{tabular}
\end{table}
\begin{figure}[ht]
\centering
\includegraphics[width=.8\textwidth]{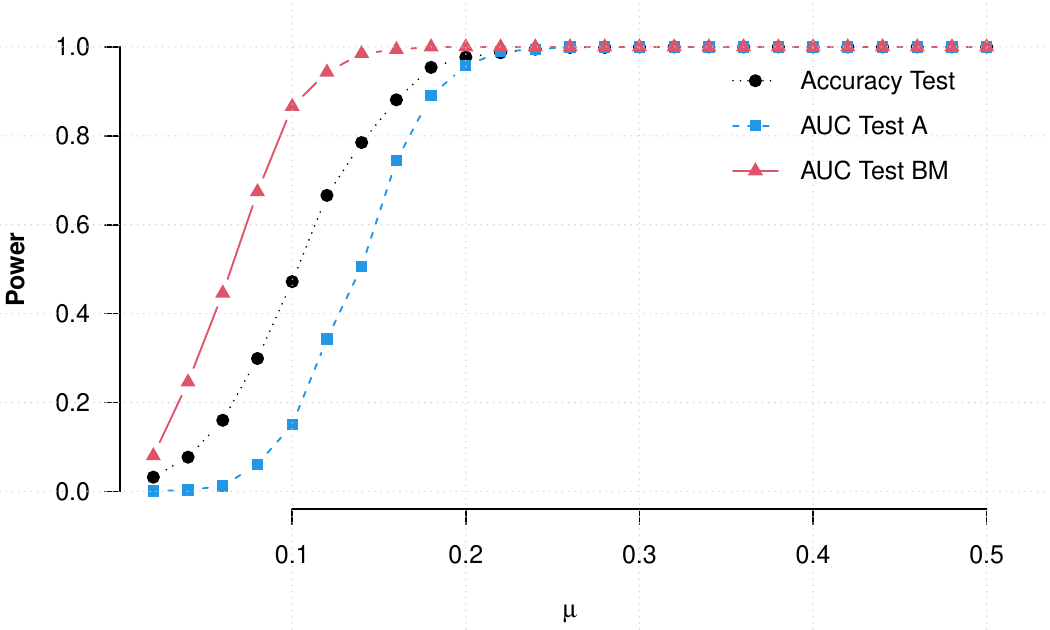}\hfill
\includegraphics[width=.8\textwidth]{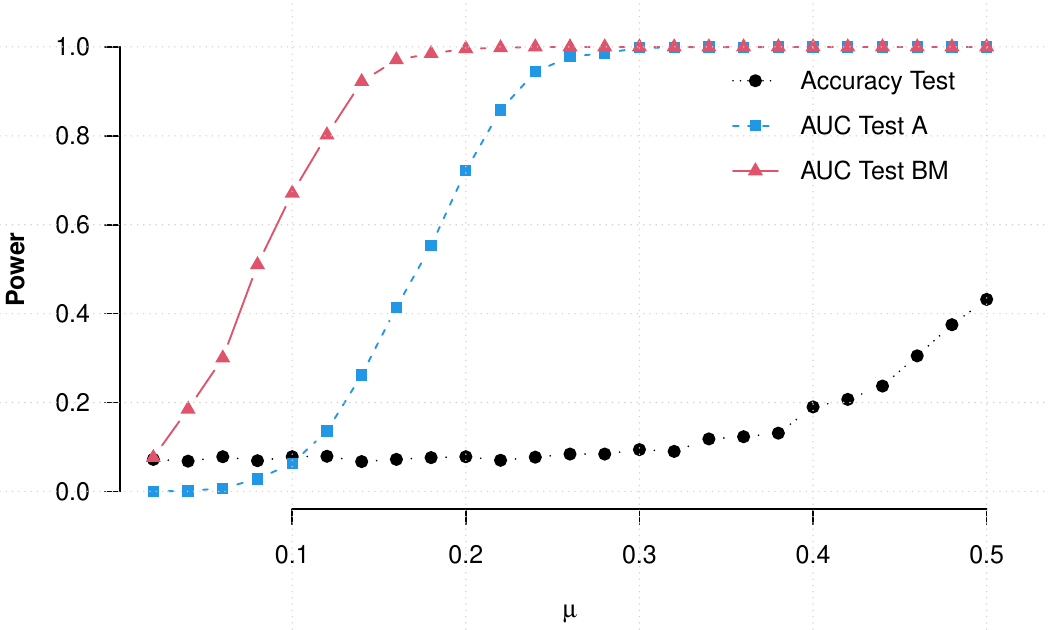}
\caption{Comparing the power (the higher, the better) of C2ST based on AUC with Brunner-Munzel test (AUC Test BM) vs Accuracy vs AUC with an asymptotic normal approximation of the Wilcoxon–Mann–Whitney statistics (AUC Test A). Upper: balanced groups  ($P(Z=1)=0.5$). Lower: unbalanced groups ($P(Z=1)=0.2$).}
\label{fig:power}
\end{figure}

We also compare the power of the C2ST based on the AUC using the Brunner-Munzel test against the two-sample test of \citet{DBLP:conf/iclr/Lopez-PazO17}, which is based on accuracy, and against the AUC test of \citet{chakravarti2023model}, which is based on the asyntotic normal approximation of the Wilcoxon–Mann–Whitney statistics.
We generate synthetic datasets from $\mathbf{X} \times Z$, where $Z \sim Ber(0.5)$ (balanced groups) or $Z \sim Ber(0.2)$ (unbalanced groups),
and $\mathbf{X} = (\mathbf{X}_1, \mathbf{X}_2)$ with positives ($Z=1$) distributed as $N( \begin{bmatrix}\mu \ \mu\\ \end{bmatrix}, \mathbf{\Sigma})$ and negatives ($Z=0$) distributed as $N( \begin{bmatrix}-\mu \ -\mu\\ \end{bmatrix}, \mathbf{\Sigma})$, where $\mathbf{\Sigma} = \begin{bmatrix}1 & 0.5 \\ 0.5 & 1 \end{bmatrix}$. Thus, the larger the parameter $\mu$, the easier is to distinguish the distributions of positives and negatives. Figure \ref{fig:power} reports the power\footnote{Probability of rejecting $H_0$ when it does not hold.} of the three tests using in all cases a logistic regression classifier. The power is
estimated by 1,000 runs for each of the $\mu$'s ranging from $0.005$ to $0.5$. The figure highlights that, under such a setting, testing the AUC using the Brunner-Munzel test achieves better power than using accuracy or using an asymptotic test. Our approach exhibits the best power, and the difference is higher in the case groups are unbalanced.

\subsection{Explaining Un-Equal Treatment}\label{app:xai.eval}

The following example showcases one of our main contributions: detecting \textit{the sources} of un-equal treatment through interpretable by-design (linear) inspectors. Here, we assume that the model is also linear. In the Appendix \ref{app:xai.eval}, we will experiment with non-linear models.

\begin{example} Let $X = X_1,X_2,X_3$ be independent features 
such that $E[X_1] = E[X_2] = E[X_3] = 0$, and $X_1, X_2 \perp Z$, and $X_3 \not \perp Z$. 
Given a random sample of i.i.d.~observations from $(X, Y)$, a linear model $f_\beta(x_1, x_2, x_3) = \beta_0 + \beta_1 \cdot x_1 + \beta_2 \cdot x_2 + \beta_3 \cdot x_3$ can be built by OLS (Ordinary Least Square) estimation, possibly with $\beta_1, \beta_2, \beta_3 \neq 0$. By reasoning as in the proof of Lemma \ref{lemma:1}, $\Ss(f_\beta, x)_i = \beta_i \cdot x_i$. Consider now a linear Equal TreatmentInspector $g_\psi(s) = \psi_0 + \psi_1 \cdot s_1 + \psi_2 \cdot s_2+\psi_3 \cdot s_3$, which can be written in terms of the $x$'s as: $g_\psi(x) = \psi_0 + \psi_1 \cdot \beta_1 \cdot x_1 + \psi_2 \cdot \beta_2 \cdot x_2+\psi_3 \cdot \beta_3 \cdot x_3$. By OLS estimation properties, we have $\psi_1 \approx cov(\beta_1 \cdot X_1, Z)/var(\beta_1 \cdot X_1) = cov(X_1, Z)/(\beta_1 \cdot var(X_1))  = 0$ and analogously $\psi_2 \approx 0$. Finally, $\psi_3 \approx cov(X_3, Z)/(\beta_3 \cdot var(X_3)) \neq 0$. In summary,  the coefficients of $g_\psi$ provide information about which feature contributes (and how much it contributes) to the dependence between the explanation $\Ss(f_\beta, X)$ and the protected feature $Z$. Notice that also $f_\beta(X) \not \perp Z$, but we cannot explain which features contribute to such a dependence by looking at $f_\beta(X)$, since $\beta_i \approx cov(X_i, Y)/var(X_i)$ can be non-zero also for $i = 1, 2$.
\end{example}

\subsection{Using AUC as Classifier Two-Sample Test evaluation metric}\label{et.auc.metric}

We complement the above results of the experimental with a further experiment relating the correlation hyperparameter $\gamma$ to the coefficients of an explainable Equal Treatmentinspector. We consider a synthetic dataset with one more feature, by drawing $10,000$ samples from a  $X_1 \sim N(0, 1)$, $X_2 \sim N(0, 1)$, and $(X_3,X_5)$ and $(X_4,X_5)$ following bivariate normal distributions $N\left(\begin{bmatrix}0 \ 0 \end{bmatrix},\begin{bmatrix}1 & \gamma \ \gamma & 1 \end{bmatrix}\right)$ and $N\left(\begin{bmatrix}0 \ 0 \end{bmatrix},\begin{bmatrix}1 & \gamma0.5 \ \gamma0.5 & 1 \end{bmatrix}\right)$, respectively. We define the binary protected feature $Z$ with values $Z=1$ if $X_5 > 0$ and $Z=0$ otherwise. 
As in Section \ref{exp:synthetic}, we consider two experimental scenarios. In the first scenario, the indirect case, we have unfairness in the data and in the model. The targe feature is $Y = \sigma(X_1 + X_2 + X_3 + X_4)$, where $\sigma$ is the logistic function. In the second scenario, the uninformative case, we have unfairness in the data and fairness in the model. The target feature is $Y = \sigma(X_1 + X_2)$.

Figure \ref{fig:xai.coef} shows how the coefficients of the inspector $g_{\psi}$ vary with correlation $\gamma$ in both scenarios. In the indirect case, coefficients for $\Ss(f_{\theta}, X_1)_1$ and $\Ss(f_{\theta}, X_1)_2$  correctly attributes zero importance to such variables, while coefficients for $\Ss(f_{\theta}, X_1)_3$ and $\Ss(f_{\theta}, X_1)_4$ grow linearly with $\gamma$, and with the one for $\Ss(f_{\theta}, X_1)_3$ with higher slope as expected. In the uninformative case, coefficients are correctly zero for all variables.

\begin{figure}[ht]
\centering
\includegraphics[width=0.8\textwidth]{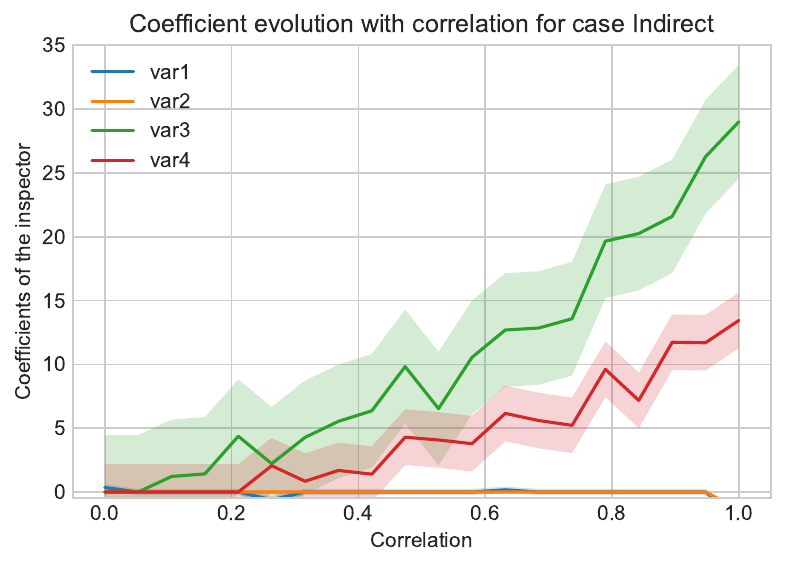}
\includegraphics[width=0.8\textwidth]{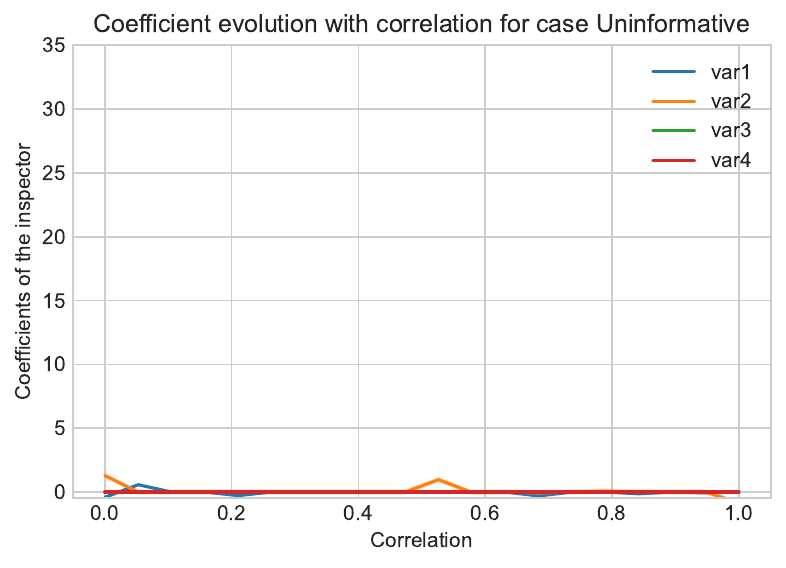}
\caption{Coefficient of $g_{\psi}$ over $\gamma$ for synthetic datasets in two experimental scenarios.}
\label{fig:xai.coef}
\end{figure}
\clearpage
\section{Experiments}

In our experiments, we aim to rigorously assess equal treatment by employing a methodology that includes:
\begin{itemize}

    \item Experiments on synthetic data in Section~\ref{et.exp.synt}

    \item Different types of Shapley values estimation in Section~\ref{et.exp.TrueModel.TrueData} 

    \item Using LIME as Local Feature Attribution Method~\ref{et.app:LIME} 
        
    \item  Compare AUC vs accuracy for the Classifier Two Sample independence test ~\ref{et.auc.metric}

    \item  Experiments natural datasets in Section~\ref{et.exp.real}
    \item Systematically varying the model $f$ in Section~\ref{subsec:EstimatorInspectorVariations}
    
    \item Systematically varying the parameters $\theta$ in Section~\ref{subsec:EstimatorInspectorVariations}
    
    \item Extend the comparison against Demographic Parity in Section~\ref{app:stat.DP.ET}

\end{itemize}

\subsection{Comparison methods and baselines:}
We evaluate the "Equal Treatment Inspector" that does Classifier Two-Sample Test on the explanation, $g_\psi$(eq. \ref{eq:fairDetector}),against the same test on other distributions: input data distributions $g_\Upsilon$ (also known as bias on the data), prediction distributions $g_\upsilon$(also known as demographic parity) and a combination of both $g_\phi$ (see Equation~\ref{eq:equal:c2st.all.dist}). These experiments are grounded on previously discussed theory (Section ~\ref{sec:theory}):

\begin{gather}
    \Upsilon = \argmin_{\tilde{\Upsilon}} \sum_{(x, z) \in \Dd{val}} \ell( g_{\tilde{\Upsilon}}(\textcolor{blue}{x}) , z) 
\end{gather}
\begin{gather}
    \upsilon = \argmin_{\tilde{\upsilon}} \sum_{(x, z) \in \Dd{val}} \ell( g_{\tilde{\upsilon}}(\textcolor{blue}{f_\theta(x)}) , z )
\end{gather}
\begin{gather}
    \phi = \argmin_{\tilde{\phi}} \sum_{(x, z) \in \Dd{val}} \ell( g_{\tilde{\phi}}(\textcolor{blue}{{f_\theta(x),x}}) , z )
\end{gather}\label{eq:equal:c2st.all.dist}
\clearpage
\subsection{Experiments with Synthetic Data}\label{et.exp.synt}

\textbf{Dataset:} To generate a synthetic dataset for both cases, we first draw $10,000$ samples from a  normal distribution $X_1 \sim N(0,1), X_2 \sim N(0,1), (X_3,X_4) \sim  N\left(\begin{bmatrix}0  \\ 0 \end{bmatrix},\begin{bmatrix}1 & \gamma \\ \gamma & 1 \end{bmatrix}\right)$. We then define a binary feature $Z$ with values $1\quad\texttt{if} \quad X_4>0,\quad \texttt{else}\quad 0$. We compare the fairness auditing methods while increasing the correlation $\gamma = r(X_3, Z)$ from 0 to 1. In both experimental scenarios below, our model $f_\beta$, is a function over the domain of the features $X_1, X_2, X_3$.

\textbf{Experimental scenarios:}\\
\textbf{Indirect Case: }\textit{Bias in the data and un-equal treatment model.} There is a demographic parity violation in the input data that is learned by the model. 
The predictor's domain is $(X_1, X_2,X_3)$, and the features are independent of each other. The protected attribute is $Z$ (binary-valued) and its correlation with the predictor's domain $(X_3, Z)$ is parameterized by $\gamma(X_3, Z)$, allowing to adjust for discrimination. To generate the synthetic demographic parity violation in the model we create the target $Y = \sigma(X_1 + X_2 + X_3)$, where $\sigma$ is the logistic function.

\textbf{Uninformative Case: }\textit{Bias on the data but equal treatment model.} The demographic parity violation on the input data remains the same but the relationship of the target variable changes, it is now independent of the protected attribute.
The target function is defined as $Y = \sigma(X_1 + X_2)$, and, the $\gamma$ parameter allows adjusting for discrimination in the training data even if the model does not capture it. The target is then  independent of the protected attribute, $Y \perp X_3 \Rightarrow Y \perp Z \Rightarrow f_\beta \perp Z$.

\begin{figure*}[ht]
\centering
\includegraphics[width=.8\textwidth]{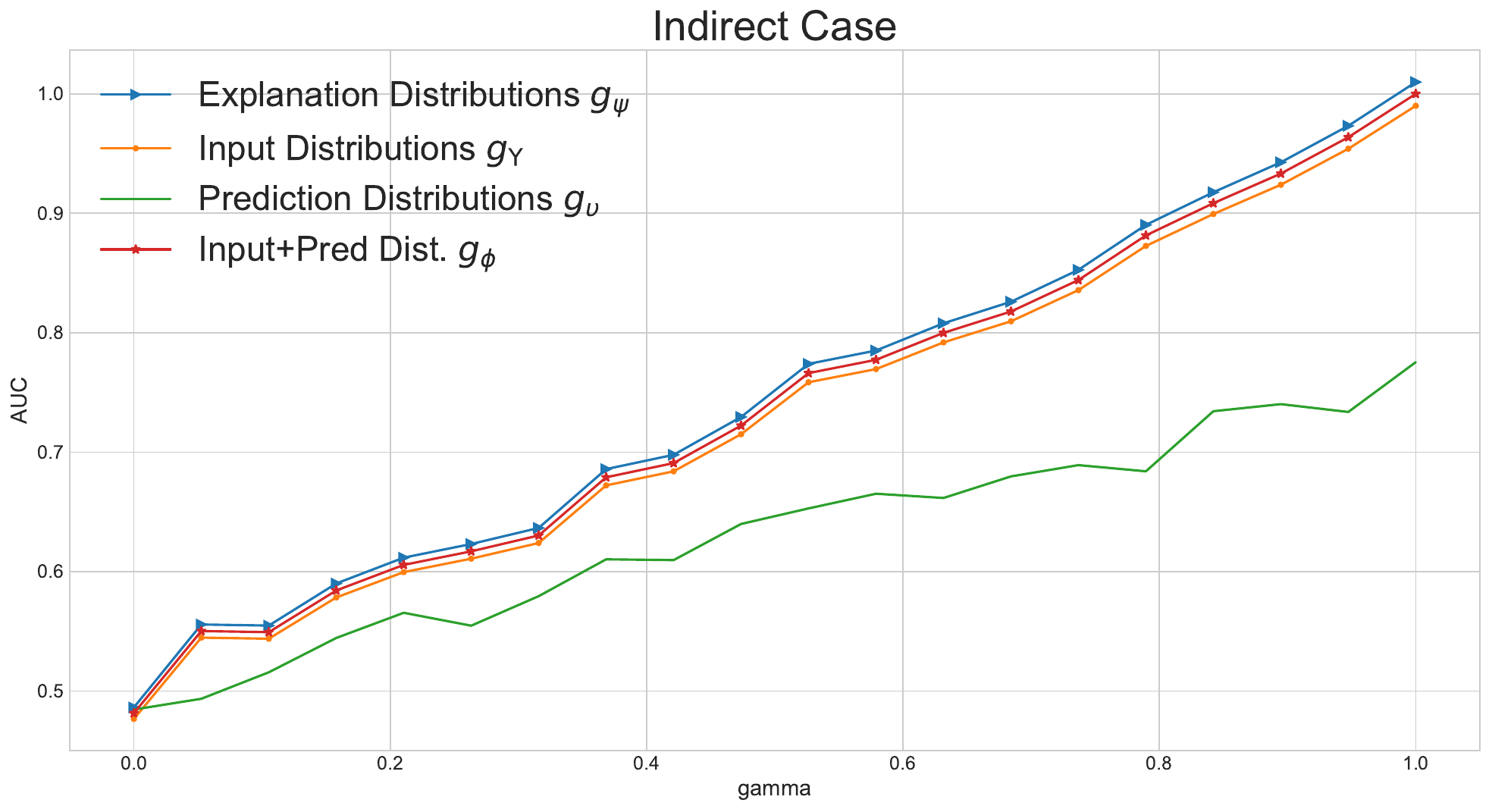}
\includegraphics[width=.8\textwidth]{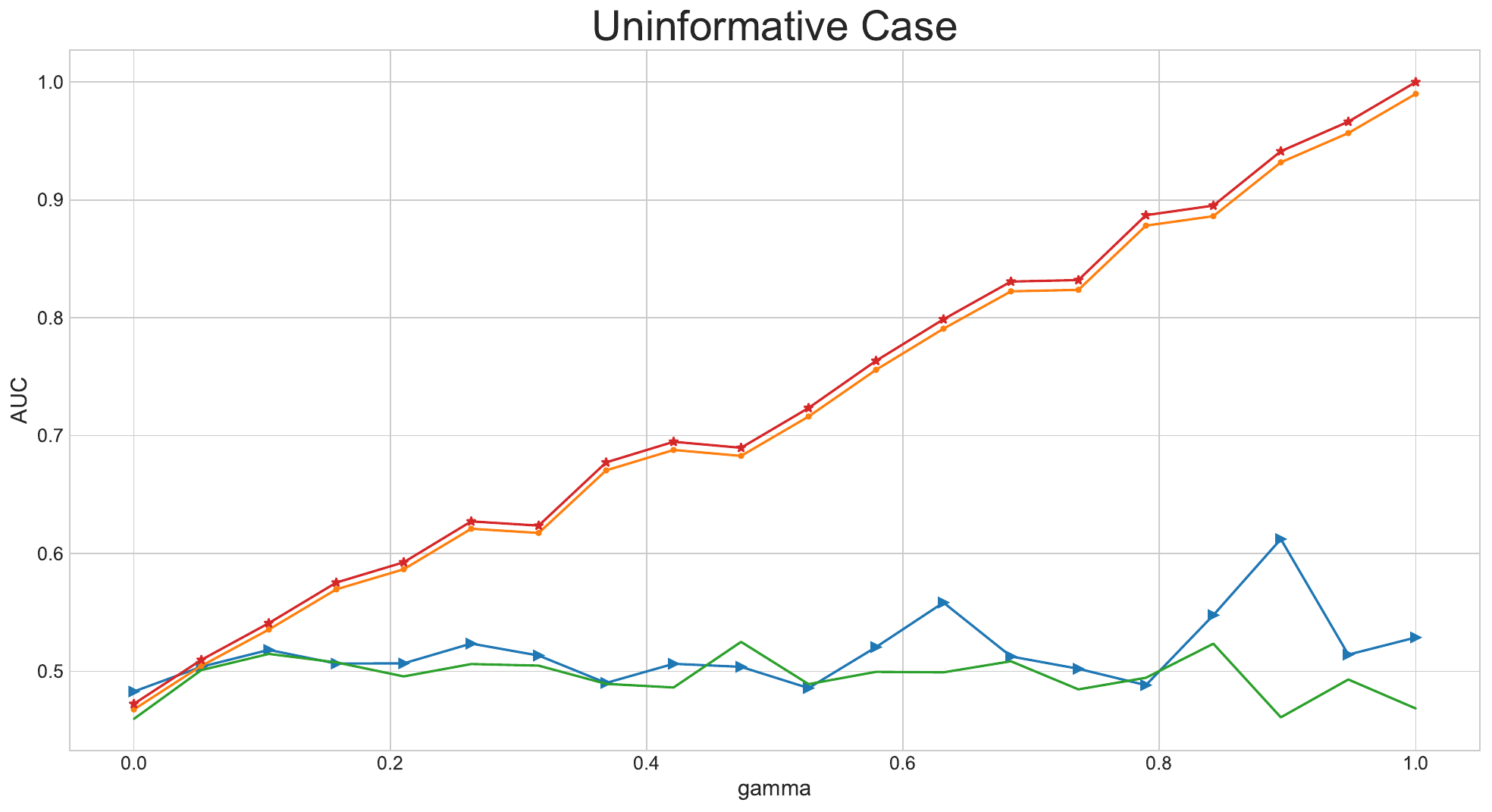}
\caption{In the figure above, \enquote{Indirect case}: All distributions capture this fairness violation; only the prediction distributions are less sensitive due to their low dimensionality. In the figure below, \enquote{Uninformative case}, only the explanation distributions and prediction distributions detect that the model is non-discriminant, input data flags a false positive detection.}\label{fig:equal:fairSyn}

\end{figure*}

In Figure \ref{fig:equal:fairSyn} we present and compare the different experiments on synthetic data. Overall, we find that learning on the explanations distributions captures un-equal treatment of the model in both situations. We say that a method is \textit{Accountable} if the feature attributions identified are the ones that indeed contribute towards the synthetically generated discrimination for both methods, see Appendix~\ref{app:xai.eval} for the experiment.

\subsubsection{Alternative Feature Attribution Explanation Methods: Shapley Value Calculations
True to the Model or True to the Data?}\label{et.exp.TrueModel.TrueData}
The \enquote{Equal Treatment Inspector} proposed in this work relies on the explanation distributions that satisfy efficiency and uninformative theoretical properties. 
We have used the Shapley values as an explainable AI method that satisfies these properties. 
A variety of (current) papers discusses the application of Shapley values, for feature attribution in machine learning models~\cite{DBLP:journals/kais/StrumbeljK14,DBLP:journals/natmi/LundbergECDPNKH20,lundberg2018explainable}. 
However, the correct way to connect a model to a coalitional game, which is the central concept of Shapley values, is a source of controversy, with two main approaches $(i)$ an interventional \citep{DBLP:journals/ai/AasJL21,DBLP:conf/nips/FryeRF20,Zern2023Interventional} or $(ii)$ an observational formulation of the conditional expectation~\citep{DBLP:conf/icml/SundararajanN20,DBLP:conf/sp/DattaSZ16,DBLP:journals/corr/abs-1911-00467}.

In the following experiment, we compare what are the differences between estimating the Shapley values using one or the other approach.  We benchmark this experiment on the four prediction tasks based on the US census data~\cite{DBLP:conf/nips/DingHMS21} and using the \enquote{Equal Treatment Inspector}, where both the model $f_\theta(X)$ and $g_\psi(\Ss(f_\theta,X))$ are linear models. We will calculate the Shapley values using the SHAP linear explainer. \footnote{\url{https://shap.readthedocs.io/en/latest/generated/shap.explainers.Linear.html}}

The comparison depends on a feature perturbation hyperparameter: whether the approach to compute the SHAP values is either \textit{interventional} or \textit{correlation dependent}. The interventional SHAP values break the dependence structure between features in the model to uncover how the model would behave if the inputs are changed (as it was an intervention). 
This option is said to stay \enquote{true to the model} meaning it will only give allocation credit to the features that are actually used by the model.

On the other hand, the full conditional approximation of the SHAP values respects the correlations of the input features. If the model depends on one input that is correlated with another input, then both get some credit for the model’s behaviour. 
This option is said to say \enquote{true to the data}, meaning that it only considers how the model would behave when respecting the correlations in the input data~\cite{DBLP:journals/corr/ShapTrueModelTrueData}.

In our case, we will measure the difference between the two approaches by looking at the linear coefficients of the model $g_\psi$ and comparing the performance of predicting protected attributes, for this case only between White-Other.

\begin{table}[ht]
\centering
\caption{AUC comparison of the \enquote{Explanation Shift Detector} between estimating the Shapley values between the interventional and the correlation-dependent approaches for the four prediction tasks based on the US census dataset~\cite{DBLP:conf/nips/DingHMS21}. The $\%$ character represents the relative difference. The performance differences are negligible.}\label{tab:t2mt2d.auc}
\begin{tabular}{l|llc}
           & Interventional                                                          & Observational & \%           \\ \hline
Income      & 0.736438                                                                & 0.736439              & 1.1e-06 \\
Employment  & 0.747923                                                                & 0.747923              & 4.44e-07 \\
Mobility    & 0.690734                                                                & 0.690735              & 8.2e-07 \\
Travel Time & 0.790512 & 0.790512              & 3.0e-07
\end{tabular}
\end{table}

\begin{table}[ht]
\caption{Linear regression coefficients comparison of the \enquote{Explanation Shift Detector} between estimating the Shapley values between the interventional and the correlation-dependent approaches for one of the US census based prediction tasks (ACS Income). The $\%$ character represents the relative difference. The coefficients show negligible differences between the calculation methods}\label{tab:t2mt2d.coefficients}
\centering
\begin{tabular}{l|rrr}

                & \multicolumn{1}{l}{Interventional} & \multicolumn{1}{l}{Observational} & \multicolumn{1}{c}{\%} \\ \hline
Marital         & 0.348170                           & 0.348190                        & 2.0e-05                 \\
Worked Hours    & 0.103258                           & -0.103254                       & 3.5e-06                 \\
Class of worker & 0.579126                           & 0.579119                        & 6.6e-06                 \\
Sex             & 0.003494                           & 0.003497                        & 3.4e-06                 \\
Occupation      & 0.195736                           & 0.195744                        & 8.2e-06                 \\
Age             & -0.018958                          & -0.018954                       & 4.2e-06                 \\
Education       & -0.006840                          & -0.006840                       & 5.9e-07                 \\
Relationship    & 0.034209                           & 0.034212                        & 2.5e-06                

\end{tabular}
\end{table}

Tables \ref{tab:t2mt2d.auc} and \ref{tab:t2mt2d.coefficients} compare the effects of using interventional and correlation-dependent approaches to train the \enquote{Explanation Shift Detector}. Table \ref{tab:t2mt2d.auc} presents the AUC values for the four prediction tasks in the US Census dataset, showing negligible performance differences between the two methods. Table \ref{tab:t2mt2d.coefficients} provides a comparison of linear regression coefficients for one prediction task (ACS Income), illustrating similarly minimal variation between the approaches. While the two methods differ theoretically, the observed differences are negligible both in the resulting AUC scores and in the linear regression coefficients for explaining the protected characteristic. This experiment underscores that, although estimation methods for Shapley values may diverge at the individual sample level, these discrepancies largely vanish when evaluating distributions of SHAP values, highlighting the robustness of the proposed approach.

\subsubsection{Alternative Feature Attribution Explanation Methods: LIME as an Alternative to Shapley Values}\label{et.app:LIME}

The definition of ET (Def. \ref{def:et}) is parametric in the explanation function. We used Shapley values for their theoretical advantages (see Appendix~\ref{sec:xai.foundations}).
Another widely used feature attribution technique is 
LIME (Local Interpretable Model-Agnostic Explanations). The intuition behind LIME is to create a local linear model that approximates the behavior of the original model in a small neighbourhood of the instance to explain~\citep{ribeiro2016why,ribeiro2016modelagnostic}, whose mathematical intuition is very similar to the Taylor/Maclaurin series. 
This section discusses the differences in our approach when adopting LIME instead of the SHAP implementation of Shapley values. First of all, LIME has certain drawbacks:

\begin{itemize}
    \item \textbf{Computationally Expensive:} Its current implementation is more computationally expensive than current SHAP implementations such as TreeSHAP~\citep{DBLP:journals/natmi/LundbergECDPNKH20}, Data SHAP \citep{DBLP:conf/aistats/KwonRZ21,DBLP:conf/icml/GhorbaniZ19}, or Local and Connected SHAP~\citep{DBLP:conf/iclr/ChenSWJ19}. This problem is exacerbated when producing explanations for multiple instances (as in our case). In fact, LIME requires sampling data and fitting a linear model, which is a computationally more expensive approach than the aforementioned model-specific approaches to SHAP. A comparison of the execution  time is reported in the next sub-section.
    
    \item \textbf{Local Neighborhood:} The randomness in the calculation of local neighbourhoods can lead to instability of the LIME explanations. 
    Works including \citep{DBLP:conf/aies/SlackHJSL20,alvarez2018towards,adebayo2018sanity} highlight that several types of feature attributions explanations, including LIME, can vary greatly.
    
    \item \textbf{Dimensionality:} LIME requires, as a hyperparameter, the number of features to use for the local linear model. For our method, all the features in the explanation distribution should be used. However, linear models suffer from the curse of dimensionality. In our experiments, this is not apparent, since our synthetic and real datasets are low-dimensional.  
\end{itemize}

\begin{figure}[ht]
\centering
\includegraphics[width=.8\textwidth]{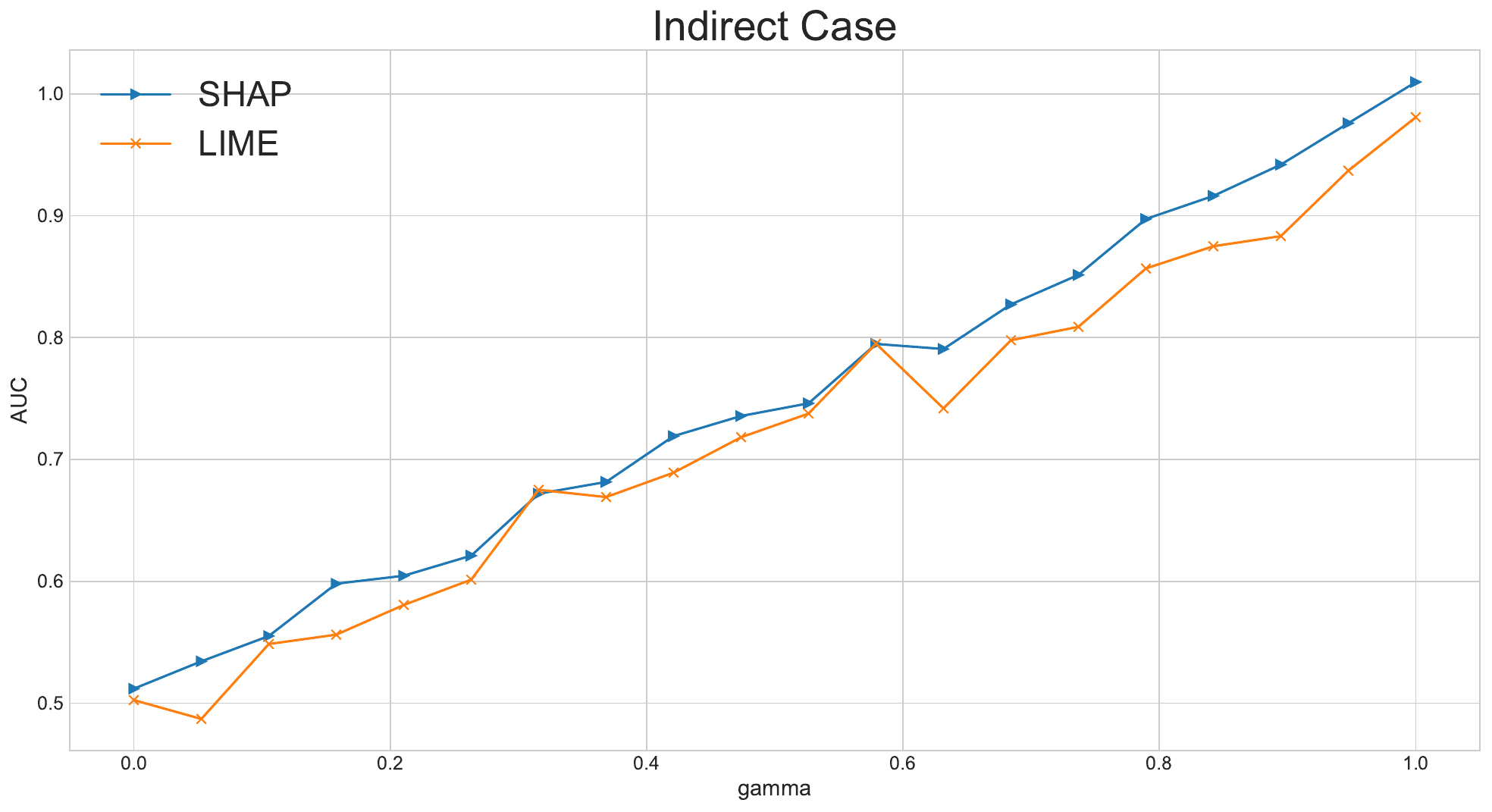}\hfill
\includegraphics[width=.8\textwidth]{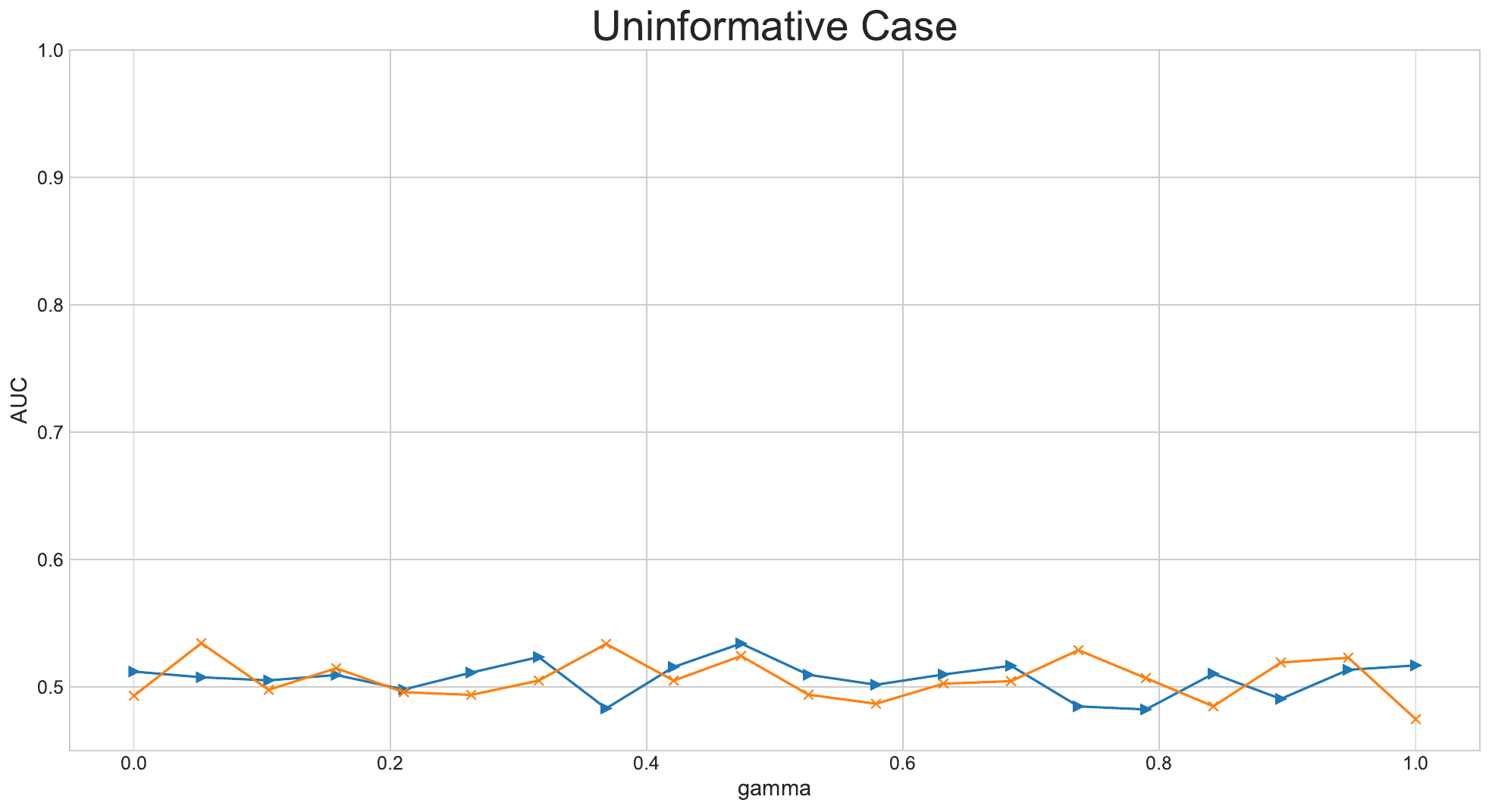}
\caption{AUC of the ET inspect using SHAP vs using LIME.}
\label{fig:lime.et}
\end{figure}

Figure \ref{fig:lime} compares the AUC of the Equal Treatment inspector using SHAP and LIME as explanation functions over the synthetic dataset of Section \ref{et.exp.synt}. In both scenarios (indirect case and uninformative case), the two approaches have similar results.

\clearpage
\section{Experiments on Real Data}\label{et.exp.real}

We experiment here with datasets derived from the ACS data\footnote{ACS PUMS documentation: \url{https://www.census.gov/programs-surveys/acs/microdata/documentation.html}} \citep{DBLP:conf/nips/DingHMS21},. The fairness notions are tested against all pairs of groups from the protected attribute \enquote{Race}. 

We divide a dataset into three equal splits $\{\Dd{tr},\Dd{val},\Dd{te} \} \subseteq \D$ and select our protected attribute, $Z$, to be a feature that indicates the ethnicity of an individual. We train our model $f_\beta$ on $\{X^{tr},Y^{tr}\}$ 
and, the \enquote{Equal Treatment Inspector} $g_\psi$ on $\{\Ss(f_\beta,X^{val}),Z^{val}\}$. Both methods are evaluated on $\{X^{te},Z^{te},y^{te}\}$. For the type of models, $f_\beta$, as we are in tabular data we focus on we choose $f_\theta$ to be a \texttt{xgboost}\cite{DBLP:conf/kdd/ChenG16}  that achieve state-of-the-art model performance~\cite{DBLP:conf/nips/GrinsztajnOV22,DBLP:journals/corr/abs-2101-02118,BorisovNNtabular},and for the inspector $g_\psi$ a logistic regression. The final explanations are given by the coefficients of the logistic regression. We compare the AUC performances of several inspectors: $g_\psi$ (see Eq. \ref{eq:fairDetector}) for Equal Treatment (see Def. \ref{def:et}), $g_v$  for Demographic Parity (see Def. \ref{def:dp}), $g_\Upsilon$ for fairness of the input (i.e., $X \perp Z$ as discussed in Section \ref{subsec:ETvsEOvsInd}), and a combination  $g_\phi$ of the last two inspectors to test $f_\theta(X), X \perp Z$.

\subsection{Equal Treatment vs Demographic Parity: ACS Census}
\subsubsection{ACS Income}

Figure~\ref{fig:xaifolks} (above) shows the AUC performances of the Equal Treatment inspector $g_{\psi}$ and the DT inspector $g_v$. The standard deviation of the AUC is calculated over $30$ bootstrap runs, each one splitting the data into $\nicefrac{1}{3}$ for training the model, $\nicefrac{1}{3}$ for training the inspectors and $\nicefrac{1}{3}$ for testing them. In the Secion~\ref{app:stat.independence.exp}, the results of the C2ST test of Section~\ref{sec:expSpaceIndependence} are reported. The AUCs for the EP inspectors are greater than for the Demographic Parity inspectors, as expected due to Lemmma \ref{lemma:inc}.

\begin{figure*}[ht]
\centering
\includegraphics[width=.8\textwidth]{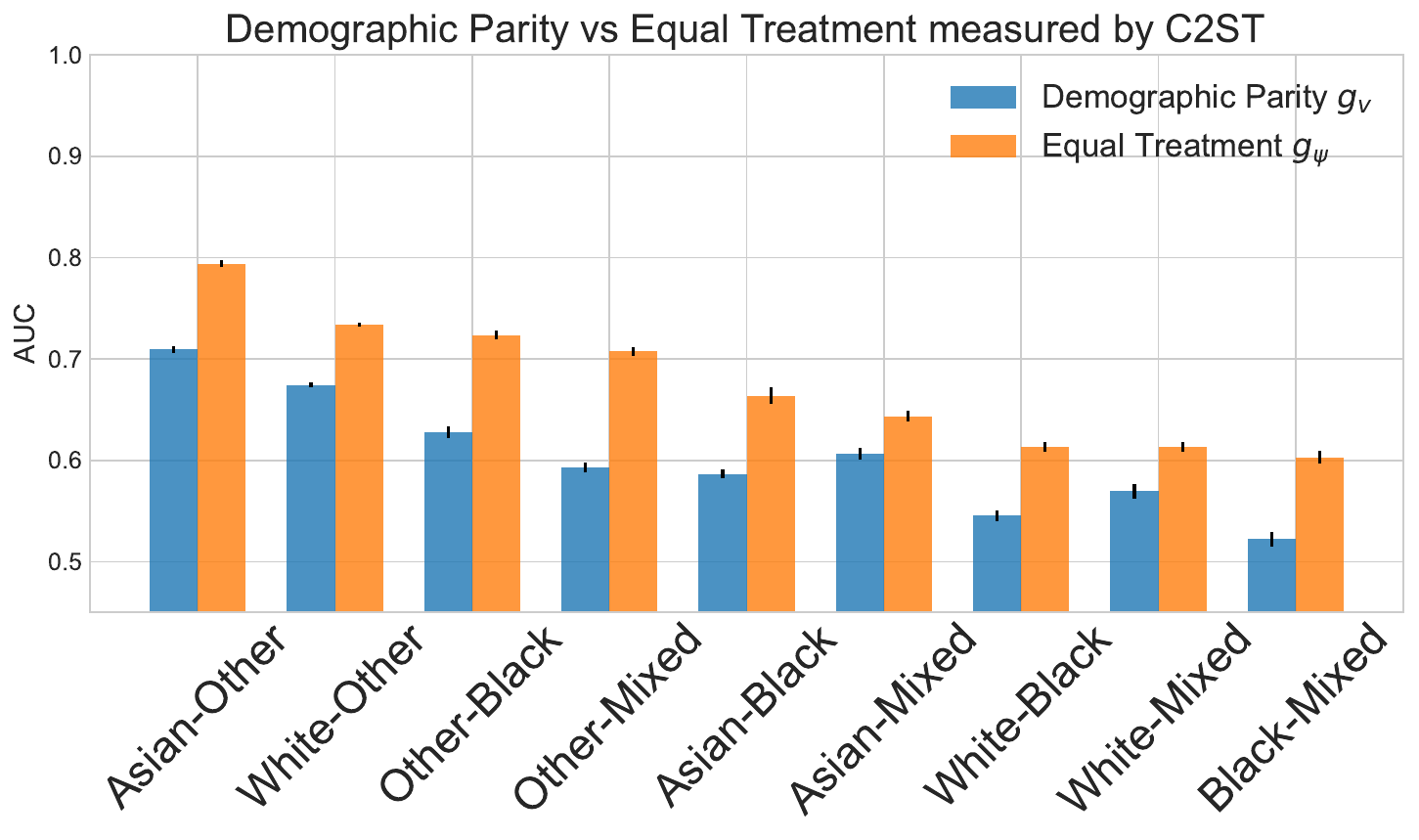}
\includegraphics[width=.8\textwidth]{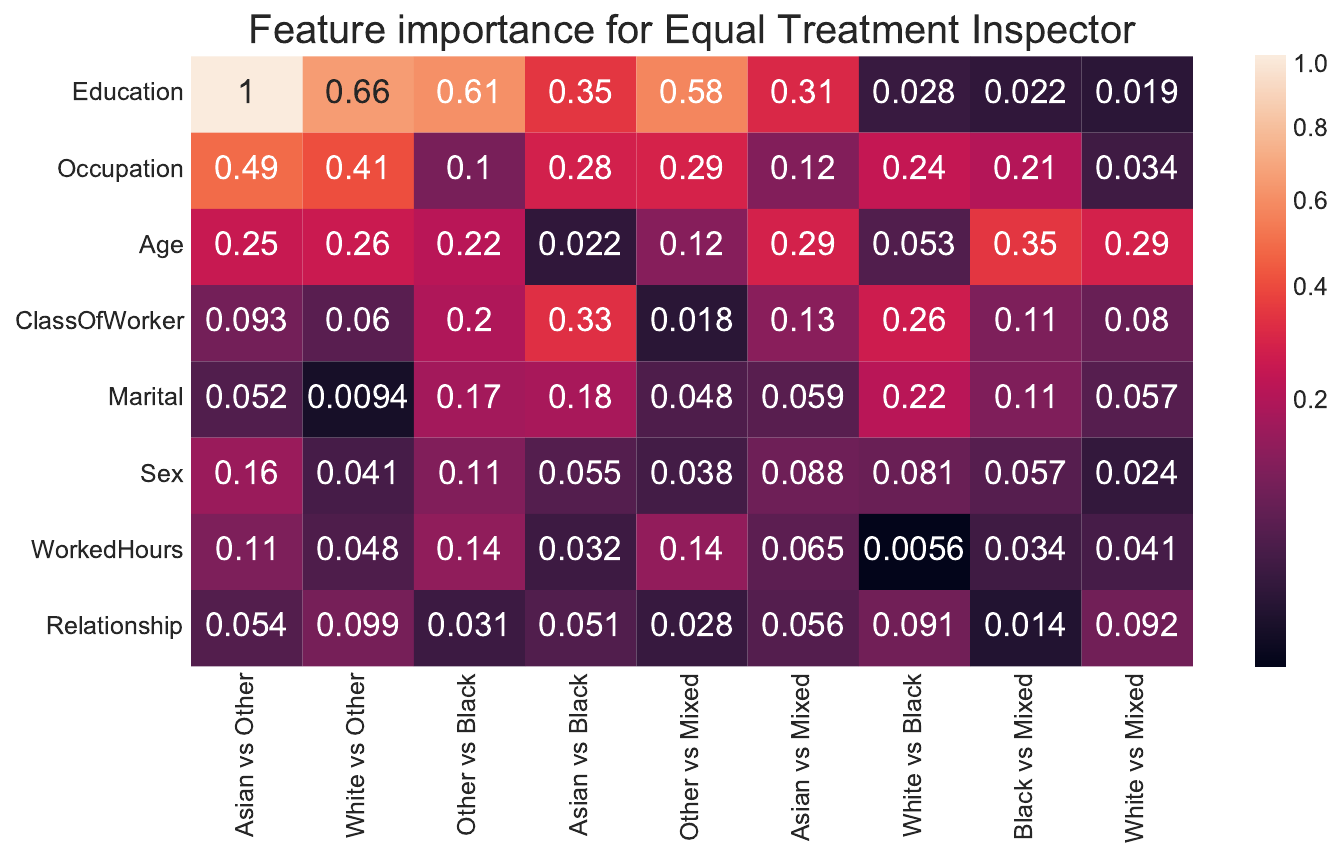}
\caption{In the figure abobe, a comparison of Equal Treatment and Demographic Parity measures on the US Income data. The AUC range for Equal Treatment is notably wider, and aligning with the theoretical section, there are indeed instances where Demographic Parity fails to identify discrimination that Equal Treatment successfully detects. For a detailed statistical analysis, please refer to Appendix~\ref{app:stat.DP.ET}. The figure below provides insight into the influential features contributing to unequal treatment. Higher feature values correspond to a greater likelihood of these features being the underlying causes of unequal treatment.}\label{fig:xaifolks}
\end{figure*}

Figure~\ref{fig:xaifolks} (below) shows the Wasserstein distance between the coefficients of the linear regressor $g_{\psi}$ compared to a baseline where groups are assigned at random in the input dataset. This feature importance post-hoc explanation method provides insights into the impact of different features as sources of unfairness. We observe \enquote{Education} as a highly discriminatory proxy while the role of the feature \enquote{Worked Hours Per Week} is less relevant. This allows us to identify areas where adjustments or interventions may be needed to move closer to the ideal of equal treatment.

\subsubsection{ACS Employment}
The goal of this task is to predict whether an individual is employed. Figure \ref{fig:xai.employment} shows a low
DP violation, compared to the other prediction tasks based on the US census dataset. The AUC of the \enquote{Equal Treatment Inspector} is ranging from $0.55$ to $0.70$. For Asian vs Black un-equal treatment, we see a significant variation of the AUC, indicating that the method achieves different values on the bootstrapping folds. 
Looking at the features driving the Equal Treatment violation, we see particularly high values when comparing \enquote{Asian} and \enquote{Black} populations, and for features \enquote{Citizenship} and \enquote{Employment}. 
On average, the most important features across all group comparisons are also \enquote{Education} and \enquote{Area}. Interestingly, features such as \enquote{difficulties on the hearing or seeing}, do not play a role.

\begin{figure*}[ht]
\centering
\includegraphics[width=.8\textwidth]{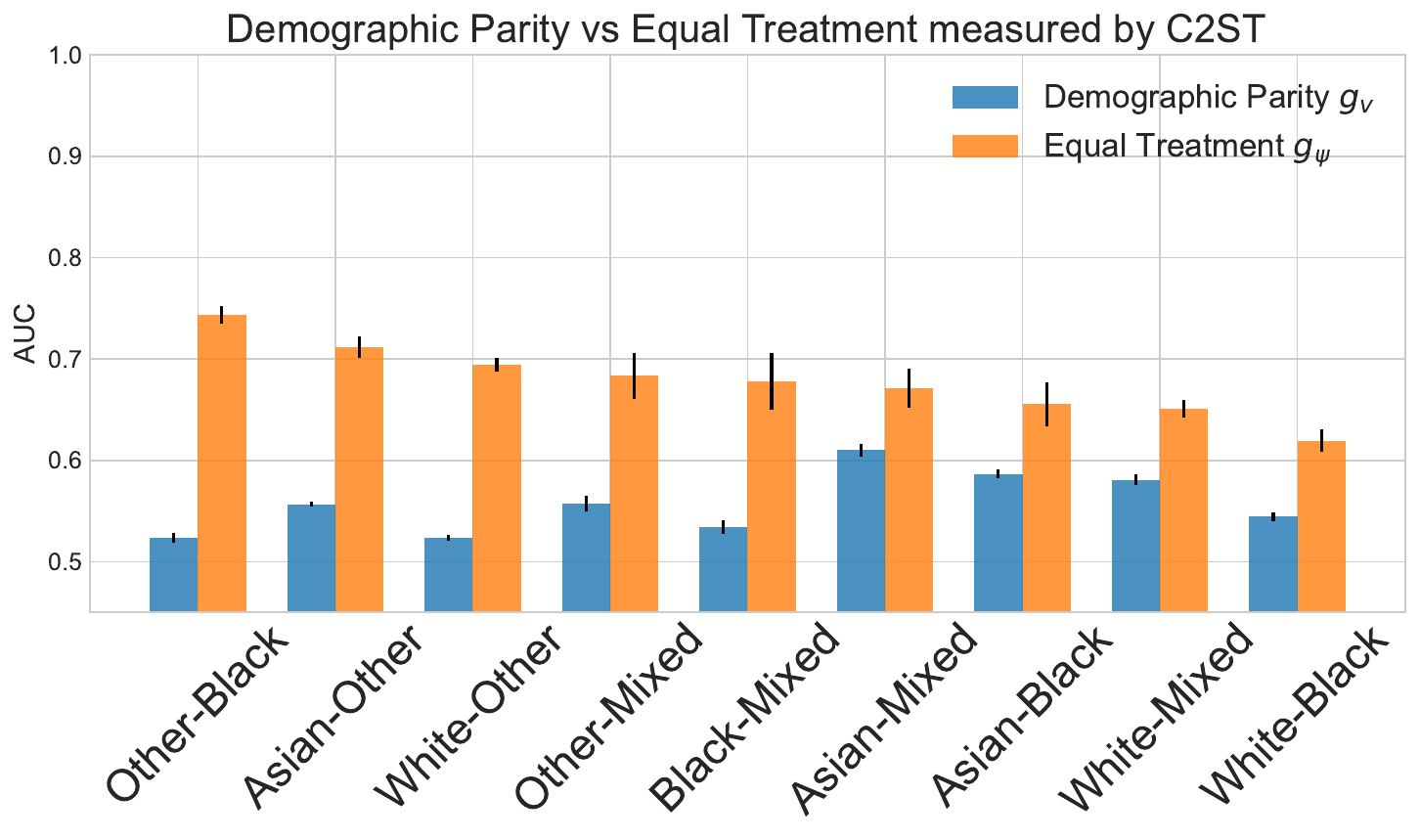}\hfill
\includegraphics[width=.8\textwidth]{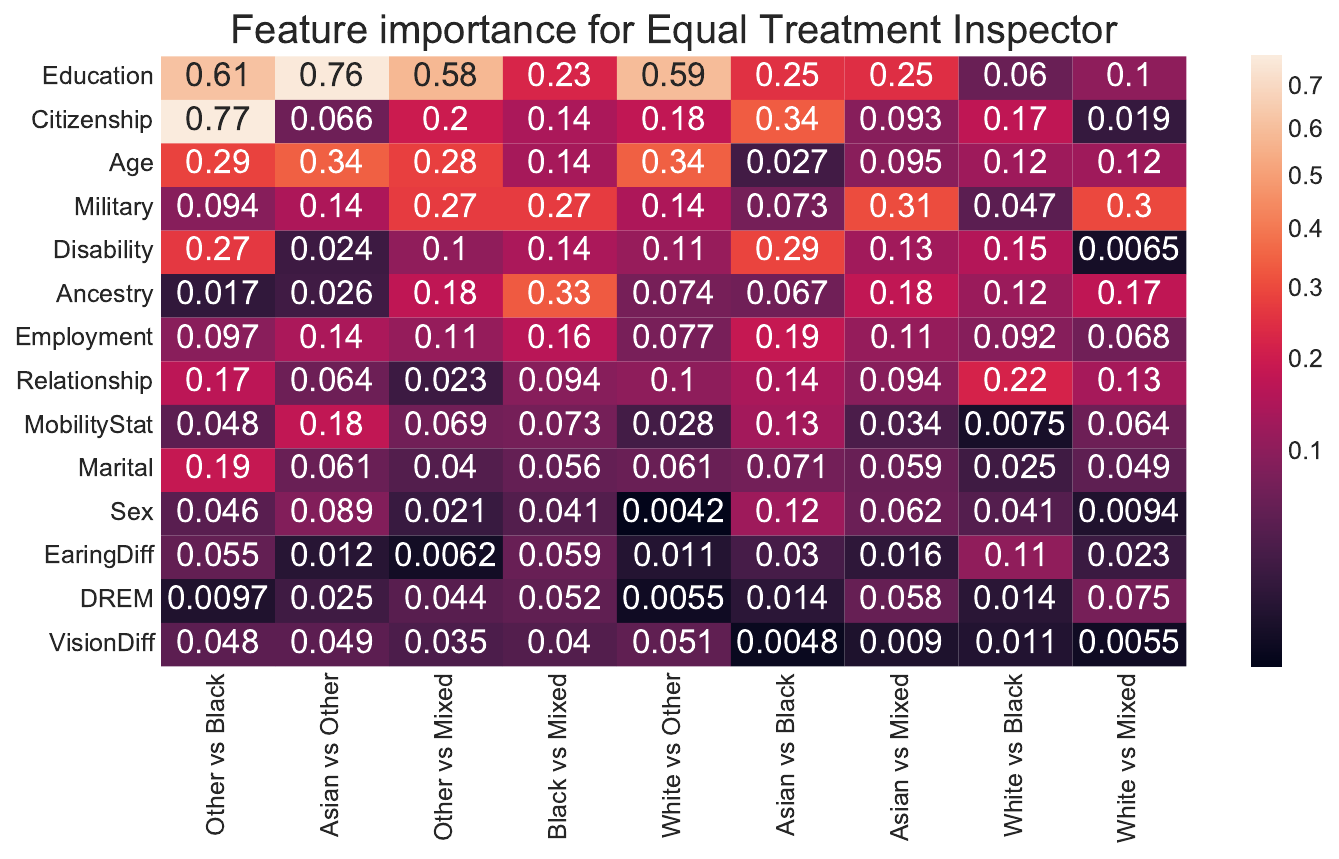}
\caption{Above: AUC of the inspector for Equal Treatment and DP, over the district of California 2014 for the ACS Employment dataset. Below: contribution of features to the Equal Treatment inspector performance.}
\label{fig:xai.employment}
\end{figure*}

\subsubsection{ACS Travel Time}

The goal of this task is to predict whether an individual has a commute to work that is longer than 20 minutes. The threshold of 20 minutes was chosen as it is the US-wide median travel time to work based on 2018 data. Figure \ref{fig:xai.traveltime} shows an AUC for the Equal Treatment inspector in the range of $0.50$ to $0.60$. 
By looking at the features, they highlight different Equal Treatment drivers depending on the pair-wise comparison made. 
In general, the features \enquote{Education}, \enquote{Citizenship} and \enquote{Area} are the those with the highest difference. Even though for Asian-Black pairwise comparison \enquote{Employment} is also one of the most relevant features.

\begin{figure*}[ht]
\centering
\includegraphics[width=.8\textwidth]{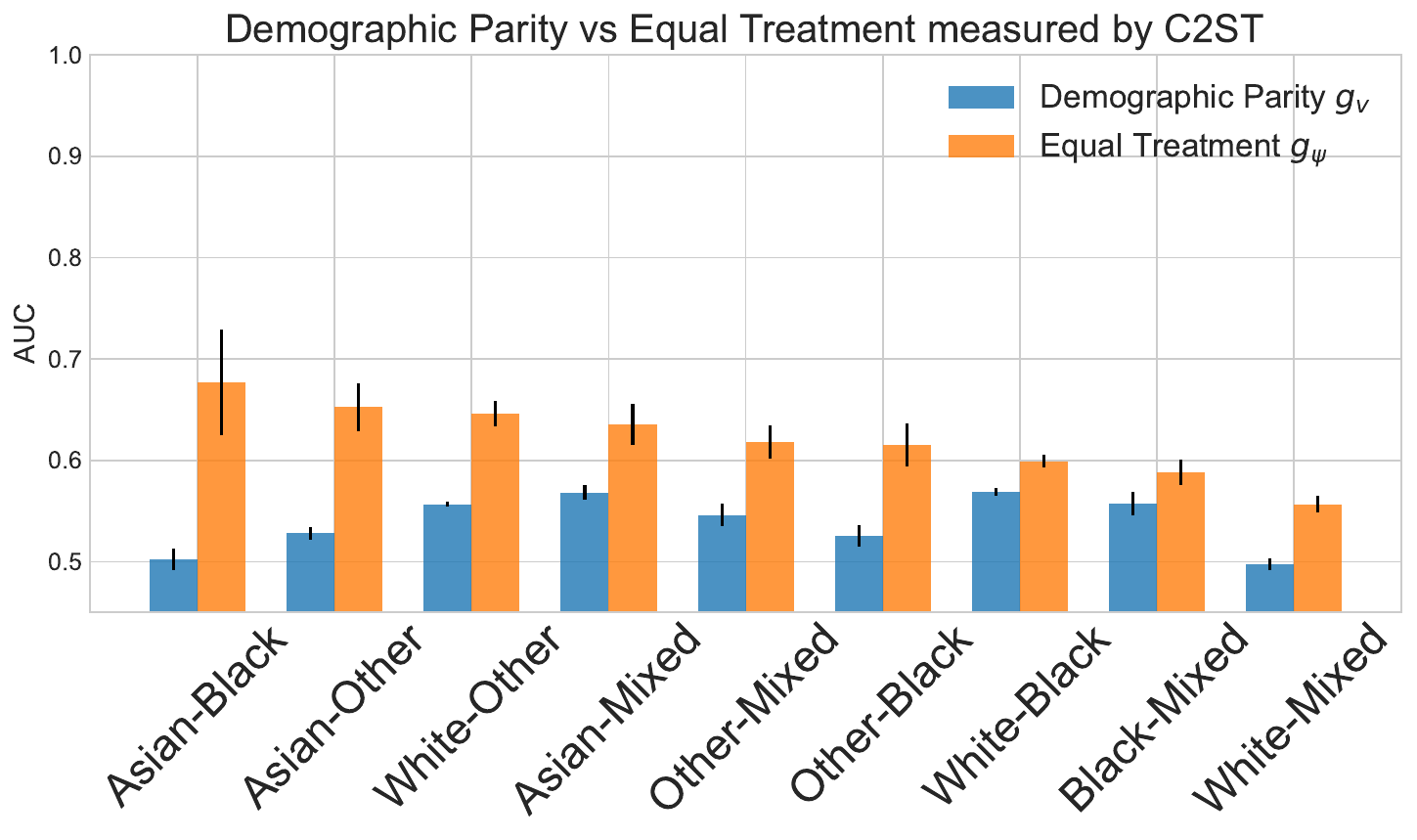}\hfill
\includegraphics[width=.8\textwidth]{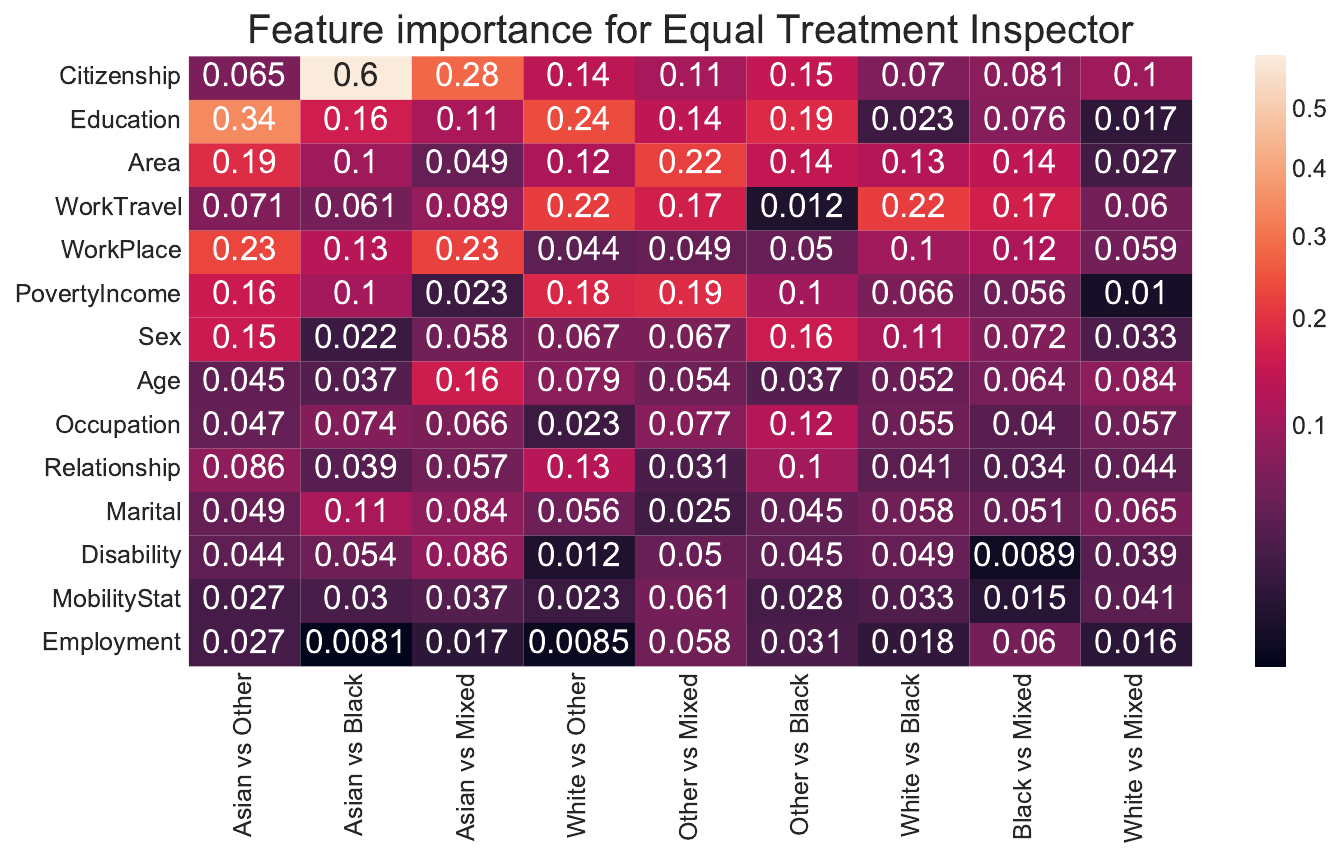}
\caption{Above AUC of the inspector for Equal Treatment and DP, over the district of California 2014 for the ACS Travel Time dataset. Below: contribution of features to the Equal Treatment inspector performance.}
\label{fig:xai.traveltime}
\end{figure*}

\subsubsection{ACS Mobility}

The goal of this task is to predict whether an individual had the same residential address one year ago, only including individuals between the ages of 18 and 35. This filtering increases the difficulty of the prediction task, as the base rate of staying at the same address is above $90\%$ for the general population~\citep{DBLP:conf/nips/DingHMS21}. Figure \ref{fig:xai.mobility} show an AUC of the Equal Treatment inspector in the range of $0.55$ to $0.80$. 
By looking at the features, they highlight different sources of the Equal Treatment violation depending on the group pair-wise comparison. 
In general, the feature \enquote{Ancestry}, i.e. ``ancestors' lives with details like where they lived, who they lived with, and what they did for a living", plays a high relevance when predicting Equal Treatment violation.

\begin{figure*}[ht]
\centering
\includegraphics[width=.8\textwidth]{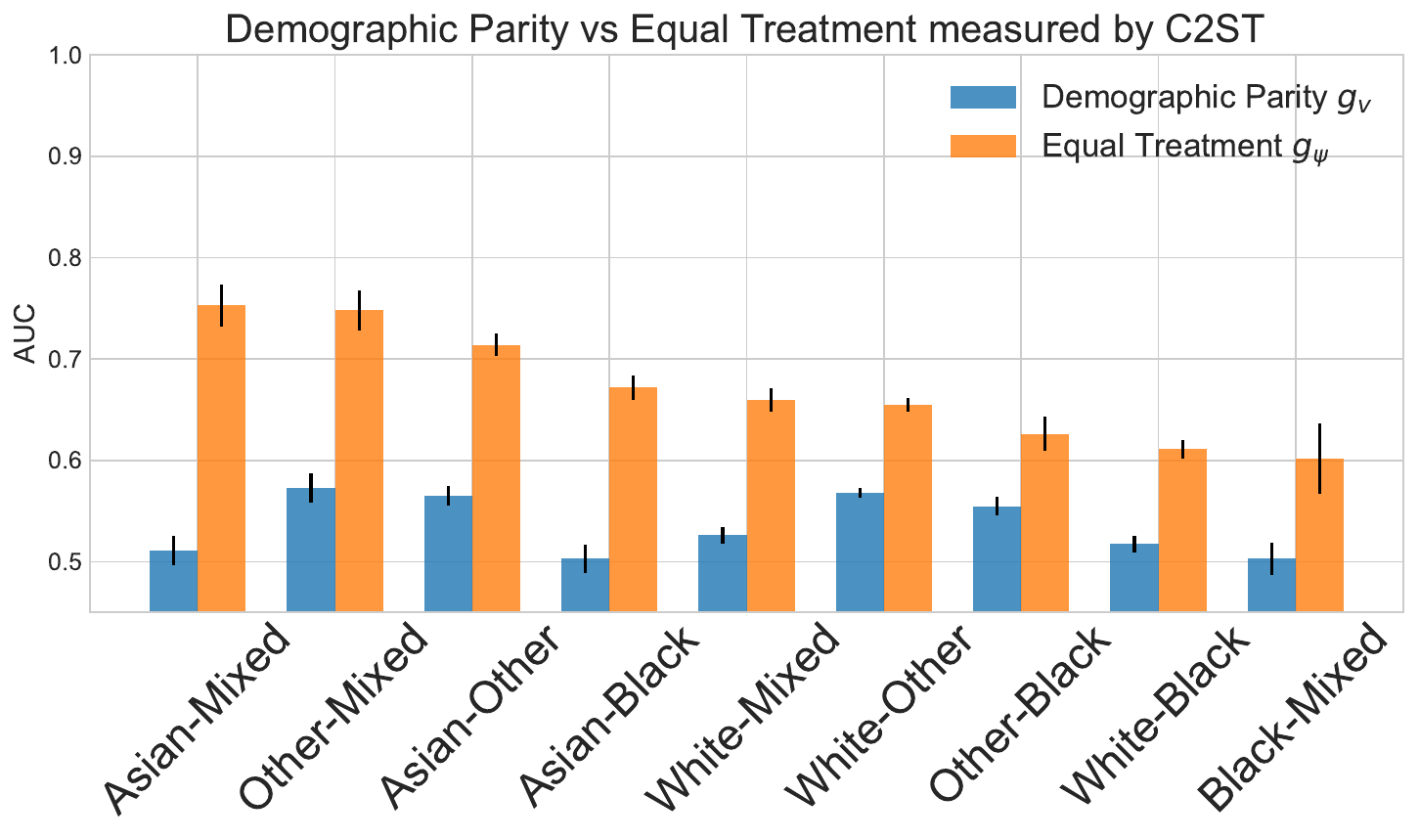}\hfill
\includegraphics[width=.8\textwidth]{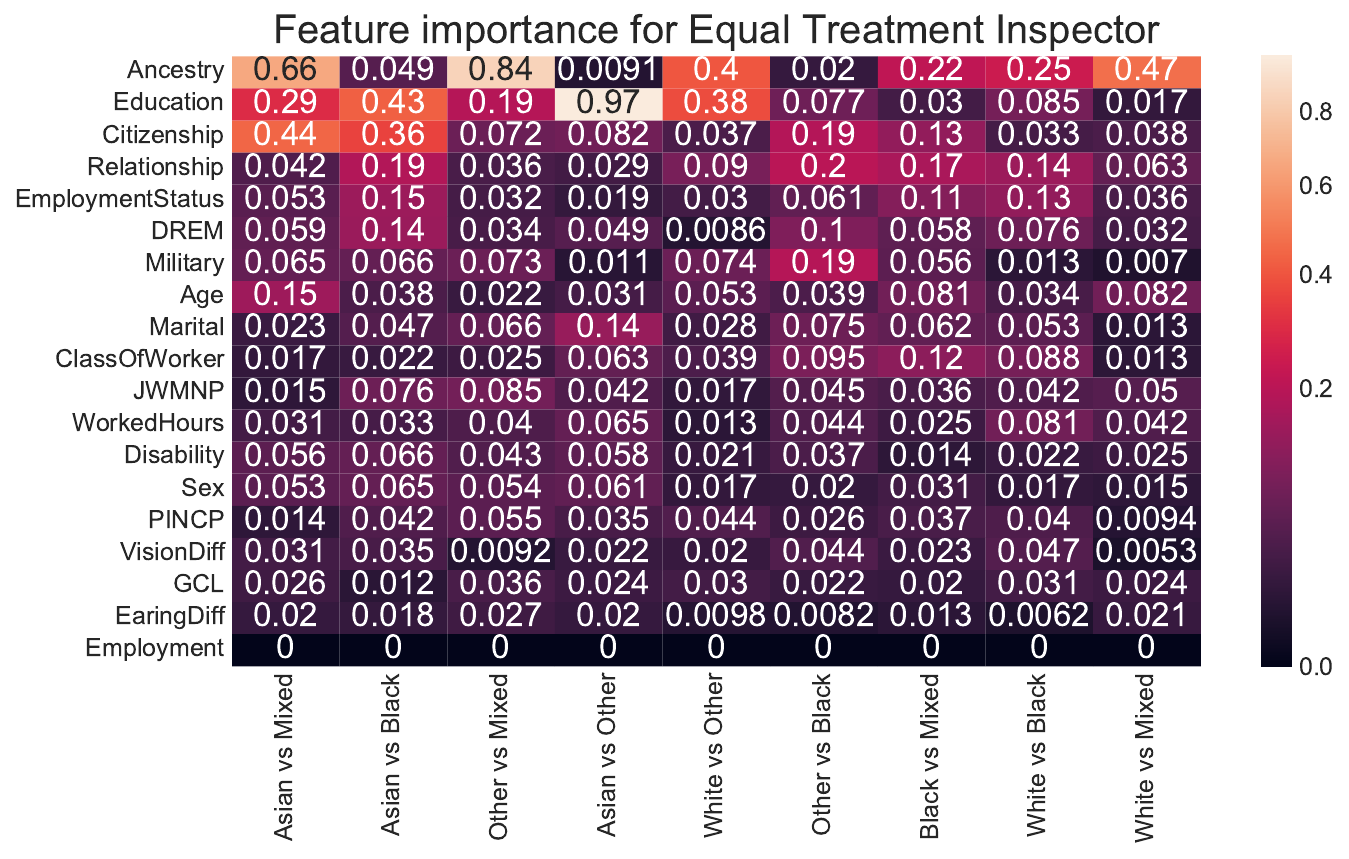}
\caption{Above: AUC of the inspector for Equal Treatment and DP, over the district of California 2014 for the ACS Mobility dataset. Below: contribution of features to the Equal Treatment inspector performance.}
\label{fig:xai.mobility}
\end{figure*}

\subsection{Quantitative Analysis with Demographic Parity}\label{app:stat.DP.ET}

So far, we measured Equal Treatment and Demographic Parity fairness using the AUC of an inspector, $g_{\psi}$ and $g_v$, respectively. For DP, however, other probability density distance metrics can be considered, including the p-value of the Kolmogorov–Smirnov (KS) test and the Wasserstein distance. Table \ref{table:dist} reports all such distances in the format ``mean $\pm$ stdev" calculated over 100 random sampled datasets. 
The pairs of group comparisons are sorted by descending AUC values. 
We highlight in red values below the threshold of $0.05$ for the KS test, of $0.55$ for the AUC of the C2ST, and of $0.05$ for the Wasserstein distance.  
They represent cases where Equal Treatment violation occurs, but no Demographic Parity violation is measured (with different metrics).

\begin{table}[ht]
\small
\caption{Comparison ofEqual Treatment and Demographic Parity measured in different ways. The cases of equal Treatment violation but no demographic parity violation are highlighted in red.}\label{table:dist}
\begin{tabular}{c|c|c|ccc}
Pair        & Data       & Equal treatment  & \multicolumn{3}{c}{Demographic Parity}                 \\ \hline
            &            & C2ST(AUC)        & C2ST(AUC)             & KS(pvalue)         & Wasserstein      \\ \hline
Asian-Other &  Income & $0.794 \pm 0.004$ & $0.709 \pm 0.004$ & $0.338 \pm 0.007$ & $0.256 \pm 0.004$ \\
White-Other &  Income & $0.734 \pm 0.002$ & $0.675 \pm 0.003$ & $0.282 \pm 0.003$ & $0.209 \pm 0.002$ \\
Other-Black &  Income & $0.724 \pm 0.004$ & $0.628 \pm 0.006$ & $0.216 \pm 0.007$ & $0.143 \pm 0.004$ \\
Other-Mixed &  Income & $0.707 \pm 0.005$ & $0.593 \pm 0.005$ & $0.169 \pm 0.006$ & $0.117 \pm 0.004$ \\
Asian-Black &  Income & $0.664 \pm 0.008$ & $0.587 \pm 0.004$ & $0.142 \pm 0.005$ & $0.111 \pm 0.004$ \\
Asian-Mixed &  Income & $0.644 \pm 0.005$ & $0.607 \pm 0.006$ & $0.159 \pm 0.008$ & $0.128 \pm 0.006$ \\
White-Mixed &  Income & $0.613 \pm 0.005$ & $0.546 \pm 0.005$ & $0.082 \pm 0.004$ & \textcolor{red}{$0.058 \pm 0.002$} \\
White-Black &  Income & $0.613 \pm 0.005$ & $0.57 \pm 0.007$ & $0.113 \pm 0.008$ & $0.08 \pm 0.006$ \\
Black-Mixed &  Income & $0.603 \pm 0.006$ & \textcolor{red}{$0.523 \pm 0.007$} & \textcolor{red}{$0.055 \pm 0.007$} & \textcolor{red}{$0.023 \pm 0.004$} \\ \hline
Asian-Black &  TravelTime & $0.677 \pm 0.052$ & \textcolor{red}{$0.502 \pm 0.011$} & \textcolor{red}{$0.021 \pm 0.009$} & \textcolor{red}{$0.01 \pm 0.003$} \\
Asian-Other &  TravelTime & $0.653 \pm 0.024$ & \textcolor{red}{$0.528 \pm 0.006$} & \textcolor{red}{$0.053 \pm 0.011$} & \textcolor{red}{$0.027 \pm 0.004$} \\
Asian-Mixed &  TravelTime & $0.647 \pm 0.013$ & $0.557 \pm 0.003$ & $0.096 \pm 0.004$ & \textcolor{red}{$0.045 \pm 0.002$} \\
White-Other &  TravelTime & $0.636 \pm 0.02$ & $0.568 \pm 0.007$ & $0.107 \pm 0.01$ & $0.06 \pm 0.005$ \\
Other-Mixed &  TravelTime & $0.618 \pm 0.017$ & \textcolor{red}{$0.546 \pm 0.011$} & $0.079 \pm 0.012$ & \textcolor{red}{$0.043 \pm 0.006$} \\
Other-Black &  TravelTime & $0.615 \pm 0.021$ & \textcolor{red}{$0.526 \pm 0.011$} & \textcolor{red}{$0.049 \pm 0.014$} & \textcolor{red}{$0.026 \pm 0.006$} \\
White-Black &  TravelTime & $0.599 \pm 0.006$ & $0.569 \pm 0.004$ & $0.12 \pm 0.006$ & \textcolor{red}{$0.057 \pm 0.003$} \\
Black-Mixed &  TravelTime & $0.588 \pm 0.012$ & $0.557 \pm 0.012$ & $0.098 \pm 0.015$ & \textcolor{red}{$0.0557 \pm 0.001$} \\
White-Mixed &  TravelTime & $0.557 \pm 0.008$ & \textcolor{red}{$0.497 \pm 0.006$} & \textcolor{red}{$0.016 \pm 0.004$} & \textcolor{red}{$0.006 \pm 0.002$} \\ \hline
Other-Black &  Employment & $0.744 \pm 0.008$ & \textcolor{red}{$0.524 \pm 0.005$} & \textcolor{red}{$0.036 \pm 0.005$} & \textcolor{red}{$0.036 \pm 0.004$} \\
Asian-Other &  Employment & $0.711 \pm 0.011$ & $0.557 \pm 0.003$ & $0.066 \pm 0.004$ & $0.066 \pm 0.003$ \\
White-Other &  Employment & $0.695 \pm 0.007$ & \textcolor{red}{$0.524 \pm 0.003$} & $0.019 \pm 0.005$ & $0.019\pm 0.002$ \\
Other-Mixed &  Employment & $0.683 \pm 0.022$ & $0.557 \pm 0.008$ & $0.083 \pm 0.005$ & $0.083 \pm 0.003$ \\
Black-Mixed &  Employment & $0.678 \pm 0.028$ & \textcolor{red}{$0.534 \pm 0.007$} & \textcolor{red}{$0.049 \pm 0.007$} & \textcolor{red}{$0.048 \pm 0.004$} \\
Asian-Mixed &  Employment & $0.671 \pm 0.019$ & $0.61 \pm 0.006$ & $0.0144 \pm 0.006$ & $0.145 \pm 0.004$ \\
Asian-Black &  Employment & $0.655 \pm 0.021$ & $0.587 \pm 0.004$ & $0.106 \pm 0.006$ & $0.106 \pm 0.004$ \\
White-Mixed &  Employment & $0.651 \pm 0.009$ & $0.581 \pm 0.006$ & $0.095 \pm 0.004$ & $0.095 \pm 0.003$ \\ 
White-Black &  Employment & $0.619 \pm 0.011$ & \textcolor{red}{$0.544 \pm 0.004$} & \textcolor{red}{$0.049 \pm 0.003$} & \textcolor{red}{$0.049 \pm 0.002$} \\\hline
Asian-Mixed &  Mobility & $0.753 \pm 0.02$ & \textcolor{red}{$0.511 \pm 0.014$} & \textcolor{red}{$0.04 \pm 0.012$} & \textcolor{red}{$0.014\pm 0.006$} \\
Other-Mixed &  Mobility & $0.748 \pm 0.02$ & $0.573 \pm 0.015$ & $0.113 \pm 0.017$ & \textcolor{red}{$0.062 \pm 0.009$} \\
Asian-Other &  Mobility & $0.714 \pm 0.011$ & $0.565 \pm 0.01$ & $0.114 \pm 0.011$ & \textcolor{red}{$0.054 \pm 0.005$} \\
Asian-Black &  Mobility & $0.672 \pm 0.012$ & \textcolor{red}{$0.503 \pm 0.014$} & \textcolor{red}{$0.032 \pm 0.011$} & \textcolor{red}{$0.012 \pm 0.004$} \\
Other-Black &  Mobility & $0.66 \pm 0.012$ & \textcolor{red}{$0.526 \pm 0.009$} & \textcolor{red}{$0.044 \pm 0.009$} & \textcolor{red}{$0.02 \pm 0.004$} \\
White-Mixed &  Mobility & $0.655 \pm 0.007$ & $0.568 \pm 0.005$ & $0.105 \pm 0.007$ & \textcolor{red}{$0.044 \pm 0.003$} \\ 
White-Other &  Mobility & $0.626 \pm 0.017$ & $0.555 \pm 0.009$ & $0.091 \pm 0.01$ & \textcolor{red}{$0.046\pm 0.005$} \\
White-Black &  Mobility & $0.611 \pm 0.009$ & \textcolor{red}{$0.518 \pm 0.008$} & \textcolor{red}{$0.043 \pm 0.008$} & \textcolor{red}{$0.017 \pm 0.004$} \\
Black-Mixed &  Mobility & $0.602 \pm 0.035$ & \textcolor{red}{$0.503 \pm 0.016$} & \textcolor{red}{$0.031 \pm 0.013$} & \textcolor{red}{$0.012 \pm 0.006$} \\\hline
\end{tabular}
\end{table}

\subsection{Varying Estimator and Inspector}\label{subsec:EstimatorInspectorVariations}

We vary here the model $f_{\theta}$ and the inspector $g_{\psi}$ over a wide range of well-known classification algorithms. 
Table~\ref{tab:benchmark} shows that the choice of model and inspector impacts on the measure of Equal Treatment, namely the AUC of the inspector.  

By Theorem \ref{thm:main}, the larger the AUC of any inspector the smaller is the $p$-value of the null hypothesis $\Ss(f_\theta, X) \perp Z$. Therefore, inspectors able to achive the best AUC should be considered. Weak inspectors have lower probability of rejecting the null hypothesis when it does not hold.

\begin{table}[ht]\centering
\begin{tabular}{r|ccccc}
\multicolumn{1}{l|}{}       & \multicolumn{5}{c}{\textbf{Model $f_\theta$}}                                                     \\ \cline{2-6} 
\textbf{Inspector $g_\psi$}    & \textbf{DecisionTree} & \textbf{SVC} & \textbf{Logistic Reg.} & \textbf{RF} & \textbf{XGB} \\ \hline
\textbf{DecisionTree} & 0.631                 & 0.644        & 0.644                  & 0.664                  & 0.634        \\
\textbf{KNN}   & 0.737                 & 0.754        & 0.75                   & 0.744                  & 0.751        \\
\textbf{Logistic Reg.} & 0.767                 & 0.812        & 0.812                  & 0.812                  & 0.821        \\
\textbf{MLP}      & 0.786                 & 0.795        & 0.795                  & 0.813                  & 0.804        \\
\textbf{RF}       & 0.776                 & 0.782        & 0.781                  & 0.758                  & 0.795        \\
\textbf{SVC}                & 0.743                 & 0.807        & 0.807                  & 0.790                   & 0.810        \\
\textbf{XGB}      & 0.775                 & 0.780         & 0.780                   & 0.789                  & 0.790       
\end{tabular}
\caption{AUC of theEqual Treatmentinspector for different combinations of models and inspectors.}\label{tab:benchmark.et}
\end{table}

\subsection{Hyperparameters Evaluation}\label{app:hyperparameter}

This section presents an extension to our experimental setup, where we increase the model complexity by varying the model hyperparameters. We use the US Income dataset for the population of the CA14 district. 
We consider three models for $f_{\theta}$: Decision Trees, Gradient Boosting, and Random Forest. For the Decision Tree models, we vary the depth of the tree, while for the Gradient Boosting and Random Forest models, we vary the number of estimators. Shapley values are calculated by means of the TreeExplainer algorithm \citep{DBLP:journals/natmi/LundbergECDPNKH20}. For theEqual Treatmentinspector $g_{\psi}$, we consider logistic regession, and XGB.

Figure \ref{fig:xai.hyper} shows that less complex models, such as Decision Trees with maximum depth 1 or 2, are also less unfair. However, as we increase the model complexity, the unequal treatment of the model becomes more pronounced, achieving a plateau when the model has enough complexity. Furthermore, when we compare the results for different Equal Treatment inspectors, we observe minimal differences (note that the y-axis takes different ranges).

\begin{figure}[ht]
\centering
\includegraphics[width=.8\textwidth]{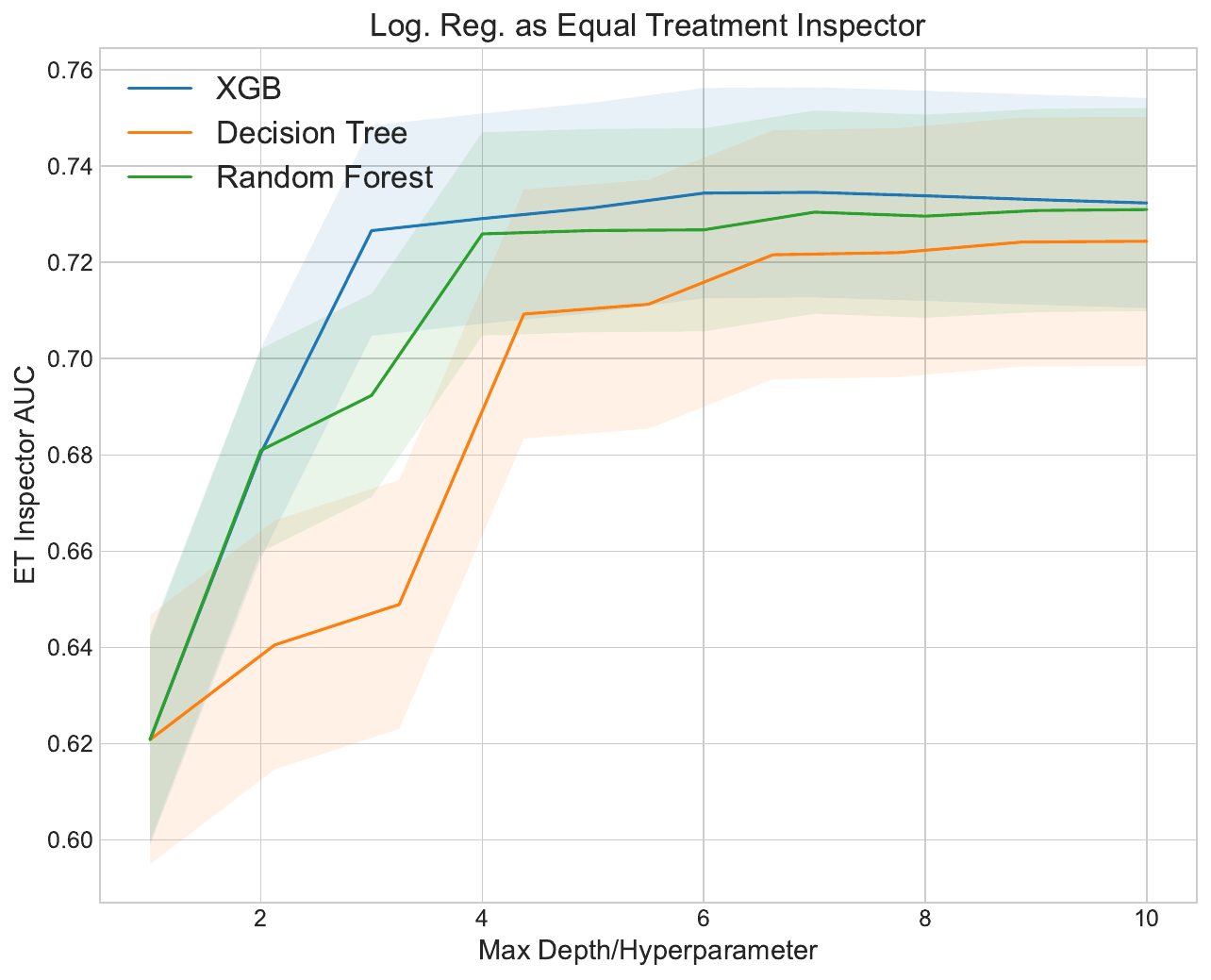}\hfill
\includegraphics[width=.8\textwidth]{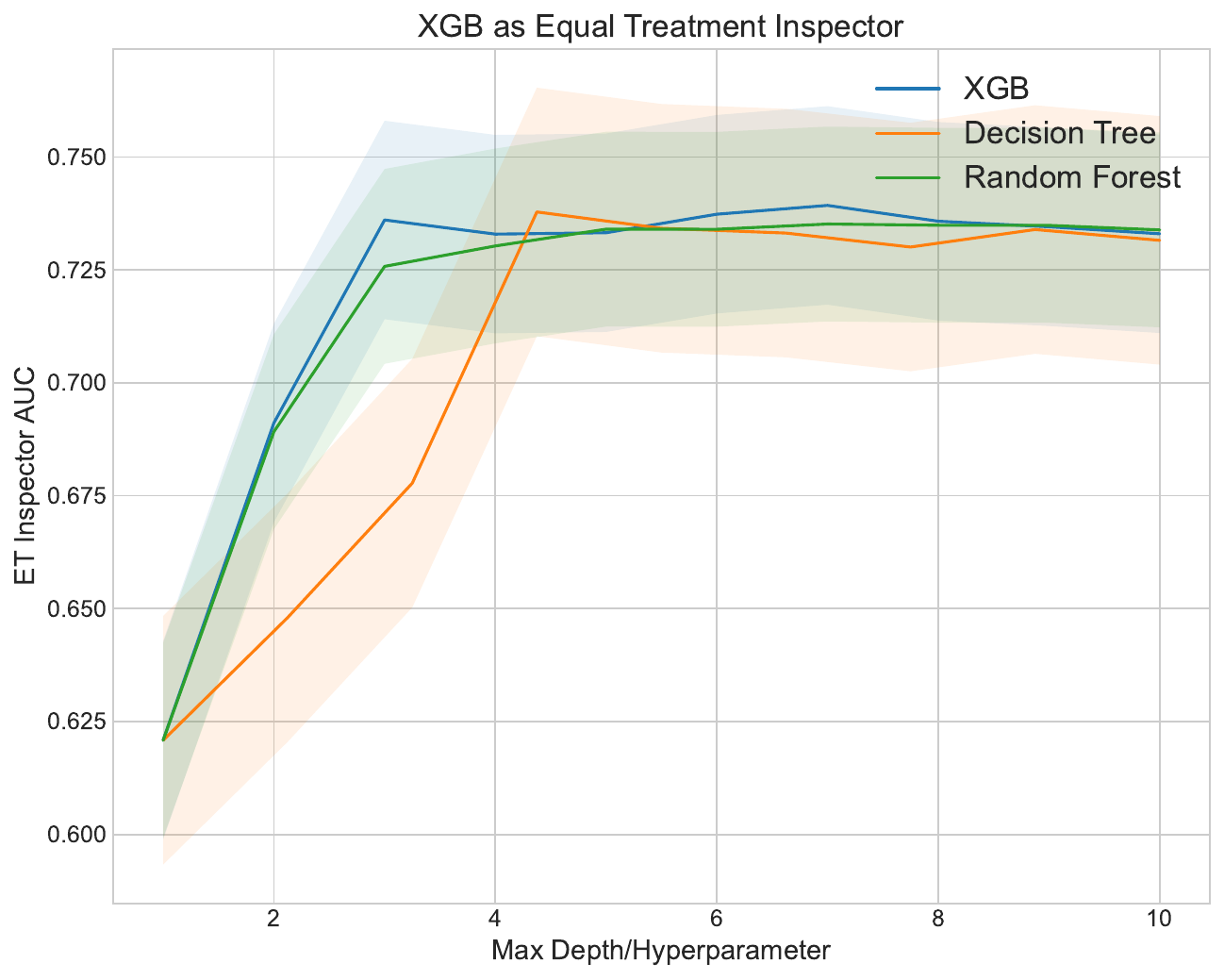}
\caption{AUC of the inspector for ET, over the district of CA14 for the US Income dataset. The model $f_\theta$ is a decision tree based algorithm, that we modify the maximum depth (x-axis), while the model $g_\psi$ varies in each of the images. The image above shows the results using a logistic regression detector, while the image below uses an XGBoost detector. By changing the model, we can see that simpler models (decision tree with depth 1 or 2) are less discriminative, while when increasing the model complexity, the unequal treatment of the model starts taking higher values.}
\label{fig:xai.hyper}
\end{figure}
\clearpage
\section{Related Work on Measuring Equal Treatment}

\subsection{Fairness and Auditing}
\enquote{Audits are tools for interrogating complex processes, often to determine whether they comply with company policy, industry standards or regulations}~\citep{DBLP:conf/fat/RajiSWMGHSTB20}. Algorithmic fairness audits are closely linked to audits studies as understood in the social sciences, with a strong emphasis on social justice~\citep{DBLP:conf/eaamo/VecchioneLB21}. A recent survey on public algorithmic audit identified four categories of \enquote{problematic machine behaviour} that can be unveiled by audit studies: discrimination, distortion, exploitation and misjudgement \citep{DBLP:journals/pacmhci/Bandy21,liu2012enterprise}. This taxonomy is highly helpful when putting forward auditing studies and also relevant for this work: our \enquote{Equal Treatment Inspector} can be seen as a tool to help understand discrimination in machine learning models, thus, falling into the first category.

Selecting a measure to compare fairness between two sensitive groups has been a highly discussed topic, where results such as~\citep{DBLP:journals/bigdata/Chouldechova17,DBLP:conf/nips/HardtPNS16,DBLP:conf/innovations/KleinbergMR17}, have highlighted the impossibility to satisfy simultaneously three type of fairness measures: demographic parity~\citep{DBLP:conf/innovations/DworkHPRZ12}, equalized odds~\citep{DBLP:conf/nips/HardtPNS16}, and predictive parity~\citep{DBLP:conf/kdd/Corbett-DaviesP17,DBLP:journals/corr/abs-2102-08453,wachter2020bias}. Previous work has relied on the notion of equal outcomes by measuring and calculating demographic and statistical parity on the model predictions \citep{kearns2018preventing}.
The arguments of ~\cite{DBLP:conf/aies/SimonsBW21} provide background and motivation towards interpretations of equal treatment as blindness and discuss its importance in practical policy applications.

\subsection{Explainability for Fair Supervised Learning}

The intersection of fairness and explainable AI has been an active topic in recent years. The most similar work is ~\citep{lundberg2020explaining} where they apply Shapley values to statistical parity.
Our work continues in-depth on this research line by formalizing the explanation distribution, deriving theoretical guarantees, and proposing novel methodology. Specifically, our approach allows for comparison across different protected groups. We also introduce the concept of accountability, which refers to the ability of the algorithm to provide insights into the sources of \emph{equal treatment} violation of the model. Additionally, we evaluate our method on multiple datasets and synthetic examples, demonstrating its effectiveness in identifying and mitigating equal treatment disparities in machine learning models.

A recent line of work assumes knowledge about causal relationships between random variables. For example,~\cite {DBLP:conf/fat/GrabowiczPM22} presents a post-processing method based on Shapley values aiming to detect and nullify the influence of a protected attribute on the output of the system. They assume known direct causal links from the data to the protected attribute and no measured confounders. Our work does not rely on causal graphs knowledge but exploits the Shapley values' theoretical properties to obtain fairness model auditing guarantees. 

~\cite{DBLP:conf/ssci/StevensDVV20} presents an approach to explaining fairness based on adapting the Shapley value function to explain model unfairness. They also introduce a new meta-algorithm  that considers the problem of learning an additive perturbation to an existing model in order to impose fairness. 
In our work, we do not adopt the Shapley value function. Instead, we use the theoretical Shapley properties to provide fairness auditing guarantees. Our \textit{Equal Treatment Inspector} is not perturbation-based but uses Shapley values to project the model to the explanation distribution, and then measures \emph{un-equal treatment}. It also allows us to pinpoint what are the specific features driving this violation.

In ~\cite{DBLP:conf/fat/GrabowiczPM22}, the authors present a post-processing method based on Shapley values aiming to detect and nullify the influence of a protected attribute on the output of the system. For this, they assumed there are direct causal links from the data to the protected attribute and that there are no measured confounders. Our work does not use causal graphs but exploits the theoretical properties of the Shapley values to obtain fairness model auditing guarantees.

Instead of using feature attribution explanation, other works have researched fairness using other explainability techniques such as counterfactual explanations \citep{DBLP:conf/nips/KusnerLRS17,DBLP:conf/cogmi/ManerbaG21,DBLP:conf/aies/MutluYG22}.
We don't focus on counterfactual explanations but on feature attribution methods that allow us to measure unequal feature contribution to the prediction. Further work can be envisioned by applying explainable AI techniques to the \enquote{Equal Treatment Inspector} or constructing the explanation distribution out of other techniques.

\subsection{Classifier Two-Sample Test}
The use of classifiers to obtain statistical tests or measure independence between two distributions has been previously explored in the literature \citep{DBLP:conf/iclr/Lopez-PazO17}. Specifically, the first authors introduced the use of \enquote{Classifier Two-Sample Tests (C2ST)} to represent the data that returns an interpretable unit test statistic, allowing it to measure how two distributions differ from each other. Their approach establishes the main theoretical properties, compares performance against state-of-the-art methods, and outlines applications of C2ST. 
Similar work of~\cite{DBLP:conf/icml/LiuXL0GS20} propose a kernel-based to two-sample
tests classification, aiming to determine whether two samples are drawn from the same underlying distribution. 
Alike work has also been used in Kaggle competitions under the name of \enquote{Adversarial Validation}~\citep{kaggleAdversarial,howtowinKaggle}, a technique which aims to detect which features are distinct between train and leaderboard datasets to avoid possible leaderboard shakes. 
Our work builds on previous approaches by adopting their interpretable power test analysis and theoretical properties. Still, we focus on testing for fairness notions of equal treatment in the explanation distribution. As the test statistic for our approach, we use the Area Under the Curve (AUC) of the \enquote{Equal Treatment Inspector}.

Another related work in the literature is by \cite{DBLP:journals/corr/EdwardsS15}, who focuses on removing statistical parity from images by
using an adversary that tries to predict the relevant sensitive variable from the model representation and censoring the learning of the representation of the model and data on images and neural networks. While methods for images or text data are often developed specifically for neural networks and cannot be directly applied to traditional machine learning techniques, we focus on tabular data where techniques such as gradient boosting decision trees achieve state-of-the-art model performance \citep{DBLP:conf/nips/GrinsztajnOV22,DBLP:journals/corr/abs-2101-02118,BorisovNNtabular}. 
Furthermore, our model and data projection into the explanation distributions leverages Shapley value theory to provide fairness auditing guarantees. In this sense, our work can be viewed as an extension of their work, both in theoretical and practical applications.

\section{Comparison with Related Fairness Metrics}

In the main body of the paper, we have made an explicit comparison against demographic parity. In this section, we compare our proposed notion and model for equal treatment with respect to other metrics proposed (see Chapter~\ref{ch:foundations}).
\begin{table}[ht]
\tiny
\centering
\caption{Fairness metrics alignment with \emph{Equal Treatment} criteria. Requirements are derived from a use case explained in Section~\ref{examples:use.case2}, and metrics are further discussed and defined in Section~\ref{sec:foundations.fairness}}
\begin{tabular}{c|c|c|c|c|c}
                             \textbf{Metric} & \begin{tabular}[c]{@{}c@{}}\textbf{Group} \\ \textbf{Discrimination} \\\emph{(R1)}\end{tabular} &\begin{tabular}[c]{@{}c@{}}\textbf{Unlabelled} \\ \textbf{Data} \\\emph{(R2)}\end{tabular} & \begin{tabular}[c]{@{}c@{}}\textbf{No Background}\\\textbf{Knowledge}\\ \emph{(R3)}\end{tabular} & \begin{tabular}[c]{@{}c@{}}\textbf{Proxy} \\ \textbf{Discrimination} \\ \emph{(R4)}\end{tabular} & \begin{tabular}[c]{@{}c@{}}\textbf{Explanation}\\ \textbf{Capabilities} \\ \emph{(R5)}\end{tabular}   \\ \hline
Equal Treatment              & $\tikzcmark$& $\tikzcmark$                                                    & $\tikzcmark$                                                       & $\tikzcmark$             & $\tikzcmark$ \\
Demographic Parity           & $\tikzcmark$& $\tikzcmark$                                                    & $\tikzcmark$                                                       & $\tikzxmark$             & $\tikzxmark$ \\
Equal Opportunity            & $\tikzcmark$& $\tikzxmark$                                                    & $\tikzcmark$                                                       & $\tikzxmark$             & $\tikzxmark$ \\
Treatment Equality           & $\tikzcmark$& $\tikzxmark$                                                    & $\tikzcmark$                                                       & $\tikzxmark$             & $\tikzxmark$ \\
Feature Importance Disparity& $\tikzcmark$& $\tikzcmark$                                                    & $\tikzcmark$                                                       & $\tikzxmark$             & $\tikzcmark$\\
Counterfactual Fairness               & $\tikzcmark$& $\tikzcmark$                                                    & $\tikzxmark$                                                       & $\tikzcmark$             & $\tikzcmark$\\
Fairness-Unawareness               & $\tikzcmark$& $\tikzcmark$                                                    & $\tikzcmark$                                                       & $\tikzxmark$             & $\tikzxmark$\\
Individual Fairness               & $\tikzxmark$& $\tikzcmark$                                                    & $\tikzcmark$                                                       & $\tikzcmark$             & $\tikzcmark$
\end{tabular}\label{tab:metrics}
\end{table}

\

\subsection{Equal Opportunity/Equalized Odds}

\begin{gather}
\text{TPR} = \frac{TP}{TP + FN}\\
\text{EOF}_{z}= \text{TPR}_z - \text{TPR}
\end{gather}

A negative value in Equal Opportunity is due to the worse ability of a ML model to find actual True Positives for the protected group in comparison with the reference group. The reliance on true positive rates in formulating Equal Opportunity presents an additional challenge involving the calculation of False Negatives. In the context of the loan scenario, a False Negative occurs when a loan is denied to someone capable of repayment, but acquiring this data may not be feasible. An alternative approach involves maintaining a holdout set of randomly selected users to whom loans are uniformly granted, allowing for monitoring in that controlled environment.

Nevertheless, the implementation of a holdout set comes with potential economic and societal implications. Careful consideration is needed, as this approach may impact both financial outcomes and broader societal dynamics.

A similar logic can be applied to the other  Equalized Odds, that also relies on having access to labeled data,hence it does not meet \emph{(R2)}.

From the AI Alignment perspective, the alignment of both metrics—Equal Opportunity and Equalized Odds—with egalitarian ideals is evident, as they prioritize access to good outcomes through error ratios. These metrics also demonstrate a parallel alignment with utilitarian principles by focusing on minimising disparate errors, emphasizing a utilitarian approach to optimizing overall societal welfare.

In contrast, our conceptualization of equal treatment diverges in its alignment, finding resonance with distributive justice principles rooted in liberal ideology and deontological ethical considerations. This divergence reflects multiple perspectives and values and the nature of ethical foundations and justice paradigms. While Equal Opportunity and Equalized Odds lean towards utilitarian and egalitarian perspectives, our equal treatment notion is grounded in a more liberal-oriented distributive justice framework, emphasizing fairness and individual rights.

This diversity in alignment underscores, as expressed in the AI alignment foundation in Chapter~\ref{ch:foundations} the multiple values of society and the challenge of defining them mathematically.

\subsection{Treament Equality}

\begin{gather}
    \frac{FP_z}{FN_z}=\frac{FP}{FN}
\end{gather}

From a technical perspective the challenge relies on obtaining false positive and false negative data, specially the later one as we have discussed in the previous section,hence it does not meet \emph{(R2)}.  

Regarding AI alignment, the original authors do not specify the principle or value with which the notion should align, and this clarity is absent from their definition. An illustrative sentence is \enquote{\textit{men and women are being treated differently by the algorithm}}. However, this statement may not be entirely accurate, as the algorithm may treat men and women equally while exhibiting different error rates. For instance, a constant classifier that grants loans to everyone treats everyone equally. Still, if males and females have distinct patterns of loan repayment, the classifier will yield different errors, resulting in distinct $\frac{FP}{FN}$ values. 

Then this metric of \enquote{treament equality} is very distinct from our proposed notion due to not belonging to a clear distributive justice perspective and the technical implementation. Identically to equal opportunity, it's a utility measure; therefore, it is utilitarian rather than our notion that is based on principles rather than consequences, deontological.

\subsection{Demographic Parity}

Demographic parity, is not able to detect proxy discrimination \emph{(R3)} nor capable to account for the sources behind discrimination \emph{(R4)}, as we have discussed in the main body of the chapter.

Thus, our notion of equal treatment is an improvement with respect to demographic parity in the two dimensions that we have discussed in the foundation following \cite{DBLP:journals/mima/Gabriel20}: \emph{(i)} better AI alignment with philosophical notions and \emph{(ii)} translated as a stronger and more sensitive metric. 

\subsection{Individual Fairness}

We say that a model achieves individual fairness w.r.t. two individual samples if:

\begin{gather}
d (f(x),f(x')) \leq \mathcal{L} \cdot d(x,x') \\ \forall x, x' \in X
\end{gather}

From a philosophical perspective, individual fairness is closest to the liberal and deontological  point of view. However, it fails the requirements of liberalist arguments. Liberalism argues for meritocracy, i.e.\ disparate treatment may be considered fair if it is based on varying efforts or preferences of individuals but unfair if it is based on protected characteristics that individuals did not choose. E.g.\ it's fair to hire someone because of better grades, but it is unfair if these grades depend on ethnicity. 
The definition of individual fairness does not consider protected characteristics or proxies thereof failing research requirement \emph{(R1)} of group discrimination.. 
While individual fairness may be easily combined with blindness or adjusted definitions of similarities~\citep{DBLP:conf/aies/Fleisher21,DBLP:conf/iclr/YurochkinBS20}, the treatment of proxies to protected characteristics is hard to avoid, rendering the alignment of an AI with distributive justice values difficult.  
Therefore, individual fairness is not a popular metric --- unlike group fairness metrics~\citep{DBLP:conf/aies/Fleisher21}.

\subsection{Counterfactual Fairness}

Counterfactual fairness, as defined by~\cite{DBLP:conf/nips/KusnerLRS17}:

\begin{definition}\textit{Counterfactual Fairness}
We say that the model $f_\theta$  is counterfactually fair if under any context $X = x$ and $A = a$
\begin{gather}
    P(f_\theta(x_{Z=z}) = y | X = x, Z = z) = P(f_\theta(x_{Z=z'})  = y | X = x, Z = z)\\
    \forall y \in Y \quad \forall z \in Z \quad \forall x \in X
\end{gather}
\end{definition}

\citet{DBLP:conf/aaai/RosenblattW23} have shown that: (i) an algorithm that satisfies counterfactual fairness also satisfies demographic parity; and (ii) all algorithms that satisfy demographic parity can be modifed to satisfy counterfactual fairness.
These results conclude that counterfactual fairness is equivalent to demographic parity, therefore failing the same requirements \emph{(R4)}, proxy discrimination and \emph{(R5)} explanation capabilities.  Moreover, this fairness definition fails on our requirement \emph{(R3)}, of not needing background knowledge -- since counterfactuals can only be computed for a given causal graph.

\subsection{Fairness Through Unawareness}

This metric was initially proposed as a baseline by \cite{DBLP:conf/aaai/Grgic-HlacaZGW18}

\begin{definition}\textit{(Fairness Through Unawareness}. An algorithm is fair if any protected
attributes $Z$ are not explicitly used in decision-making.
\end{definition}

Any mapping $f: \mathbf{X} \rightarrow Y$ for which $Z \not\in \mathbf{X}$ satisfies such a definition. It has a clear shortcoming as features in $\mathbf{X}$ can contain discriminatory information, known as proxy features for discrimination.
Further, the research requirement of detecting \emph{(R4) \textit{Proxy discrimination}} is not met. One of the main contributions of our notion of equal treatment is that it captures statistical relations of all the features with the protected attribute. In our approach, the model does not consider the protected attribute to be present in the covariates $X$, therefore implying fairness through unawareness. Furthermore, in Section~\ref{sec:expSpaceIndependence} of the main body, we discuss the limitations of analyzing the Shapley values of the protected attribute. 

\subsubsection{Decomposition Method Specific of Demographic Parity.}\label{app:relatedwork.Lundberg}

We formally compare our approach to the prior work of \citet{lundberg2020explaining} and the related SHAP Python package documentation. The authors addressed DP using SHAP value estimation.  This brief workshop paper emphasizes the importance of \enquote{decomposing a fairness metric among each of a model’s inputs to reveal which input features may be driving any observed fairness disparities}.
In terms of statistical independence, the approach can be rephrased as decomposing $f_\theta(\mathbf{X}) \perp Z$ by examining $S(f_\theta,\mathbf{X})_i \perp Z$ individually for $i \in [1, p]$. 
Actually, the paper limits itself to consider only differences of means, namely testing for $E[\Ss(f_\theta,\mathbf{X})_i|Z=1] \neq E[\Ss(f_\theta,\mathbf{X})_i|Z=0]$. 
However, the method proposed by \citet{lundberg2020explaining} is not sufficient nor necessary to prove DP, as shown next.

\begin{lemma}\label{lemma:lundberg}
$f_\theta(\mathbf{X}) \perp Z$ is neither implied by nor it implies ($\Ss(f_\theta,\mathbf{X})_i \perp Z$ for $i \in [1, p]$).
\end{lemma}
\begin{proof}
Consider $f_\theta(\mathbf{X}) = \mathbf{X}_1 - \mathbf{X}_2$ with $\mathbf{X}_1,\mathbf{X}_2 \sim \texttt{Ber}(0.5)$ and $Z=1$ if $\mathbf{X}_1=\mathbf{X}_2$, and $Z=0$ otherwise. Hence $Z\sim \texttt{Ber}(0.5)$. We have $\Ss(f_\theta,\mathbf{X}_1) = \mathbf{X}_1 \perp Z$ and $\Ss(f,\mathbf{X}_2) = -\mathbf{X}_2 \perp Z$. However, $f_\theta(\mathbf{X}) \not\perp Z$, e.g., $P(Z=0|f_\theta(\mathbf{X})=0) = 1 \neq 0.5$. Example \ref{ex42} illustrates a case where $f_\theta(\mathbf{X}) \perp Z$ yet $\Ss(f_\theta, \mathbf{X})_1$ and $\Ss(f_\theta,\mathbf{X})_2$ are not independent of $Z$. 
\end{proof}

Our approach to ET considers the independence of the \textit{multivariate} distribution of $\Ss(f,\mathbf{X})$ with respect to $Z$, rather than the independence of each marginal distribution $\Ss(f_\theta, \mathbf{X})_i \perp Z$. With our definition, we obtain a sufficient condition for DP, as shown in Lemma \ref{lemma:inc}.

\citet{DBLP:journals/corr/abs-2303-01704} follows a similar approach but from the perspective of discovering which subgroup $z \in Z$ exhibits the largest importance disparity relative to a feature $\mathbf{X}_i$. 

\begin{definition}\textit{(Feature Importance Disparity)}
Assume $Z$ is discrete, not necessarily binary. The feature importance disparity of $z \in Z$ relatively to a feature $\mathbf{X}_i$ is defined as:
\begin{equation*}
    \texttt{FID}(z, i) = |E[\Ss(f_\theta,\mathbf{X})_i|Z=z] - E[\Ss(f_\theta,\mathbf{X})_i]|
\end{equation*}  
\end{definition}

From an alignment perspective, the authors of this method don't clarify which equality philosophical criteria the proposed metric measures. Moreover, assuming DP as the reference notion, the method shares the same pitfalls illustrated by the lemma above for \citet{lundberg2020explaining}.
From the requirements perspective,  the definition does not meet \emph{(R4)} proxy discrimination detection due to using the expected value to measure distribution differences.

\clearpage
\section{Discussion}
\textbf{Alignment of AI Methodology with Equal Treatment Specifications}

The methodology's alignment with equal treatment is  structured to comply with the stipulated requirements from the identified Table~\ref{tab:requirements}.  Each aspect of our method addresses mandates from this table, ensuring the alignment of our AI tools with the foundational principles of liberal political philosophy and deontological ethics.

\begin{itemize}
\item \textbf{ET-01 Unlabeled Data:} Our methodology analyses the model behaviour disparities without reliance on labelled data, thereby resonating with both the deontological and liberal perspectives of assessing fairness in the decisions over errors. 

\item \textbf{ET-02 Group Discrimination:} Our metric aligns with ethical standards and societal expectations by emphasizing equal treatment and addressing discrimination based on characteristics non-chosen at birth.

\item \textbf{ET-03 Measuring Proxy Discrimination:} Our analytical tools are designed to identify and address proxy discrimination, where seemingly neutral features could indirectly generate bias against protected groups. 

\item \textbf{ET-04 Explanation Capability:} Through the implementation of explanation capabilities, our metric offers transparency regarding how decisions are made, providing theoretical and empirical insights into the factors driving any detected biases, hence empowering end-users with the knowledge to understand and rectify potential discrimination.

\item \textbf{ET-05 Documentation:} Our software and documentation are Python-based hosted in \url{https://explanationspace.readthedocs.io/}. This aligns with the requirement for accessibility and clarity. Documentation supports the community of developers and researchers, facilitating understanding, deployment, and further innovation.

\item \textbf{ET-06 Reproducibility:} Our metric is implemented in an open source python package \texttt{explanationspace} fosters the reproducibility of fairness evaluations, enabling comparisons across different models or iterations thereof. 

\item \textbf{ET-07 Evaluation:} Providing extensive tutorials and feedback mechanisms by \texttt{GitHub} in \url{https://github.com/cmougan/explanationspace} enriches the user experience, promoting active community participation. 
\end{itemize}

\section{Chapter Conclusions}

In this chapter of the thesis, we have introduced a novel approach for fairness in machine learning by measuring \emph{equal treatment}. While related work
reasoned over model predictions to measure  \emph{equal outcomes}, our notion of {equal treatment} is more fine-grained, accounting for the usage of attributes by the model via explanation distributions. Consequently, {equal treatment} implies {equal outcomes}, but the converse is not necessarily true, which we confirmed both theoretically and experimentally.

This paper also seeks to improve the understanding of how theoretical concepts of fairness from liberalism-oriented political philosophy and moral deontology aligns with technical measurements. Rather than merely comparing one social group to another based on disparities within decision distributions, our concept of equal treatment takes into account differences through the explanation distribution of all non-protected attributes, which often act as proxies for protected characteristics. Implications warrant further techno-philosophical discussions. Implications warrant further techno-philosophical discussions.


\chapter{How Did Distribution Shift Impact the Model?}\label{ch:exp.shift}

Machine Learning (ML) theory provides means to forecast the quality of ML models on unseen data, provided that this data is sampled from the same distribution as the data used to train and evaluate the model. If unseen data is sampled from a different distribution than the training data, model quality may deteriorate, making monitoring how the model's behavior changes crucial.

Recent research has highlighted the impossibility of reliably estimating the performance of machine learning models on unseen data sampled from a different distribution in the absence of further assumptions about the nature of the shift~\cite{DBLP:journals/jmlr/Ben-DavidLLP10,DBLP:conf/icml/LiptonWS18,garg2022leveraging}. State-of-the-art techniques attempt to model statistical distances between the distributions of the training and unseen data~\cite{continual_learning,clouderaff} or the distributions of the model predictions~\cite{garg2022leveraging,garg2021ratt,lu2023predicting}. However, these measures of \emph{distribution shifts} only partially relate to changes of interaction between new data and trained models or they rely on the availability of a causal graph or types of shift assumptions, which limits their applicability. Thus, it is often necessary to go beyond detecting such changes and understand how the feature attribution changes~\citep{DBLP:conf/kdd/KenthapadiLNS22, haug2022change,mougan2022monitoring,continual_learning}.

The field of explainable AI has emerged as a way to understand model decisions ~\cite{xai_concepts,molnar2019} and interpret the inner workings of ML models~\cite{guidotti_survey}. The core idea of this paper is to go beyond the modeling of distribution shifts and monitor for \emph{explanation shifts} to signal a change of interactions between learned models and dataset features in tabular data. We newly define explanation shift as the statistical comparison between how predictions from training data are explained and how predictions on new data are explained. In summary, our contributions are:
\begin{itemize}

    \item We propose measures of explanation shifts as a key indicator for investigating the interaction between distribution shifts and learned models.

    \item We define an \textit{Explanation Shift Detector} that operates on the explanation distributions allowing for more sensitive and explainable changes of interactions between distribution shifts and learned models.
    
    \item We compare our monitoring method that is based on explanation shifts with methods that are based on other kinds of distribution shifts. We find that monitoring for explanation shifts results in  more sensitive indicators for varying model behavior.

    \item We release an open-source Python package \texttt{skshift}, which implements our \enquote{\textit{Explanation Shift Detector}}, along usage tutorials for reproducibility.
    
\end{itemize}

\section{A Model for Explanation Shift Detection}
\label{sec:Detector}

Our model for explanation shift detection is sketched in Fig.~\ref{fig:workflow}. We define it step-by-step as follows:

\begin{definition}[Explanation distribution]
An explanation function $\mathcal{S}:F\times \mbox{dom}(X)\to \mathbb{R}^p$ maps a model $f_\theta$  and data $x\in \mathbb{R}^p$ to a vector of attributions $\mathcal{S}(f_\theta, x)\in \mathbb{R}^p$. We call $\Ss(f_\theta, x)$ an explanation.
We write $\Ss(f_\theta, \mathcal{D})$ to refer to the empirical \emph{explanation distribution} generated by $\{\Ss(f_\theta, x) | x\in \mathcal{D}\}$. \label{def:explanationdistribution}
\end{definition}

We use local feature attribution methods SHAP and LIME as explanation functions $\Ss$.
\begin{definition}[Explanation shift]
Given a model $f_\theta$ learned from  $\Dd{tr}$, explanation shift with respect to the model $f_\theta$ occurs if $\Ss(f_\theta,\Dd{new}_X) \nsim \Ss(f_\theta,\Dd{tr}_X)$.
\end{definition}
\begin{definition}[Explanation shift metrics]
Given a measure of statistical distances $d$,
 {explanation shift} is measured as the distance between two explanations of the model $f_\theta$ 
by  $d(\mathcal{S}(f_\theta, \Dd{tr}_X),\mathcal{S}(f_\theta, \Dd{new}_X))$.
\end{definition}




We follow Lopez et al.~\cite{DBLP:conf/iclr/Lopez-PazO17} to define an explanation shift metrics based on a two-sample test classifier. We proceed as depicted in Figure~\ref{fig:workflow}.
To counter overfitting, given the model $f_\theta$  trained on $\D^{\text{tr}}$, we compute explanations $\{\Ss(f_\theta,x)| x \in \D_X^{\text{val}}\}$ on an in-distribution validation data set  $\D_X^{\text{val}}$. 
Given a dataset $\D_X^{new}$, for which the status of in- or out-of-distribution is unknown, we compute its explanations $\{\Ss(f_\theta,x)| x\in \D_X^{\text{new}}\}$.
Then, we construct a two-samples dataset $E=\{(S(f_\theta,x),a_x)|x\in \D_X^{\text{val}},a_x=0\}\cup \{(S(f_\theta,x),a_x)|x\in \D_X^{\text{new}},a_x=1\}$ and  we train a discrimination model $g_\psi:R^p \rightarrow\{0,1\}$  on  $E$, to predict if an explanation should be classified as in-distribution (ID) or out-of-distribution (OOD):  

\begin{gather}\label{eq:explanationShift}
\psi = \argmin_{\tilde{\psi}} \sum_{x\in \D_X^{\text{val}}\cup \D_X^{\text{new}}} \ell( g_{\tilde{\psi}}(\Ss(f_\theta,x)) , 
a_x ),
\end{gather}

where $\ell$ is a classification loss function (e.g.\ cross-entropy). 
$g_\psi$ is our two-sample test classifier, based on which AUC yields a test statistic that measures the distance between the $D^{tr}_X$ explanations and the explanations of new data $D^{new}_X$.

Explanation shift detection allows us to detect \emph{that} a novel dataset $D^{new}$ changes the model's behavior. Beyond recognizing explanation shift, using feature attributions for the model $g_\psi$, we can interpret \emph{how} the features of the novel dataset $D^{new}_X$ interact differently with model $f_\theta$ than the features of the validation dataset $D^{val}_X$. These features are to be considered for model monitoring and for classifying new data as out-of-distribution.



\begin{figure}[ht]
    \centering
    \includegraphics[width=1\columnwidth]{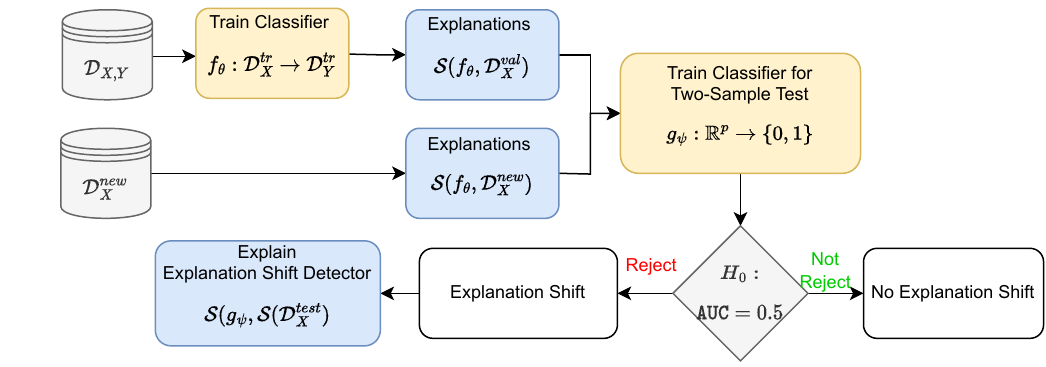}
    \caption{{Our model for explanation shift detection}. The model $f_\theta$ is trained on $\Dd{tr}$ implying explanations for distributions $\Dd{val}_X, \Dd{new}_X$. The AUC of the two-sample test classifier $g_\psi$ decides for or against explanation shift. If an explanation shift occurred, it could be explained which features of the $\Dd{new}_X$ deviated in $f_\theta$ compared to $\Dd{val}_X$.} 
    \label{fig:workflow}
\end{figure}

\section{Relationships between Common Distribution Shifts and Explanation Shifts}\label{sec:math.analysis}
This section analyses and compares data shifts, prediction shifts, with explanation shifts.

\subsection{Explanation Shift vs Data Shift}\label{subsec:explanationShiftMethods}

\subsubsection{Covariate Shift}
One type of distribution shift that is challenging to detect comprises cases where the univariate distributions for each feature $j$ are equal between the source $\Dd{tr}_X$ and the unseen dataset $\Dd{new}_X$, but where interdependencies among different features change.  Multi-covariance statistical testing is a hard taks with high sensitivity that can lead to false positives. The following example demonstrates that Shapley values account for co-variate interaction changes while a univariate statistical test will provide false negatives. \\

\begin{example}
\textit{(\textbf{Covariate Shift})\label{ex:covariate}
Let $D^{tr} \sim   N\left(\begin{bmatrix}\mu_{1}  \\ \mu_{2} \end{bmatrix},\begin{bmatrix}\sigma^2_{X_1} & 0 \\0 & \sigma^2_{X_2} \end{bmatrix}\right) \times Y$.
We fit a linear model $f_\theta(x_1,x_2) = \gamma + a\cdot x_1 + b \cdot x_2.$  If 
$\D_X^{new} \sim N\left(\begin{bmatrix}\mu_{1}  \\ \mu_{2} \end{bmatrix},\begin{bmatrix} \sigma^2_{X_1} & \rho\sigma_{X_1}\sigma_{X_2}  \\ \rho\sigma_{X_1}\sigma_{X_2} & \sigma^2_{X_2}\end{bmatrix}\right)$, then 
$\PP(\D_{X_1}^{tr})$ and $\PP(\D_{X_2}^{tr})$ are identically distributed with $\PP(\D_{X_1}^{new})$ and $\PP(\D_{X_2}^{new})$, respectively, while this does not hold for the corresponding $\Ss_j(f_\theta,\D_X^{tr})$ and $\Ss_j(f_\theta,\D_X^{new})$. 
}
\end{example}

\begin{gather}
\Ss_1(f_\theta,x) = a(x_1 - \mu_1)\\
\Ss_1(f_\theta,x^{new}) =\\
=\frac{1}{2}[\mathrm{val}(\{1,2\}) - \mathrm{val}(\{2\})] + \frac{1}{2}[\mathrm{val}(\{1\}) - \mathrm{val}(\emptyset)] \\
\mathrm{val}(\{1,2\}) = E[f_\theta|X_1=x_1, X_2=x_2] = a x_1 + b x_2\\
\mathrm{val}(\emptyset) = E[f_\theta]= a \mu_1 + b  \mu_2 \\
\mathrm{val}(\{1\}) = E[f_\theta(x) | X_1 = x_1] +b\mu_2 \\
\mathrm{val}(\{1\}) = \mu_1 +\rho \frac{\rho_{x_1}}{\sigma_{x_2}}(x_1-\sigma_1)+b \mu_2\\
\mathrm{val}(\{2\}) = \mu_2 +\rho \frac{\sigma_{x_2}}{\sigma_{x_1}}(x_2-\mu_2)+a\mu_1 \\
\Rightarrow \Ss_1(f_\theta,x^{new})\neq a(x_1 - \mu_1)
\end{gather}

\subsubsection{Uninformative Features}

False positives frequently occur in out-of-distribution data detection when a statistical test recognizes differences between a source distribution and a new distribution, thought the differences do not affect the model behavior~\cite{grinsztajn2022why,DesignMLSystems}. Shapley values satisfy the \emph{Uninformativeness} property, where a feature $j$ that does not change the predicted value  has a Shapley value of $0$ (Property~\ref{eq:SHAP-uninformativeness}).

\begin{example}
\textbf{\textit{Unused features}}\textit{
Let $\Dd{tr}=\{(x_0^{tr},y_0^{tr})\ldots, (x_n^{tr},y_n^{tr})\}$  
with predictors $\D_{X}^{tr} \sim X_1 \times X_2 \times X_3  = N(\mu ,\mathrm{diag}(c))$, and create a synthetic target $y_i=a_0 + a_1 \cdot x_{i,1} + a_2 \cdot x_{i,2} + \epsilon$ that is independent of $X_3$. As new data we have $\D_{X}^{new}\sim X_1^{new} \times X_2^{new} \times X_3^{new} = N(\mu,\mathrm{diag}(c'))$, where $\mathrm{c}$ and $\mathrm{c'}$ are an identity matrix of order three and $\mu = (\mu_1,\mu_2,\mu_3)$. We train a linear regression $f_\theta$ on $\D_{X,Y}^{tr}$, with coefficients $a_0,a_1,a_2,a_3$. Then if $\mu_3'\neq \mu_3$ or $c_3' \neq c_3$, then $\PP(\D_{X_3})$ can be different from $\PP(\D_{X_3}^{new})$ but $\Ss_3(f_\theta, \D_X^{tr}) = \Ss_3(f_\theta,\D_X^{new})$}
\end{example}
\begin{gather}
\D_{X_3}\sim N(\mu_3,c_3),\D_{X_3}^{new} \sim N(\mu_3^{'}, c_3^{'})\\
\mathrm{If} \quad  \mu_3^{'}\neq \mu_3 \quad\mathrm{or} \quad c_3^{'}\neq c_3 \rightarrow P(X_3)\neq P(X_3^{new})\\
\Ss(f_\theta,X) = \left(\begin{bmatrix} a_1(X_1 - \mu_1)  \\a_2(X_2 - \mu_2)  \\a_3(X_3 - \mu_3)   \end{bmatrix} \right) = \left(\begin{bmatrix} a_1(X_1 - \mu_1)  \\a_2(X_2 - \mu_2)  \\0   \end{bmatrix} \right)\\
\Ss_3(f_\theta,\D_X)= \Ss_3(f_\theta, \D_X^{new})
\end{gather}

\subsection{Explanation Shift vs Prediction Shift}\label{sec:exp.vs.pred}

Analyses of the explanations  detect distribution shifts that interact with the model. In particular, if  a prediction shift occurs, the explanations produced are also shifted. 
\begin{prop}
    Given a model $f_\theta:\D_X \to \D_Y$. If $f_\theta(x')\neq f_\theta(x)$, then $\Ss(f_\theta,x') \neq \Ss(f_\theta,x)$.
\end{prop}

By efficiency property of the Shapley values~\cite{DBLP:journals/ai/AasJL21} (equation (\eqref{eq:SHAP-efficiency})), if  the prediction between two instances is different, then they differ in at least one component of their explanation vectors. 

The opposite direction does not always hold:
\begin{example}\textit{(\textbf{Explanation shift not affecting prediction distribution}) Given $\mathcal{D}^{tr}$ is generated from $(X_1 \times X_2 \times Y), X_1 \sim U(0,1), X_2 \sim U(1,2), Y = X_1+X_2+\epsilon$ and thus the optimal model is $f(x)=x_1+x_2$. If $\mathcal{D}^{new}$ is generated from $X_1^{new}\sim U(1,2), X_2^{new}\sim U(0,1),\quad Y^{new} = X_1^{new}+X_2^{new}+\epsilon$, the prediction distributions are identical $f_\theta(\mathcal{D}^{tr}_X),f_\theta(\mathcal{D}^{new}_X)\sim U(1,3)$, but explanation distributions are different $S(f_\theta,\mathcal{D}^{tr}_X)\nsim S(f_\theta,\mathcal{D}^{new}_X)$, because $\Ss_i(f_\theta,x) = \alpha_i \cdot x_i$.
}
\end{example}
Thus, an explanation shift does not always imply a prediction shift.

\subsection{Explanation Shift vs Concept Shift}\label{sec:conceptShift}
One of the most challenging types of distribution shift to detect are cases where distributions are equal between source and unseen data-set $\PP(\D_X^{tr}) = \PP(\D_X^{new})$ and the target variable  $\PP(\D_Y^{tr}) = \PP(\D_Y^{new})$ and what changes are the relationships that features have with the target $\PP(\D_Y^{tr}|\D_X^{tr}) \neq  \PP(\D_Y^{new}|\D_X^{new})$, this kind of distribution shift is also known as concept drift or posterior shift~\citep{DesignMLSystems} and is especially difficult to notice, as it requires labeled data to detect. The following example compares how the explanations change for two models fed with the same input data and different target relations.

Concept shift comprises cases where the covariates retain a given distribution, but their relationship with the target variable changes (cf. Section~\ref{subsec:typesShift}). This example shows the negative result that  concept shift cannot be indicated by the detection of explanation shift.

\begin{example}

\textbf{ \textit{Concept shift}}\textit{
Let $\D_X = (X_1,X_2) \sim N(\mu,I)$, and $\D_X^{new}= (X^{new}_1,X^{new}_2) \sim N(\mu,I)$, where $I$ is an identity matrix of order two and $\mu = (\mu_1,\mu_2)$. We now create two synthetic targets $Y=a + \alpha \cdot X_1 + \beta \cdot X_2 + \epsilon$ and $Y^{new}=a + \beta \cdot X_1 + \alpha \cdot X_2 + \epsilon$. Let $f_\theta$ be a linear regression model trained on $f_\theta:\D_X\rightarrow \D_Y)$ and $h_\phi$ another linear model trained on $h_\phi:\D_X^{new}\rightarrow \D_Y^{new})$. Then $\PP(f_\theta(X)) = \PP(h_\phi(X^{new}))$, $P(X) = \PP(X^{new})$ but $\Ss(f_\theta,X)\neq \Ss(h_\phi, X)$}. 
\end{example}
\begin{gather}
X  \sim N(\mu,\sigma^2\cdot I), X^{new}\sim N(\mu,\sigma^2\cdot I)\\
\rightarrow P(\D_X) = P(\D_X^{new})\\
Y \sim a + \alpha N(\mu, \sigma^2) + \beta N(\mu, \sigma^2) + N(0, \sigma^{'2})\\
Y^{new} \sim a + \beta N(\mu, \sigma^2) + \alpha N(\mu, \sigma^2) + N(0, \sigma^{'2})\\
\rightarrow P(\D_Y) = P(\D_Y^{new})\\
\Ss(f_\theta,\D_X) = \left( \begin{matrix}\alpha(X_1 - \mu_1)  \\\beta(X_2-\mu_2) \end{matrix}\right) \sim \left(\begin{matrix}N(\mu_1,\alpha^2 \sigma^2)  \\N(\mu_2,\beta^2 \sigma^2) \end{matrix}\right)\\
\Ss(h_\phi,\D_X) =  \left( \begin{matrix}\beta(X_1 - \mu_1)  \\\alpha(X_2-\mu_2) \end{matrix}\right)\sim \left(\begin{matrix}N(\mu_1,\beta^2 \sigma^2)  \\N(\mu_2,\alpha^2 \sigma^2) \end{matrix}\right)\\
\mathrm{If} \quad \alpha \neq \beta \rightarrow \Ss(f_\theta,\D_X)\neq \Ss(h_\phi,\D_X)
\end{gather}

In general, concept shift cannot be detected because $\D_Y^{new}$ is unknown \cite{garg2022leveraging}. Some research studies have made specific assumptions about the conditional  $\frac{P(\D^{new})}{P(\D^{new}_X)}$ in order to monitor models and detect distribution shift~\cite{lu2023predicting,alvarez2023domain}.

\subsection{Analytical Experiments on Synthetic Data with Non-Linear Models}\label{app:exp.synthetic}

This experimental section explores the detection of distribution shifts in the previous synthetic examples but using non linear models.

\subsubsection{Detecting multivariate shift}\label{sec:multivariate}

Given two bivariate normal distributions $\D_X = (X_1,X_2) \sim  N\left(0,\begin{bmatrix}1 & 0 \\0& 1 \end{bmatrix}\right)$ and $\D_X^{new} = (X^{new}_1,X^{new}_2) \sim  N \left( 0,\begin{bmatrix}1 & 0.2 \\0.2 & 1 \end{bmatrix}\right)$, then, for each feature $j$ the underlying distribution is equally distributed between $\D_X$ and $\D_X^{new}$, $\forall j \in \{1,2\}: P(\D_{X_j}) = P(\D_{X_j}^{new})$, and what is different are the interaction terms between them. We now create a synthetic target $Y=X_1\cdot X_2 + \epsilon$ with $\epsilon \sim N(0,0.1)$ and fit a gradient boosting decision tree  $f_\theta(\D_X)$. Then we compute the SHAP explanation values for $\mathcal{S}(f_\theta,\D_X)$ and $\mathcal{S}(f_\theta,\D_X^{new})$

\begin{table}[ht]
\centering
\caption{Displayed results are the one-tailed p-values of the Kolmogorov-Smirnov test comparison between two underlying distributions. Small p-values indicate that compared distributions would be very unlikely  to be equally distributed. SHAP values correctly indicate the interaction changes that individual distribution comparisons cannot detect}\label{table:multivariate}

\begin{tabular}{c|cc}
Comparison                                 & \textbf{p-value} & \textbf{Conclusions} \\ \hline
$\PP(\D_{X_1})$, $\PP(\D_{X_1}^{new})$                        & 0.33                        & Not Distinct                         \\
$\PP(\D_{X_2})$, $\PP(\D_{X_2}^{new})$                        & 0.60                        & Not Distinct                          \\
$\Ss_1(f_\theta,\D_X)$, $\Ss_1(f_\theta,\D_X^{new})$ & $3.9\mathrm{e}{-153}$        & Distinct                              \\
$\Ss_2(f_\theta,\D_X)$, $\Ss_2(f_\theta,\D_X^{new})$ & $2.9\mathrm{e}{-148}$        & Distinct   
\end{tabular}
\end{table}

Having drawn $50,000$ samples from both $\D_X$ and $\D_X^{new}$, in Table~\ref{table:multivariate}, we evaluate whether changes in the input data distribution or in the explanations are able to detect changes in covariate distribution. For this, we compare the one-tailed p-values of the Kolmogorov-Smirnov test between the input data distribution and the explanations distribution.  Explanation shift correctly detects the multivariate distribution change that univariate statistical testing can not detect.

\subsubsection{Detecting concept shift}

As mentioned before, concept shift cannot be detected if new data comes without target labels. If new data is labelled, the explanation shift can still be a useful technique for detecting concept shifts.

Given a bivariate normal distribution  $\D_X = (X_1,X_2) \sim  N(1,I)$ where $I$ is an identity matrix of order two. We now create two synthetic targets $Y= X_1^2 \cdot X_2 + \epsilon$ and $Y^{new}=X_1 \cdot X_2^2 + \epsilon$ and fit two machine learning models $f_\theta:\D_X \rightarrow \D_Y)$ and $h_\Upsilon:\D_X \rightarrow \D_Y^{new})$. Now we compute the SHAP values for $\mathcal{S}(f_\theta,\D_X)$ and $\mathcal{S}(h_\Upsilon,\D_X)$
\begin{table}[ht]
\centering
\caption{Distribution comparison for synthetic concept shift. Displayed results are the one-tailed p-values of the Kolmogorov-Smirnov test comparison between two underlying distributions }\label{table:concept}
\begin{tabular}{c|c}
Comparison                                              & \textbf{Conclusions} \\ \hline
$\PP(\D_X)$, $\PP(\D_X^{new})$                                    & Not Distinct         \\
$\PP(\D_Y)$, $\PP(\D_Y^{new})$                                    & Not Distinct         \\
$\PP(f_\theta(\D_X))$, $\PP(h_\Upsilon(\D_X^{new}))$                  & Not Distinct         \\
$\PP(\Ss(f_\theta,\D_X)$), $\PP(\Ss(h_\Upsilon,\D_X))$                    & Distinct             \\
\end{tabular}
\end{table}

In Table~\ref{table:concept}, we see how the distribution shifts are not able to capture the change in the model behavior while the SHAP values are different. The \enquote{Distinct/Not distinct} conclusion is based on the one-tailed p-value of the Kolmogorov-Smirnov test with a $0.05$ threshold drawn out of $50,000$ samples for both distributions. As in the synthetic example, in table \ref{table:concept} SHAP values can detect a relational change between $\D_X$ and $\D_Y$, even if both distributions remain equivalent.

\subsubsection{Uninformative features on synthetic data}

To have an applied use case of the synthetic example from the methodology section, we create a three-variate normal distribution $\D_X = (X_1,X_2,X_3) \sim N(0,I_3)$, where $I_3$ is an identity matrix of order three. The target variable is generated  $Y=X_1\cdot X_2 + \epsilon$ being independent of $X_3$. For both, training and test data, $50,000$ samples are drawn. Then out-of-distribution data is created by shifting $X_3$, which is independent of the target, on test data $\D_{X_3}^{new}= \D_{X_3}^{te}+1$.

\begin{table}[ht]
\centering
\caption{Distribution comparison when modifying a random noise variable on test data. The input data shifts while explanations and predictions do not.}\label{table:unused}
\begin{tabular}{c|c}
Comparison                                              & \textbf{Conclusions} \\ \hline
$\PP(\D_{X_3}^{te})$, $\PP(\D_{X_3}^{new})$                                       & Distinct                \\
$f_\theta(\D_X^{te})$, $f_\theta(\D_X^{new})$                     & Not Distinct            \\
$\Ss(f_\theta,\D_X^{te})$, $\Ss(f_\theta,\D_X^{new})$                    & Not Distinct            \\
\end{tabular}
\end{table}

In Table~\ref{table:unused}, we see how an unused feature has changed the input distribution, but the explanation distributions and performance evaluation metrics remain the same. The \enquote{Distinct/Not Distinct} conclusion is based on the one-tailed p-value of the Kolmogorov-Smirnov test drawn out of $50,000$ samples for both distributions.

\subsubsection{Explanation shift that does not affect the prediction}

In this case we provide a situation when we have changes in the input data distributions that affect the model explanations but do not affect the model predictions due to positive and negative associations between the model predictions and the distributions cancel out producing a vanishing correlation in the mixture of the distribution (Yule's effect~\ref{sec:exp.vs.pred}).  

We create a train and test data by drawing $50,000$ samples from a bi-uniform distribution  $X_1 \sim U(0,1), \quad X_2 \sim U(1,2)$ the target variable is generated  by $Y = X_1+X_2$ where we train our model $f_\theta$. Then if out-of-distribution data is sampled from $X_1^{new}\sim U(1,2)$, $X_2^{new}\sim U(0,1)$

\begin{table}[ht]
\centering
\caption{Distribution comparison over how the change on the contributions of each feature can cancel out to produce an equal prediction (cf. Section \ref{sec:exp.vs.pred}), while explanation shift will detect this behaviour changes on the predictions will not.}\label{table:predShift}
\begin{tabular}{c|c}
Comparison                                              & \textbf{Conclusions} \\ \hline
$f(\D_X^{te})$, $f(\D_X^{new})$                                  & Not Distinct            \\
$\Ss(f_\theta,\D_{X_2}^{te})$, $\Ss(f_\theta,\D_{X_2}^{new})$                    & Distinct            \\
$\Ss(f_\theta,\D_{X_1}^{te})$, $\Ss(f_\theta,\D_{X_1}^{new})$                    & Distinct            \\
\end{tabular}
\end{table}

In Table~\ref{table:predShift}, we see how an unused feature has changed the input distribution, but the explanation distributions and performance evaluation metrics remain the same. The \enquote{Distinct/Not Distinct} conclusion is based on the one-tailed p-value of the Kolmogorov-Smirnov test drawn out of $50,000$ samples for both distributions.

\section{Empirical Evaluation}\label{sec:experiments}


We evaluate the effectiveness of explanation shift detection on tabular data by comparing it against methods from the literature, which are all based on discovering distribution shifts.
For this comparison, we systematically vary models $f$, model parametrizations $\theta$, and input data distributions $\D_X$.  We complement core experiments described in this section by adding further experimental results in the appendix that \emph{(i)} add details on experiments with synthetic data (Appendix~\ref{app:exp.synthetic}), \emph{(ii)} add experiments on further natural datasets (Appendix~\ref{app:real.data}), \emph{(iii)} exhibit a larger range of modeling choices (Appendix~\ref{app:modelingmethods}),and \emph{(iv)} contrast our SHAP-based method against the use of LIME, an alternative explanation approach (Appendix~\ref{app:LIME}). Core observations made in this section will only be confirmed and refined but not countered in the appendix.


\subsection{Baseline Methods and Datasets} 

\textbf{Baseline Methods.} We compare our method of explanation shift detection (Section~\ref{sec:Detector}) with several methods that aim to detect that input data is out-of-distribution: \emph{(B1)} statistical Kolmogorov Smirnov test on input data~\citep{DBLP:conf/nips/RabanserGL19}, \emph{(B2)} prediction shift detection by Wasserstein distance~\citep{lu2023predicting}, \emph{(B3)} NDCG-based test of feature importance between the two distributions~\citep{DBLP:conf/kdd/NigendaKZRTDK22}, \emph{(B4)} prediction shift detection by Kolmogorov-Smirnov test~\citep{continual_learning}, and \emph{(B5)} model agnostic uncertainty estimation~\citep{mougan2022monitoring,DBLP:conf/nips/KimXB20}. All 
Distribution Shift Metrics are scaled between 0 and 1. We also compare against Classifier Two-Sample Test~\citep{DBLP:conf/iclr/Lopez-PazO17} on different distributions as discussed in Section~\ref{sec:math.analysis}, viz.\   \emph{(B6)} classifier two-sample test on input distributions ($g_\phi$) following~\cite{barrabes2023adversarial} and \emph{(B7)} classifier two-sample test on the predictions distributions ($g_\Upsilon$): 
\begin{gather}
\phi = \argmin_{\tilde{\phi}} \sum_{x\in \D_X^{val}\cup \D_X^{new}} \ell( g_{\tilde{\phi}}(\textcolor{magenta}{x})), a_x )\\
\Upsilon = \argmin_{\tilde{\Upsilon}} \sum_{x\in \D_X^{val}\cup \D_X^{new}} \ell( g_{\tilde{\Upsilon}}(\textcolor{magenta}{f_\theta(x)}), a_x )
\end{gather}
\noindent\textbf{Datasets.} In the paper's main body, we base our comparisons on the UCI Adult Income dataset~\cite{Dua:2019} and on synthetic data. In the Appendix, we extend experiments to several other datasets, which confirm our findings: ACS Travel Time~\citep{ding2021retiring}, ACS Employment, Stackoverflow dataset~\citep{so2019}.



 
 
 

\subsection{Synthetic Data}\label{exp:synthetic}

Our first experiment on synthetic data showcases the two main contributions of our method: $(i)$ being more sensitive to changes in the model than prediction shift and input shift and $(ii)$ accounting for its drivers. We first generate a synthetic dataset $\D^{\rho}$, with a parametrized multivariate shift between $(X_1,X_2)$, where $\rho$ is the correlation coefficient, and an extra variable $X_3 = N(0,1)$ and generate our target $Y=X_1 \cdot X_2 + X_3$.
We train the $f_\theta$ on $\D^{tr,\rho=0}$ using a gradient boosting decision tree, while for $g_\psi : \Ss(f_\theta,\D_X^{val,\rho})\rightarrow \{0,1\},$ we train on different datasests with different values of $\rho$. For $g_\psi$ we use a logistic regression. In Appendix~\ref{app:modelingmethods}, we benchmark other models $f_\theta$ and detectors $g_\psi$.

\begin{figure*}[ht]
\centering
\includegraphics[width=.8\textwidth]{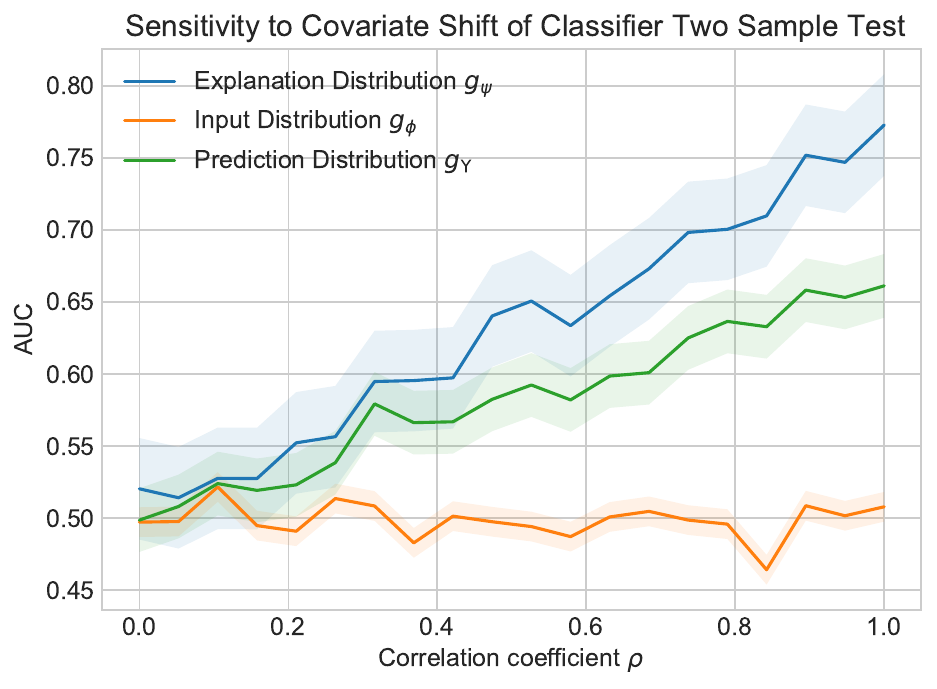}\hfill
\includegraphics[width=.8\textwidth]{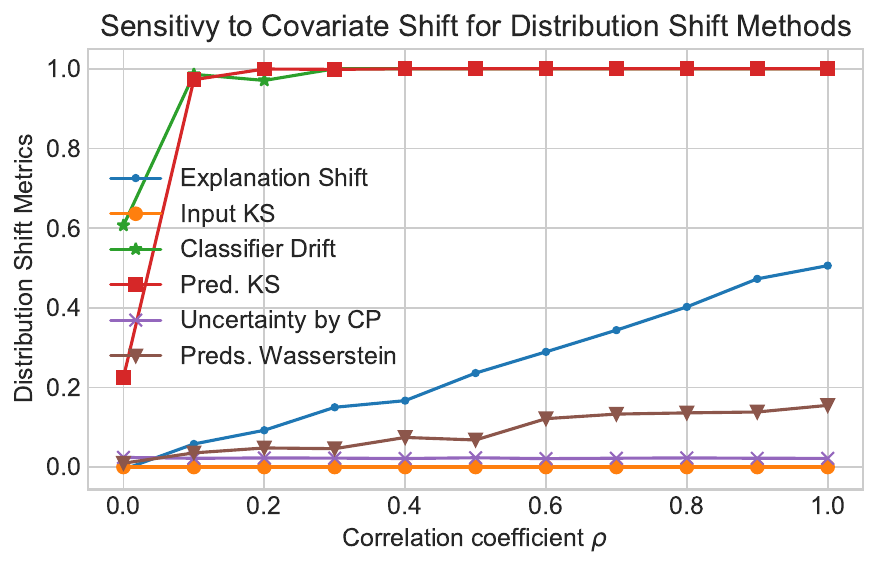}
\caption{In the figure above, we compare the Classifier Two-Sample Test on explanation distribution (ours) versus input distribution \emph{(B6)} and prediction distribution \emph{(B7)}. Explanation distribution shows the highest sensitivity. The figure below, related work comparison of distribution shift methods\emph{(B1-B5)}, as the experimental setup has a gradual distribution shift, good indicators should follow a progressive steady positive slope, following the correlation coefficient, as our method does. In Table~\ref{tab:corr.Synth} we provide a quantitative evaluation. }
\label{fig:sensitivity}
\end{figure*}

The above image in Figure \ref{fig:sensitivity} compares our approach against C2ST on input data distribution\emph{(B6)} and on the predictions distribution \emph{(B7)} different data distributions, for detecting multi-covariate shifts on different distributions. In our covariate experiment, we observed that using the explanation shift led to higher sensitivity towards detecting distribution shift. We interpret the results with the efficiency property of the Shapley values, which decomposes the vector $f_{\theta}(\D_X)$ into the matrix $\Ss(f_{\theta},\D_X)$. Moreover, we can identify the features that cause the drift by extracting the coefficients of $g_\psi$, providing global and local explainability.

The below image in Figure \ref{fig:sensitivity} features the same setup compared to the other out-of-distribution detection methods (B1-B5).  Table~\ref{tab:corr.Synth} quantitatively evaluates how the baselines correlate with the covariate correlation coefficient ($\rho$). One can see how our method behaves favourably compared to the others.

\begin{table}[ht]
    \centering
    \small
    \caption{Pearson Correlation of the correlation coefficient $\rho$ and baseline methods, extending Figure~\ref{fig:sensitivity}. Explanation Shift achieves better covariate shift detection on synthetic data.}

    \begin{tabular}{l c}
        \hline
        \textbf{Baseline} & \textbf{Pearson Correlation with $\rho$} \\
        \hline
        B1 Input KS & 0.01 \\
        B2 Prediction Wasserstein & 0.97 \\
        B3 Explanation NDCG & 0.52 \\
        B4 Prediction KS & 0.70 \\
        B5 Uncertainty & 0.26 \\
        B6 C2ST Input & 0.18 \\
        B7 C2ST Output & 0.96 \\
        (Ours) Explanation Shift & \textbf{0.99} \\
        \hline
    \end{tabular}
    \label{tab:corr.Synth}
\end{table}

\subsection{Novel Group Shift}\label{sec:sub:novel.covariate}

The distribution shift in this experimental setup is constituted by the appearance of a hitherto unseen group at prediction time (the group information is not present in the training features). We vary the ratio of presence of this unseen group in $\D_X^{new}$ data. The experiment is done with two $f_\theta$ models: a gradient-boosting decision tree and a logistic regression; for $g_\psi$, we use a logistic regression. Results are presented in Figure~\ref{fig:nove.group}, Figure~\ref{fig:shift.ndcg} and Table~\ref{tab:novel.group}.  Confidence intervals are extracted out of 10 bootstraps. Furthermore, we compare the performance of different algorithms for $f_\theta$ and $g_\psi$ in Appendix~\ref{app:varying}, and varying hyperparameters in Appendix~\ref{fig:shift.ndcg}.
\begin{figure*}[ht!]
\centering
\includegraphics[width=.6\textwidth]{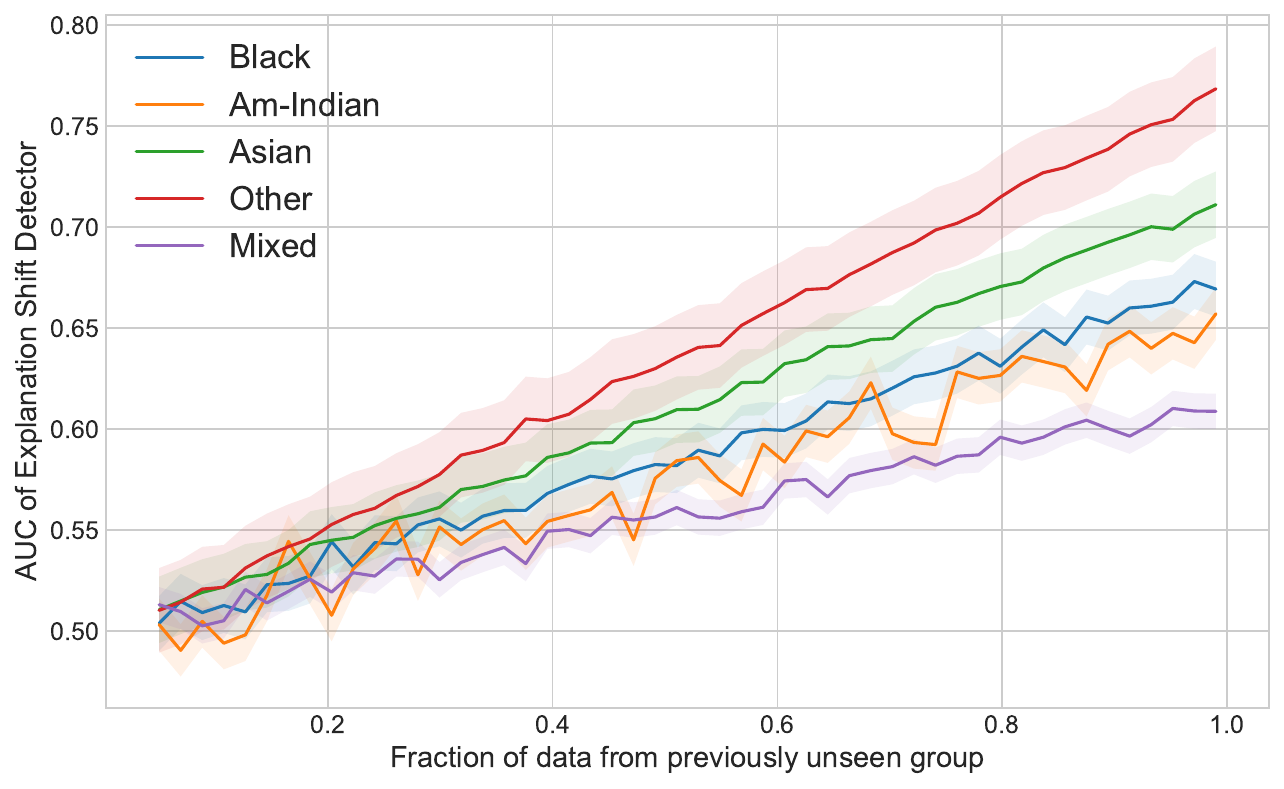}\hfill
\includegraphics[width=.6\textwidth]{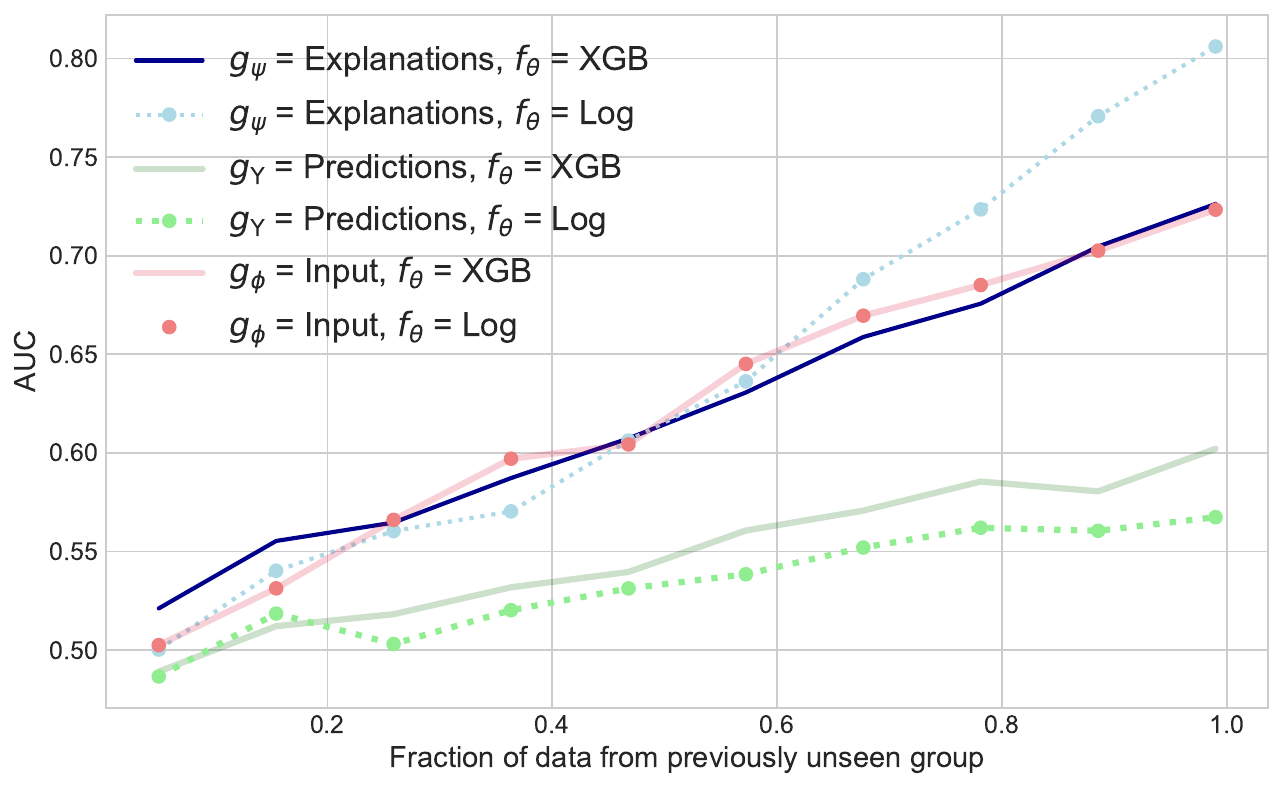}\hfill
\includegraphics[width=.6\textwidth]{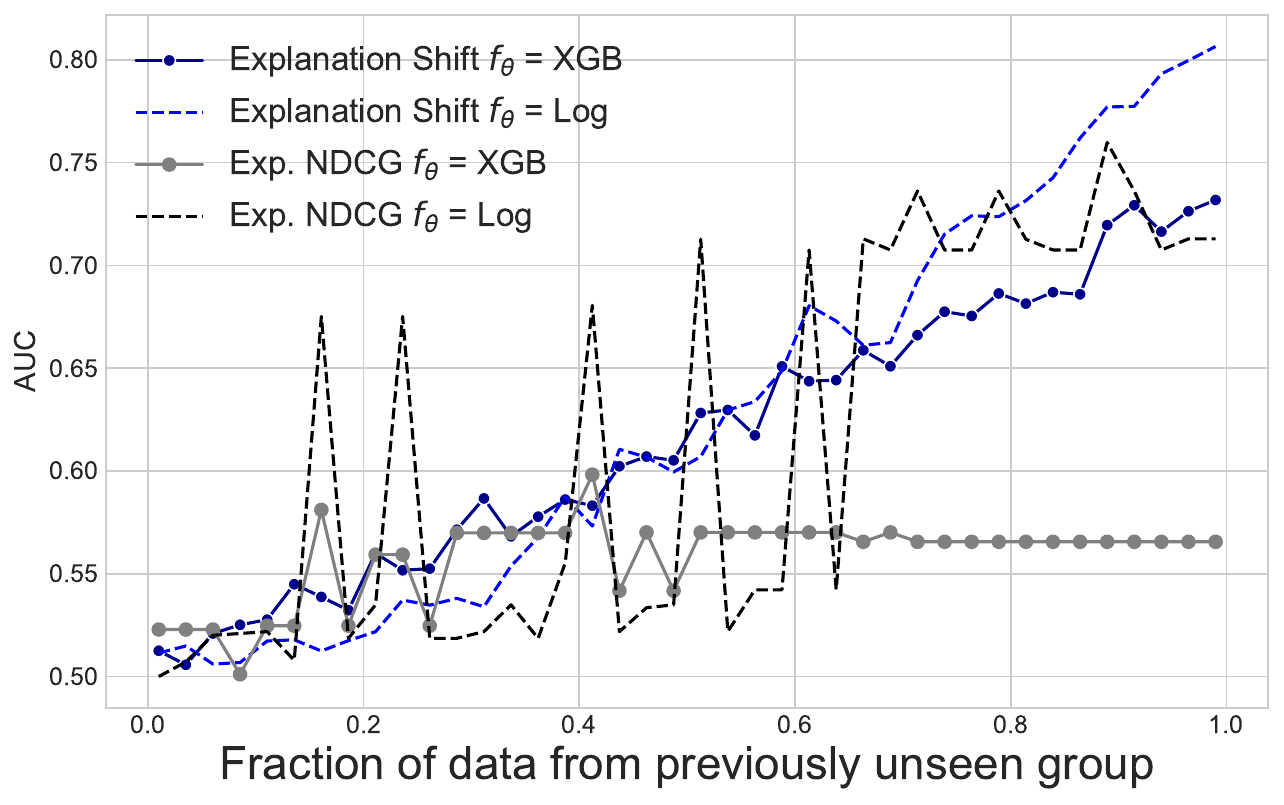}
\caption{Novel group shift experiment on the US Income dataset. Sensitivity (AUC) increases with the growing fraction of previously unseen social groups. \emph{Left figure}: The explanation shift indicates that different social groups exhibit varying deviations from the distribution on which the model was trained (White).
\emph{Middle Figure}: We vary the model $f_\theta$ by training it using both \texttt{XGBoost} (solid lines) and Logistic Regression (dots). The novel ethnicity group is Black. We compare Explanation Shift against C2ST on input \emph{(B6)} and output \emph{(B7)}. \emph{Right figure}: Comparison of Explanation Shift against Exp. NDCG  \emph{(B4)}.  We see how  monitoring method\emph{(B4)} is more unstable with a linear model, and with an \texttt{xgboost} it erroneously finds a horizontal asymptote. We don't compare against methods relying purely on input data such as \emph{(B1)} as we are changing the model, which they don't take into consideration.
}\label{fig:nove.group}
\end{figure*}

\begin{table}[ht]
\centering
\small
\caption{Pearson correlation between baselines and the ratio of presence of the unseen group. The \emph{Accountable} column relates to how well a method, like the Explanation Shift Detector, is underpinned by theoretical guarantees (e.g., Shapley value properties) and supported by empirical evidence (e.g., synthetic tests and real-world validations), as discussed in Section~\ref{sec:math.analysis}.}
\begin{tabular}{l|ccc}
\hline
\multirow{2}{*}{\textbf{Baseline}} & \multicolumn{2}{c}{\textbf{Pearson Correlation}} & \multirow{2}{*}{\textbf{Accountable}} \\
                                   & \textbf{$f_\theta = $ Log}    & \textbf{$f_\theta = $ XGB}   &                                       \\ \hline
B1 Input KS                        & $\mathbf{0.99\pm 0.01}$       & $\mathbf{0.99\pm 0.01}$      & \tikzxmark  \\ 
B2 Pred. Wass.                     & $0.95\pm 0.02$                & $\mathbf{0.98\pm 0.01}$      & \tikzxmark  \\ 
B3 NDCG                            & $0.37\pm 0.25$                & $0.81\pm 0.10$               & \tikzxmark  \\ 
B4 Pred. KS                        & $0.97\pm 0.02$                & $0.96\pm 0.01$               & \tikzxmark  \\ 
B5 Uncertainty                     & $0.73\pm 0.10$                & $0.74\pm 0.12$               & \tikzcmark  \\ 
B6 C2ST Input                      & $0.95\pm 0.03$                & $0.95\pm 0.03$               & \tikzxmark  \\ 
B7 C2ST Output                     & $0.67\pm 0.13$                & $0.96\pm 0.02$               & \tikzxmark  \\ 
Explanation Shift                  & $\mathbf{0.98\pm 0.01}$       & $\mathbf{0.98\pm 0.01}$      & \tikzcmark  \\ 
\end{tabular}\label{tab:novel.group}
\end{table}

\begin{figure*}[ht]
\centering
\includegraphics[width=.49\textwidth]{content/expShift/images/NewCategoryBenchmarkNDCGACSIncome.pdf}\hfill
\includegraphics[width=.49\textwidth]{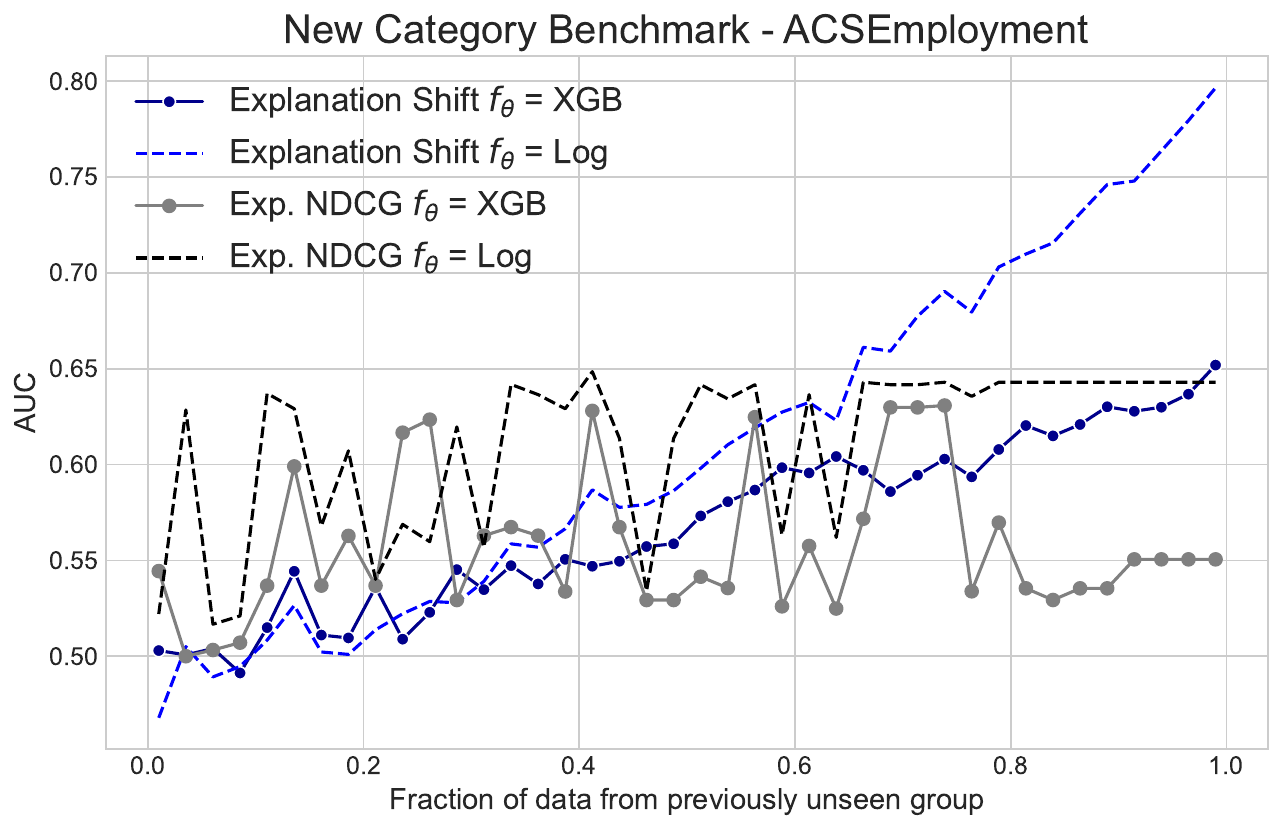}
\includegraphics[width=.49\textwidth]{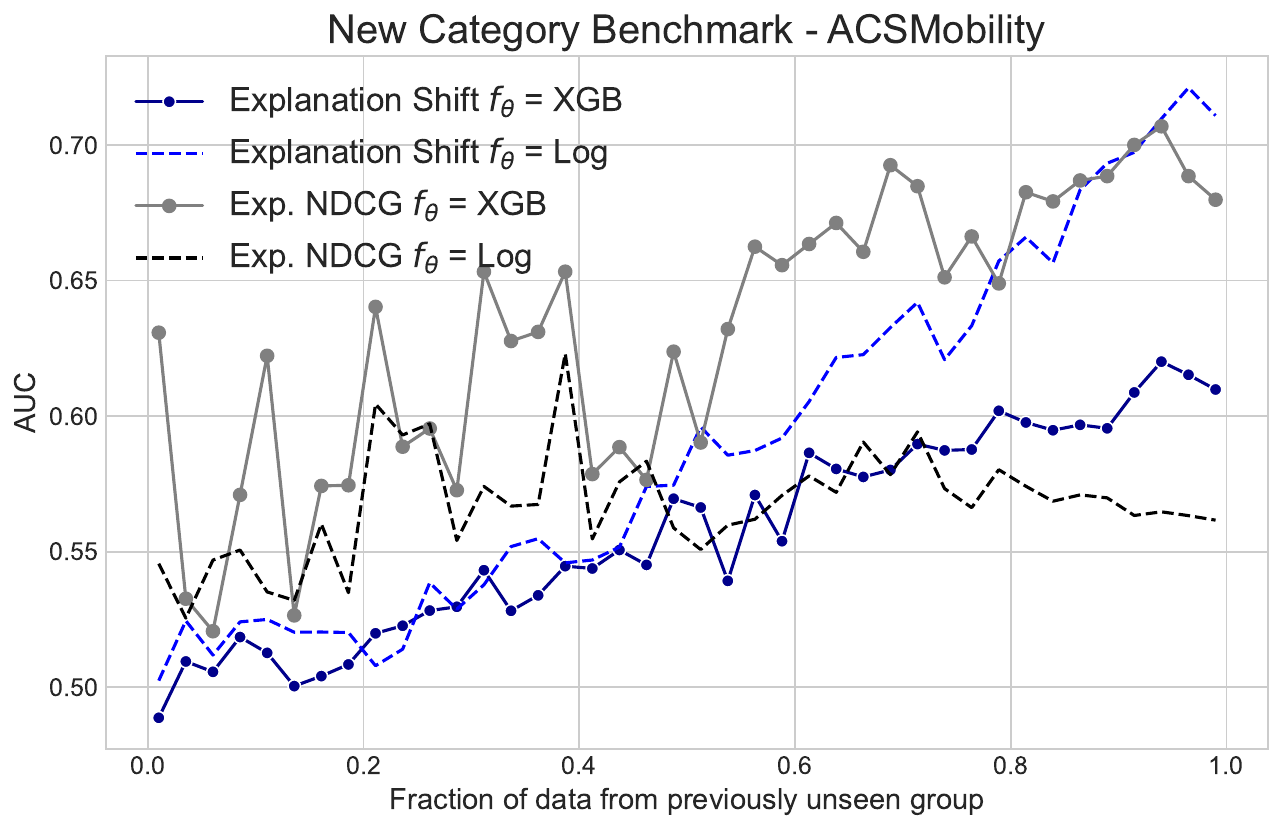}\hfill
\includegraphics[width=.49\textwidth]{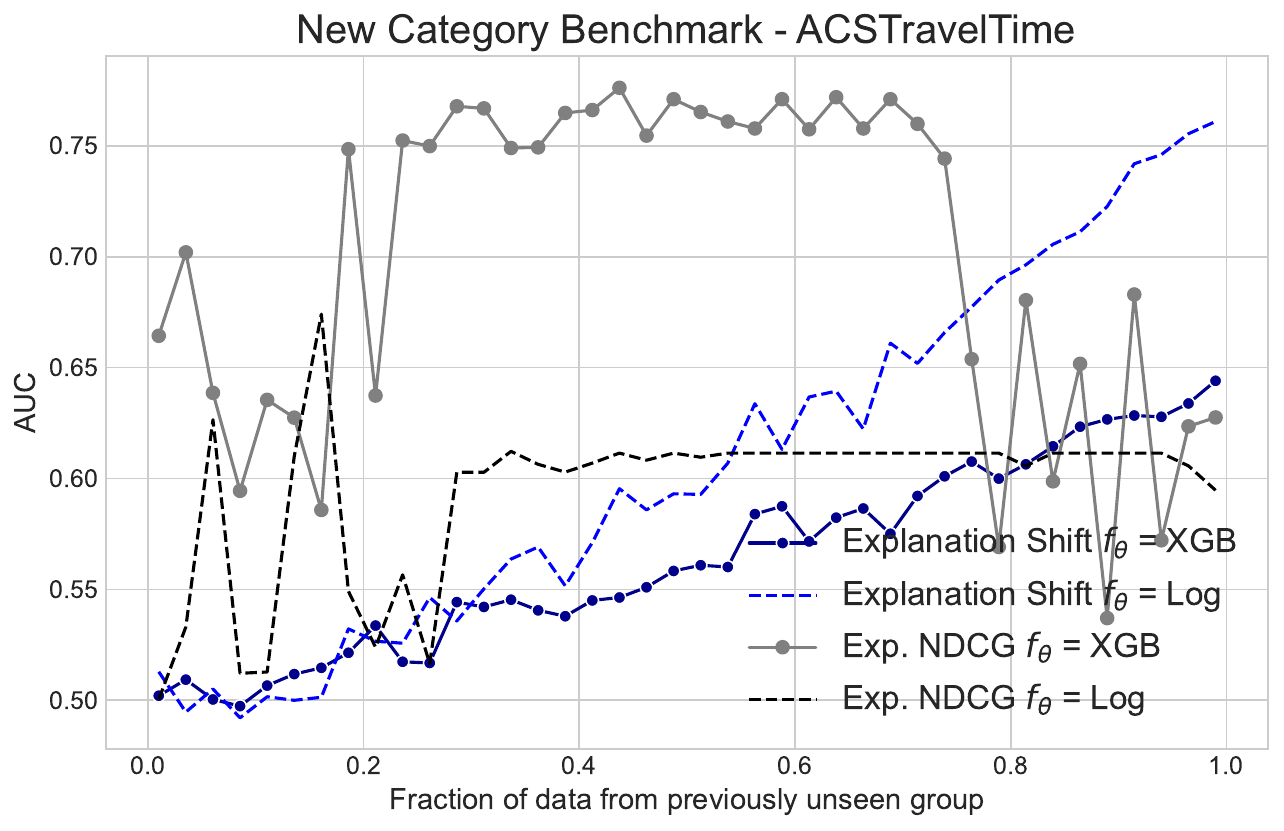}
\caption{Novel group shift experiment conducted on the 4 Datasets Comparison to NDCG. Sensitivity (AUC) increases as the proportion of previously unseen social groups grows.  As the experimental setup has a gradual distribution shift, ideal indicators should exhibit a steadily increasing slope. However, in all figures, NDCG exhibits saturation and instability. These observations align with the analysis presented in the synthetic experiment section, as discussed in Section~\ref{exp:synthetic} of the main paper}\label{fig:shift.ndcg}
\end{figure*}

This experiment further validates the effectiveness of the “Explanation Shift Detector” in addressing novel group changes in real-world datasets. The method demonstrates consistent performance across multiple datasets, providing valuable insights into the sensitivity of model behavior as previously unseen social groups constitute a larger portion of the prediction data. As shown in Figure~\ref{fig:shift.ndcg}, the proposed approach significantly outperforms Exp. NDCG ‘(B6)’ across four datasets. Unlike Exp. NDCG, which exhibits instability and frequently reaches horizontal asymptotes, the Explanation Shift Detector remains robust and adaptable. These limitations in Exp. NDCG arise from its reliance on feature importance rankings, which lack direct information about the value changes induced by the novel group. In contrast, the Explanation Shift Detector’s use of a Classifier Two-Sample Test on the distributions of explanations effectively captures these shifts, ensuring accurate and reliable detection of novel group effects on model behavior.

As a takeaway, the Explanation Shift Detector emerges as a highly reliable and accountable approach for detecting and adapting to novel group shifts. It outperforms traditional metrics like Exp. NDCG effectively captures distributional changes in explanations, making it an indispensable tool for robust model monitoring in diverse and evolving datasets.

\clearpage
\subsection{StackOverflow Survey Data}

In this extended experiment from novel covariate group, we evaluate the new unseen group in the new data over the StackOverflow dataset. As a training data country, we use the United States. The model $f_\theta$used is a gradient-boosting decision tree or logistic regression, and logistic regression is used for the detector. The results show that the AUC of the Explanation Shift Detector varies depending on the quantification of OOD explanations, and it shows more sensitivity concerning model variations than other state-of-the-art techniques.

The dataset used is the StackOverflow annual developer survey, with over 70,000 responses from over 180 countries examining aspects of the developer experience~\citep{so2019}. The data has high dimensionality, leaving it with $+100$ features after data cleansing and feature engineering. The goal of this task is to predict the total annual compensation.

\begin{figure}[ht]
\centering
\includegraphics[width=.49\textwidth]{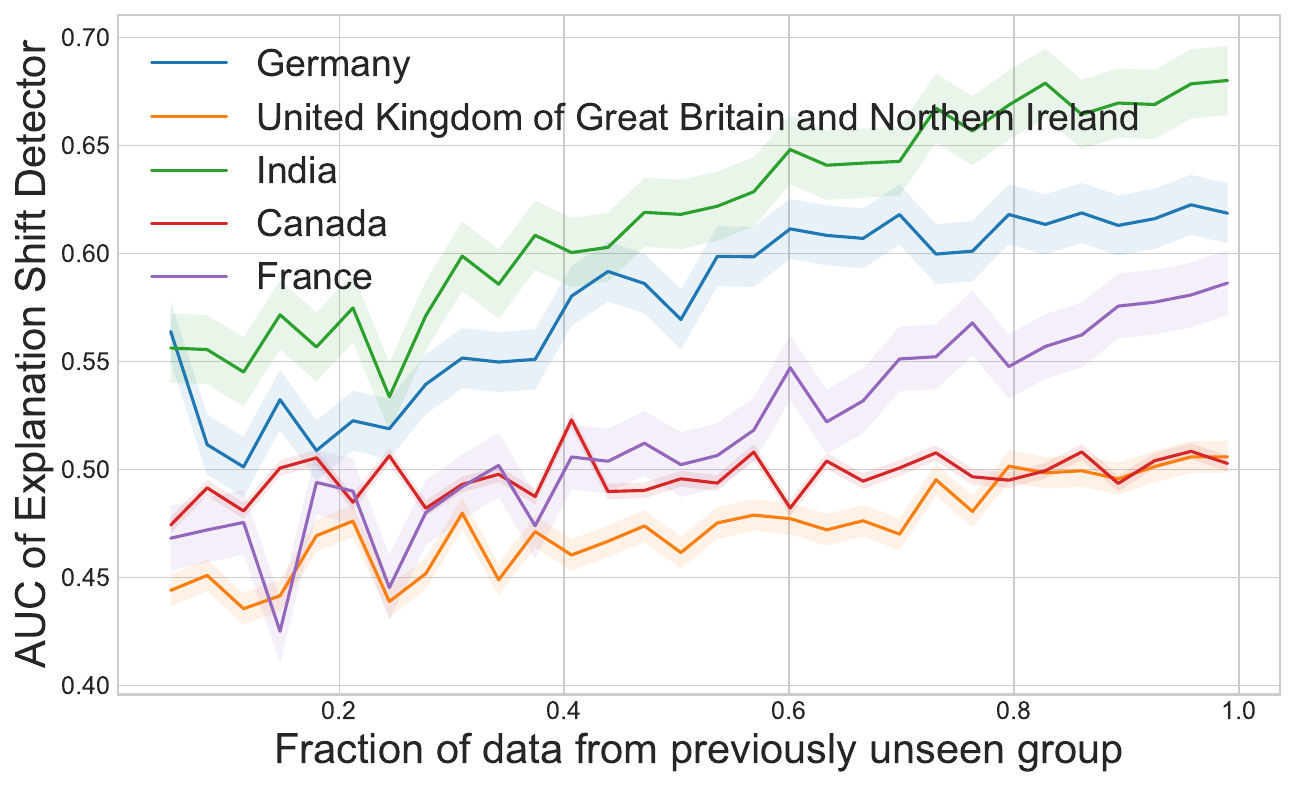}\hfill
\includegraphics[width=.49\textwidth]{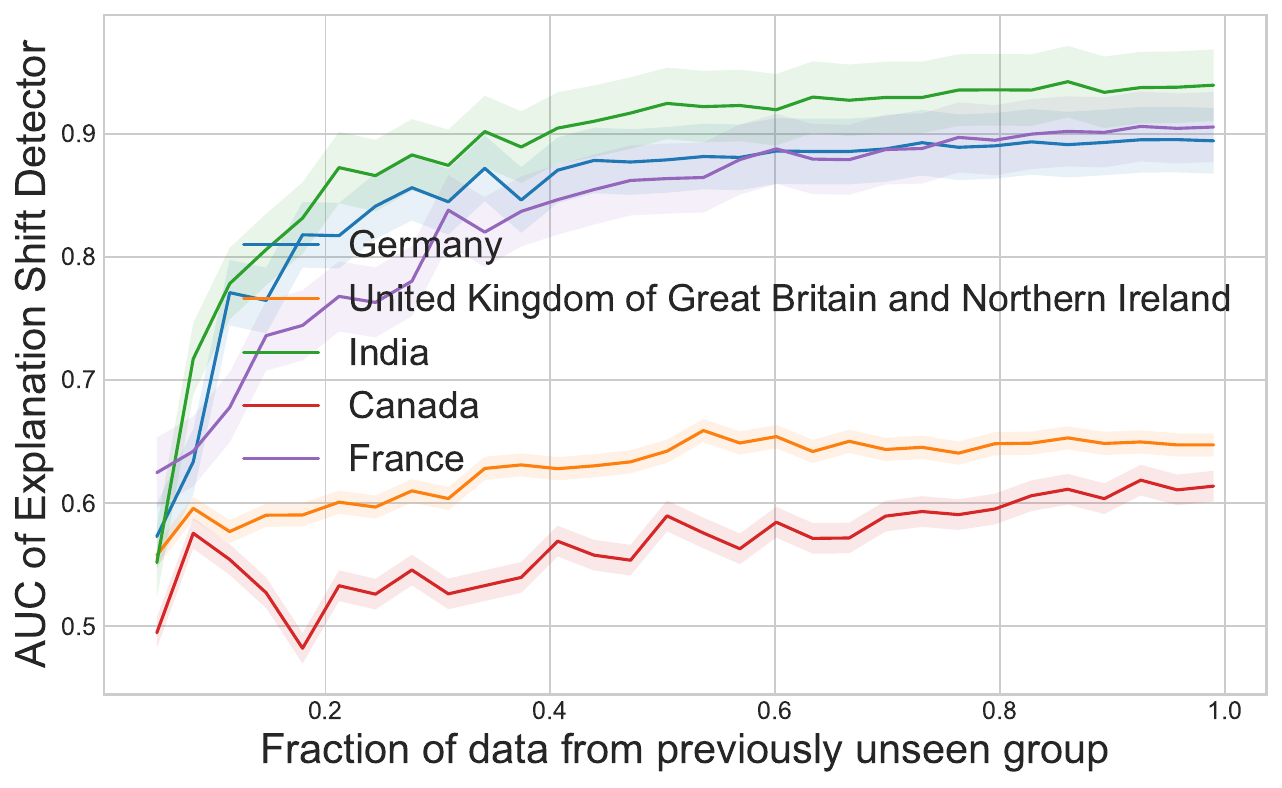}
\caption{Both images represent the AUC of the \textit{Explanation Shift Detector} for different countries on the StackOverflow survey dataset under novel group shift. In the left image, the estimator, $f_\theta$, is a gradient boosting decision tree; in the right image, for both cases the detector, $g_\psi$, is a logistic regression. By changing the type of estimator model, we can see how different types of models are affected differently for the same distribution shift}
\label{fig:xai.hyperSO.shift}
\end{figure}

\subsection{Geopolitical and Temporal Shift}\label{exp:geopolitical}
\begin{figure*}[ht]
\centering
\includegraphics[width=.8\textwidth]{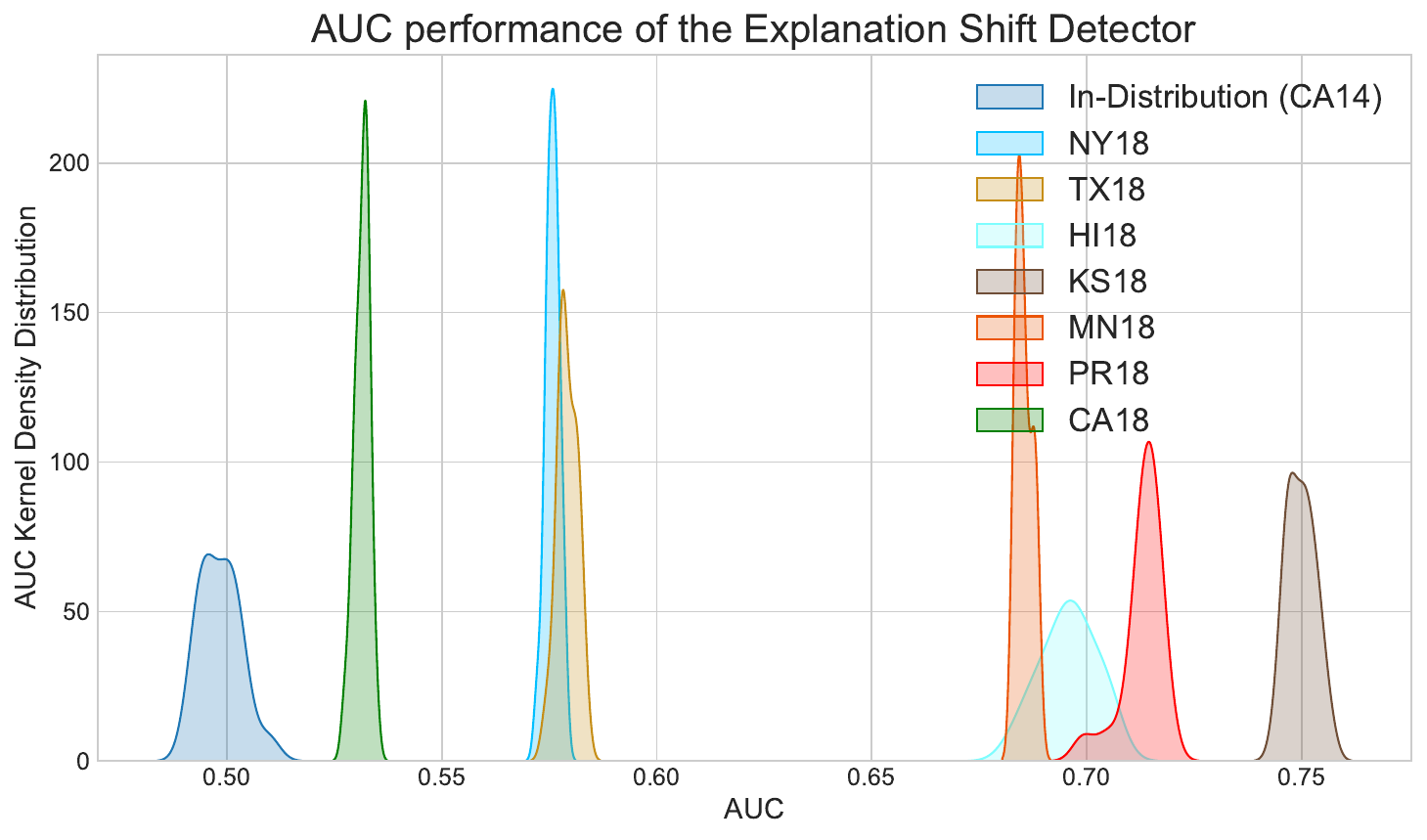}\hfill
\includegraphics[width=.8\textwidth]{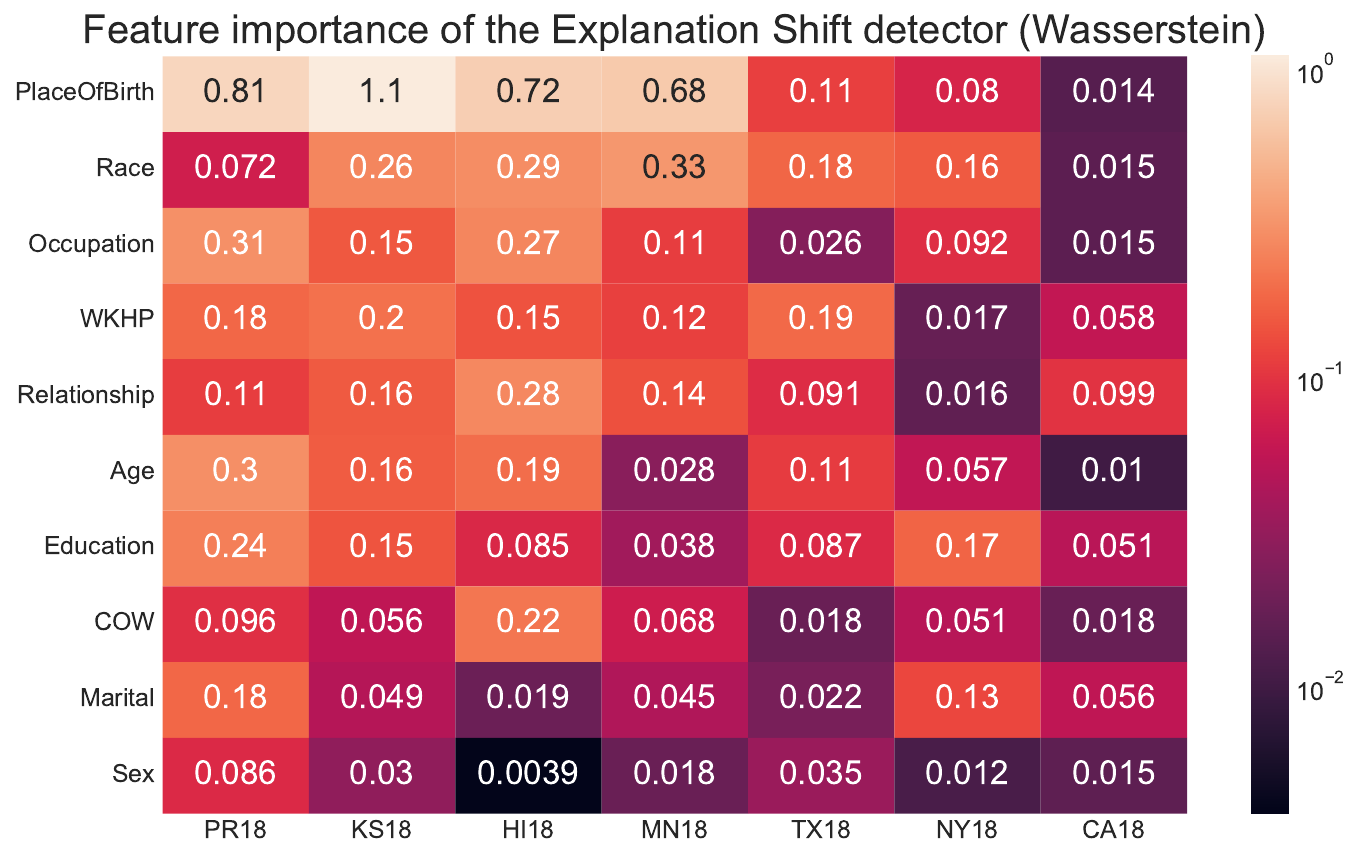}
\caption{In the above figure, a comparison of the performance of \textit{Explanation Shift Detector} in different states. In the below figure, the strength analysis of features driving the change in the model, on the y-axis are the features, and on the x-axis are the different states.  Explanation shifts allow us to identify how the distribution shift of different features impacted the model.
}
\label{fig:xai.income.shift}
\end{figure*}
In this section, we tackle a geopolitical and temporal distribution shift; for this, the training data $\D^{tr}$ for the model $f_\theta$ is composed of data from  California in 2014 and a $\D^{new}$ for each of the states in 2018. The model $g_\psi$ is trained each time on each state using only the $\D_X^{new}$ in the absence of the label, and a 50/50 random train-test split evaluates its performance. As models, we use \texttt{xgboost} as $f_\theta$ and logistic regression for the \textit{Explanation Shift Detector} $(g_\psi)$.

We hypothesize that the AUC of the Explanation Shift Detector on new data will be distinct from that on in-distribution data, primarily owing to the distinctive nature of out-of-distribution model explanations. Figure~\ref{fig:xai.income.shift} illustrates the performance of our method on different data distributions, where the baseline is a ID hold-out set of CA14. The AUC for CA18, where there is only a temporal shift, is the closest to the baseline, and the OOD detection performance is better in the rest of the states. The most disparate state is Puerto Rico (PR18).

Our next objective is to identify the features where the explanations differ between $\D_X^{tr}$ and  $\D_X^{new}$ data. To achieve this, we compare the distribution of linear coefficients of the detector between both distributions.
We use the Wasserstein distance as a distance measure, generating 1000 in-distribution bootstraps using a $63.2\%$ sampling fraction from California-14 and 1000 bootstraps from other states in 2018. In the below image of Figure~\ref{fig:xai.income.shift}, we observe that for PR18, the most crucial feature is the Place of Birth.

Furthermore, we conduct an across-task evaluation by comparing the performance of the \enquote{Explanation Shift Detector} on another prediction task in the Appendix~\ref{app:real.data}. Although some features are present in both prediction tasks, the weights and importance order assigned by the "Explanation Shift Detector" differ. One of this method's advantages is that it identifies differences in distributions and how they relate to the model.

\subsubsection{Further Experiments on Real Data}\label{app:real.data}
In this section, we extend the prediction task of the main body of the paper.  The methodology used follows the same structure. We start by creating a distribution shift by training the model $f_\theta$ in California in 2014 and evaluating it in the rest of the states in 2018, creating a geopolitical and temporal shift. The model $g_\theta$ is trained each time on each state using only the $X^{New}$ in the absence of the label, and its performance is evaluated by a 50/50 random train-test split. As models, we use a gradient boosting decision tree\citep{xgboost,catboost} for $f_\theta$, approximating the Shapley values by TreeExplainer \citep{lundberg2020local2global}, and using logistic regression for the \textit{Explanation Shift Detector}.

For further understanding of the meaning of the features the ACS PUMS data dictionary contains a comprehensive list of available variables \url{https://www.census.gov/programs-surveys/acs/microdata/documentation.html}.

\paragraph{ACS Employment}
The objective of this task is to determine whether an individual aged between 16 and 90 years is employed or not. The model's performance was evaluated using the AUC metric in different states, except PR18, where the model showed an explanation shift. The explanation shift was observed to be influenced by features such as Citizenship and Military Service. The performance of the model was found to be consistent across most of the states, with an AUC below 0.60. The impact of features such as difficulties in hearing or seeing was negligible in the distribution shift impact on the model. The left figure in Figure \ref{fig:xai.employment.shift} compares the performance of the Explanation Shift Detector in different states for the ACS Employment dataset.

\begin{figure*}[ht]
\centering
\includegraphics[width=.8\textwidth]{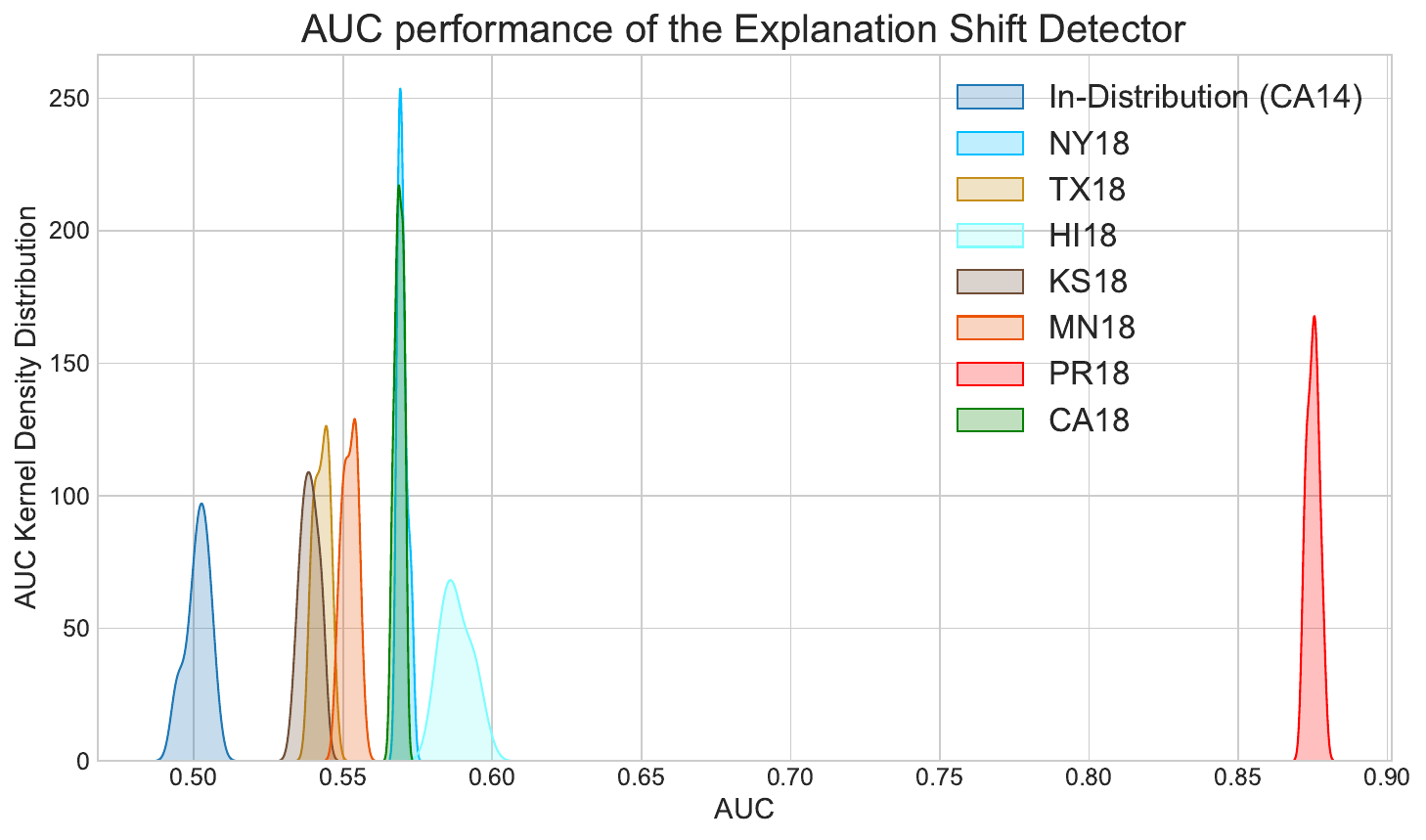}\hfill
\includegraphics[width=.8\textwidth]{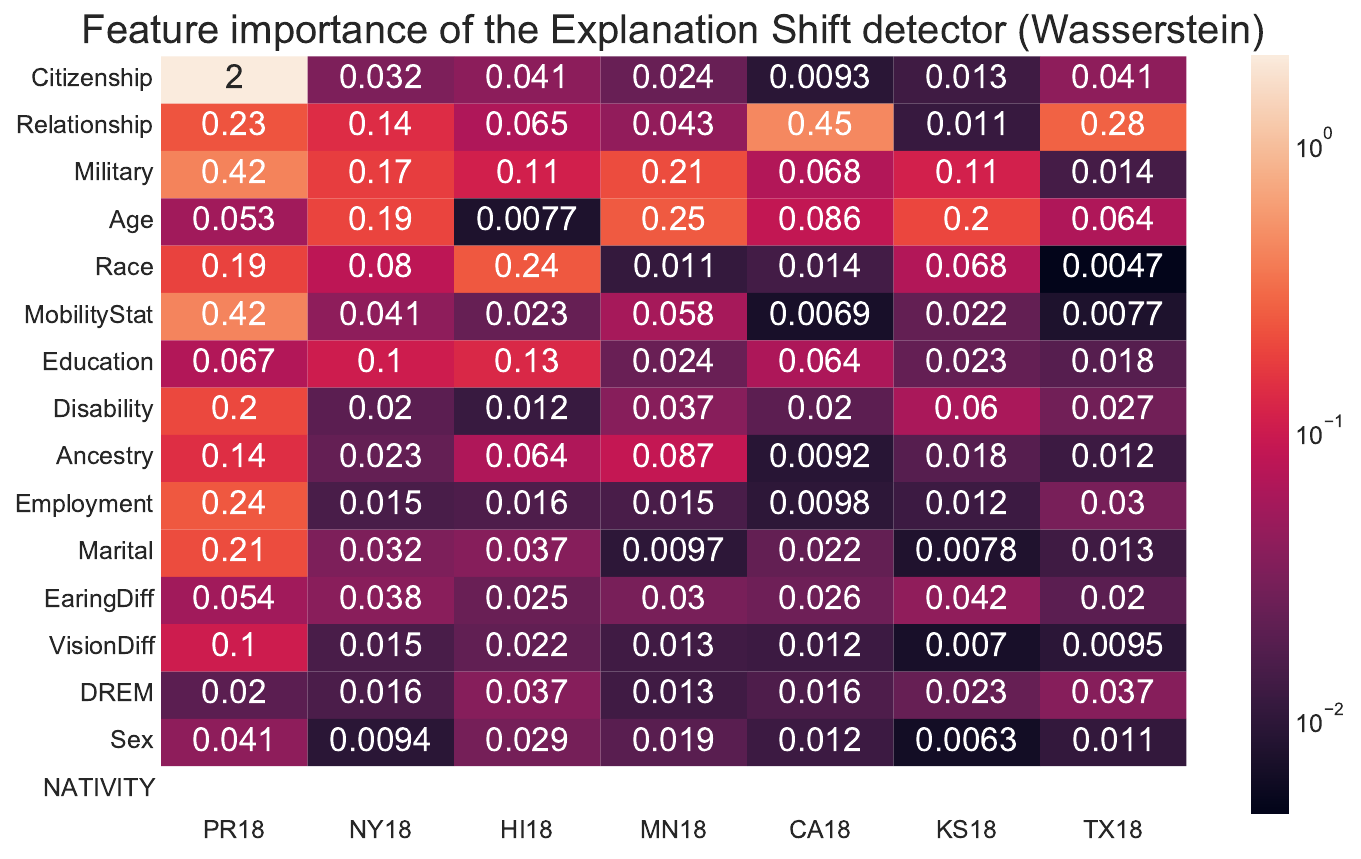}
\caption{The above figure shows a comparison of the performance of the Explanation Shift Detector in different states for the ACS Employment dataset. The below figure shows the feature importance analysis for the same dataset.}
\label{fig:xai.employment.shift}
\end{figure*}

Additionally, the feature importance analysis for the same dataset is presented in the below figure in Figure \ref{fig:xai.employment.shift}.

\paragraph{ACS Travel Time}

The goal of this task is to predict whether an individual has a commute to work that is longer than $+20$ minutes. For this prediction task, the results are different from the previous two cases; the state with the highest OOD score is $KS18$, with the \enquote{Explanation Shift Detector} highlighting features as Place of Birth, Race or Working Hours Per Week. The closest state to ID is CA18, where there is only a temporal shift without any geospatial distribution shift.
\begin{figure*}[ht]
\centering
\includegraphics[width=.8\textwidth]{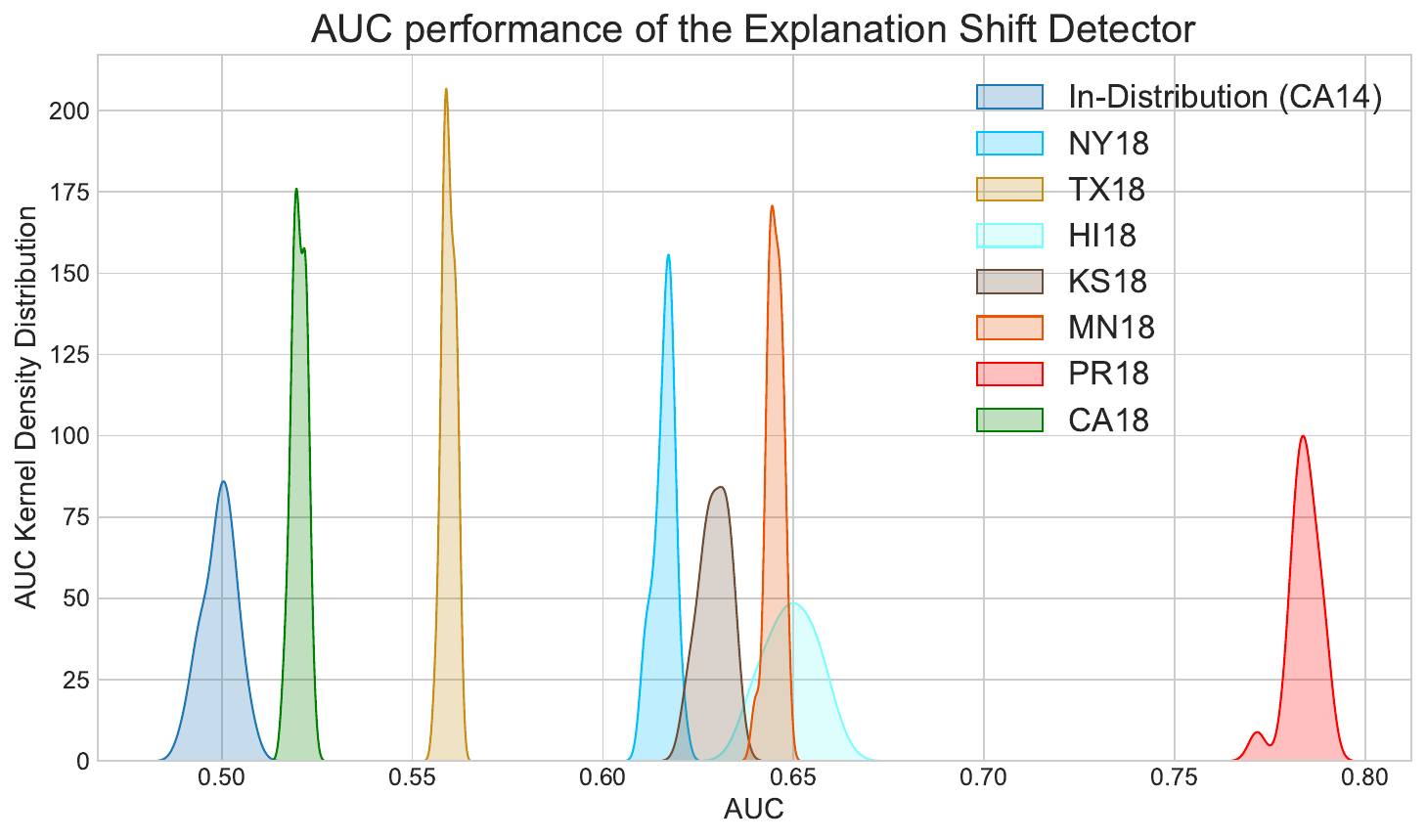}\hfill
\includegraphics[width=.8\textwidth]{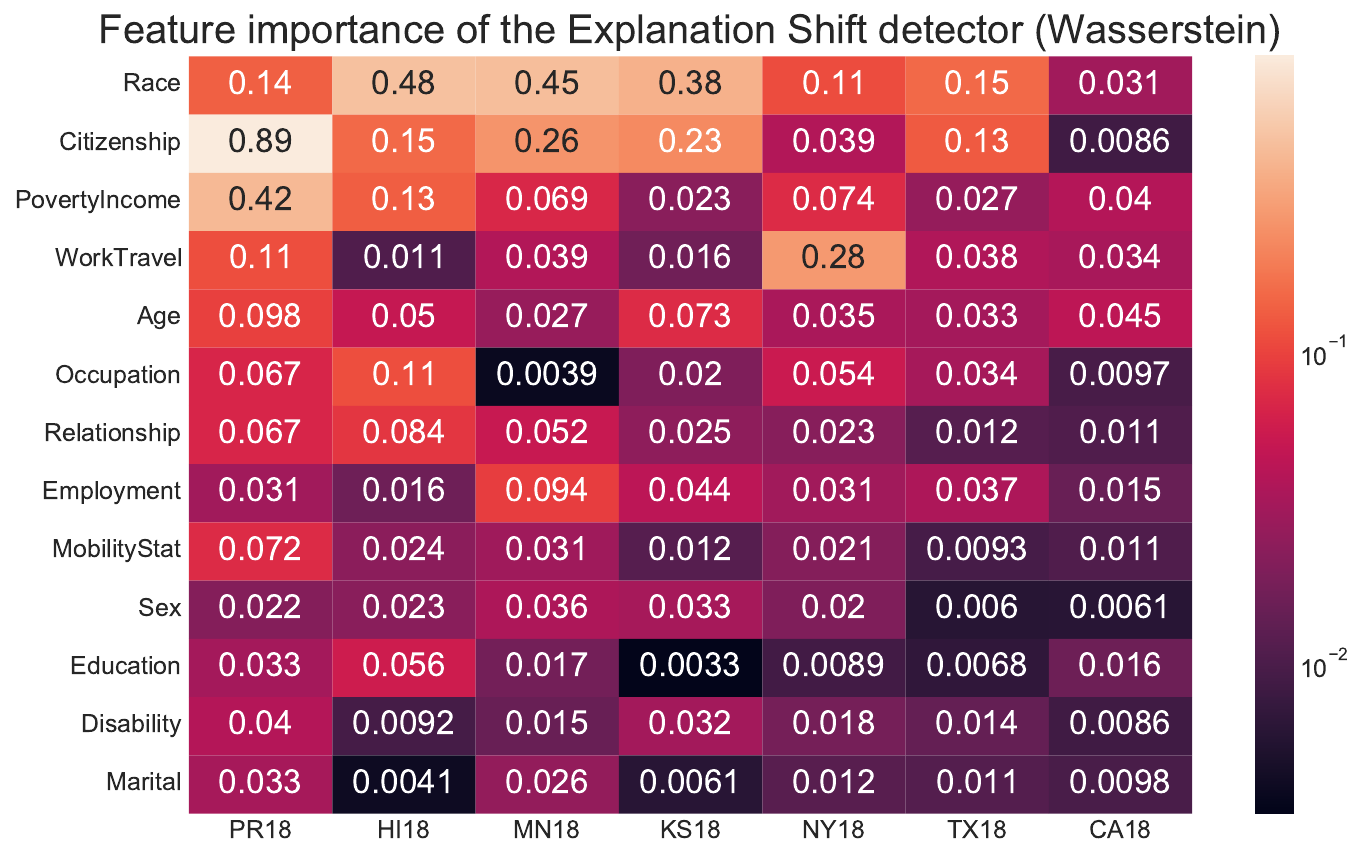}
\caption{In the left figure, comparison of the performance of \textit{Explanation Shift Detector}, in different states for the ACS TravelTime prediction task.  In the left figure, we can see how the state with the highest OOD AUC detection is KS18 and not PR18 as in other prediction tasks; this difference with respect to the other prediction task can be attributed to \enquote{Place of Birth}, whose feature attributions the model finds to be more different than in CA14.}
\label{fig:xai.traveltime.shift}
\end{figure*}

\paragraph{ACS Mobility}
The objective of this task is to predict whether an individual between the ages of 18 and 35 had the same residential address as a year ago. This filtering is intended to increase the difficulty of the prediction task, as the base rate for staying at the same address is above $90\%$ for the population \citep{ding2021retiring}.

The experiment shows a similar pattern to the ACS Income prediction task (cf. Section \ref{fig:xai.income.shift}), where the inland US states have an AUC range of $0.55-0.70$, while the state of PR18 achieves a higher AUC. For PR18, the model has shifted due to features such as Citizenship, while for the other states, it is Ancestry (Census record of your ancestors' lives with details like where they lived, who they lived with, and what they did for a living) that drives the change in the model.

As depicted in Figure \ref{fig:xai.mobility.shift}, all states, except for PR18, fall below an AUC of explanation shift detection of $0.70$. Protected social attributes, such as Race or Marital status, play an essential role for these states, whereas for PR18, Citizenship is a key feature driving the impact of distribution shift in model.

\begin{figure*}[ht]
\centering
\includegraphics[width=.8\textwidth]{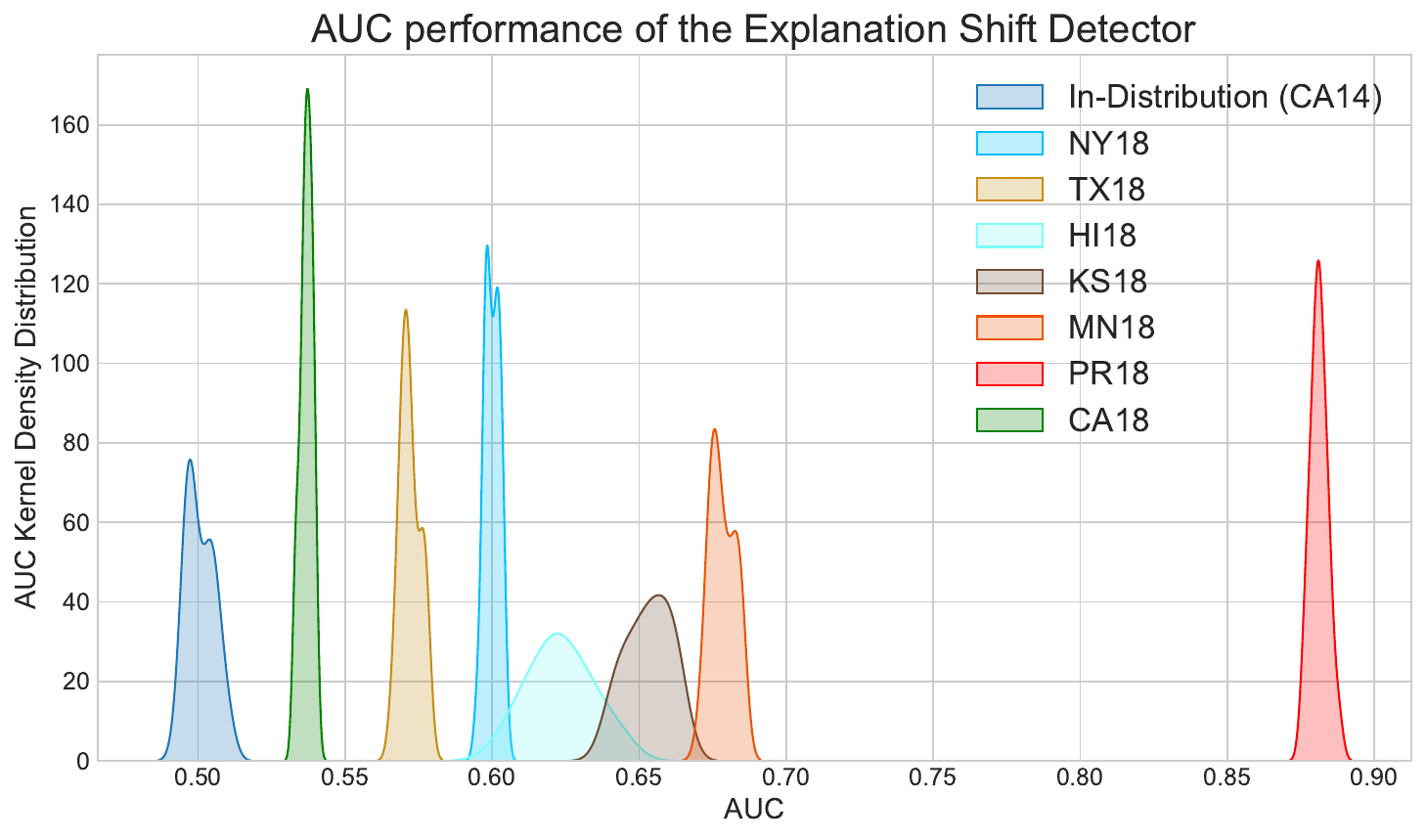}\hfill
\includegraphics[width=.8\textwidth]{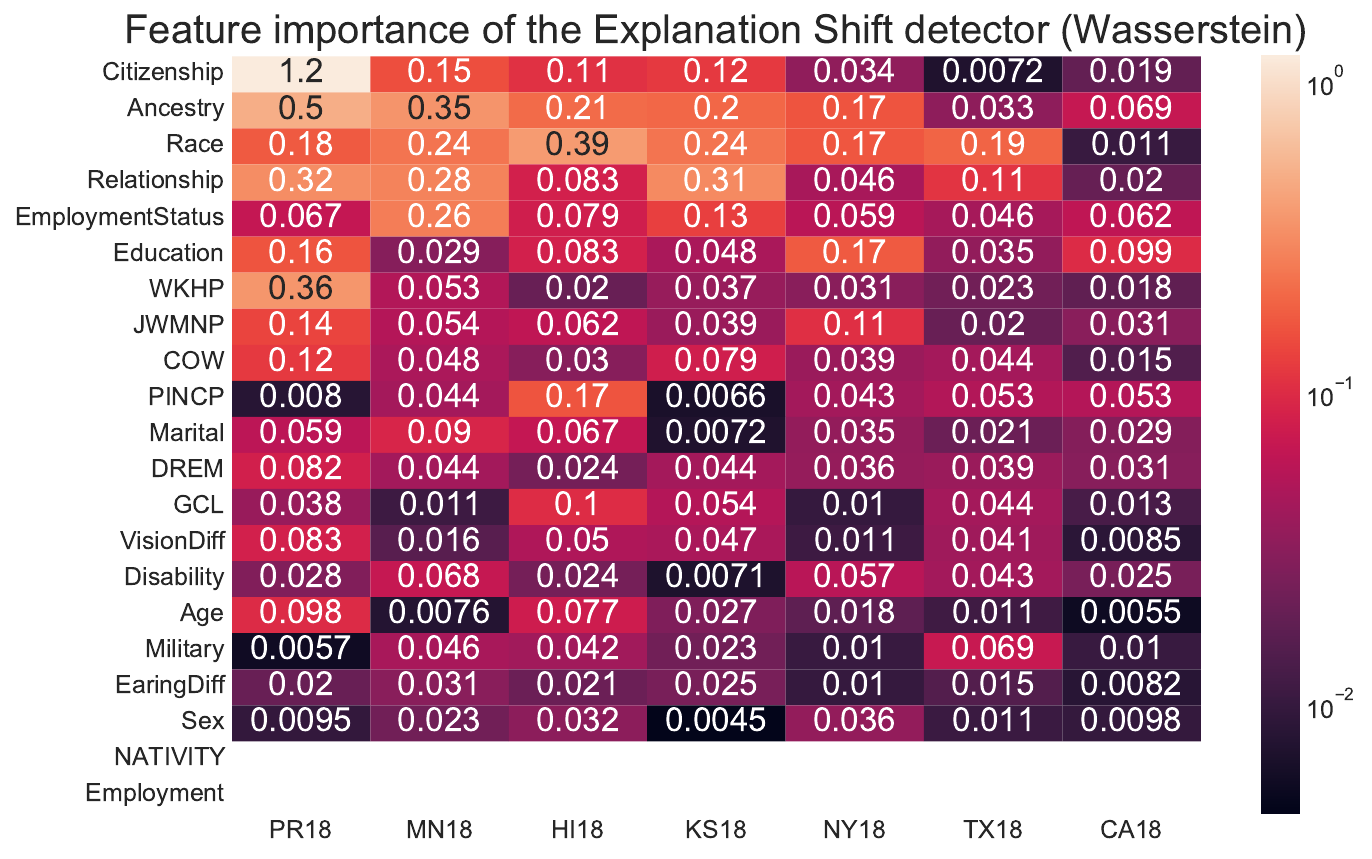}
\caption{Left figure shows a comparison of the \textit{Explanation Shift Detector}'s performance in different states for the ACS Mobility dataset. Except for PR18, all other states fall below an AUC of explanation shift detection of $0.70$. The features driving this difference are Citizenship and Ancestry relationships. For the other states, protected social attributes, such as Race or Marital status, play an important role.}
\label{fig:xai.mobility.shift}
\end{figure*}

\clearpage

\section{Experiments with Modeling Methods and Hyperparameters}\label{app:modelingmethods}
In the next sections, we are going to show the sensitivity or our method to variations of  the model $\textcolor{magenta}{f}$, the detector $\textcolor{magenta}{g}$, and the parameters of the estimator $f\textcolor{magenta}{_\theta}$. 

As an experimental setup, In the main body of the paper, we have focused on the UCI Adult Income dataset. The experimental setup has been using Gradient Boosting Decision Tree as the model $f_\theta$ and then as \enquote{Explanation Shift Detector} $g_\psi$ a logistic regression. In this section, we extend the experimental setup by providing experiments by varying the types of algorithms for a given experimental set-up: the UCI Adult Income dataset using the Novel Covariate Group Shift for the \enquote{Asian} group with a fraction ratio of $0.5$ (cf. Section \ref{sec:experiments}).

\subsection{Varying Models and Explanation Shift Detectors}\label{app:varying}

OOD data detection methods based on input data distributions only depend on the type of detector used, being independent of the model $f_\theta$. OOD Explanation methods rely on both the model and the data. Using explanations shifts as indicators for measuring distribution shifts impact on the model enables us to account for the influencing factors of the explanation shift. Therefore, in this section, we compare the performance of different types of algorithms for explanation shift detection  using the same experimental setup. The results of our experiments show that using Explanation Shift enables us to see differences in the choice of the original model $f_\theta$ and the Explanation Shift Detector $g_\phi$

\begin{table}[ht]
\begin{tabular}{cccccccc}
\multicolumn{1}{l}{}                       & \multicolumn{7}{c}{Estimator $f_\theta$}                                                                                                         \\ \cline{2-8} 
\multicolumn{1}{c|}{Detector $g_\phi$}              & \textbf{XGB} & \textbf{Log.Reg} & \textbf{Lasso} & \textbf{Ridge} & \textbf{Rand.Forest} & \textbf{Dec.Tree} & \textbf{MLP} \\ \hline
\multicolumn{1}{c|}{\textbf{XGB}}          & 0.583        & 0.619                 & 0.596          & 0.586          & 0.558                 & 0.522                 & 0.597        \\
\multicolumn{1}{c|}{\textbf{LogisticReg.}} & 0.605        & 0.609                 & 0.583          & 0.625          & 0.578                 & 0.551                 & 0.605        \\
\multicolumn{1}{c|}{\textbf{Lasso}}        & 0.599        & 0.572                 & 0.551          & 0.595          & 0.557                 & 0.541                 & 0.596        \\
\multicolumn{1}{c|}{\textbf{Ridge}}        & 0.606        & 0.61                  & 0.588          & 0.624          & 0.564                 & 0.549                 & 0.616        \\
\multicolumn{1}{c|}{\textbf{RandomForest}} & 0.586        & 0.607                 & 0.574          & 0.612          & 0.566                 & 0.537                 & 0.611        \\
\multicolumn{1}{c|}{\textbf{DecisionTree}} & 0.546        & 0.56                  & 0.559          & 0.569          & 0.543                 & 0.52                  & 0.569       
\end{tabular}
\caption{Comparison of explanation shift detection performance, measured by AUC, for different combinations of explanation shift detectors and estimators on the UCI Adult Income dataset using the Novel Covariate Group Shift for the \enquote{Asian} group with a fraction ratio of $0.5$ (cf. Section \ref{sec:experiments}). The table shows that the algorithmic choice for $f_\theta$ and $g_\psi$ can impact the OOD explanation performance.  We can see how, for the same detector, different $f_\theta$  models flag different OOD explanations performance. On the other side, for the same $f_\theta$ model, different detectors achieve different results.}\label{tab:benchmark}
\end{table}

\subsection{Hyperparameters Sensitivity Evaluation}\label{app:hyperparameter.shift}
This section presents an extension to our experimental setup where we vary the model complexity by varying the model hyperparameters $\Ss(f\textcolor{magenta}{_\theta},X)$. We use the UCI Adult Income dataset with the Novel Covariate Group Shift for the \enquote{Asian} group with a fraction ratio of $0.5$ as described in Section \ref{sec:experiments}. And for the Stackoverflow as training data we use the United States of America and a novel covariate group France.

In this experiment, we changed the hyperparameters of the original model: for the decision tree, we varied the depth of the tree, while for the gradient-boosting decision,  we changed the number of estimators, and for the random forest, both hyperparameters. We calculated the Shapley values using TreeExplainer \citep{lundberg2020local2global}. For the Detector choice of model, we compare Logistic Regression and XGBoost models.

\begin{figure}[ht]
\centering
\includegraphics[width=.8\textwidth]{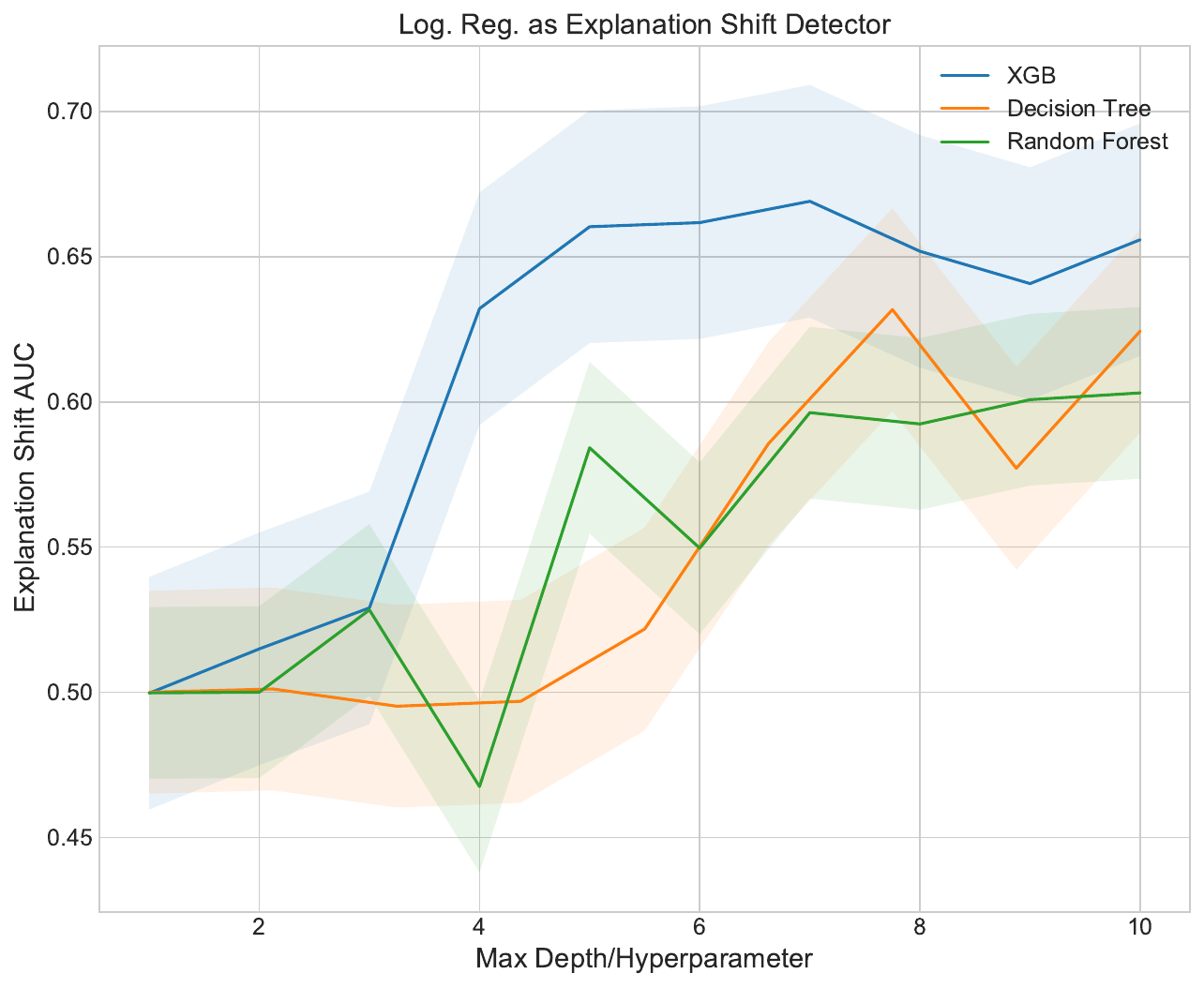}\hfill
\includegraphics[width=.8\textwidth]{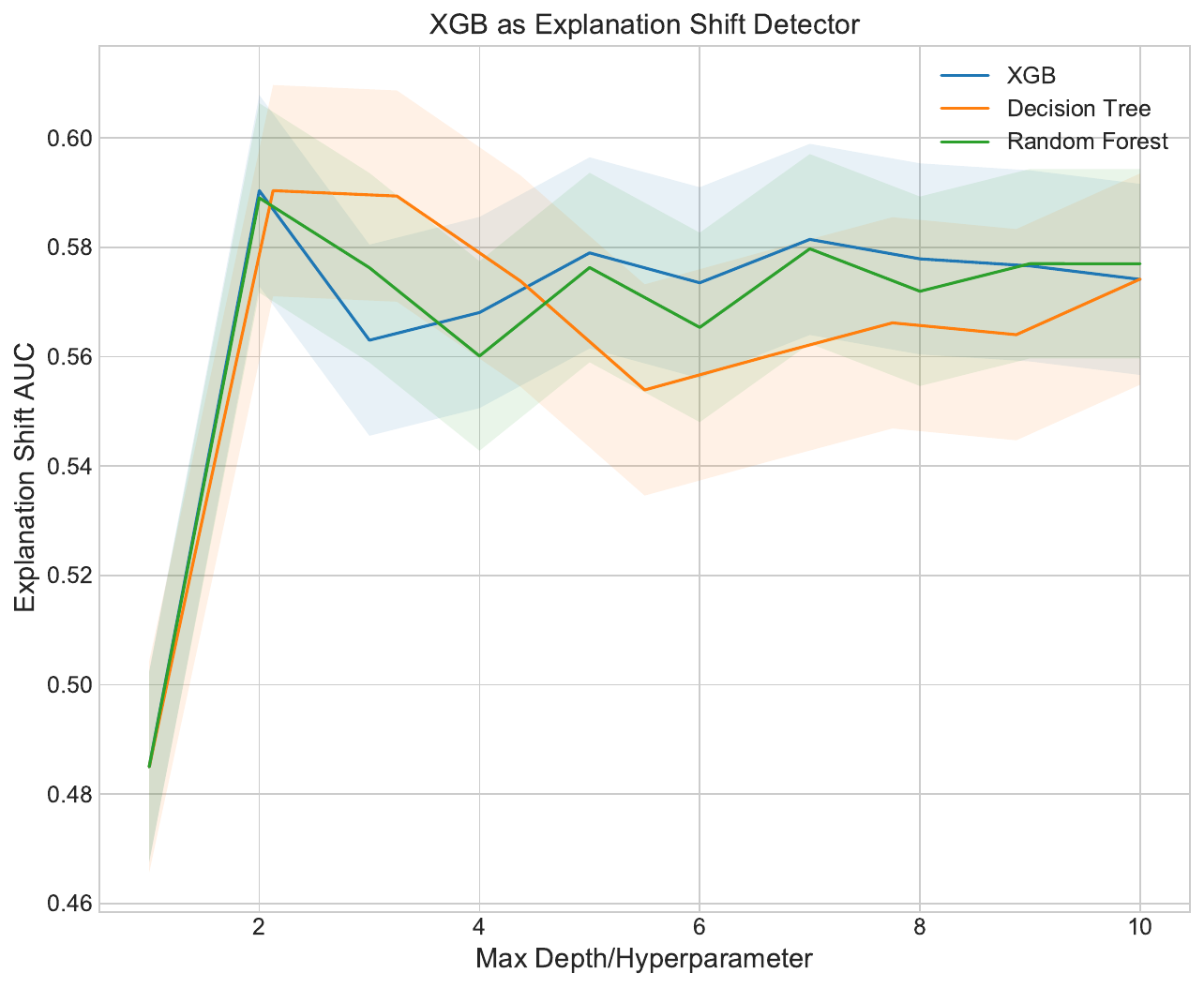}
\caption{Images represent the AUC of the \textit{Explanation Shift Detector}, on ACS Income and last two Stackoverflow under novel group shift. In the image above, the detector is a logistic regression, and in the images below, it is a gradient-boosting decision tree classifier. By changing the model, we can see that vanilla models (decision tree with depth 1 or 2) are unaffected by the distribution shift, while when increasing the model complexity, the out-of-distribution impact of the data in the model  starts to be tangible}
\label{fig:xai.hyper.shift}
\end{figure}

\begin{figure}[ht]
\centering
\includegraphics[width=.8\textwidth]{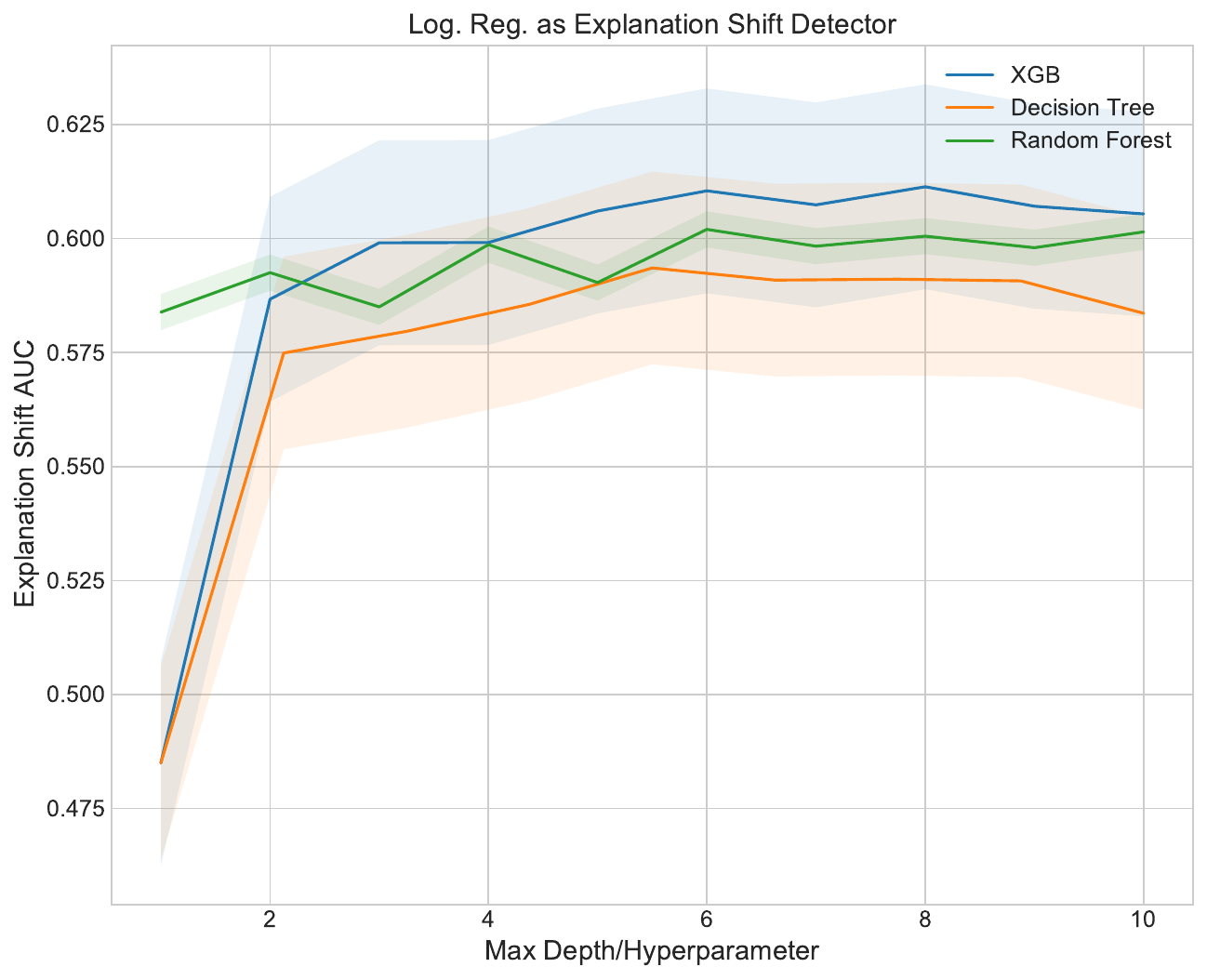}\hfill
\includegraphics[width=.8\textwidth]{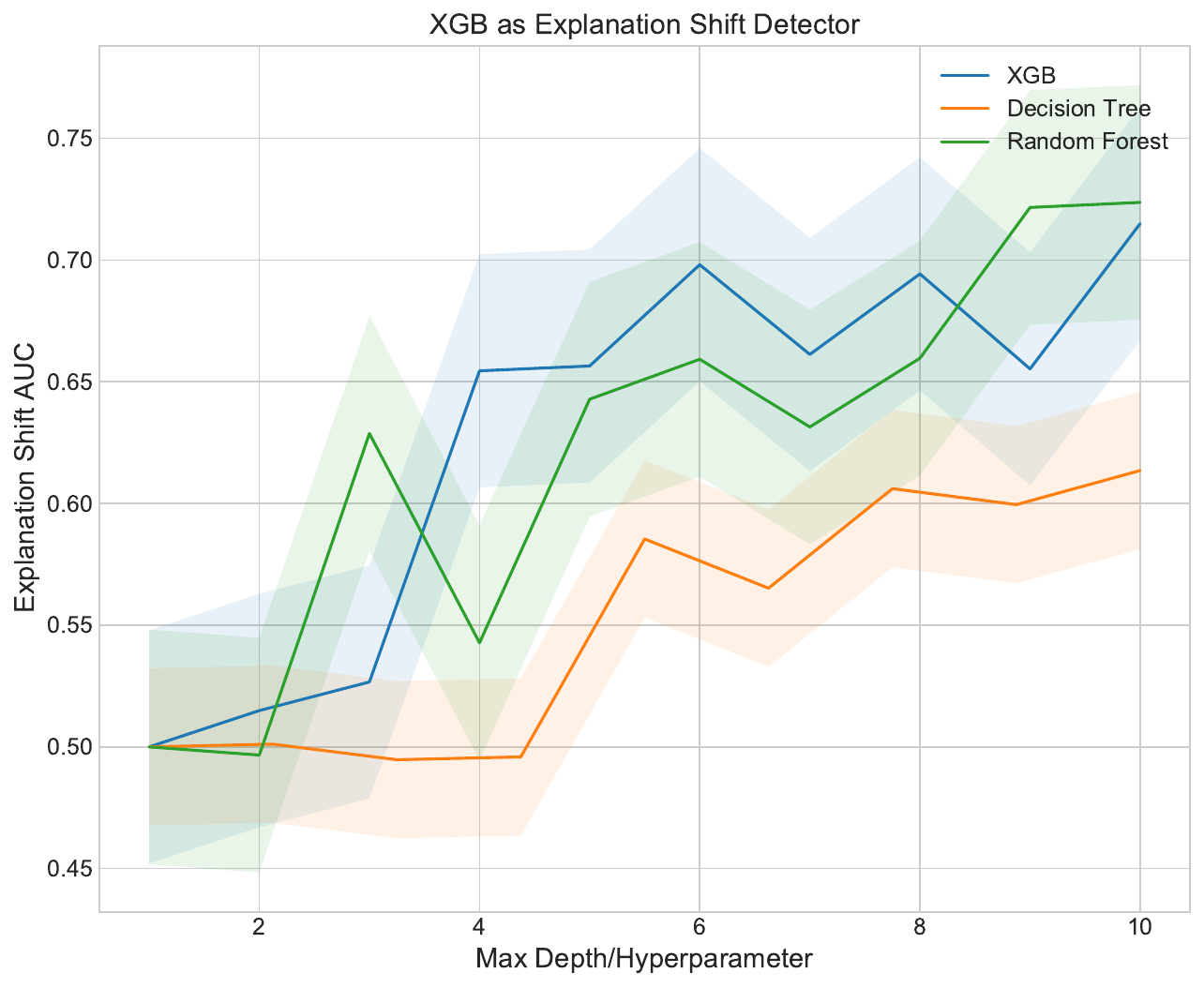}
\caption{Images represent the AUC of the \textit{Explanation Shift Detector}, on Stackoverflow under novel group shift. In the image above, the detector is a logistic regression, and in the image below, it is a gradient-boosting decision tree classifier. By changing the model, we can see that vanilla models (decision tree with depth 1 or 2) are unaffected by the distribution shift, while when increasing the model complexity, the out-of-distribution impact of the data in the model  starts to be tangible}
\label{fig:xai.hyper2.shift}
\end{figure}

The results presented in Figure \ref{fig:xai.hyper.shift} show the AUC of the \textit{Explanation Shift Detector} for the ACS Income dataset under novel group shift. We observe that the distribution shift does not affect very simplistic models, such as decision trees with depths 1 or 2. However, as we increase the model complexity, the out-of-distribution data impact on the model becomes more pronounced. Furthermore, when we compare the performance of the \textit{Explanation Shift Detector} across different models, such as Logistic Regression and Gradient Boosting Decision Tree, we observe distinct differences(note that the y-axis takes different values).

In conclusion, the explanation distributions serve as a projection of the data and model sensitive to what the model has learned. The results demonstrate the importance of considering model complexity under distribution shifts.

\clearpage
\section{LIME as an Alternative Explanation Method}\label{app:LIME}

Another feature attribution technique that satisfies the aforementioned properties (efficiency and uninformative features Section~\ref{sec:xai.foundations}) and can be used to create the explanation distributions is LIME (Local Interpretable Model-Agnostic Explanations). The similar discussion that we had in section~\ref{et.app:LIME}, also applies here.

\begin{figure}[ht]
\centering
\includegraphics[width=.8\textwidth]{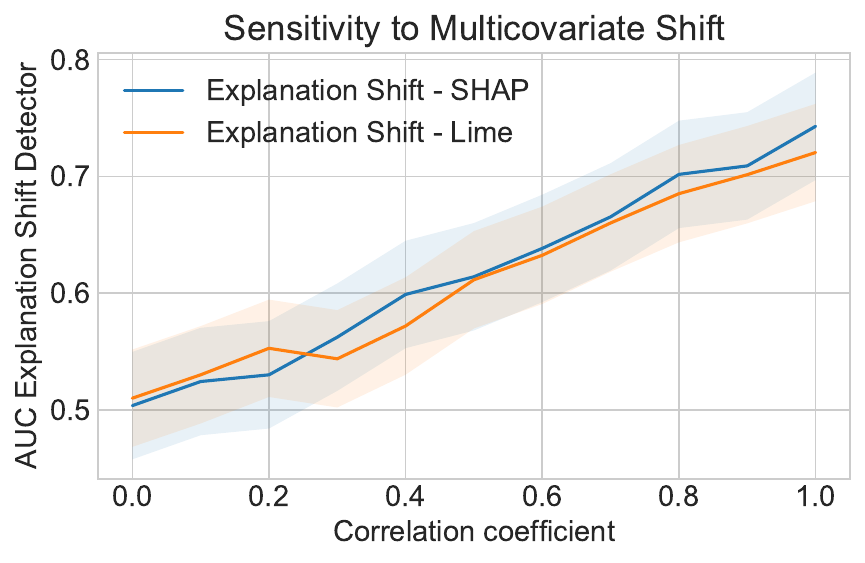}\hfill
\includegraphics[width=.8\textwidth]{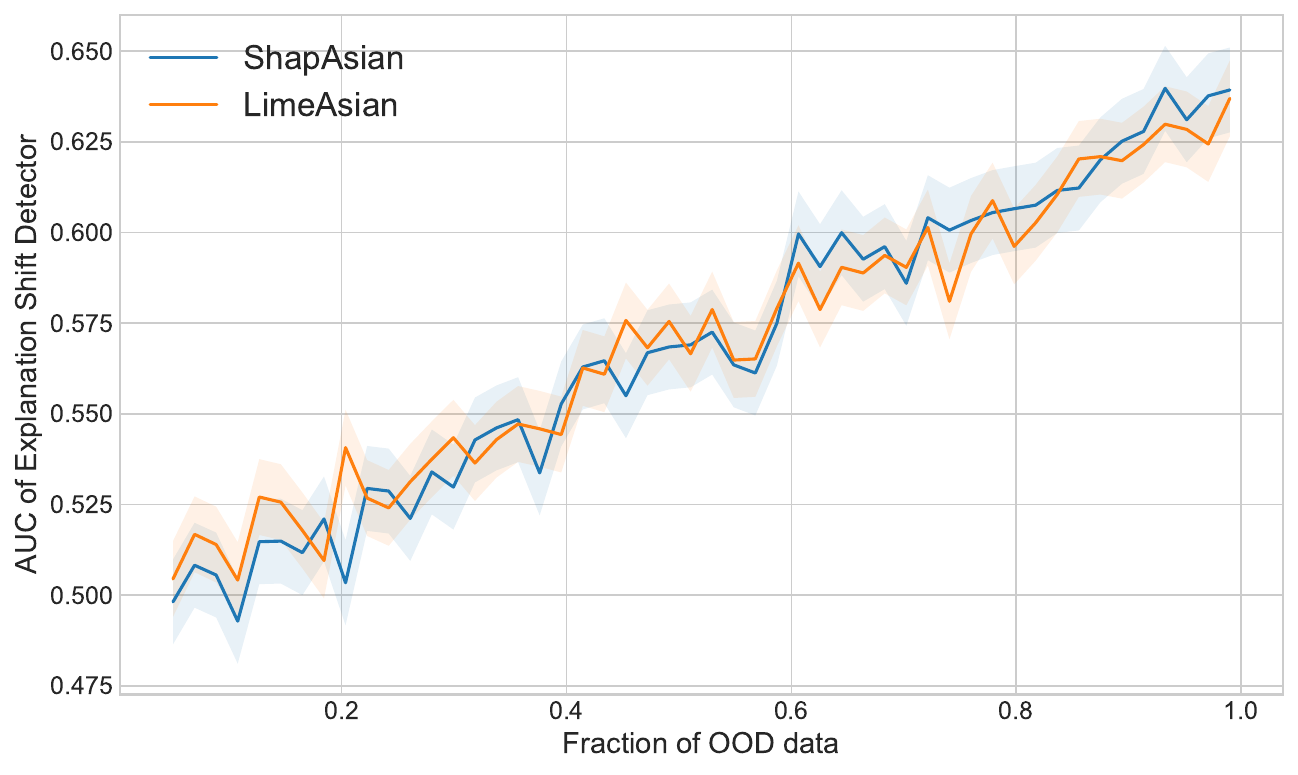}
\caption{Comparison of the explanation distribution generated by LIME and SHAP. The above figure shows the sensitivity of the predicted probabilities to multi covariate changes using the synthetic data experimental setup of \ref{fig:sensitivity} on the main body of the paper. The figure below shows the distribution of explanation shifts for a New Covariate Category shift (Asian) in the ACS Income dataset.}
\label{fig:lime}
\end{figure}

Figure \ref{fig:lime} compares the explanation distributions generated by LIME and SHAP. The left plot shows the sensitivity of the predicted probabilities to multicovariate changes using the synthetic data experimental setup from Figure \ref{fig:sensitivity} in the main body of the paper. The below plot shows the distribution of explanation shifts for a New Covariate Category shift (Asian) in the ASC Income dataset. The performance of OOD explanations detection is similar between the two methods, but LIME suffers from two drawbacks: its theoretical properties rely on the definition of a local neighborhood, which can lead to unstable explanations (false positives or false negatives on explanation shift detection), and its computational runtime required is much higher than that of SHAP (see experiments below).

\subsection{Runtime}
We conducted an analysis of the runtimes of generating the explanation distributions using the two proposed methods. The experiments were run on a server with 4 vCPUs and 32 GB of RAM. We used \texttt{shap} version $0.41.0$ and \texttt{lime} version $0.2.0.1$ as software packages. In order to define the local neighborhood for both methods in this example we use all the data provided as background data. As an $f_\theta$  model, we use an \texttt{xgboost} and compare the results of TreeShap against LIME.  When varying the number of samples we use 5 features and while varying the number of features we use $1000$ samples.

\begin{figure}[ht]
\centering
\includegraphics[width=.8\textwidth]{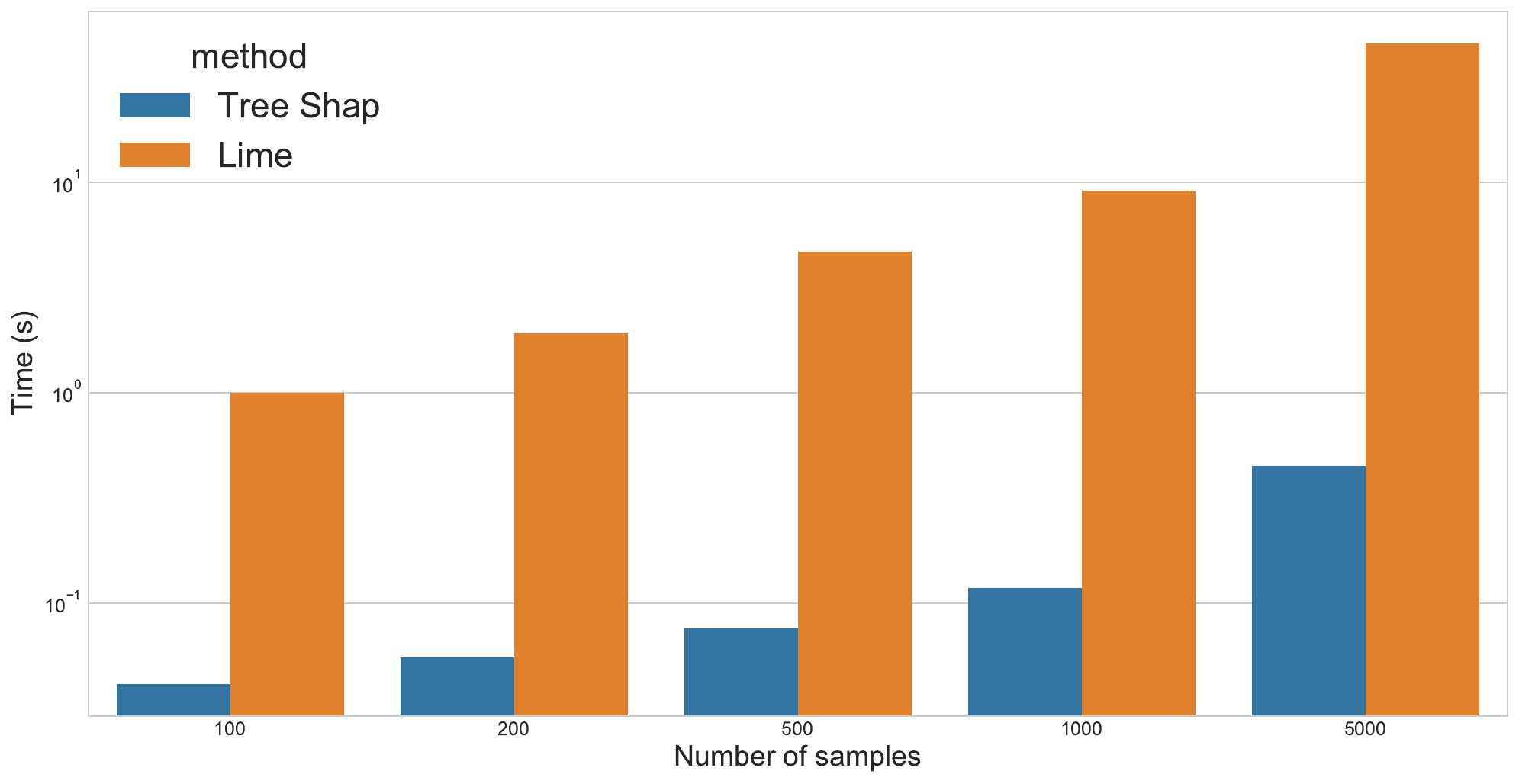}\hfill
\includegraphics[width=.8\textwidth]{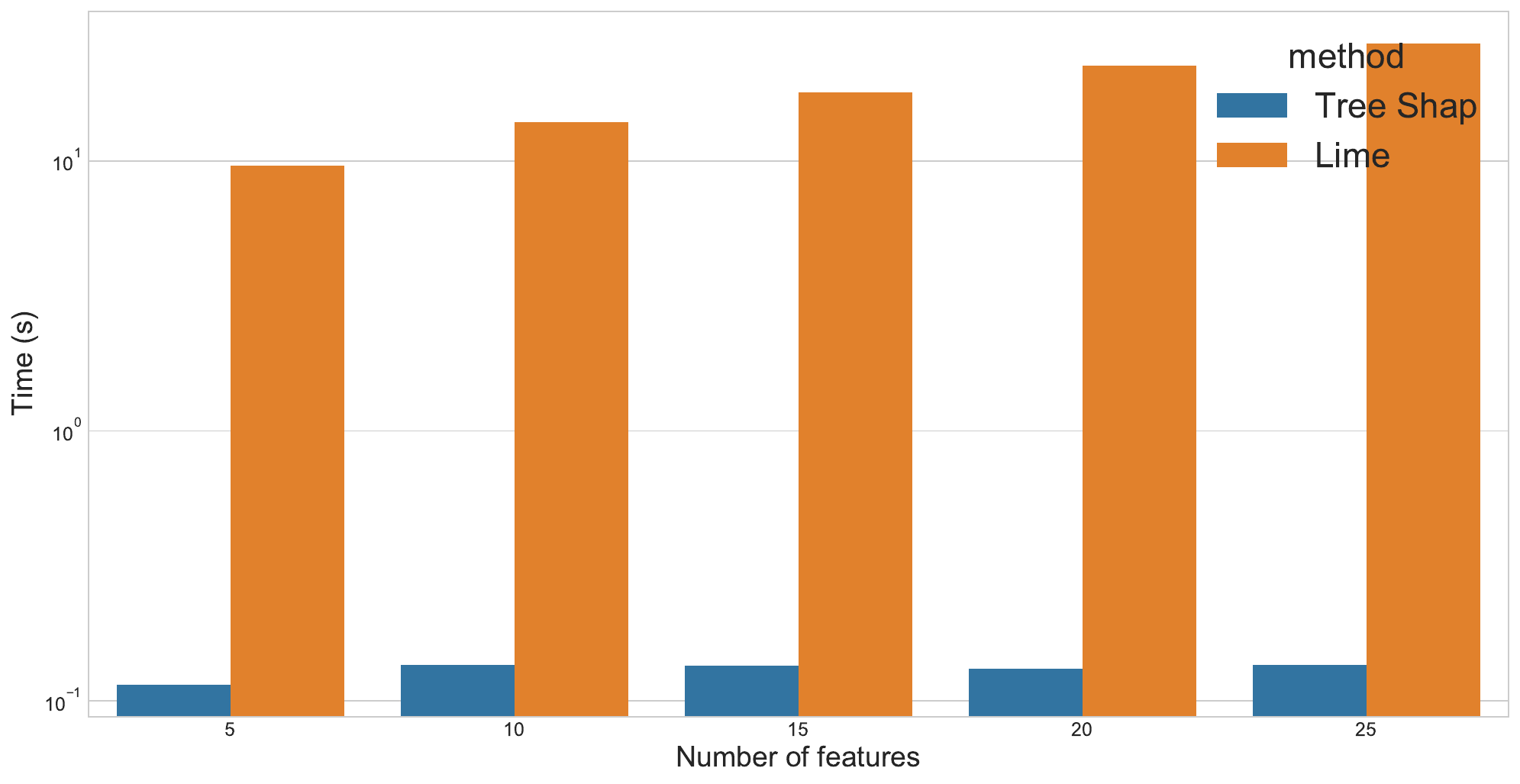}
\caption{Wall time for generating explanation distributions using SHAP and LIME with different numbers of samples (above) and different numbers of columns (below). Note that the y-scale is logarithmic. The experiments were run on a server with 4 vCPUs and 32 GB of RAM. The runtime required to create an explanation distributions with LIME is far greater than SHAP for a gradient-boosting decision tree}
\label{fig:computational}
\end{figure}

Figure \ref{fig:computational}, shows the wall time required for generating explanation distributions using SHAP and LIME with varying numbers of samples and columns. The runtime required of generating an explanation distributions using LIME is much higher than using SHAP, especially when producing explanations for distributions. This is due to the fact that LIME requires training a local model for each instance of the input data to be explained, which can be computationally expensive. In contrast, SHAP relies on heuristic approximations to estimate the feature attribution with no need to train a model
for each instance. The results illustrate that this difference in computational runtime becomes more pronounced as the number of samples and columns increases. 

We note that the computational burden of generating the explanation distributions can be further reduced by limiting the number of features to be explained, as this reduces the dimensionality of the explanation distributions, but this will inhibit the quality of the explanation shift detection as it won't be able to detect changes on the distribution shift that impact model on those features.

Given the current state-of-the-art of software packages we have used SHAP values due to lower runtime required and that theoretical guarantees hold with the implementations. In the experiments performed in this paper, we are dealing with a medium-scaled dataset with around $\sim 1,000,000$ samples and $20-25$ features. Further work can be envisioned on developing novel mathematical analysis and software that study under which conditions which method is more suitable.

\section{Related Work on Tabular Data}

\subsection{Classifier two-sample test} 

Evaluating how two distributions differ has been a widely studied topic in the statistics and statistical learning literature~\citep{statisticallearning,datasetShift,DBLP:conf/icml/LiuXL0GS20} and has advanced in recent years ~\citep{DBLP:conf/nips/ParkAKP21,DBLP:conf/nips/LeeLLS18,DBLP:conf/icml/ZhangSMW13}. The use of supervised learning classifiers to measure statistical tests has been explored by~\cite{DBLP:conf/iclr/Lopez-PazO17} proposing a classifier-based approach that returns test statistics to interpret differences between two distributions. We adopt their power-test analysis and interpretability approach but apply it to the explanation distributions instead of input data distributions.
\subsection{Out-Of-Distribution Detection}
Evaluating how two distributions differ has been a widely studied topic in the statistics and statistical learning literature~\cite{statisticallearning,datasetShift,DBLP:conf/icml/LiuXL0GS20}, that have advanced recently in last years ~\cite{DBLP:conf/nips/ParkAKP21,DBLP:conf/nips/LeeLLS18,DBLP:conf/icml/ZhangSMW13}.~\cite{DBLP:conf/nips/RabanserGL19} provides a comprehensive empirical investigation, examining how dimensionality reduction and two-sample testing might be combined to produce a practical pipeline for detecting distribution shifts in real-life machine learning systems.  Other methods to detect if new data is OOD have relied on neural networks based on the prediction distributions~\cite{fort2021exploring,NEURIPS2020_219e0524}. They use the maximum softmax probabilities/likelihood as a confidence score~\cite{DBLP:conf/iclr/HendrycksG17}, temperature or energy-based scores ~\cite{DBLP:conf/nips/RenLFSPDDL19,DBLP:conf/nips/LiuWOL20,DBLP:conf/nips/WangLBL21}, they extract information from the gradient space~\cite{DBLP:journals/corr/GradientShift}, relying on the latent space~\cite{DBLP:conf/nips/CrabbeQIS21}, they fit a Gaussian distribution to the embedding, or they use the Mahalanobis distance for out-of-distribution detection~\cite{DBLP:conf/nips/LeeLLS18,DBLP:journals/corr/reliableShift}. 

Many of these methods are explicitly developed for neural networks that operate on image and text data, and often, they can not be directly applied to traditional ML techniques. For image and text data, one may build on the assumption that the relationships between relevant predictor variables ($X$) and
response variables ($Y$) remain unchanged, i.e., \ that no \emph{concept shift} occurs. For instance, the essence of how a dog looks remains unchanged over different data sets, even if contexts may change. Thus, one can define invariances on the latent spaces of deep neural models, which do not apply to tabular data in a similar manner. For example, predicting buying behaviour before, during, and after the COVID-19 pandemic constitutes a conceptual shift that is not amenable to such methods. We focus on such tabular data where techniques such as gradient boosting decision trees achieve state-of-the-art model performance~\cite{grinsztajn2022why,DBLP:journals/corr/abs-2101-02118,BorisovNNtabular}.

Other research lines, such as membership inference (or membership classifiers), aim to detect that, given a data record and a machine learning model, whether the record was in the model's training dataset or not ~\cite{DBLP:conf/sp/ShokriSSS17}. Our approach is slightly different, as we don't attempt to predict if a certain instance belongs to the training data or not, but if its distribution is as the training dataset. 

\subsection{Detecting distribution shift and its impact on model behaviour}
Extensive related work has aimed at detecting that data is from out-of-distribution. To this end, they have created several benchmarks that measure whether data comes from in-distribution or not~\citep{wilds,wilds_ext,malinin2021shifts,malinin2022shifts,DBLP:conf/nips/MalininBGGGCNPP21}. In contrast, our main aim is to evaluate the impact of the distribution shift on the use of model.



A typical example is two-sample testing on the latent space such as described by~\cite{DBLP:conf/nips/RabanserGL19}. However, many of the methods developed for detecting out-of-distribution data are specific to neural networks processing image and text data and can not be applied to traditional machine learning techniques. These methods often assume that the relationships between predictor and response variables remain unchanged, i.e., no concept shift occurs. 
Our work is applied to  tabular data where techniques such as gradient boosting decision trees achieve state-of-the-art model performance~\citep{grinsztajn2022why,DBLP:journals/corr/abs-2101-02118,BorisovNNtabular}.

\subsection{Impossibility of model monitoring}

Recent research findings have formalized the limitations of monitoring machine learning models in the absence of labelled data. Specifically~\citep{garg2022leveraging,chen2022estimating} prove the impossibility of predicting model degradation or detecting out-of-distribution data with certainty~\citep{fang2022is,DBLP:conf/icml/ZhangGR21,guerin2022outofdistribution}. Although our approach does not overcome these limitations, it provides valuable insights for machine learning engineers to better understand changes in interactions between learned models and shifting data distributions.

\subsection{Model monitoring and distribution shift under specific assumptions:} Under specific types of assumptions, model monitoring and distribution shift become feasible tasks. One type of assumption often found in the literature is to leverage causal knowledge to identify the drivers of distribution changes~\citep{DBLP:conf/aistats/BudhathokiJBN21,zhang2022why,schrouff2022diagnosing}. For example,~\cite{DBLP:conf/aistats/BudhathokiJBN21} use graphical causal models and feature attributions based on Shapley values to detect changes in the distribution.
Similarly, other works aim to detect specific distribution shifts, such as covariate or concept shifts.  Our approach does not rely on additional information, such as a causal graph, labelled test data, or specific types of distribution shift. Still, by the nature of pure concept shifts, the model behaviour remains unaffected and new data need to come with labelled responses to be detected.

\subsection{Explainability and distribution shift}

~\cite{lundberg2020local2global} applied Shapley values to identify possible bugs in the pipeline by visualizing univariate SHAP contributions. Following this line of work, \cite{DBLP:conf/kdd/NigendaKZRTDK22} compare the order of the feature importance using the NDCG between training and unseen data.  We go beyond their work and formalize the multivariate explanation distributions on which we perform a two-sample classifier test to detect how distribution shift impacts interaction with the model. Furthermore, we provide a mathematical analysis of how the SHAP values contribute to detecting distribution shift. In Appendix~\ref{fig:shift.ndcg} we provide a formal comparison against~\cite{DBLP:conf/kdd/NigendaKZRTDK22}.

An approach using Shapley values by ~\cite{li2022enabling} allows for tracking distributional shifts and their impact among for categorical time series using slidSHAP, a novel method for unlabelled data streams. This approach is particularly useful for unlabelled data streams, offering insights into the changing data distribution dynamics. In contrast, our work focuses on defining explanation distributions and leveraging their theoretical properties in the context of distribution shift detection, employing a two-sample classifier test for detection.

Another perspective in the field of explainability is explored by~\cite{10.5555/3495724.3495784,DBLP:conf/nips/AdebayoGMGHK18}, who investigate the effectiveness of post-hoc model explanations for diagnosing model errors. They categorize these errors based on their source. While their work is geared towards model debugging, our research takes a distinct path by aiming to quantify the influence of distribution shifts on the model.

\cite{DBLP:conf/esann/HinderAVH22} proposes to explain concept drift by contrasting explanations
describing characteristic changes of spatial features. ~\cite{haug2021learning} track changes in the distribution of model parameter values that are directly related to the input features to identify concept drift early on in data streams. In a more recent paper,~\cite{haug2022change} also exploits the idea that local changes to feature attributions and distribution shifts are strongly intertwined and uses this idea to update the local feature attributions efficiently. Their work focuses on model retraining and concept shift, in our work, the original model $f_\theta$ remains unaltered, and since we are in an unsupervised monitoring scenario, we can't detect concept shifts see discussion in Section \ref{sec:discussion}


\section{Discussion}\label{sec:discussion}

In this study, we conducted a comprehensive evaluation of explanation shift by systematically varying models ($f$), model parametrizations ($\theta$), feature attribution explanations ($\Ss$), and input data distributions ($\D_X$). Our objective was to investigate the impact of distribution shift on the model by explanation shift and gain insights into its characteristics and implications.

The Shapley value, a key component in our method, describes how a model's prediction for a specific data point deviates from the mean. These theoretical considerations, which we laid out in Section~\ref{sec:math.analysis}, have been confirmed by our experimental sections.

In Section~\ref{sec:sub:novel.covariate}, we have studied input distribution shift. Our experiment shows that explanation shift detects input distribution shifts better than the best baseline methods. Table~\ref{tab:novel.group} showcases the two top-performing methods—comparing input distributions with Kolmogorov-Smirnoff \emph{(B1)} and our method — with statistically insignificant differences.

In Section~\ref{exp:synthetic}, we have studied co-variate shift. Considering, Table~\ref{tab:corr.Synth}, the best method for detecting input distribution shifts, \emph{(B1)}, fails completely on this task. The second best method is \emph{(B2)} comparison of prediction distributions using the Wasserstein distance, which also did quite well wrt.\ predicting input distribution shifts and came rather close behind our approach in both experiments.

Moreover, in our geopolitical and temporal shift experiment (Section~\ref{exp:geopolitical}), we demonstrate the ability to account for the drivers of model changes under such input data shifts. Cross-task comparisons in experiments (Figure~\ref{fig:xai.mobility} or Figure~\ref{fig:xai.traveltime}) highlight how explanation shift feature importance varies even when input distribution shifts remain constant during cross-task. These capabilities are not offered by any of the competing baselines. These observations are further supported by additional experiments in Appendix~\ref{app:hyperparameter}, where we solely vary model complexity, showcasing the adaptability of explanation shifts to changes in model characteristics.

Our approach cannot detect concept shifts, as concept shift requires understanding the interaction between prediction and response variables. By the nature of pure concept shifts, such changes do not affect the model. To be understood, new data need to come with labelled responses. We work under the assumption that such labels are not available for new data, nor do we make other assumptions; therefore, our method is not able to predict the degradation of prediction performance under distribution shifts. All papers such as \citep{garg2022leveraging,baek2022agreementontheline,chen2022estimating,fang2022is,DBLP:conf/icml/MillerTRSKSLCS21,lu2023predicting} that address the monitoring of prediction performance have the same limitation. Only under specific assumptions, e.g., no occurrence of concept shift or causal graph availability, can performance degradation be predicted with reasonable reliability.

The potential utility of explanation shifts as distribution shift indicators that affect the model in computer vision or natural language processing tasks remains an open question. We have used feature attribution explanations to derive indications of explanation shifts, but other AI explanation techniques may be applicable and come with their advantages.

\section{Conclusions}

Commonly, the problem of detecting the impact of the distribution shift on the model has relied on measurements for detecting shifts in the input or output data distributions or relied on assumptions either on the type of distribution shift or causal graphs availability. In this paper, we proposed explanation shifts as an indicator for detecting and identifying the impact of distribution shifts on machine learning models. We provide software, mathematical analysis examples, synthetic data, and real-data experimental evaluation. We found that measures of explanation shift can provide more insights than input distribution and prediction shift measures when monitoring machine learning models. 

\subsection*{Limitations}
The potential utility of explanation shifts as distribution shift indicators that affect the model in computer vision or natural language processing tasks remains an open question. We have used feature attribution explanations to derive indications of explanation shifts, but other AI explanation techniques may be applicable and come with their advantages. Also, our approach cannot detect concept shifts, as concept shift requires understanding the interaction between input data and response variables. By the nature of pure concept shifts, such changes do not affect the model. We work under the assumption that such labels are not available for new data, nor do we make other assumptions; therefore, our method is not able to predict the degradation of prediction performance under distribution shifts.

Furthermore, our use of the \texttt{shap} Python package for Shapley values approximation can introduce known drawbacks, as highlighted in recent literature~\citep{DBLP:journals/corr/abs-2212-11870,10.1145/3375627.3375830}. Additionally, our current implementation relies on linear Shapley value interaction approximations, which can be extended following the work of~\citet{fumagalli2023shapiq,DBLP:conf/aistats/BordtL23}.

\chapter{Building Indicators of Model Performance Deterioration}\label{ch:monitoring}

Monitoring machine learning models in production is not an easy task. There are situations when the true label of the deployment data is available, and performance metrics can be monitored. But there are cases where it is not, and performance metrics are not so trivial to calculate once the model has been deployed. Model monitoring aims to ensure that a machine learning application in a production environment displays consistent behavior over time. 

Being able to explain or remain accountable for the performance or the deterioration of a deployed model is crucial, as a drop in model performance can affect the whole business process, potentially having catastrophic consequences\footnote{The Zillow case is an example of consequences of model performance degradation in an unsupervised monitoring scenario, see
\url{https://edition.cnn.com/2021/11/09/tech/zillow-ibuying-home-zestimate/index.html} (Online accessed January 26, 2022).}. Once a deployed model has deteriorated, models are retrained using previous and new input data in order to maintain high performance. This process is called continual learning ~\citep{continual_learning} and it can be computationally expensive and put high demands on the software engineering system. Deciding when to retrain machine learning models is paramount in many situations.

Traditional machine learning systems assume that training data has been generated from a stationary source, but \textit{data is not static, it evolves}. This problem can be seen as a distribution shift, where the data distributions of the training set and the test set differ. Detecting distribution shifts has been a longstanding problem in the machine learning (ML) research community~\citep{SHIMODAIRA2000227,sugiyama,sugiyama2,priorShift,learningSampleSelection,intuitionSampleSelection,varietiesSelectionBias,cortes2008sample,CorrectingSampleSelection,practicalFB}, as it is one of the main sources of model performance deterioration ~\citep{datasetShift}. Furthermore, data scientists in machine learning competitions claim that finding the train/validation split that better resembles the test (evaluation) distribution is paramount to winning a Kaggle competition ~\citep{howtowinKaggle}. 

However, despite the fact that a shift in data distribution can be a source of model deterioration, the two are not identical. Indeed, if we shift a random noise feature we have caused a change in the data distribution, but we should not expect the performance of a model to decline when evaluated on this shifted dataset. Thus, we emphasize here that our focus is on \textit{model deterioration} and not distribution shift, despite the correlation between the two.

Established ways of monitoring distribution shift when the real target distribution is not available are based on statistical changes either the input data ~\citep{continual_learning,DBLP:conf/nips/RabanserGL19} or on the model output~\cite{garg2022leveraging}. These statistical tests correctly detect univariate changes in the distribution but are completely independent of the model performance and can therefore be too sensitive, indicating a change in the covariates but without any degradation in the model performance. This can result in false positives, leading to unnecessary model retraining. It is worth noting that several authors have stated the clear need to identify how non-stationary environments affect the behavior of models~\citep{continual_learning}.

Aside from merely indicating that a model has deteriorated, it can in some circumstances be beneficial to identify the \textit{cause} of the model deterioration by detecting and explaining the lack of knowledge in the prediction of a model. Such explainability techniques can provide algorithmic transparency to stakeholders and to the ML engineering team ~\citep{desiderataECB,bhatt2021uncertainty,koh2020understanding,ribeiro2016why,sundararajan2017axiomatic}.

This chapter's primary focus is on non-deep learning models and small to medium-sized tabular datasets, a size of data that is very common in the average industry, where, non-deep learning-based models achieve state-of-the-art results~\cite{grinsztajn:hal-03723551,survey_DL_tabular,do_we_need_DL}.

Our contributions are the following:

\begin{enumerate}
    
    \item We use this non-parametric uncertainty estimation method to develop a machine learning monitoring system for regression models, which outperforms previous monitoring methods in terms of detecting deterioration of model performance.
    
    \item We use explainable AI techniques to identify the source of model deterioration for both entire distributions as a whole as well as for individual samples, where classical statistical indicators can only determine distribution differences. 

    \item We release an open source Python package, \texttt{doubt}, which implements our uncertainty estimation method and is compatible with all \texttt{scikit-learn} models~\citep{pedregosa2011scikit}.
\end{enumerate}

\section{Related Work}\label{sec:relatedwork}

\subsection{Model Monitoring}

From a software perspective, a machine learning model is another software component whose functionalities need to be tested before deployment. This approach has one structural weakness: \textit{data is not static, it evolves} ~\citep{continual_learning}. The conditions in which a system is developed can differ from deployment time for a range of reasons, from an evolutionary environment, a sampling bias at training time, or the inability to reproduce test conditions at training time ~\citep{datasetShift}.

Model monitoring techniques help to detect unwanted changes in the behavior of a machine learning application in a production environment. One of the biggest challenges in model monitoring is distribution shift, which is also one of the main sources of model degradation ~\citep{datasetShift,continual_learning}. 

Diverse types of model monitoring scenarios require different supervision techniques. We can distinguish two main groups: Supervised learning and unsupervised learning. Supervised learning is the appealing one from a monitoring perspective, where performance metrics can easily be tracked. Whilst attractive, these techniques are often unfeasible as they rely either on having ground truth labeled data available or maintaining a hold-out set, which leaves the challenge of how to monitor ML models to the realm of unsupervised learning~\cite{continual_learning}. Popular unsupervised methods that are used in this respect are the Population Stability Index (PSI) and the Kolmogorov-Smirnov test (K-S), all of which measure how much the distribution of the covariates in the new samples differs from the covariate distribution within the training samples. These methods are often limited to real-valued data, low dimensions, and require certain probabilistic assumptions~\citep{continual_learning,ShiftsData}.


Another approach suggested by~\citet{shapTree} is to monitor the SHAP value contribution of input features over time together with decomposing the loss function across input features in order to identify possible bugs in the pipeline as well as distribution shift. This technique can account for previously unaccounted bugs in the machine learning production pipeline but fails to monitor the model degradation.

Is worth noting that prior work~\citep{garg2022leveraging,jiang2021assessing} has focused on monitoring models either on out-of-distribution data or in-distribution data~\citep{neyshabur2017exploring,neyshabur2018understanding}. Such a task, even if challenging, does not accurately represent the different types of data a model encounters in the wild. In a production environment, a model can encounter previously seen data (training data), unseen data with the same distribution (test data), and statistically new and unseen data (out-of-distribution data). That is why we focus our work on finding an unsupervised estimator that replicates the behavior of the model performance.

The idea of mixing uncertainty with dataset shift was introduced by~\citet{trustUncertainty}. Our work differs from theirs, in that they evaluate uncertainty by shifting the distributions of their dataset, where we aim to detect model deterioration under dataset shift using uncertainty estimation. Their work is also focused on deep learning classification problems, while we estimate uncertainty using model agnostic regression techniques. Further, our contribution allows us to pinpoint the features/dimensions that are main causes of the model degradation.

\citet{garg2022leveraging} introduces a monitoring system for classification models, based on imposing thresholds on the softmax values of the model. Our method differs from theirs in that we work with regression models and not classification models, and that our method utilizes external uncertainty estimation methods, rather than relying on the model's own ``confidence'' (i.e., the outputted logits and associated softmax values).

\citet{DBLP:conf/nips/RabanserGL19}, presents a comprehensive empirical investigation of dataset shift, examining how dimensionality reduction and two-sample testing might be combined to produce a practical pipeline for detecting distribution shift in a real-life machine learning system. They show that the two-sample-testing-based approach performs best. This serves as a baseline comparison within our models, even if their idea is more focused on binary classification, whereas our works focus on building a regression indicator. 

\subsection{Uncertainty}

Uncertainty estimation is being developed at a fast pace. Model averaging ~\citep{kumar2012bootstrap,gal2016dropout,lakshminarayanan2017simple,arnez2020} has emerged as the most common approach to uncertainty estimation. Ensemble and sampling-based uncertainty estimates have been successfully applied to many use cases such as detecting misclassifications~\citep{unc_ood_class}, out-of-distribution inputs~\citep{dangelo2021uncertaintybased}, adversarial attacks ~\citep{adversarialUncertainty,smith2018understanding}, automatic language assessments~\citep{malinin2019uncertainty} and active learning~\citep{kirsch2019batchbald}. In our work, we apply uncertainty to detect and explain model performance for seen data (train), unseen and identically distributed data (test), and statistically new and unseen data (out-of-distribution). \citet{barber2021predictive} recently introduced a new non-parametric method of creating prediction intervals using the Jackknife+. Our method differs from theirs in that we are using general bootstrapped samples for our estimates rather than leave-one-out estimates.

In this thesis, we use the method introduced by \cite{mougan2022monitoring}, which is an extension of the method proposed by \cite{kumar2012bootstrap}, that takes into account the model's variance in the construction of the prediction intervals. \cite{kumar2012bootstrap}, introduced a non-parametric method to compute prediction intervals for any ML model using a bootstrap estimate with theoretical guarantees.




\section{Methodology}\label{sec:methodology}

\subsection{Evaluation of Deterioration Detection Systems}\label{sec:evaluation-deterioration-systems}

The problem we are tackling in this paper is evaluating and accounting for model predictive performance deterioration, which in general is an impossible task. In order to achieve possible results, we simulate a gradual covariate distribution shift scenario in which we \textit{have} access to the true labels, which we can use to measure the model deterioration and thus evaluate the monitoring system. A naive simulation in which we simply manually shift a chosen feature of a dataset would not be representative, as the associated true labels could have changed if such a shift happened "in the wild".

Therefore, we propose the following alternative approach. Starting from a real-life dataset $\D$ and a numerical feature $F$ of $\D$, we sort the data samples of $\D$ by the value of $F$, and split the sorted $\D$ in three equally sized sections: $\{\Dd{below},\Dd{tr},\Dd{upper}\} \subseteq\D$. The model is then fitted to the middle section ($\Dd{tr}$) and evaluated on all of $\D$. The goal of the monitoring system is to input the model, the labelled data segment $\Dd{tr}$ and a sample of unlabelled data $\Dd{te}_X\subseteq\D$, and output a ``monitoring value'' which behaves like the model's performance on $\Dd{te}$. Such a prediction will thus have to take into account the training performance, generalization performance, and the out-of-distribution performance of the model. This setup aims to evaluate a gradual covariate shift, in which we assume that in the vicinity of the training data, concept shift has not occurred. 

In the experimental section, we compare our monitoring technique to several other such systems. To enable comparison between the different monitoring systems, we standardize all monitoring values as well as the performance metrics of the model. From these standardized values, we can now directly measure the goodness-of-fit of the model monitoring system by computing the absolute difference between its (standardized) monitoring values and the (standardized) ground truth model performance metrics. Our chosen evaluation method is very similar to the one used by \citet{garg2022leveraging}. They focus on classification models and their systems output estimates of the model's accuracy on the dataset. They evaluate these systems by computing the absolute difference between the system's accuracy estimate and the actual accuracy that the model achieves on the dataset.

As we are working with regression models in this paper, we will only operate with a single model performance metric: mean squared error. We will introduce our monitoring system, which is based on an uncertainty measure, and will compare our monitoring system against statistical tests based on input data or prediction data. In that section, we will also compare our uncertainty estimation method to current state-of-the-art uncertainty estimation methods.

\subsection{Uncertainty Estimation}\label{sec:uncertaintyestimation}

In order to estimate uncertainty in a general way for all machine learning models, we use a non-parametric regression technique. The method aims to determine prediction intervals for outputs of general non-parametric regression models using bootstrap methods.

We estimate the uncertainty of a model by computing a bootstrap estimate of the variance of the predictions. This is part of the method used in \cite{kumar2012bootstrap} and is computed as follows. From a dataset $X$ and a new sample $x_0\notin X$, we first bootstrap $X$ $B>0$ times, yielding bootstrapped subsamples $X_b$ for each $b<B$. We next fit our model on each of the $X_b$'s, call the resulting fitted models $M_b$, and then computing the predictions of the fitted models of $x_0$, $M_b(x_0)$. We then use the length of the 95\% confidence interval of these bootstrapped predictions as our uncertainty measure. Concretely, this is computed as
\begin{equation}
    \texttt{uncertainty}_M(x_0) := (q_{0.025}(\{M_b(x_0)\mid b<B\}), q_{0.975}(\{M_b(x_0)\mid b<B\})),
\end{equation}

with $q_\alpha(A)$ being the $\alpha$'th quantile of the set $A$.

\subsection{Detecting the Source of Uncertainty/Model Deterioration}
Using uncertainty as a method to monitor the performance of an ML model does not provide any information on \textit{what} features are the cause of the model degradation, only a goodness-of-fit to the model performance. We propose to solve this issue with the use of Shapley values. 

We start by fitting a model $f_{\theta}$ to the training data, $X^{\text{train}}$. We next shift the test data by five standard deviations (call the shifted data $X^{\text{ood}}$) and compute uncertainty estimates $Z$ of $f_{\theta}$ on $X^{\text{ood}}$. We next fit a second model $g_{\psi}$ on $(X^{\text{ood}})$ to predict the uncertainty estimate $Z$, and compute the associated Shapley values \cite{shapTree} of $g_{\psi}$. These Shapley values thus signify which features are the ones contributing the most to the uncertainty values. With the correlation between uncertainty values and model deterioration that we hope to conclude from the experiment described in the experimental section, this thus also provides us with a plausible cause of the model deterioration, if deterioration has taken place. Particularly, this methodology can be extended to large-scale datasets and deep learning-based models.
\section{Experiments}
\label{sec:experiments.monitoring}
Our experiments have been
organized into two main groups: Firstly, we assess the performance of our proposed uncertainty method for monitoring the performance of a machine learning model. Then, we evaluate the usability of the explainable uncertainty for identifying the features that are driving model degradation in local and global scenarios. We present the results over several real-world datasets, Table~\ref{tab:uncertaintydatasets}; we provide experiments on synthetic datasets that exhibit non-linear and linear behaviour.

\begin{table}[ht]
\begin{center}
\caption{Statistics of the regression datasets used in this paper.}\label{tab:uncertaintydatasets}
\begin{tabular}{l|c|c}
    Dataset & \# Samples & \# Features\\
    \hline
    Airfoil Self-Noise             & 1,503            & 5  \\
    Bike Sharing                   & 17,379           & 16 \\
    Concrete Strength              & 1,030            & 8  \\
    QSAR Fish Toxicity             & 908              & 6  \\
    Forest Fires                   & 517              & 12 \\
    Parkinsons                     & 5,875            & 22 \\
    Power Plant                    & 9,568            & 4  \\
    Protein                        & 45,730           & 9 \\

\end{tabular}
\end{center}
\end{table}

\subsection{Evaluating Model Deterioration}

The scenario we are addressing is characterized by regression data sets with statistically seen data (train data), IID statistically unseen data (test data), and out-of-distribution data. Following the open data for reproducible research guidelines described in~\citet{the_turing_way_community_2019_3233986} and for measuring the performance of the proposed methods, we have used eight open-source datasets (cf. Table~\ref{tab:uncertaintydatasets}) for an empirical comparison coming from the UCI repository~\citep{uci_data}. As described in the methodology, to benchmark our algorithm we, for each feature $F$ in each dataset $\D$, sort $\D$ according to $F$ and split $\D$ into three equally sized sections $\{\Dd{below},\Dd{tr},\Dd{upper}\} \subseteq\D$. We then train the model on $\Dd{tr}$ and test the performance of all of $\D$. In this way we obtain a mixture of train, test, and out-of-distribution data, allowing us to evaluate our monitoring techniques in all three scenarios.

In evaluating a monitoring system we need to make a concrete choice of the sampling method to get the unlabelled data $\mathcal{S}\subseteq\D$. We are here using a rolling window of fifty samples, which has the added benefit of giving insight into the performance of the monitoring system on each of the three sections $\Dd{lower}$, $\Dd{tr}$ and $\Dd{upper}$ (cf. Figure~\ref{fig:distribution}).

\begin{figure}[ht]
  \centering
  \includegraphics[width=1\columnwidth]{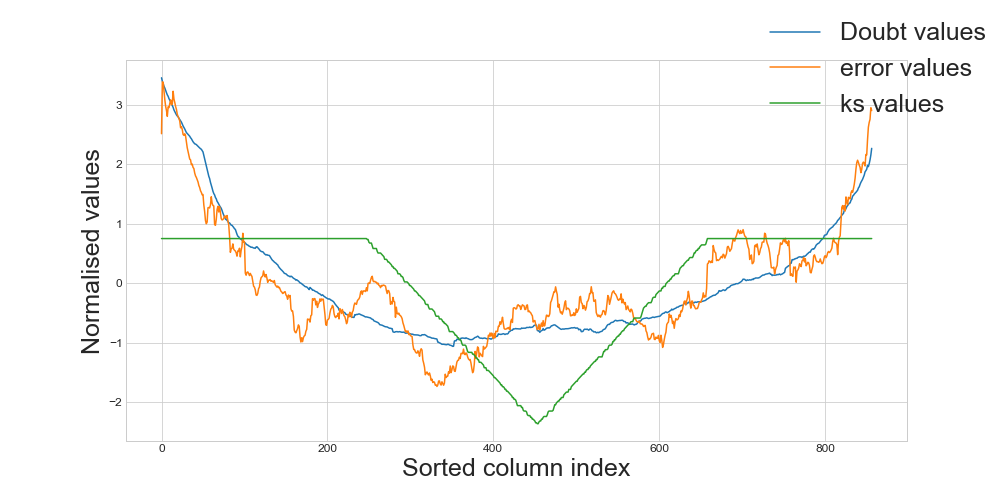}
  \caption{Comparison of different model degradation detection methods for the Fish Toxicity dataset. Each of the plots represents an independent experiment where each of the six features has been shifted, using the method described in the methodology section. Doubt achieves a better goodness-of-fit than previous statistical methods. A larger version of this figure can be found in Appendix.}\label{fig:distribution}
\end{figure}

We compare our monitoring system using the uncertainty estimation method against: $(i)$ two classical statistical methods on input data: the Kolmogorov-Smirnov test statistic (K-S) and the Population-Stability Index (PSI)~\citep{continual_learning}, $(ii)$ a Kolmogorov-Smirnov statistical test on the predictions between train and test~\cite{garg2022leveraging} that we denominate prediction shift and $(iii)$ the previous state-of-the-art  uncertainty estimate MAPIE. We evaluate the monitoring systems on a variety of model architectures: generalized linear models, tree-based models as well as neural networks.

The average performance across all datasets can be found in Table~\ref{tab:scores}.\footnote{See the appendix for a more detailed table.} From these we can see that our methods outperform K-S and PSI in all cases except for the Random Forest case, where our method is still on par with the best method, in that case, K-S. We have included a table with each dataset and all the estimators in the appendix, where it can be seen that both K-S and PSI easily identify a shift in the distribution but fail to detect when the model performance degrades, giving too many false positives.

\begin{table*}[ht]
\small
\begin{center}
\begin{tabular}{l|cccccc}
    & PSI & K-S & Pred.Shift & MAPIE & Doubt & Exp.Shift \\
    \hline
    Linear Reg. & $0.87 \pm 0.08$ & $0.81 \pm 0.10$ & $0.86 \pm 0.13$ & $0.77 \pm 0.10$ & $\mathbf{0.71 \pm 0.14}$ & $0.85 \pm 0.10$\\
    Poisson & $0.93 \pm 0.08$ & $0.94 \pm 0.20$ & $1.00 \pm 0.15$ & $0.83 \pm 0.18$ & $\mathbf{0.79 \pm 0.14}$& $0.94 \pm 0.14$ \\
    Decision Tree & $0.97 \pm 0.10$ & $0.52 \pm 0.12$ & $0.80 \pm 0.14$ & $0.60 \pm 0.16$ & $\mathbf{0.49 \pm 0.10}$& $0.77 \pm 0.16$ \\
    Random Forest & $0.95 \pm 0.08$ & $\mathbf{0.50 \pm 0.12}$ & $0.73 \pm 0.18$ & $0.86 \pm 0.15$ & $0.74 \pm 0.18$ & $0.89 \pm 0.20$\\
    Gradient Boosting & $0.95 \pm 0.08$ & $0.61 \pm 0.19$ & $0.75 \pm 0.20$ & $0.73 \pm 0.18$ & $\mathbf{0.58 \pm 0.23}$ & $0.87 \pm 0.18$\\
    MLP & $0.84 \pm 0.16$ & $0.72 \pm 0.22$ & $0.74 \pm 0.22$ & $0.74 \pm 0.38$ & $\mathbf{0.68 \pm 0.38}$& $0.84 \pm 0.10$ \\
\end{tabular}
\caption{Performance of model monitoring systems for model deterioration for a variety of model architectures on eight regression datasets from the UCI repository~\cite{uci_data}. The scores are the means and standard deviations of the absolute deviation from the true labels on $\Dd{lower}$ and $\Dd{upper}$ (lower is better). K-S and PSI are the monitoring systems obtained by computing the Kolmogorov-Smirnov test values and the Population Stability Index, respectively, Prediction Shift is the statistical comparison of the model prediction,  and Doubt is our method. The best results for each model architecture are shown in bold. See the Appendix for all the raw scores.}
\label{tab:scores}
\end{center}
\end{table*}

\subsection{Detecting the Source of Uncertainty}\label{sec:unceratainty_source}

For this experiment, we make use of two datasets: a synthetic one and the popular House Prices regression dataset\footnote{\url{https://www.kaggle.com/c/house-prices-advanced-regression-techniques}}, where the goal is to predict the selling price of a given property. We select two of the features that are the most correlated with the target, \texttt{GrLivArea} and \texttt{TotalBsmtSF}, and also create a new feature of random noise, to have an example of a feature with minimum correlation with the target. A model deterioration system should therefore highlight the \texttt{GrLivArea} and \texttt{TotalBsmtSF} features, and \textit{not} highlight the random features.

Concretely, we compute an estimation of the Shapley values using TreeSHAP~\cite{shapTree}, which is an efficient estimation approach values for tree-based models, that allows for this second model to identify the features that are the source of the uncertainty, and thus also provide an indicator for what features might be causing the model deterioration.

We fitted an MLP on the training dataset, which achieved a $R^2$ value of $0.79$ on the validation set. We then shifted all three features by five standard deviations and trained a gradient boosting model on the uncertainty values of the MLP on the validation set, which achieves a good fit (an $R^2$ value of $0.94$ on the hold-out set of the validation). We then compare the SHAP values of the gradient boosting model with the PSI and K-S statistics for the individual features.


\begin{figure}[ht]
  \centering
  \includegraphics[width=1\linewidth]{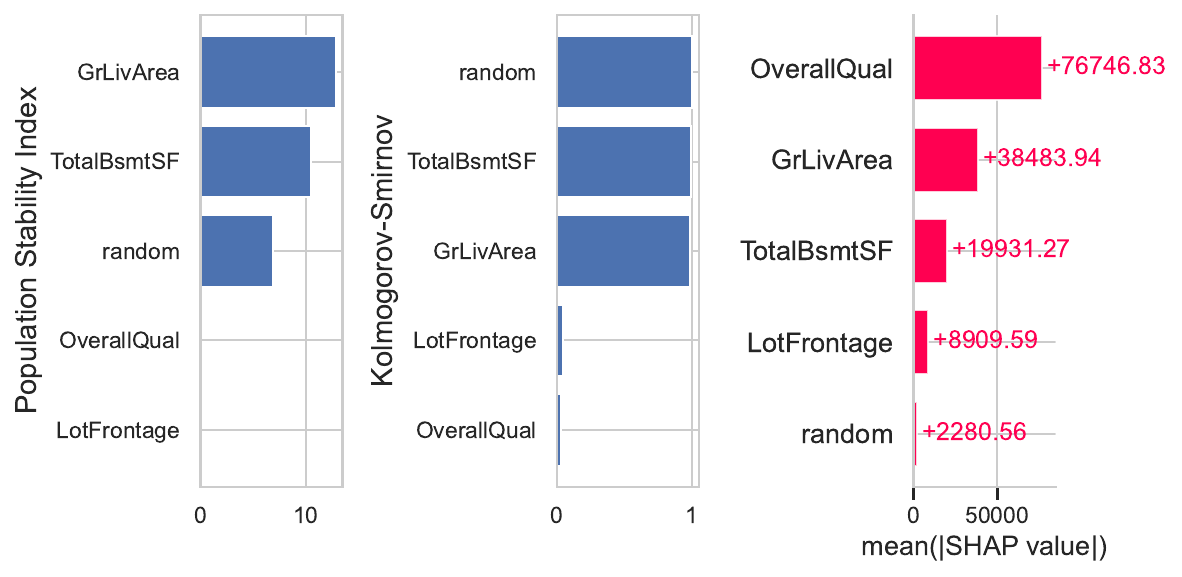}
  \caption{Global comparison of different distribution shift detection methods. Statistical methods correctly indicate that there exists a distribution shift in the shifted data. Shapley values indicate the contribution of each feature to the drop in predictive performance of the model.}\label{fig:shap}
\end{figure}

In Figure~\ref{fig:shap}, classical statistics and SHAP values to detect the source of the model deterioration are compared. We see that the PSI and K-S value correctly capture the shift in each of the three features (including the random noise). On the other hand, our SHAP method highlights the two substantial features (\texttt{GrLivArea} and \texttt{TotalBsmtSF}) and correctly does not assign a large value to the random feature, despite the distribution shift.

\begin{figure}[ht]
  \centering
  \includegraphics[width=1\linewidth]{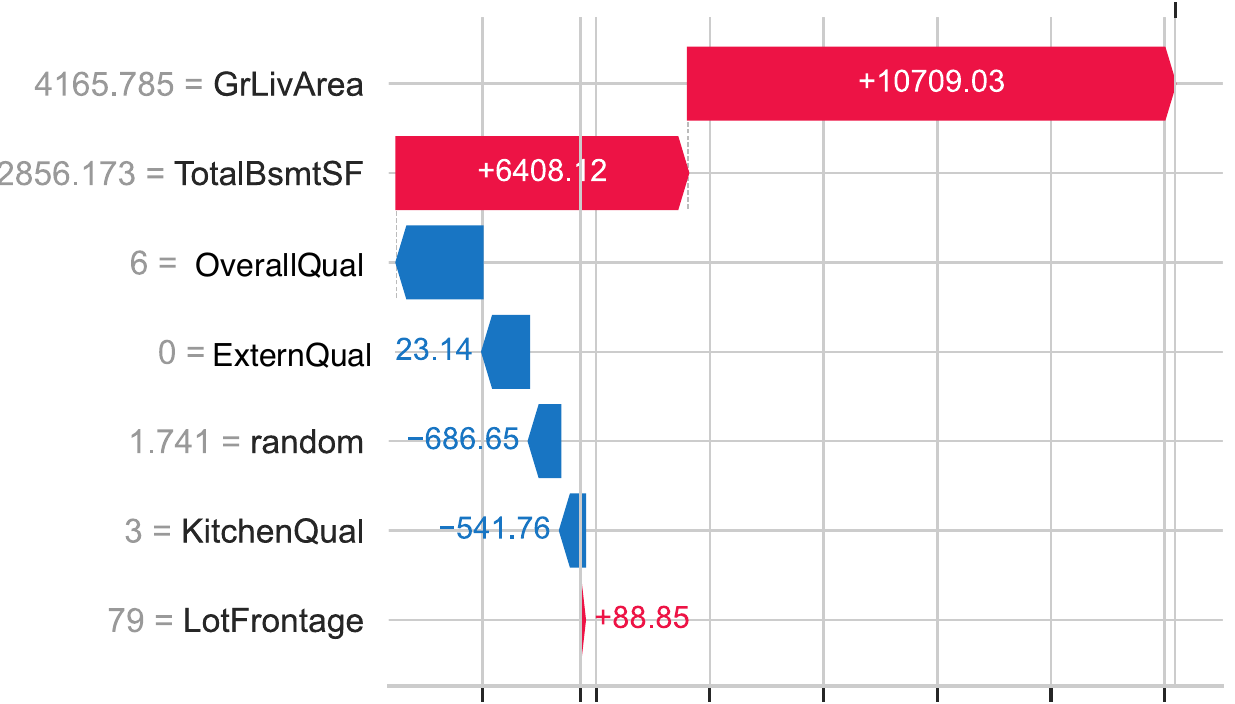}
  \caption{Individual explanation that displays the source of uncertainty for one instance. The previous method allowed only for comparison between distributions, now with explainable uncertainty, we are able to account for individual instances. In red, features pushing the uncertainty prediction higher are shown; in blue, those pushing the uncertainty prediction lower.}\label{fig:local_shap}
\end{figure}

Figure~\ref{fig:local_shap} shows features contributing to pushing the model output from the base value to the model output. Features pushing the uncertainty prediction higher are shown in red, and those pushing the uncertainty prediction lower are in blue~\citep{lundberg2018,lundberg2020local2global,lundberg2017unified}. From these values, we can, at a local level, also identify the two features (\texttt{GrLivArea} and \texttt{TotalBsmtSF}) causing the model deterioration in this case.

\section{Synthetic Data Experiments}
In this section we apply our model monitoring method to synthetic data to demonstrate main differences between our approach, methods from classical statistics and other model agnostic uncertainty approaches.

We sample data from a three-variate normal distribution $X = (X_1,X_2,X_3) \sim N(1,0.1\cdot I_3)$ with $I_3$ being an identity matrix of order 3. We construct the target variable as $Y:=X_1^2 + X_2 + \varepsilon$, with $\varepsilon \sim N(0,0.1)$ being random noise. We thus have a non-linear feature $X_1^2$, a linear feature $X_2$ and a feature $X_3$ which is not used, all of them being independent among each other. We draw $10,000$ random samples for both training and test data, yielding the training set $(X^{tr},Y^{tr})$ and the test set $(X^{te},Y^{te})$ and train a linear model $f_\theta$. 

\subsection{Evaluating Model Deterioration}
In this experiment we aim to compare indicators of model deterioration by replacing the value of each feature, $j=1,2,3$, with a continuous vector with evenly spaced numbers over a specified range $(-3,4)$

\begin{figure}[ht]
  \includegraphics[width=1\linewidth]{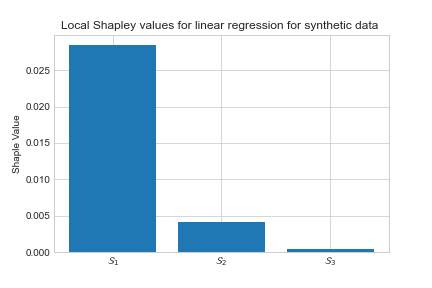}
  \caption{Shapley values for the data point $x:=(10,10,10)$  and the model $g_\theta$. Explainable uncertainty estimation allows to
  account for the drivers of uncertainty that serves as a proxy for model predictive performance degradation. }\label{fig:syntheticAnalytical}
\end{figure}

\begin{figure*}[ht]
  \includegraphics[width=1\linewidth]{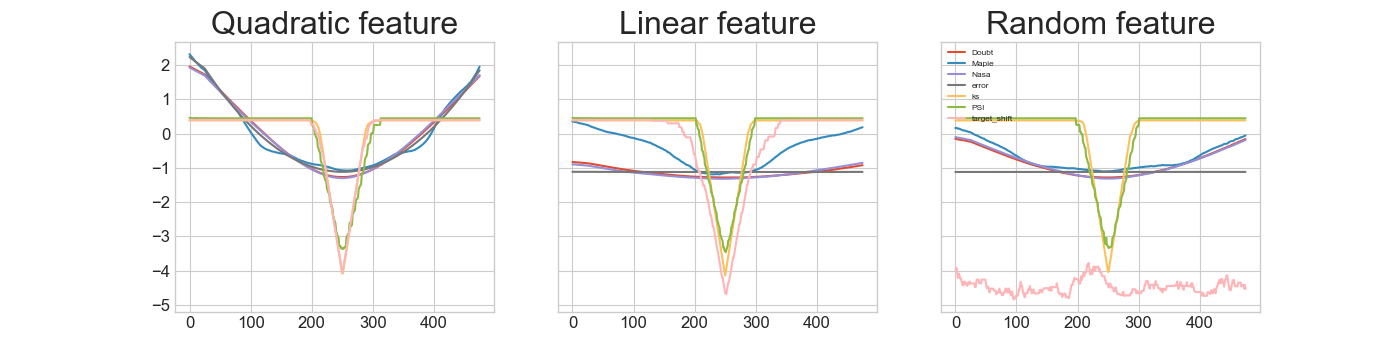}
  \caption{Comparison of different monitoring methods over the synthetic dataset. Each of the plots represents an independent experiment where each of the features has been replaced by a continuous vector with  evenly spaced numbers  in the range $(-3,4)$.  The x-axis represents the index of the sorted column. Doubt achieves a better goodness-of-fit than previous statistical methods (cf. Table \ref{table:synthetic}) }\label{fig:syntheticDegradation}
\end{figure*}

\begin{table*}[ht]
\begin{center}
\begin{tabular}{c|ccccccc}
                  & \textbf{Doubt} & \textbf{NASA}  & \textbf{MAPIE}   & \textbf{KS} & \textbf{PSI} & \textbf{Pred. Shift}& \textbf{Exp. Shift} \\ \hline
Quadratic feature &  \textbf{0.09}&     0.11        &    \textbf{0.09}&     0.96    &0.95          &   0.94  &   0.98                  \\
Linear feature    &  \textbf{0.11}&     0.12        &     0.70       &     0.35     & 1.39         &  1.46   &   1.20                    \\
Random feature    &  \textbf{0.34}&     0.35        &     0.42       &     1.39     &1.46          &   5.58  &   1.21                  \\ \hline
Mean              & \textbf{0.18} &     0.20        &     0.40       &      1.24    &1.29          &   2.69  &   1.13                    
\end{tabular}
\caption{Performance of model monitoring systems for model deterioration for a linear regression model on a synthetic dataset (cf. Figure \ref{fig:syntheticDegradation}).  K-S and PSI are the monitoring systems obtained by computing the Kolmogorov-Smirnov test values and the Population Stability Index, respectively, and Doubt is our method. Explanation Shift is the method presented in the previous chapter of the thesis. Best results are shown in bold.}\label{table:synthetic}
\end{center}
\end{table*}

In Figure \ref{fig:syntheticDegradation} and Table \ref{table:synthetic} we can see what is the impact of the synthetic distribution shift is in terms of model mean squared error. We see that our uncertainty method follows a better goodness-of-fit than both the Kolmogorov-Smirnov test and the Population Stability Index on the input data and and target data. The latter two methods have identical values for the linear and random features since they are independent of the target values.

\subsection{Detecting the Source of Uncertainty/Model Deterioration}

For this example, we simulate out-of-distribution data by shifting $X$ by 10; i.e., $X^{ood}_j := X^{tr}_j + 10$ for $j=1,2,3$. We train a linear regression model $f_\theta$ on $(X^{tr},Y^{tr})$ and use Doubt to get uncertainty estimates $Z$ on both $X^{te}$ and $X^{ood}$. We then train another linear model $g_\psi$ on $((X^{te},X^{ood}))$ to predict $Z$. The coefficients of $g_\psi$ are $\beta_1= 0.03478, \beta_2=0.005023 ,\beta_3= 0.000522$. 
Since the features are independent and we are dealing with linear regression, the interventional conditional expectation Shapley values can be computed as $\beta_i(x_i-\mu_i)$~\cite{DBLP:journals/corr/ShapTrueModelTrueData}, where $\mu$ is the mean of the training data . So for the data point $x:=(10,10,10)$, the Shapley values are $(0.0291,0.0042,0.0004)$, where the most relevant shifted feature in the model is the one that receives the highest Shapley value. In this experiment with synthetic data, statistical testing on the input data would have flagged the three feature distributions as equally shifted. With our proposed method, we can identify the more meaningful features towards identifying the source of model predictive performance deterioration. It is worth noting that this explainable AI uncertainty approach can be used with other uncertainty estimation techniques.

\section{Computational Performance Comparison}

\begin{figure*}[ht]
\centering
\includegraphics[width=.4\textwidth]{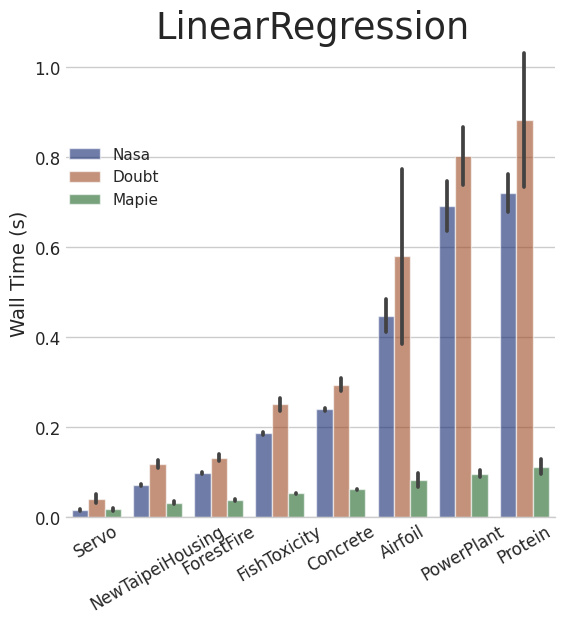}\hfill
\includegraphics[width=.4\textwidth]{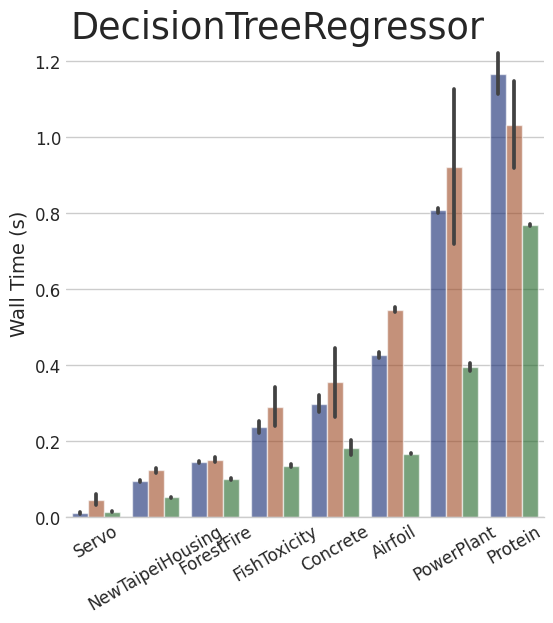}\hfill
\includegraphics[width=.4\textwidth]{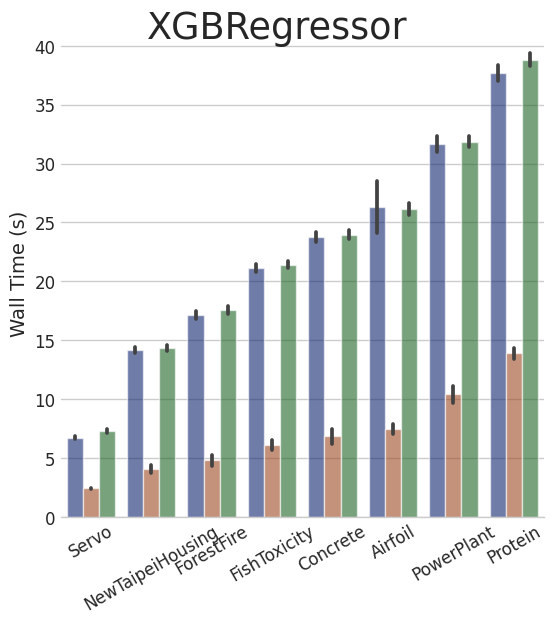}

\caption{Wall time computational performance of the three uncertainty estimation methods described through the paper. Datasets are sorted by the number of rows and the y-axis is not shared accross plots.}
\label{fig:figure3}
\end{figure*}

In this section we compare the wall clock time of the NASA method~\cite{kumar2012bootstrap}, MAPIE~\cite{barber2021predictive} and Doubt with three estimators: linear regression, decision tree and  gradient boosting decision tree, over the range of dataset previously used. We set the number of model bootstraps in the Doubt method to be equal to the number of cross-validation splits in the MAPIE method; namely, the square root of the total number of samples in the dataset. The experiment are performed on a virtual machine with 8 vCPUs and 60 GB RAM.

See the results in Figure \ref{fig:figure3}. We see here that the MAPIE approach is significantly faster than the NASA and Doubt methods in the cases where the model used is a linear regression model or a decision tree. However, interestingly, in the case where the model is an XGBoost model, the Doubt method is significantly faster than the other two.

\begin{table}[ht]
\centering
\caption{Non-aggregated version of Table \ref{tab:scores}. Model monitoring systems for model deterioration for a variety of model architectures on eight regression datasets from the UCI repository.}
\label{tab:appendix}
\resizebox{\textwidth}{!}{%
\begin{tabular}{l|l|llllllll|ll}
                            &             & Airfoil & Concrete & ForestFire & Parkinsons & PowerPlant & Protein & Bike & FishToxicity & Mean               & Std                 \\\hline
Linear Regression           & Uncertainty & 0.67    & 0.59     & 0.6        & 0.6        & 1.0        & 0.72    & 0.86 & 0.64         & 0.71               & 0.14  \\
                            & MAPIE       & 0.72    & 0.79     & 0.71       & 0.62       & 0.85       & 0.74    & 0.79 & 0.94         & 0.77               & 0.09 \\
                            & K-S         & 0.78    & 0.67     & 0.81       & 0.92       & 0.8        & 0.99    & 0.85 & 0.71         & 0.81               & 0.10 \\
                            & PSI         & 1.02    & 0.9      & 0.78       & 0.83       & 0.93       & 0.79    & 0.91 & 0.81         & 0.87               & 0.08 \\
                            & Target      & 0.94    & 0.68     & 0.7        & 0.84       & 0.8        & 0.97    & 0.97 & 1.0          & 0.86               & 0.12 \\\hline
Poisson Regressor           & Uncertainty & 0.81    & 0.49     & 0.93       & 0.82       & 0.97       & 0.83    & 0.79 & 0.71         & 0.79               & 0.14 \\
                            & MAPIE       & 0.84    & 0.55     & 0.89       & 0.86       & 0.87       & 0.91    & 0.61 & 1.1          & 0.82               & 0.17 \\
                            & K-S         & 0.83    & 0.68     & 1.08       & 1.08       & 0.8        & 1.16    & 1.16 & 0.79         & 0.94               & 0.19 \\
                            & PSI         & 1.01    & 0.9      & 1.04       & 1.04       & 0.94       & 0.85    & 0.85 & 0.86         & 0.93               & 0.08 \\
                            & Target      & 1.02    & 0.75     & 1.15       & 1.15       & 0.79       & 1.1     & 1.1  & 0.95         & 1.00               & 0.15  \\\hline
Decision Tree               & Uncertainty & 0.48    & 0.44     & 0.74       & 0.51       & 0.45       & 0.44    & 0.46 & 0.47         & 0.49               & 0.10 \\
                            & MAPIE       & 0.59    & 0.52     & 0.96       & 0.61       & 0.61       & 0.51    & 0.49 & 0.45         & 0.59               & 0.15 \\
                            & K-S         & 0.49    & 0.56     & 0.66       & 0.49       & 0.3        & 0.49    & 0.68 & 0.5          & 0.52               & 0.11  \\
                            & PSI         & 1.14    & 1.08     & 1.02       & 0.9        & 0.92       & 0.9     & 0.86 & 0.95         & 0.97               & 0.09  \\
                            & Target      & 0.96    & 0.77     & 0.81       & 0.9        & 0.64       & 0.54    & 0.94 & 0.79         & 0.79               & 0.14 \\\hline
Random Forest               & Uncertainty & 0.57    & 0.577    & 0.93       & 0.5        & 0.8        & 0.92    & 0.67 & 0.98         & 0.74               & 0.18 \\
                            & MAPIE       & 0.76    & 0.91     & 0.95       & 0.58       & 0.99       & 1.0     & 0.75 & 0.98         & 0.86               & 0.15 \\
                            & K-S         & 0.48    & 0.51     & 0.74       & 0.41       & 0.41       & 0.41    & 0.62 & 0.45         & 0.50               & 0.11 \\
                            & PSI         & 1.11    & 1.02     & 1.0        & 0.89       & 0.92       & 0.92    & 0.85 & 0.93         & 0.95               & 0.08 \\
                            & Target      & 0.95    & 0.7      & 0.7        & 0.87       & 0.61       & 0.38    & 0.89 & 0.8          & 0.73               & 0.18 \\\hline
Gradient Boosting           & Uncertainty & 0.48    & 0.49     & 1.0        & 0.45       & 0.45       & 0.32    & 0.57 & 0.88         & 0.58               & 0.23 \\
                            & MAPIE       & 0.66    & 0.86     & 0.94       & 0.55       & 0.67       & 0.51    & 0.64 & 1.03         & 0.73               & 0.18 \\
                            & K-S         & 0.45    & 0.47     & 0.65       & 0.9        & 0.9        & 0.42    & 0.64 & 0.48         & 0.61               & 0.19 \\
                            & PSI         & 1.11    & 1.01     & 0.99       & 0.91       & 0.92       & 0.9     & 0.85 & 0.98         & 0.95               & 0.08 \\
                            & Target      & 1.05    & 0.73     & 0.54       & 0.92       & 0.65       & 0.46    & 0.88 & 0.8          & 0.75               & 0.19 \\\hline
MultiLayer Perceptron       & Uncertainty & 1.5     & 0.93     & 0.81       & 0.43       & 0.43       & 0.42    & 0.52 & 0.41         & 0.68               & 0.38 \\
                            & MAPIE       & 1.6     & 0.85     & 0.95       & 0.49       & 0.58       & 0.42    & 0.55 & 0.55         & 0.74               & 0.38  \\
                            & K-S         & 1.03    & 0.72     & 1.18       & 0.61       & 0.62       & 0.64    & 0.62 & 0.62         & 0.75               & 0.22 \\
                            & PSI         & 1.1     & 0.75     & 1.08       & 0.76       & 0.88       & 0.79    & 0.74 & 0.66         & 0.84               & 0.16  \\
                            & Target      & 0.63    & 0.8      & 1.24       & 0.6        & 0.59       & 0.54    & 0.73 & 0.78         & 0.73               & 0.22 \\
\end{tabular}
}
\end{table}

\section{Conclusion}
\label{sec:conclusion}

In this work, we have provided methods and experiments to monitor and identify machine learning model deterioration via non-parametric bootstrapped uncertainty estimation methods and use explainability techniques to explain the source of the model deterioration under gradual covariate shift. 

Our monitoring system is based on a model agnostic uncertainty estimation method, which produces prediction intervals with theoretical guarantees and which achieves better coverage than the current state-of-the-art. The resulting monitoring system more accurately detects model deterioration than methods using classical statistics. Finally, we used SHAP values in conjunction with these uncertainty estimates to identify the features that are driving the model deterioration at both a global and local level, and qualitatively showed that these more accurately detect the source of the model deterioration compared to classical statistical methods.

\subsection{Limitations:}
Detecting performance degradation may be impossible in situations where no labelled data is available and no characterizations of the shift can be made. In this work, we have engineered an experiment of a gradual covariate shift, where uncertainty estimation by model averaging achieves the best result.

We emphasize here that due to computationally limitations, we have only benchmarked on datasets of relatively small to medium size (cf. Table~\ref{tab:uncertaintydatasets}), and further work needs to be done to see if these results are also valid for datasets of significantly larger size. This work also focused on tabular data and non-deep learning models.

\chapter{Conclusion and Remarks}\label{ch:conclusion}

\section{Summary of Findings}

This thesis studies the problem of model monitoring in the absence of labelled deployment data, aiming to improve the understanding of supervised machine learning systems in production. It begins by exploring the AI system life cycle, emphasizing the critical role of model monitoring in ensuring the functionality and reliability of deployed models and the unsupervised monitoring stage. 

In Chapter~\ref{ch:foundations}, we navigated through the foundations needed for this thesis. Starting from feature attribution distribution using both Shapley values and LIME in Section\ref{sec:xai.foundations}. In section~\ref{sec:foundations.alginemnt} we provide the foundations of AI alignment, drawing basic common understanding from political and moral philosophy. In section~\ref{sec:foundations.fairness}, we reviewed the current state-of-the-art of technical fairness metrics. Finally in section~\ref{sec:foundations.monitoring} we reviewed the problem of model monitoring under distribution shift with its inherent limitations.

Chapter \ref{ch:et} proposed a metric and a model for measuring \emph{equal treatment}, aiming to align AI systems with diverse human values encompassing political philosophy and moral principles. We have examined the principle of equality in machine learning, specifically on liberalism-oriented philosophical notions, emphasizing the importance of equal treatment for all individuals, regardless of their protected characteristics. By proposing a new formalization for equal treatment incorporating feature values' influence through Shapley values, we have sought to provide a better understanding of how can we align machine learning systems with equal treatment notions from the social sciences. The theoretical properties of our formalization have been analyzed, and a classifier two-sample test based on the AUC of an equal treatment inspector has been devised. Our experiments on synthetic and natural datasets have demonstrated the efficacy of our formalization, shedding light on the complexities of achieving equality in machine learning.

Chapter~\ref{ch:exp.shift} proposed to usage of explanation distributions shift to better detect changes in the model due to distribution shifts.From the predictive performance perspective, we have investigated the challenges posed by distribution shifts, which can significantly impact the performance of machine learning models on new data. By introducing the concept of explanation shift, we have developed a sensitive and explainable indicator to monitor changes in model behavior resulting from distribution shifts. Through theoretical analysis and empirical evaluations, we have demonstrated the effectiveness of explanation shifts and facilitated understanding of the interaction between distribution shifts and learned models.

Chapter~\ref{ch:monitoring} distinguishes between understanding changes in model behaviour and building indicators of performance degradation which remains an impossible task in many scenarios. We engineered an experiment where there is a gradual covariate shift and proposed model agnostic uncertainty estimation techniques to build an indicator of model performance degradation. Finally, we account for the drivers of model deterioration, training another model to predict the uncertainty, and explaining this model.

Finally, in Chapter~\ref{ch:conclusion}, we summarise the findings, make the contributions explicit and propose future work directions.

\section{Contributions to the Field}

The contributions made in this thesis help to tackle some of the challenges of model monitoring in machine learning, particularly in the context of deployment without labelled data. By defining and making explicit some the intricacies of post deployment monitoring, it flags the importance of unsupervised model monitoring to maintain the functionality and reliability of machine learning systems in production. The exploration of feature attribution distributions for AI alignment and technical fairness metrics, along with the challenges of model monitoring under distribution shift, provided techniques and methods to monitor models in production before the predictions are served in the real world.

We now summarize the key findings for each of the research questions posed.

\subsection{(RQ-1) How can we use feature attributions distributions for post-deployment monitor in the absence of labeled truth data?}

We addressed this issue by segmenting it into performance monitoring and AI alignment monitoring, and divide into two different research question aiming to explore in depth.

This thesis's technical solution leverages feature attribution distributions for model monitoring in unlabeled data scenarios. These distributions represent a projection of the model and data into a distribution space that retains distinct theoretical properties, depending on the feature attribution method utilized. We explored LIME and Shapley feature attribution methods.

\subsection{(RQ-2) Can feature attribution distributions be used to measure alignment of post-deployed models in a unsupervised monitoring set up??}

We initiated this by gathering research requirements from two specific philosophical perspectives—liberal-oriented political philosophy and deontological ethics—to align with the notion of equal treatment.

We introduced a new metric that improved existing equal treatment measures based on the collected requirements. By integrating feature attribution methods and philosophical principles, our proposed metric provides a more precise approach to measuring AI systems alignment with human values of equality and fairness.

\subsection{(RQ-3) How did Distribution Shift Impact the Model?}

We studied the interactions between deployed models and shifting data distributions, introducing the concept of explanation shift to detect and understand changes in model behaviour due to distribution shifts. This method offers a novel perspective on maintaining consistent model behaviour over time.

\subsection{(RQ-4) How can we identify and explain changes in model performance using feature attribution distribution?}

We proposed the use of Shapley values in conjunction with a second model that predicts uncertainty estimates to identify the features driving model deterioration at both a global and local level.

\subsection{Software Contributions}
The developed software, \texttt{skshift} and \texttt{explanationspace}, stands as a practical implementation of these contributions, offering researchers and practitioners a valuable tool for model monitoring. This work not only advances our understanding of model monitoring in unlabelled environments but also sets a new standard for the alignment of AI systems with ethical principles and human values.

To ensure reproducibility, we have made the data, code repositories, and experiments publicly available. The data, code, and tutorials can be accessed at the following URLs:

\begin{itemize}
\item \textbf{Explanation Shift}, Dataset and code repositories: \url{https://github.com/cmougan/ExplanationShift}

\item Open-source Python package with documentation and tutorials \texttt{skshift}: \url{https://skshift.readthedocs.io}

\item \textbf{Equal Treatment}, Data, data preparation routines, and code repositories: \url{https://github.com/cmougan/xAIAuditing}

\item Open-source Python package with documentation and tutorials \texttt{explanationspace}: \url{https://explanationspace.readthedocs.io}

\item  Data, data preparation routines, and code repositories for model monitoring with uncertainty estimation \url{https://github.com/cmougan/MonitoringUncertainty}

\end{itemize}

For our experiments, we used default parameters from the \texttt{scikit-learn} library \cite{pedregosa2011scikit}, unless stated otherwise. The system requirements and software dependencies for reproducing our experiments are provided. All experiments were conducted on a 4 vCPU server with 32 GB RAM.

\subsection{Limitations and Reflections}

As this thesis draws to a close, we now reflect on its scope, limitations, and areas requiring further exploration. While, I believe, the proposed methodologies and contributions are significant, they come with challenges and constraints that require reflection to provide a holistic view of the work's impact and applicability.

\subsubsection{Reliability of Explanations}

This thesis introduces an integrated framework for model monitoring that combines value alignment and performance maintenance in deployment scenarios. A central component of this framework is the reliance on algorithmic explanations, particularly feature attribution methods, to assess model behavior and detect shifts in explanation distributions. However, it relies on explanations.

Feature attribution explanation methods, such as those produced by LIME or Shapley values, are approximations and can introduce new biases. As previously stated in the thesis (cf. Section~\ref{ch:foundations}), previous research has focused on studying the reliability of local explanations; for example, LIME’s reliance on local perturbations may inadvertently favour majority groups due to its neighbourhood sampling strategy. However, when studying distributions, SHAP and LIME achieved very similar results, seeming to overcompensate those biases.  Further unexplored biases could compromise the fairness assessments derived from distributions of explanation-based methods. Therefore, a deeper technical exploration of these explanations is necessary. Can distributions of algorithmic explanations be considered objective or reliable enough to guide fairness and performance decisions? This question remains partially unanswered and warrants further investigation.

\subsubsection{Alternative Explanation Methods}

Feature attribution methods, such as those used in this thesis to monitor ML models after deployment, can be useful tools. Although this thesis predominantly focuses on Shapley values and LIME, other explanation techniques, such as counterfactual explanations or gradient-based methods, might offer complementary or superior insights in certain scenarios. Each method has its own set of assumptions, strengths, and weaknesses, which should be considered when selecting a monitoring strategy.

\subsubsection{Beyond Tabular Data}

Moreover, while the two explored feature attribution methods proposed in this thesis, have demonstrated efficacy in both theoretical and experimental settings, it might benefit from further validation across a wider range of models, data types, or deployment environments. The adaptability of these methods to different architectures, such as transformer models or non-tabular data like images and text, remains an open question. Future research could aim to generalize these techniques to broader contexts and evaluate their performance comprehensively.

\subsubsection{Complexity and Practicality of Feature Attribution-Based Monitoring}

Incorporating feature attribution methods into model monitoring introduces an additional layer of complexity. While this complexity may be justified in some scenarios where complexity is needed, it may not always be necessary or practical. Mathematically simpler or less computationally intensive methods might suffice in certain cases, particularly where the risk of model failure is lower or where interpretability is not a primary concern.

Furthermore, this work has been developed under specific assumptions, including controlled experimental setups and the availability of computational resources. These assumptions are explicitly stated to avoid overgeneralization. For instance, the impact of population imbalances or other real-world data challenges has not been thoroughly investigated. The findings should, therefore, be framed as promising but context-dependent, highlighting the need for cautious interpretation and application.

\section{Future Directions}

As we conclude this thesis, we highlight potential avenues for future research and development in model monitoring in general and in each of the chapter presented on this thesis.

\subsection{Software Improvement}

One avenue for future work involves expanding and refining the \emph{Equal Treatment Inspector} and the \emph{Explanation Shift Detector} software, allocated in the Python package \texttt{explanationspace} and \texttt{skshift} respectively. Currently, it provides insights into the features with different distributions of explanations, helping identify unequal treatment and distribution shift that impacts the model. Improvementts could include the development of a more user-friendly interface, incorporating statistical significance tests to validate differences in explanation distributions, extending to non-tabular data and exploring additional visualization techniques to facilitate a deeper understanding of the root sources of unequal treatment.

\subsection{Integration with Real-world Applications}

To validate the practical utility of the proposed approach, future work can potentially involve integrating the \emph{Equal Treatment Inspector} and \emph{Explanation Shift Detector} into real-world machine learning applications. This would entail applying the method to diverse datasets in various domains, such as finance, healthcare, or criminal justice, to assess its effectiveness in uncovering and addressing disparate treatment issues. Additionally, feedback from practitioners and stakeholders in these domains could be incorporated to improve the tool's applicability and relevance.

\subsection{Ethical and Societal Implications}

Future research should also delve into the ethical and societal implications of incorporating political liberal and Kantian deontological notions into machine learning fairness. This involves exploring how aligning AI systems with the principles of treating individuals as ends in themselves and equally might impact various stakeholders, including marginalized communities, policymakers, and the general public. Addressing these implications is crucial, as their implications within a machine learning context might be different than the ones from the social sciences.

\subsection{Longitudinal Studies}

Longitudinal studies could be conducted to assess the long-term impact and effectiveness of the proposed approach. These studies would involve monitoring the deployment of machine learning models incorporating the \emph{Equal Treatment Inspector} and \emph{Explanation Shift Detector} over an extended period, evaluating its ability to mitigate disparate treatment, and adapting the tool based on real-world feedback and evolving ethical standards.

By addressing these future directions, researchers can contribute to the ongoing discourse on fairness in machine learning and foster the development of ethically sound and socially responsible AI systems.

\listoffigures
\listoftables
\bibliographystyle{apalike}
\bibliography{UOS}

\backmatter
\raggedright                  
\copyrightDeclaration{} 
\end{document}